\documentclass{article}

\usepackage[final]{neurips_2021}

\usepackage{xr}
\makeatletter
\newcommand*{\addFileDependency}[1]{% argument=file name and extension
  \typeout{(#1)}
  \@addtofilelist{#1}
  \IfFileExists{#1}{}{\typeout{No file #1.}}
}
\makeatother

\newcommand*{\myexternaldocument}[1]{%
    \externaldocument{#1}%
    \addFileDependency{#1.tex}%
    \addFileDependency{#1.aux}%
}
\myexternaldocument{supplementary}

\usepackage[utf8]{inputenc} % allow utf-8 input
\usepackage[T1]{fontenc}    % use 8-bit T1 fonts
\usepackage{hyperref}       % hyperlinks
\usepackage{url}            % simple URL typesetting
\usepackage{booktabs}       % professional-quality tables
\usepackage{amsfonts}       % blackboard math symbols
\usepackage{nicefrac}       % compact symbols for 1/2, etc.
\usepackage{microtype}      % microtypography
\usepackage{xcolor}         % colors
\usepackage{tikz}
\usepackage{caption}
\usepackage{float}
\usetikzlibrary{arrows.meta}
\usetikzlibrary{calc}
 \usepackage[utf8]{inputenc}   % LaTeX, comprends les accents !
\usepackage[T1]{fontenc}      % Police contenant les caractÃ¨res franÃ§ais
\usepackage{comment}

\usepackage{natbib}
\usepackage[tbtags]{amsmath}
\usepackage{amsthm}
\allowdisplaybreaks
\usepackage{amssymb,mathrsfs}
\usepackage{amsfonts}
\usepackage{upgreek}
\usepackage{xspace}

\usepackage{graphicx}
\usepackage{subfig}
\usepackage{color}
\usepackage{algorithm, algorithmic}
\usepackage{stmaryrd}
\usepackage[inline]{enumitem}
\usepackage{url}

\usepackage{tikz}
\usetikzlibrary{calc}

\newcommand\xNodemoinstiny{-1}

\newcommand\xNodeMoins{-3}

\usepackage{pgfplots}
\usepackage{xcolor}
\usepackage{bbm}
\usepackage{ifthen}
\usepackage{xargs}
\usepackage[textwidth=1.8cm]{todonotes}

\usepackage{aliascnt}
\usepackage{cleveref}
\usepackage{autonum}
\makeatletter
\newtheorem{theorem}{Theorem}
\crefname{theorem}{theorem}{Theorems}
\Crefname{Theorem}{Theorem}{Theorems}

\newtheorem*{lemma_nonumber*}{Lemma}

\newaliascnt{lemma}{theorem}
\newtheorem{lemma}[lemma]{Lemma}
\aliascntresetthe{lemma}
\crefname{lemma}{lemma}{lemmas}
\Crefname{Lemma}{Lemma}{Lemmas}

\newaliascnt{corollary}{theorem}

\aliascntresetthe{corollary}
\crefname{corollary}{corollary}{corollaries}
\Crefname{Corollary}{Corollary}{Corollaries}

\newaliascnt{proposition}{theorem}
\newtheorem{proposition}[proposition]{Proposition}
\aliascntresetthe{proposition}
\crefname{proposition}{proposition}{propositions}
\Crefname{Proposition}{Proposition}{Propositions}

\newaliascnt{definition}{theorem}

\aliascntresetthe{definition}
\crefname{definition}{definition}{definitions}
\Crefname{Definition}{Definition}{Definitions}

\newaliascnt{remark}{theorem}

\aliascntresetthe{remark}
\crefname{remark}{remark}{remarks}
\Crefname{Remark}{Remark}{Remarks}

\crefname{example}{example}{examples}
\Crefname{Example}{Example}{Examples}

\newtheorem{technique}{Technique}
\crefname{technique}{technique}{techniques}
\Crefname{Technique}{Technique}{Techniques}

\crefname{figure}{figure}{figures}
\Crefname{Figure}{Figure}{Figures}

\newtheorem{assumption}{\textbf{A}\hspace{-3pt}}
\crefformat{assumption}{{\textbf{A}}#2#1#3}

\newtheorem{assumptionF}{\textbf{F}\hspace{-3pt}}
\crefformat{assumptionF}{{\textbf{F}}#2#1#3}

\newtheorem{assumptionB}{\textbf{B}\hspace{-3pt}}
\Crefname{assumptionB}{\textbf{B}\hspace{-3pt}}{\textbf{B}\hspace{-3pt}}
\crefname{assumptionB}{\textbf{B}}{\textbf{B}}

\Crefname{assumptionC}{\textbf{C}\hspace{-3pt}}{\textbf{C}\hspace{-3pt}}
\crefname{assumptionC}{\textbf{C}}{\textbf{C}}

\Crefname{assumptionH}{\textbf{H}\hspace{-3pt}}{\textbf{H}\hspace{-3pt}}
\crefname{assumptionH}{\textbf{H}}{\textbf{H}}

\Crefname{assumptionT}{\textbf{T}\hspace{-3pt}}{\textbf{T}\hspace{-3pt}}
\crefname{assumptionT}{\textbf{T}}{\textbf{T}}

\Crefname{assumptionT}{\textbf{T}\hspace{-3pt}}{\textbf{T}\hspace{-3pt}}
\crefname{assumptionT}{\textbf{T}}{\textbf{T}}

\Crefname{assumptionL}{\textbf{L}\hspace{-3pt}}{\textbf{L}\hspace{-3pt}}
\crefname{assumptionL}{\textbf{L}}{\textbf{L}}

\Crefname{assumptionQ}{\textbf{Q}\hspace{-3pt}}{\textbf{Q}\hspace{-3pt}}
\crefname{assumptionQ}{\textbf{Q}}{\textbf{Q}}

\Crefname{assumptionAR}{\textbf{AR}\hspace{-3pt}}{\textbf{AR}\hspace{-3pt}}
\crefname{assumptionAR}{\textbf{AR}}{\textbf{AR}}

\newcommand\diaW{11}
\newcommand\diaH{5}
\newcommand\diaJump{2.75}
\newcommand\nextRow{1.25}
\newcommand\imW{0.08}
\newcommand\imWB{0.1}
\newcommand\imOp{0.6}
\newcommand\bend{5}

\newcommand\offset{2}
\newcommand\offsety{2.3}

\newcommand\hsmall{1.75}
\newcommand\ww{3.25}
\newcommand\www{1.8}
\newcommand\wwww{3.5}
\newcommand\wwwww{4.8}

\usepackage{bm}
\usepackage{wrapfig}

 \newcommand{\detLigne}[1]{\det(#1)}
\def\hlf{\hat{\ell}^f}
\def\hlb{\hat{\ell}^b}
\def\Ent{\mathrm{H}}
\def\lyap{V_{p,t,x_t}}
\def\lyapp{V_{p}}

\def\contspace{\mathcal{C}}
\def\pdata{p_{\textup{data}}}

\def\pprior{p_{\textup{prior}}}

\def\Pens{\mathscr{P}}
\def\Mens{\mathscr{M}}

\newcommand{\schro}{Schr\"{o}dinger\xspace}

\newcommand{\tta}{\mathtt{A}}

\newcommand{\Capprox}{\tta}

\newcommandx\ctun[1][1=T]{\Capprox_{#1,1}}

\newcommand{\rref}[1]{\tup{\Cref{#1}}}

\newcommandx{\expec}[2]{{\mathbb E}\left[#1 \middle \vert #2  \right]} %%%% esperance conditionnelle

\def\dim{d}

\newcommand{\rme}{\mathrm{e}}

\newcommand{\Mtt}{\mathtt{M}}

\newcommandx{\norm}[2][1=]{\ifthenelse{\equal{#1}{}}{\left\Vert #2 \right\Vert}{\left\Vert #2 \right\Vert^{#1}}}
\newcommandx{\normLigne}[2][1=]{\ifthenelse{\equal{#1}{}}{\Vert #2 \Vert}{\Vert #2\Vert^{#1}}}

\def\bfc{\mathbf{c}}
\def\bfY{\mathbf{Y}}
\def\bbfY{\bar{\mathbf{Y}}}
\def\bfX{\mathbf{X}}
\def\tbfX{\tilde{\mathbf{X}}}
\def\tbfY{\tilde{\mathbf{Y}}}
\def\bfs{\mathbf{s}}
\def\bfZ{\mathbf{Z}}

\def\bfZ{\mathbf{Z}}

\def\bfB{\mathbf{B}}

\def\msa{\mathsf{A}}
\def\msd{\mathsf{D}}

\def\msb{\mathsf{B}}
\def\msc{\mathsf{C}}
\def\mse{\mathsf{E}}
\def\msf{\mathsf{F}}

\def\msu{\mathsf{U}}

\def\msx{\mathsf{X}}

\def\msy{\mathsf{Y}}

\newcommand{\mcb}[1]{\mathcal{B}(#1)}

\def\mcy{\mathcal{Y}}
\def\mcx{\mathcal{X}}
\def\mce{\mathcal{E}}

\def\mcf{\mathcal{F}}

\def\Qbb{\mathbb{Q}}

\def\Mbb{\mathbb{M}}
\def\Pbb{\mathbb{P}}

\def\rset{\mathbb{R}}

\def\nset{\mathbb{N}}

\def\rmP{\mathrm{P}}

\def\rmd{\mathrm{d}}

\def\rms{\mathrm{s}}

\def\rme{\mathrm{e}}

\def\rmc{\mathrm{C}}

\def\rmg{\mathrm{g}}
\def\rmh{\mathrm{h}}

\def\trace{\operatorname{Tr}}
\newcommandx{\functionspace}[2][1=+]{\mathbb{F}_{#1}(#2)}
\newcommand{\argmax}{\operatorname*{arg\,max}}
\newcommand{\argmin}{\operatorname*{arg\,min}}

\newcommandx{\VarDeux}[3][3=]{\operatorname{Var}^{#3}_{#1}\left\{#2 \right\}}

\newcommand{\1}{\mathbbm{1}}

\newcommand{\LeftEqNo}{\let\veqno\@@leqno}

\newcommand{\N}{\ensuremath{\mathbb{N}}}

\newcommand{\PE}{\mathbb{E}}

\newcommand{\abs}[1]{\left\vert #1 \right\vert}
\newcommand{\absLigne}[1]{\vert #1 \vert}

\newcommand{\tvnormLigne}[1]{\| #1 \|_{\mathrm{TV}}}

\newcommandx{\Vnorm}[2][1=V]{\| #2 \|_{#1}}
\newcommandx{\VnormEq}[2][1=V]{\left\| #2 \right\|_{#1}}
\newcommand{\parenthese}[1]{\left(#1 \right)}
\newcommand{\parentheseLigne}[1]{(#1 )}

\newcommand{\parentheseDeuxLigne}[1]{[ #1 ]}
\newcommand{\defEns}[1]{\left\lbrace #1 \right\rbrace }
\newcommand{\defEnsLigne}[1]{\lbrace #1 \rbrace }

\newcommand{\probaLigne}[1]{\mathbb{P}( #1 )}
\newcommandx\probaMarkovTilde[2][2=]
{\ifthenelse{\equal{#2}{}}{{\widetilde{\mathbb{P}}_{#1}}}{\widetilde{\mathbb{P}}_{#1}\left[ #2\right]}}

\newcommand{\expeLigne}[1]{\PE [ #1 ]}

\newcommand{\expeMarkovLigne}[2]{\PE_{#1} [ #2 ]}

\newcommand{\bigO}{\ensuremath{\mathcal O}}

\def\ie{\textit{i.e.}}

\def\eqsp{\;}
\newcommand{\coint}[1]{\left[#1\right)}
\newcommand{\ocint}[1]{\left(#1\right]}
\newcommand{\ooint}[1]{\left(#1\right)}
\newcommand{\ccint}[1]{\left[#1\right]}

\newcommand{\ccintLigne}[1]{[#1]}

\newcommandx{\weight}[2][2=n]{\omega_{#1,#2}^N}

\newcommandx\sequence[3][2=,3=]
{\ifthenelse{\equal{#3}{}}{\ensuremath{\{ #1_{#2}\}}}{\ensuremath{\{ #1_{#2}, \eqsp #2 \in #3 \}}}}

\newcommandx\sequenceD[3][2=,3=]
{\ifthenelse{\equal{#3}{}}{\ensuremath{\{ #1_{#2}\}}}{\ensuremath{( #1)_{ #2 \in #3} }}}

\newcommandx{\sequencen}[2][2=n\in\N]{\ensuremath{\{ #1_n, \eqsp #2 \}}}
\newcommandx\sequenceDouble[4][3=,4=]
{\ifthenelse{\equal{#3}{}}{\ensuremath{\{ (#1_{#3},#2_{#3}) \}}}{\ensuremath{\{  (#1_{#3},#2_{#3}), \eqsp #3 \in #4 \}}}}
\newcommandx{\sequencenDouble}[3][3=n\in\N]{\ensuremath{\{ (#1_{n},#2_{n}), \eqsp #3 \}}}

\def\eg{\textit{e.g.}}

\newcommand{\opnorm}[1]{{\left\vert\kern-0.25ex\left\vert\kern-0.25ex\left\vert #1
    \right\vert\kern-0.25ex\right\vert\kern-0.25ex\right\vert}}

\def\generator{\mathcal{A}}
\def\generatort{\tilde{\mathcal{A}}}

\def\Id{\operatorname{Id}}
\def\Idbf{\mathbf{I}}

\newcommandx{\CPE}[3][1=]{{\mathbb E}_{#1}\left[#2 \middle \vert #3  \right]} %%%% esperance conditionnelle
\newcommandx{\CPELigne}[3][1=]{{\mathbb E}_{#1}[#2  \vert #3  ]} %%%% esperance conditionnelle
\newcommandx{\CPEsq}[3][1=]{{\mathbb{E}^{1/2}}_{#1}\left[#2 \middle \vert #3  \right]} %%%% esperance conditionnelle
\newcommandx{\CPVar}[3][1=]{\mathrm{Var}^{#3}_{#1}\left\{ #2 \right\}}
\newcommand{\CPP}[3][]
{\ifthenelse{\equal{#1}{}}{{\mathbb P}\left(\left. #2 \, \right| #3 \right)}{{\mathbb P}_{#1}\left(\left. #2 \, \right | #3 \right)}}

\newcommandx{\osc}[2][1=]{\mathrm{osc}_{#1}(#2)}

\def\Id{\operatorname{Id}}

\def\V{V}

\def\bgamma{\bar{\gamma}}

\newcommand{\ensemble}[2]{\left\{#1\,:\eqsp #2\right\}}
\newcommand{\ensembleLigne}[2]{\{#1\,:\eqsp #2\}}

\newcommand\coupling[2]{\Gamma(\mu,\nu)}

\newcommand{\complementary}{\mathrm{c}}

\def\Leb{\lambda}

\def\vareps{\varepsilon}

\def\Psibf{\mathbf{\Psi}}

\newcommandx{\KL}[2]{\operatorname{KL}\left( #1 | #2 \right)}
\newcommandx{\KLsqrt}[2]{\operatorname{KL}^{1/2}\left( #1 | #2 \right)}
\newcommandx{\Jef}[2]{\operatorname{J}\left( #1 , #2 \right)}
\newcommandx{\JefLigne}[2]{\operatorname{J}( #1 , #2 )}
\newcommandx{\KLLigne}[2]{\operatorname{KL}( #1 | #2 )}

\def\gaStep
\def\QKer{Q}

\def\Tnplusun{\mathcal{T}_{k+1}}

\def\distance{\mathbf{d}}
\newcommandx{\wasserstein}[3][1=\distance,3=]{\mathbf{W}_{#1}^{#3}\left(#2\right)}
\newcommandx{\wassersteinLigne}[3][1=\distance,3=]{\mathbf{W}_{#1}^{#3}(#2)}
\newcommandx{\wassersteinD}[1][1=\distance]{\mathbf{W}_{#1}}
\newcommandx{\wassersteinDLigne}[1][1=\distance]{\mathbf{W}_{#1}}

\def\Rcoupling{\mathrm{R}}

\def\Kcoupling{\mathrm{K}}

\def\sigmaD{\sigma^2}

\newcommandx{\phibfs}[1][1=]{\pmb{\varphi}_{\sigmaD_{#1}}}

\newcommandx\sequenceg[3][2=,3=]
{\ifthenelse{\equal{#3}{}}{\ensuremath{( #1_{#2})}}{\ensuremath{( #1_{#2})_{ #2 \geq #3}}}}

\def\Kker{\Kcoupling}

\def\Rker{\Rcoupling}

\def\Pker{\mathrm{P}}

\def\Qker{\mathrm{Q}}

\def\rmL{\mathrm{L}}
\def\rmG{\mathrm{G}}

\newcommandx{\distV}[1][1=\bfc]{\mathbf{W}_{#1}}
\newcommandx{\distVdeux}[1][1=W_2]{\mathbf{d}_{#1}}

\def\mtt{\mathtt{m}}

\newcommand{\tup}[1]{\textup{#1}}

\def\wass{\mathcal{W}}

 \usepackage{comment}
\hypersetup{colorlinks,citecolor=blue}

\title{Diffusion Schr\"{o}dinger Bridge with Applications to\\ Score-Based Generative Modeling}

\author{%
  Valentin De Bortoli \\
  Department of Statistics,\\
  University of Oxford, UK
  \And
  James Thornton \\
  Department of Statistics,\\
  University of Oxford, UK
  \AND
  Jeremy Heng \\
  ESSEC Business School,\\
  Singapore
  \And
  Arnaud Doucet \\
  Department of Statistics,\\
  University of Oxford, UK
}

\begin{document}

\maketitle

\begin{abstract}
  Progressively applying Gaussian noise transforms complex data distributions to
  approximately Gaussian. Reversing this dynamic defines a generative
  model. When the forward noising process is given by a Stochastic Differential
  Equation (SDE), \cite{song2020score} demonstrate how the time inhomogeneous
  drift of the associated reverse-time SDE may be estimated using
  score-matching. A limitation of this approach is that the forward-time SDE must be run for a sufficiently long time for the final distribution to be approximately Gaussian while ensuring that the corresponding time-discretization error is controlled.  In contrast, solving the \schro Bridge (SB) problem,
  \ie \ an entropy-regularized optimal transport problem on path spaces, yields
  diffusions which generate samples from the data distribution in finite
  time. We present Diffusion SB (DSB), an original approximation of the
  Iterative Proportional Fitting (IPF) procedure to solve the SB problem, and
  provide theoretical analysis along with generative modeling experiments. The
  first DSB iteration recovers the methodology proposed by \cite{song2020score},
  with the flexibility of using shorter time intervals, as subsequent DSB
  iterations reduce the discrepancy between the final-time marginal of the
  forward (resp. backward) SDE with respect to the Gaussian prior (resp. data)
  distribution.  Beyond generative modeling, DSB offers a  computational optimal transport tool as the continuous state-space analogue of
  the popular Sinkhorn algorithm \citep{cuturi2013sinkhorn}.

\end{abstract}

\section{Introduction}
\label{sec:some-hist-notes}
 \emph{Score-Based Generative Modeling} (SGM) is a recently developed
approach to probabilistic generative modeling that exhibits state-of-the-art performance on several audio and image synthesis tasks; see \eg \
\cite{song2019generative,cai2020learning,chen2020wavegrad,kong2020diffwave,gao2020learning,jolicoeur2020adversarial,ho2020denoising,song2020improved,song2020denoising,song2020score,niu2020permutation,durkan2021maximum,hoogeboom2021argmax,saharia2021image,luhman2021knowledge,luhman2020diffusion,nichol2021improved,popov2021gradtts,nichol2021beatgans}. 
Existing SGMs generally consist of two parts. Firstly, noise is incrementally
added to the data in order to obtain a perturbed data distribution approximating
an easy-to-sample \emph{prior} distribution \eg \  Gaussian. Secondly, a neural network
is used to learn the reverse-time denoising dynamics, which when initialized at
this prior distribution, defines a generative model
\citep{sohl2015deep,ho2020denoising,song2019generative,song2020score}. \cite{song2020score}
have shown that one could fruitfully view the noising process as a
Stochastic Differential Equation (SDE) that progressively perturbs the initial data
distribution into an approximately Gaussian one.
\begin{figure}[H]
        \centering
        \begin{tikzpicture}
        % forward 1 ------------------------------------------------------------------
        \node[inner sep=0pt, label={\small }] (f1_T) at (\diaW,0)
            {\includegraphics[width=\imW\textwidth]{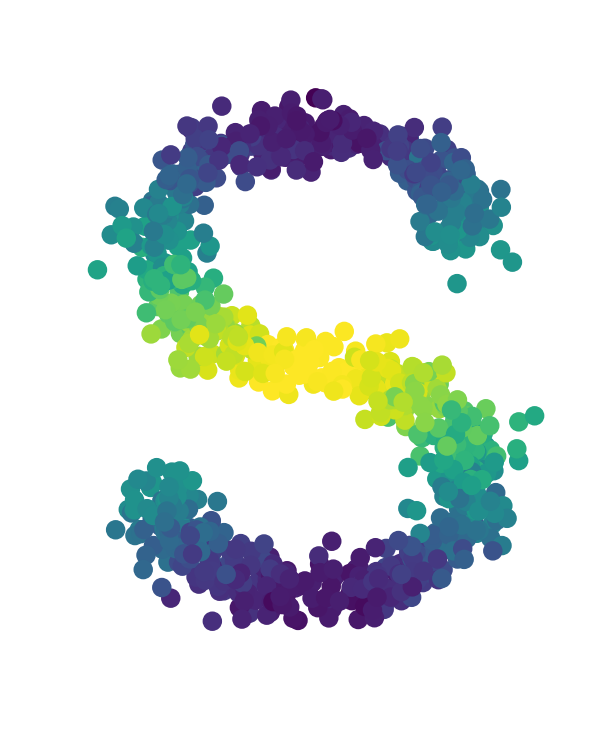}};
        \node[inner sep=0pt, label={\small },opacity=\imOp] (f1_1) at (\diaW/4,0)
            {\includegraphics[width=\imW\textwidth]{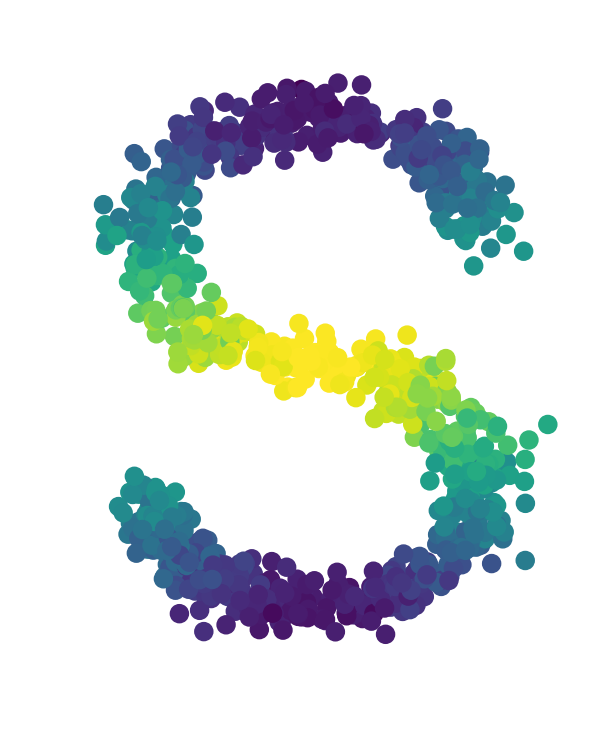}};
        \node[inner sep=0pt, label={\small },opacity=\imOp] (f1_2) at (\diaW/2,0)
            {\includegraphics[width=\imW\textwidth]{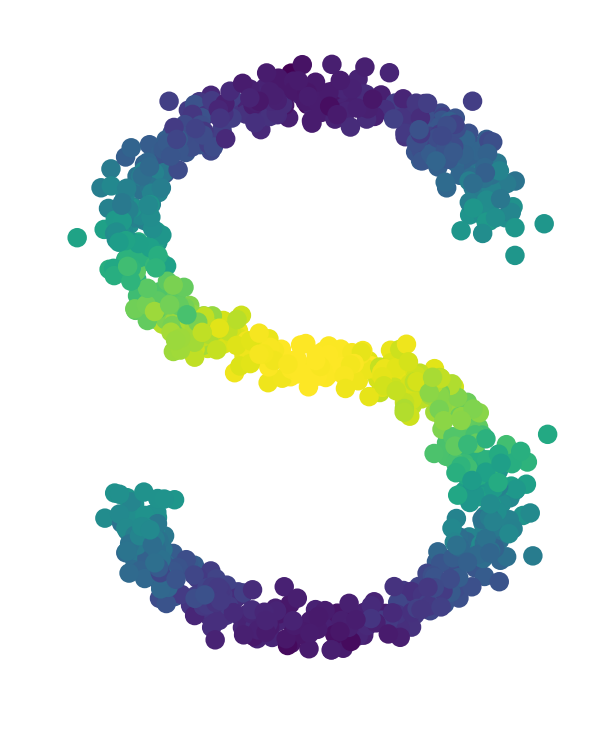}};
        \node[inner sep=0pt, label={\small },opacity=\imOp] (f1_3) at (3*\diaW/4,0)
            {\includegraphics[width=\imW\textwidth]{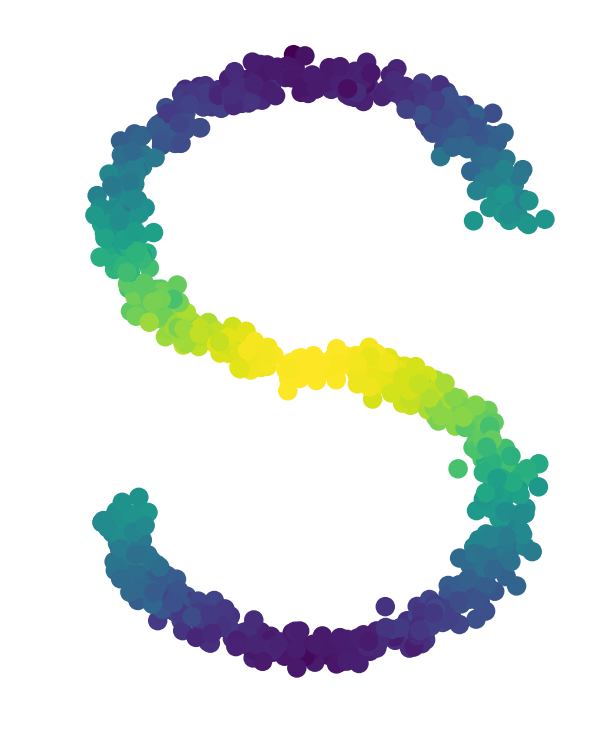}};        
        \node[inner sep=0pt, label={\small }] (f1_data) at (0,0)
            {\includegraphics[width=\imW\textwidth]{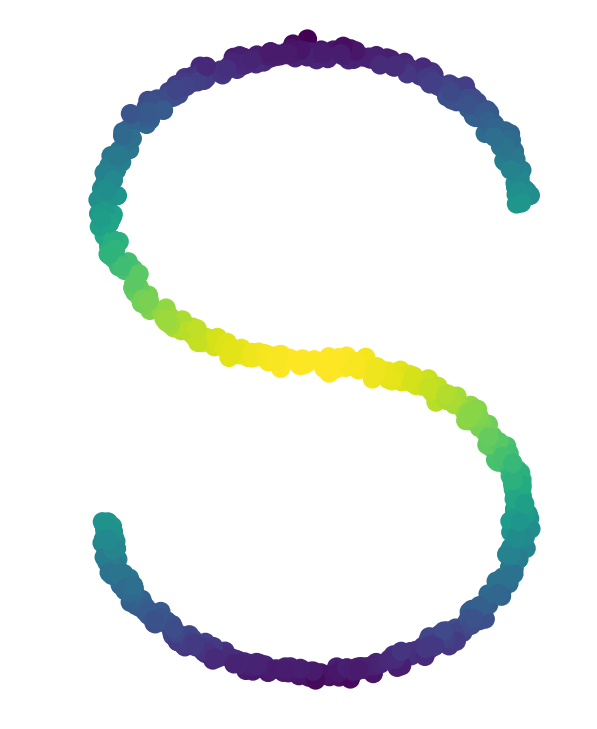}};
        
        % backward 1 ------------------------------------------------------------------    
        \node[inner sep=0pt, label={\small }] (b1_data) at (0,-\nextRow)    {\includegraphics[width=\imWB\textwidth]{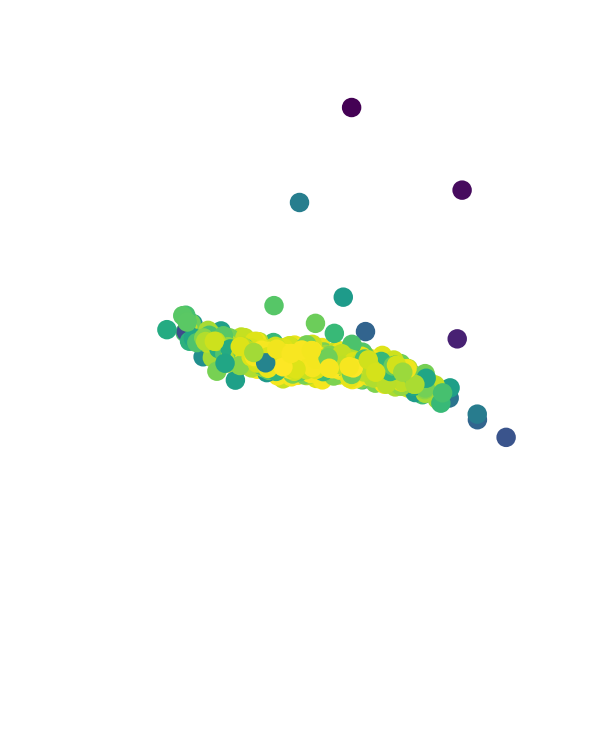}};
        \node[inner sep=0pt, label={\small },opacity=\imOp] (b1_1) at (\diaW/4,-\nextRow)
            {\includegraphics[width=\imWB\textwidth]{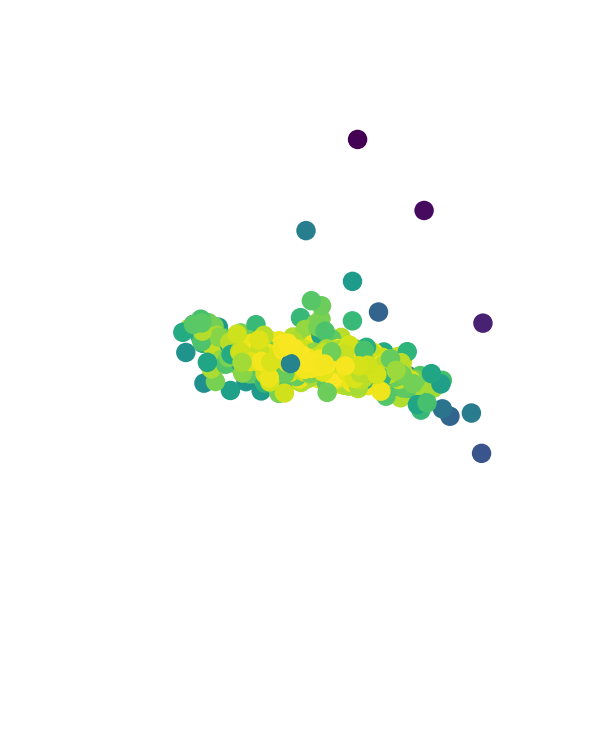}};
        \node[inner sep=0pt, label={\small },opacity=\imOp] (b1_2) at (2*\diaW/4,-\nextRow)
            {\includegraphics[width=\imWB\textwidth]{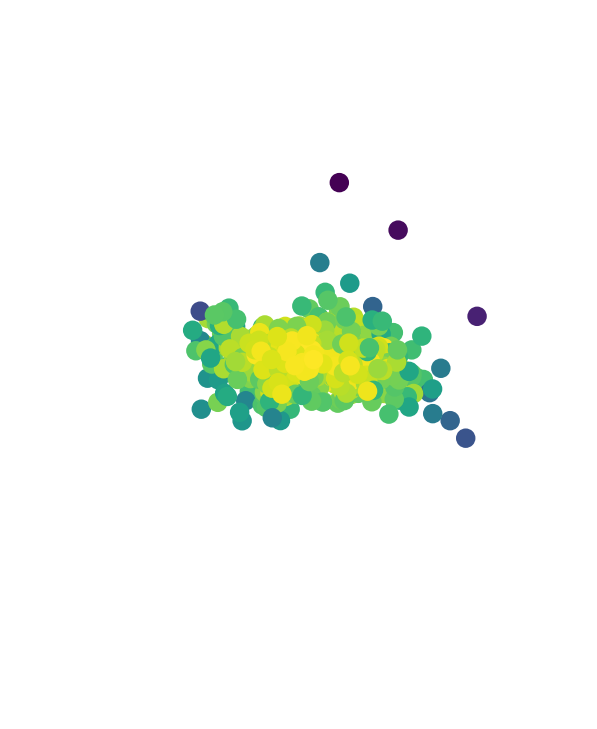}};
        \node[inner sep=0pt, label={\small },opacity=\imOp] (b1_3) at (3*\diaW/4,-\nextRow)
            {\includegraphics[width=\imWB\textwidth]{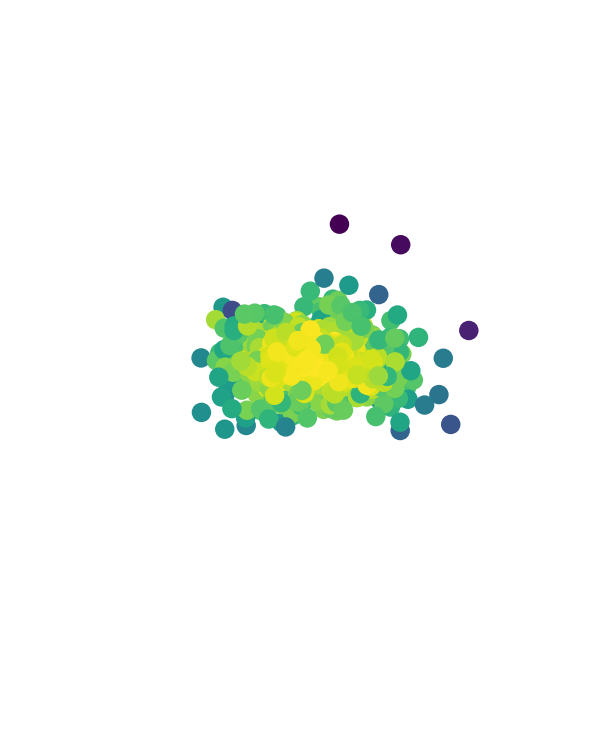}};        
        \node[inner sep=0pt, label={\small }] (b1_T) at (\diaW,-\nextRow)
            {\includegraphics[width=\imWB\textwidth]{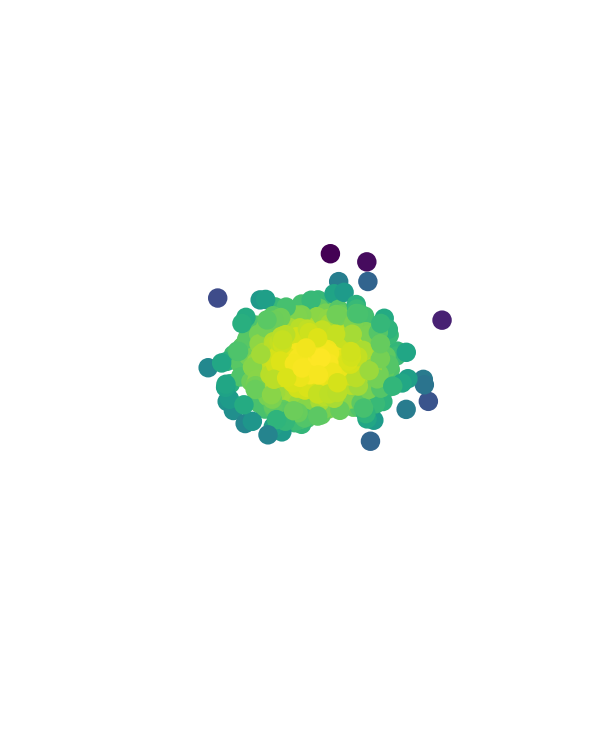}};
        
        % forward 15 ------------------------------------------------------------------
        \node[inner sep=0pt, label={\small }] (f15_data) at (0,-\diaJump)
            {\includegraphics[width=\imW\textwidth]{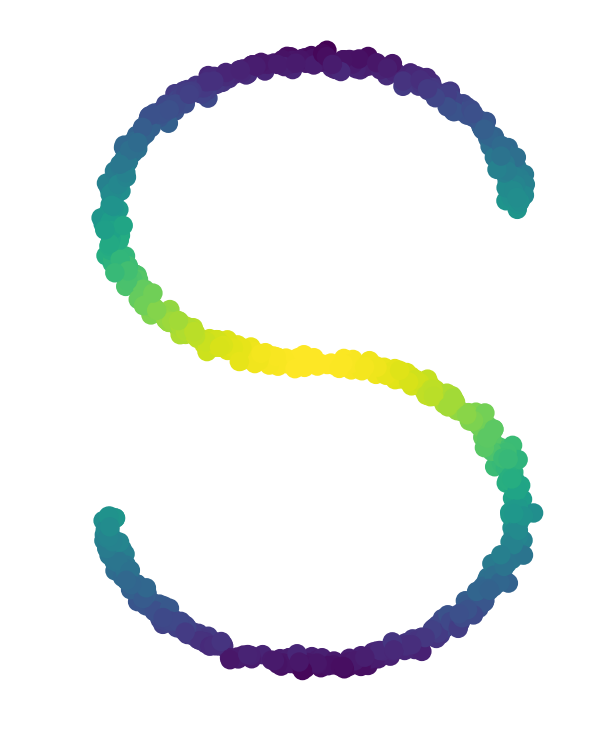}};
        \node[inner sep=0pt, label={\small },opacity=\imOp] (f15_1) at (\diaW/4,-\diaJump)
            {\includegraphics[width=\imW\textwidth]{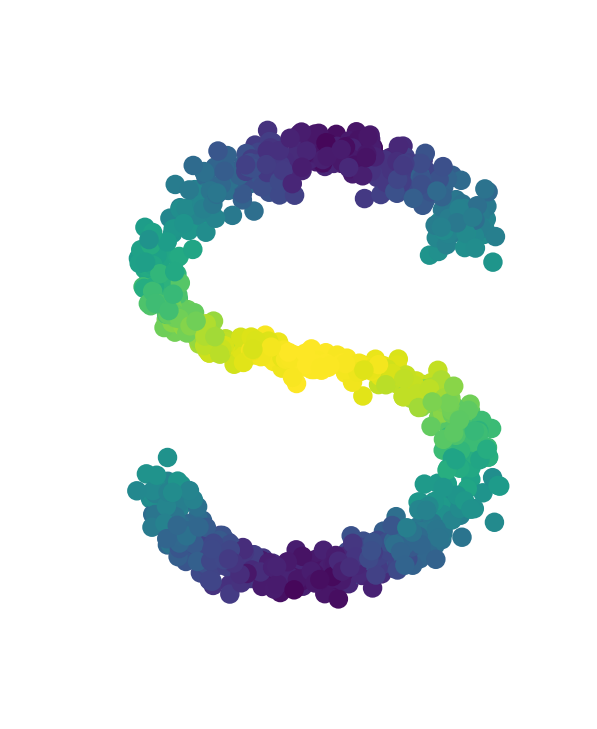}};
        \node[inner sep=0pt, label={\small },opacity=\imOp] (f15_2) at (\diaW/2,-\diaJump)
            {\includegraphics[width=\imW\textwidth]{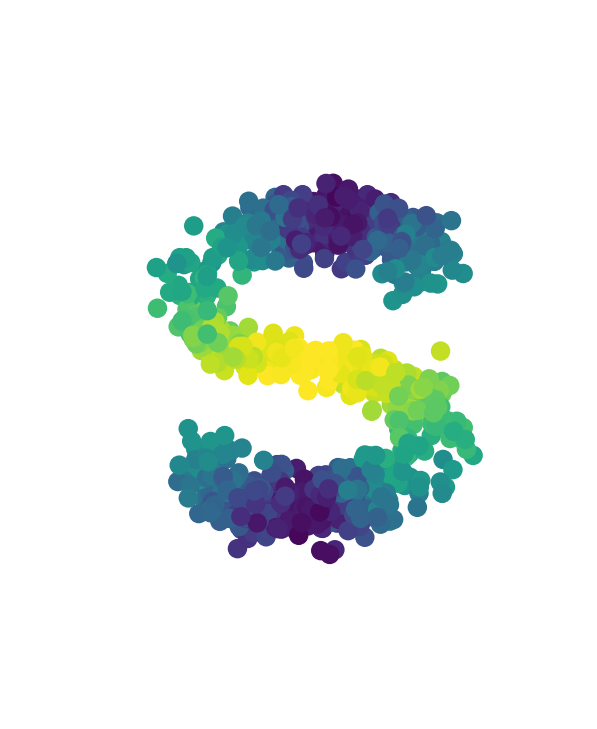}};
        \node[inner sep=0pt, label={\small },opacity=\imOp] (f15_3) at (3*\diaW/4,-\diaJump)
            {\includegraphics[width=\imW\textwidth]{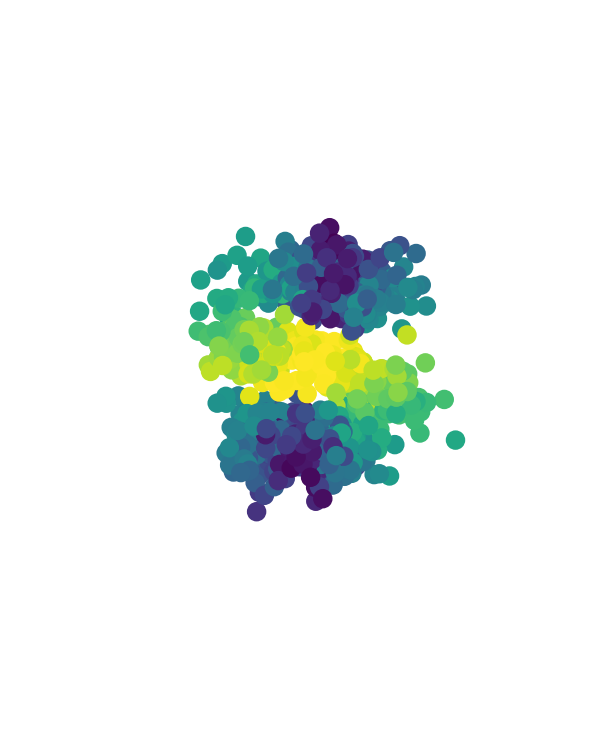}};    
        \node[inner sep=0pt, label={\small }] (f15_T) at (\diaW,-\diaJump)
            {\includegraphics[width=\imW\textwidth]{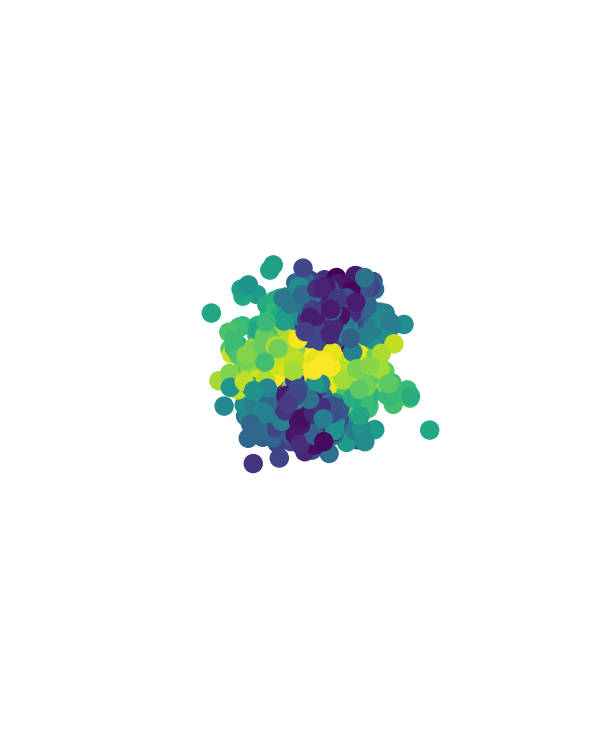}};
            
        % backward 15 ------------------------------------------------------------------
        \node[inner sep=0pt, label={\small }] (b15_data) at (0,-\diaJump-\nextRow)
            {\includegraphics[width=\imW\textwidth]{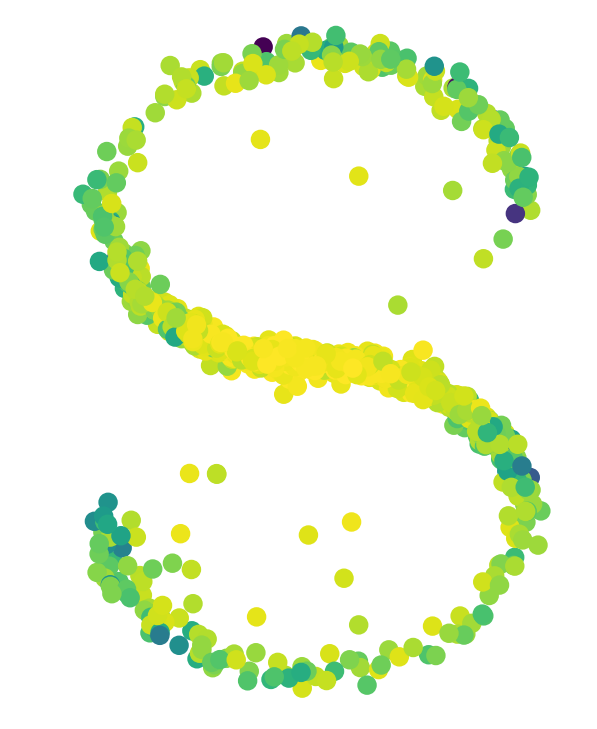}};
            
        \node[inner sep=0pt, label={\small },opacity=\imOp] (b15_1) at (\diaW/4,-\diaJump-\nextRow)
            {\includegraphics[width=\imW\textwidth]{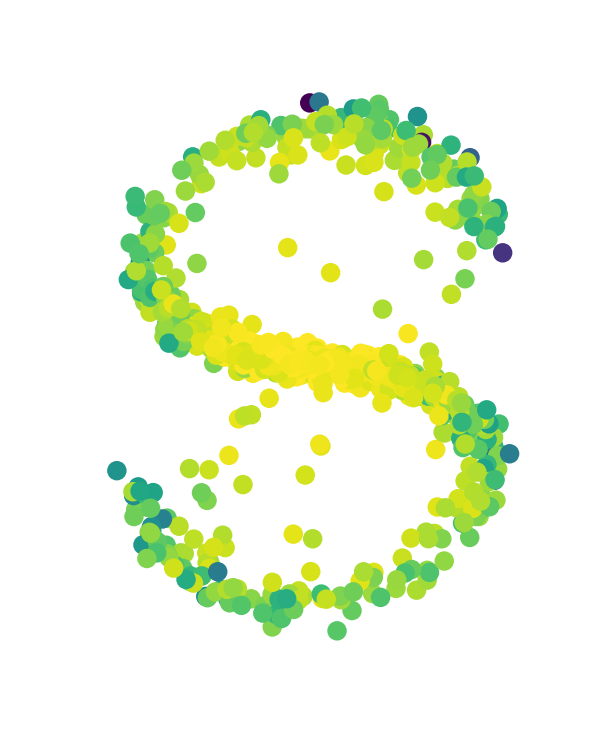}};
            
        \node[inner sep=0pt, label={\small }, opacity=\imOp] (b15_2) at (\diaW/2,-\diaJump-\nextRow)
            {\includegraphics[width=\imW\textwidth]{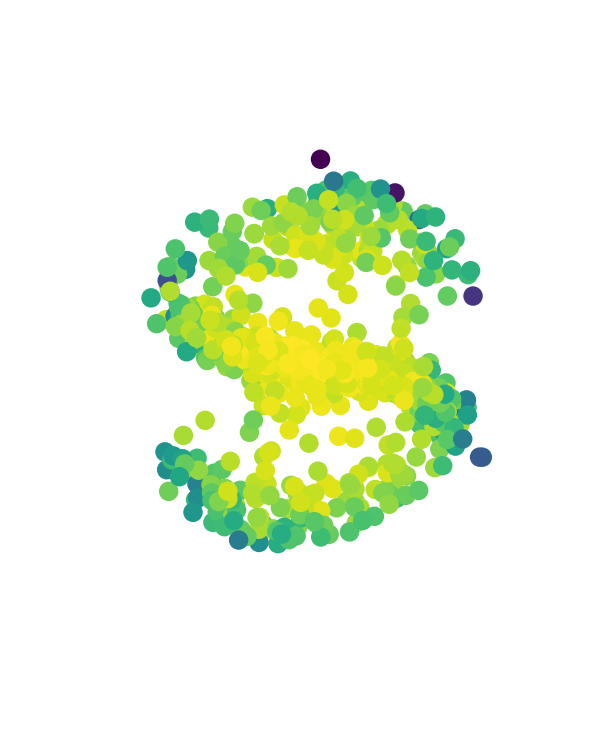}};
        
        \node[inner sep=0pt, label={\small }, opacity=\imOp] (b15_3) at (3*\diaW/4,-\diaJump-\nextRow)
            {\includegraphics[width=\imW\textwidth]{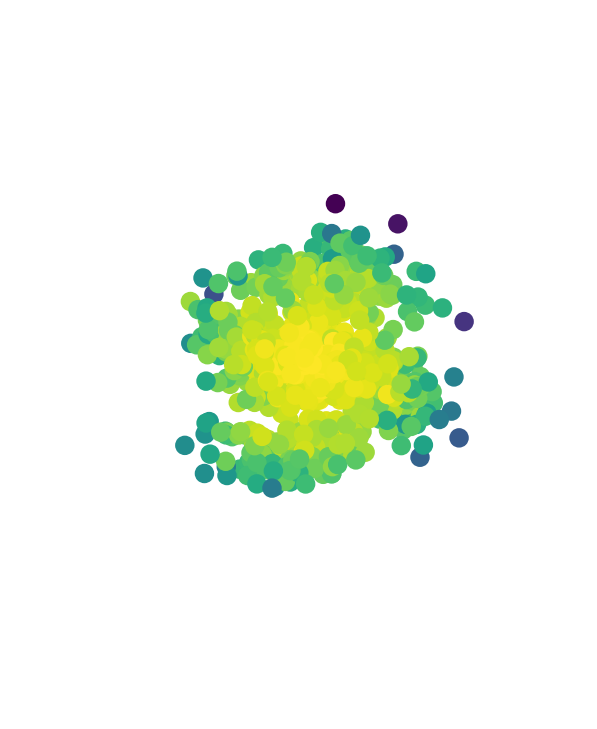}};
            
        \node[inner sep=0pt, label={\small }] (b15_T) at (\diaW,-\diaJump-\nextRow)
            {\includegraphics[width=\imW\textwidth]{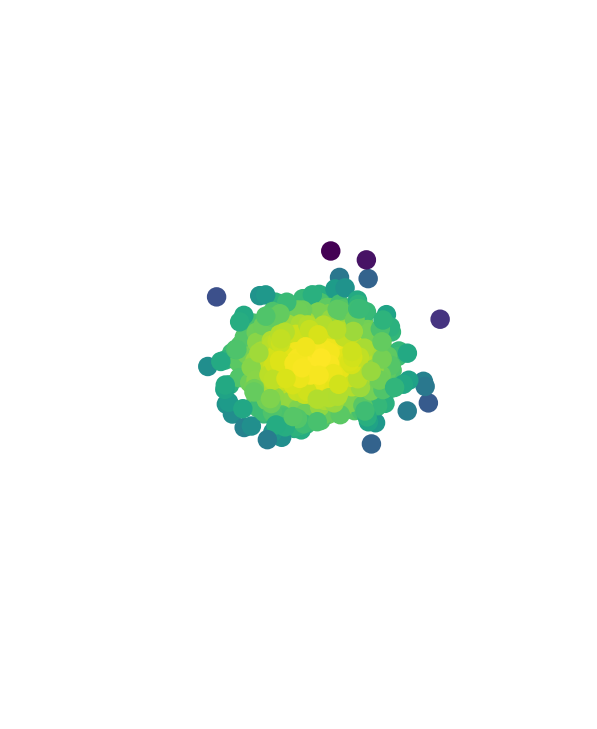}};
        
        % arrows 15 ------------------------------------------------------------------
        \draw[-{Latex[length=2mm]},dashed] (\diaW/2,-1.25*\diaH/4) -- (\diaW/2,-\diaJump+0.5) ;
        \draw[-{Latex[length=2mm]},dashed] (\diaW/4,-1.25*\diaH/4) -- (\diaW/4,-\diaJump+0.5) ;
        \draw[-{Latex[length=2mm]},dashed] (3*\diaW/4,-1.25*\diaH/4) -- (3*\diaW/4,-\diaJump+0.5) ;

        \path[-{Latex[length=3mm]}, thick]
         (f15_data) edge[bend left=\bend] node [fill=white,xshift=-35pt] {Forward} (f15_T)
         (f1_data) edge[bend left=\bend] node [fill=white,xshift=-35pt] {Forward} (f1_T)
         (b1_T) edge[bend left=\bend] node [fill=white,xshift=-35pt] {Backward} (b1_data)
         (b15_T) edge[bend left=\bend] node [fill=white,xshift=-35pt] {Backward} (b15_data);
        
        \draw[-{Latex[length=3mm]},very thick] (-1.2,0) -- (-1.2,-\diaH)
            node[pos=0.4,fill=white] {DSB Steps}
            node[pos=0.1,left,fill=white] {DSB 1}
            node[pos=0.7,left,fill=white] {DSB 5};
            
        \draw[-{Latex[length=3mm]}, very thick] (-1,-\diaH) -- (\diaW,-\diaH)
            node[midway,fill=white] {$t$}
            node[pos=0.1,below,fill=white] {$t=0$}
            node[pos=0.9,below,fill=white] {$t=T$};
        
        \draw[red,thick] ($(b15_data.north west)+(-0.05,0.05)$)  rectangle ($(b15_T.south east)+(0.05,-0.05)$)
        node[left,fill=white] {Generative Model};
      \end{tikzpicture}
      \vspace{-.5cm}
       \setlength{\belowcaptionskip}{-0.5cm}
        \caption{The reference forward diffusion initialized from the 2-dimensional data
          distribution fails to converge to the Gaussian prior in $T=0.2$ diffusion-time ($N=20$ discrete time steps), 
           and the reverse diffusion initialized from the Gaussian prior does not converge to the data distribution. However, convergence does occur after $5$ DSB iterations.}
       \label{fig:schro_bridge}          
     \end{figure}
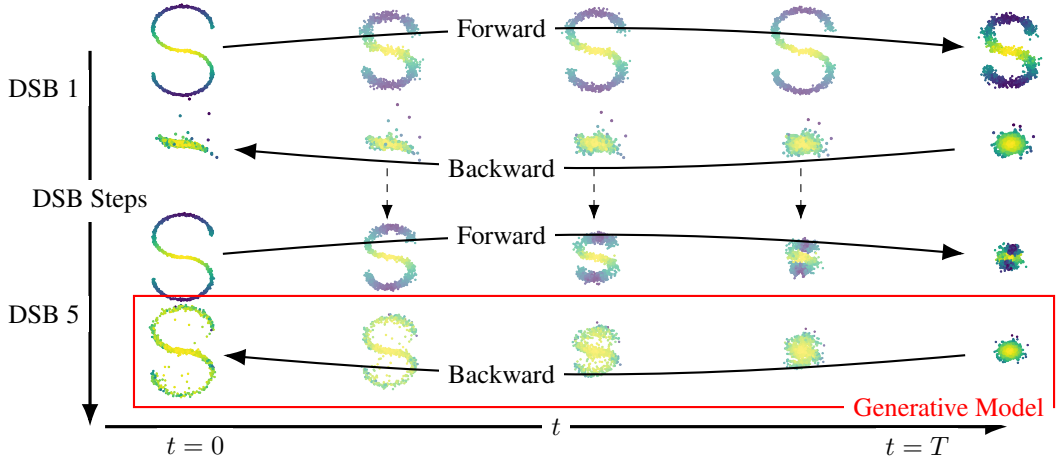
The corresponding reverse-time
SDE is an inhomogeneous diffusion whose drift depends on the logarithmic
gradients of the perturbed data distributions, \ie \ the scores. In practice,
these scores are approximated using neural networks and score-matching
techniques \citep{hyvarinen2005estimation,vincent2011connection} while numerical SDE integrators are
used for the sampling procedure.
 
Although SGM provides state-of-the-art results \citep{nichol2021beatgans},
sample generation is computationally expensive. In order to learn the
reverse-time SDE from the prior, \ie \ the generative model, the forward noising
SDE must be run for a sufficiently long time to converge to the prior and the step size
must be sufficiently small to obtain a good numerical approximation of this SDE. By reformulating
generative modeling as a \schro bridge (SB) problem, we mitigate this issue and
propose a novel algorithm to solve SB problems. Our detailed contributions are
as follows.
  
     \textbf{Generative modeling as a \schro bridge problem.} The SB problem is
     a famous entropy-regularized Optimal Transport (OT) problem introduced by
     \cite{schrodinger1932theorie}; see \eg \
     \citep{leonard2014survey,chen2020optimal} for reviews. Given a reference
     diffusion with finite time horizon $T$, a data distribution and a prior
     distribution, solving the SB amounts to finding the closest diffusion to the reference
     (in terms of Kullback--Leibler divergence on path spaces) which admits the
     data distribution as marginal at time $t=0$ and the prior at time
     $t=T$. The reverse-time diffusion solving this SB problem
     provides a new SGM algorithm which enables approximate sample generation
     from the data distribution using shorter time intervals compared to the
     original SGM methods. 
     Our method differs from the entropy-regularized OT
     formulation in \citep{genevay2018learning}, which deals with discrete distributions and relies on a static formulation of SB, as opposed
     to our dynamical approach for continuous distributions which operates on
     path spaces. It also differs from \citep{finlay2020learning} which
     approximates the SB solution by a diffusion whose drift is computed
     using potentials of the dual formulation of SB. Finally, \cite{wang2021deepschro} have recently proposed to perform generative modeling by solving not one but two SB problems. Contrary to us, they do not formulate generative modeling as computing the SB between the data and prior distributions.
    
     \textbf{Solving the \schro bridge problem using score-based diffusions.}
     The SB problem can be solved using Iterative Proportional
     Fitting (IPF)
     \citep{fortet1940resolution,kullback1968probability,chen2020optimal}. We
     propose Diffusion SB (DSB), a novel implementation of IPF using score-based
     diffusion techniques. DSB does not
     require discretizing the state-space
     \citep{chen2016entropic,reich2018data}, approximating potential functions
     using regression
     \citep{bernton2019schr,dessein2017parameter,pavon2018data}, nor performing
     kernel density estimation \citep{pavon2018data}. The first DSB iteration
     recovers the method proposed by \cite{song2020score}, with the flexibility
     of using shorter time intervals, as additional DSB iterations reduce the
     discrepancy between the final-time marginal of the forward (resp. backward)
     SDE w.r.t. the prior (resp. data) distribution; see \Cref{fig:schro_bridge}
     for an illustration. An algorithm akin to DSB has been proposed
     concurrently and independently by \cite{vargas2021solving}; the main
     difference with our algorithm is that they estimate the drifts of the SDEs
     using Gaussian processes while we use neural networks and score matching
     ideas.
     
     \textbf{Theoretical results.} We provide the first quantitative convergence results
     for the methodology of \citet{song2020score}. In particular, we show that while we simulate Langevin-type diffusions in potentially extremely high-dimensional spaces, the SGM approach does \emph{not} suffer from poor mixing times. Additionally, we derive novel
     quantitative convergence results for IPF in continuous state-space which do
     not rely on classical compactness assumptions
     \citep{chen2016entropic,ruschendorf1995convergence} and improve on the
     recent results of \cite{leger2020gradient}. Finally, we show that DSB may be viewed as the time discretization of a dynamic version
     of IPF on path spaces based on forward/backward diffusions.
     
    \textbf{Experiments.} We validate our methodology by generating image
    datasets such as MNIST and CelebA. In particular, we show that using
    multiple steps of DSB always improve the generative model. We also show how
    DSB can be used to interpolate between two data distributions.

    \smallskip\noindent\textbf{Notation.}\label{sec:notation} In the
    continuous-time setting, we set $\contspace = \rmc(\ccint{0,T}, \rset^d)$
    the space of continuous functions from $\ccint{0,T}$ to $\rset^d$ and
    $\mcb{\contspace}$ the Borel sets on $\contspace$. For any measurable space
    $(\mse, \mce)$, we denote by $\Pens(\mse)$ the space of probability measures
    on $(\mse, \mce)$. For any $\ell \in \nset$, let
    $\Pens_\ell= \Pens((\rset^d)^\ell)$. When it is defined, we denote
    $\Ent(p) = -\int_{\rset^d} p(x) \log p(x)\rmd x$ as the entropy of $p$ and
    $\KLLigne{p}{q}$ as the Kullback--Leibler divergence between $p$ and $q$.
    When there is no ambiguity, we use the same notation for distributions
    and their densities. All proofs are postponed to the supplementary.

\section{Denoising Diffusion, Score-Matching and Reverse-Time SDEs}\label{sec:SGM}

    \subsection{Discrete-Time: Markov Chains and Time Reversal}
    \label{sec:discr-sett-mark}
    Consider a data distribution with positive density $\pdata$\footnote{In this
      presentation, we assume that all distributions admit a density w.r.t. the
      Lebesgue measure for simplicity. However, the algorithms presented here
      only require having access to samples from $\pdata$ and $\pprior$.}, a
    positive prior density $\pprior$ w.r.t. Lebesgue measure both with support on
    $\rset^d$ and a Markov chain with initial density $p_0=\pdata$ on
    $\mathbb{R}^d$ evolving according to positive transition densities $p_{k+1|k}$ for
    $k \in \{0, \dots,N-1\}$. Hence for any
    $x_{0:N}=\{x_k\}_{k=0}^N \in \mcx = (\rset^d)^{N+1}$, the joint density may be expressed as
\begin{equation}\label{eq:mu_forward}
           \textstyle{p(x_{0:N}) = p_0(x_0) \prod_{k=0}^{N-1}p_{k+1|k}(x_{k+1}|x_{k}). }
        \end{equation}
        This joint density also admits the backward decomposition 
        \begin{equation}\label{eq:timereversal}
          \textstyle{
           p(x_{0:N}) = p_N(x_N) \prod_{k=0}^{N-1}p_{k|k+1}(x_{k}|x_{k+1}),  \text{with~}  p_{k|k+1}(x_{k}|x_{k+1})=\frac{p_k(x_{k}) p_{k+1|k}(x_{k+1}|x_{k})}{p_{k+1}(x_{k+1})}, }
        \end{equation}
        where
        $p_k(x_{k})=\int
        p_{k|k-1}(x_k|x_{k-1})p_{k-1}(x_{k-1})\textrm{d}x_{k-1}$ is the marginal density at step $k \geq
        1$. For the purpose of generative modeling, we will choose transition densities such that
        $p_N(x_N) =\int p(x_{0:N})\textrm{d}x_{0:N-1}
        \approx \pprior(x_N)$ for
        large $N$, where $\pprior$ is an easy-to-sample \emph{prior} density. One
        may sample approximately from $\pdata$ using ancestral sampling with the
        reverse-time decomposition \eqref{eq:timereversal},~\ie \ first sample
        $X_N\sim \pprior$ followed by $X_k\sim p_{k|k+1}(\cdot|X_{k+1})$ for 
        $k \in \{N-1, \dots, 0\}$. This idea is at the core of all recent
        SGM methods. The reverse-time transitions in \eqref{eq:timereversal} cannot be
        simulated exactly but may be approximated if we consider a forward
        transition density of the form
        \begin{equation} \label{eq:euler_maru}
          \textstyle{p_{k+1|k}(x_{k+1}|x_{k})=\mathcal{N}(x_{k+1};x_{k}+\gamma_{k+1}
            f(x_{k}),2 \gamma_{k+1} \Idbf), }
          \end{equation} 
        with drift $f: \ \rset^d \to  \rset^d$ and stepsize $\gamma_{k+1}>0$. We first make the following approximation from \eqref{eq:timereversal}
        \begin{align}\label{eq:eulerbackward}
          p_{k|k+1}(x_{k}|x_{k+1})&= p_{k+1|k}(x_{k+1}|x_{k}) \exp[\log p_{k}(x_{k}) - \log p_{k+1}(x_{k+1})]\\
          &\approx \mathcal{N}(x_{k};x_{k+1}- \gamma_{k+1} f(x_{k+1})+2 \gamma_{k+1} \nabla \log p_{k+1}(x_{k+1}), 2\gamma_{k+1} \Idbf), \label{eq:approx_disc_uno}
        \end{align}
        using that $p_k\approx p_{k+1}$, a Taylor expansion of $\log p_{k+1}$ at
        $x_{k+1}$ and $f(x_{k}) \approx f(x_{k+1})$. In practice, the
        approximation holds if $\normLigne{x_{k+1} - x_k}$ is small which is ensured by
        choosing $\gamma_{k+1}$ small enough.
                Although $\nabla\log p_{k+1}$ is not available, one may 
        obtain an approximation using denoising score-matching methods
        \citep{hyvarinen2005estimation,vincent2011connection,song2020score}. 

        Assume that the conditional density $p_{k+1|0}(x_{k+1}|x_{0})$ is
        available analytically as in \citep{ho2020denoising,song2020score}.  We
        have
        $p_{k+1}(x_{k+1})=\int
        p_{0}(x_{0})p_{k+1|0}(x_{k+1}|x_{0})\textrm{d}x_{0}$ and elementary
        calculations show that
        $\nabla \log p_{k+1}(x_{k+1})=\mathbb{E}_{p_{0|k+1}}[\nabla_{x_{k+1}} \log
        p_{k+1|0}(x_{k+1}|X_{0})]$. We can therefore formulate score estimation
        as a regression problem and use a flexible class of functions, \eg \
        neural networks, to parametrize an approximation
        $s_{\theta^\star}(k,x_k) \approx \nabla \log p_{k}(x_k)$ such that
        \begin{equation}\label{eq:scorematching}
        \textstyle{
            \theta^\star=\argmin_{\theta} \sum_{k=1}^{N} \mathbb{E}_{ p_{0,k}}[||s_\theta(k, X_{k})-\nabla_{x_k} \log p_{k|0}(X_{k}|X_{0})||^2] },
        \end{equation}
        where $p_{0,k}(x_0,x_k)=p_0(x_0)p_{k|0}(x_k|x_0)$ is the joint density
        at steps $0$ and $k$.  If $p_{k|0}$ is not available, we use
        $\theta^\star = \argmin_{\theta} \mathbb
        \sum_{k=1}^N \mathbb{E}_{p_{k-1,k}}[||s_\theta(k,X_{k})-\nabla_{x_{k}} \log
        p_{k|k-1}(X_{k}|X_{k-1})||^2]$.  In summary, SGM involves first estimating
        the score function $s_{\theta^\star}$ from noisy data, and then
        sampling $X_0$ using $X_N \sim \pprior$ and the approximation 
        \eqref{eq:approx_disc_uno}, \ie
        \begin{equation}
          \label{eq:reverse_discrete}
          X_k = X_{k+1} - \gamma_{k+1} f(X_{k+1}) + 2 \gamma_{k+1} s_{\theta^\star}(k+1, X_{k+1}) + \sqrt{2 \gamma_{k+1}} Z_{k+1},  Z_{k+1} \overset{\textup{i.i.d.}}\sim \mathcal{N}(0,  \Idbf). 
        \end{equation}
        The random variable $X_0$ is approximately $p_0=\pdata$ distributed if $p_N(x_N)\approx \pprior(x_N)$. In what
        follows, we let $\{Y_k\}_{k=0}^{N} = \{X_{N-k}\}_{k=0}^{N}$ and remark that
        $\{Y_k\}_{k=0}^{N}$ satisfies a forward recursion.
    \subsection{Continuous-Time: SDEs, Reverse-Time SDEs and Theoretical results}
    For appropriate transition densities, \cite{song2020score} showed that the
    forward and reverse-time Markov chains may be viewed as discretized
    diffusions. We derive the continuous-time limit of the procedure
    presented in \Cref{sec:discr-sett-mark} and establish convergence results. The
    Markov chain with kernel \eqref{eq:euler_maru} corresponds to an
    Euler--Maruyama discretization of $(\bfX_t)_{t \in \ccint{0,T}}$, solving
    the following SDE
        \begin{equation}
          \label{eq:forward}
          \textstyle{
            \rmd \bfX_t = f(\bfX_t) \rmd t + \sqrt{2} \rmd \bfB_t  ,\quad \bfX_0 \sim p_0=\pdata,
            }
        \end{equation}
        where $(\bfB_t)_{t \in \ccint{0,T}}$ is a Brownian motion and
        $f: \ \rset^d \to \rset^d$ is regular enough so that (strong) solutions
        exist.  Under conditions on $f$, it is well-known (see
        \cite{haussmann1986time,follmer1985entropy,cattiaux2021time} for
        instance) that the reverse-time process
        $(\bfY_t)_{t \in \ccint{0,T}}=(\bfX_{T-t})_{t \in \ccint{0,T}}$
        satisfies
        \begin{equation}
          \label{eq:time_reversed}
          \textstyle{
          \rmd \bfY_t = \defEns{-f(\bfY_t) + 2 \nabla \log p_{T-t}(\bfY_t) } \rmd t + \sqrt{2}
          \rmd \bfB_t ,
          }
        \end{equation}
        with initialization $\bfY_0 \sim p_T$, where $p_t$ denotes the marginal density of $\bfX_t$. 
 
        The reverse-time Markov chain $\{Y_k\}_{k=0}^{N}$ associated with
        \eqref{eq:reverse_discrete} corresponds to an Euler--Maruyama
        discretization of \eqref{eq:time_reversed}, where the score functions
        $\nabla \log p_{t}(x)$ are approximated by $s_{\theta^{\star}}(t,x)$.

        In what follows, we consider $f(x) = -\alpha x$ for
        $\alpha \geq 0$. This framework includes the one of
        \cite{song2019generative} ($\alpha =0$,
        $\pprior(x)= \mathcal{N}(x;0, 2T \ \Idbf)$) for which
        $(\bfX_t)_{t \in \ccint{0,T}}$ is simply a Brownian motion and
        \cite{ho2020denoising} ($\alpha > 0$,
        $\pprior(x)= \mathcal{N}(x;0, \Idbf/\alpha)$) for which it is an
        Ornstein–Uhlenbeck process, see \Cref{sec:comparison-with-ho} for more
        details.  Contrary to \cite{song2020score} we consider time homogeneous
        diffusions. Both approaches approximate \eqref{eq:reverse_discrete} using
        distinct discretizations but our setting leverages the ergodic
        properties of the Ornstein--Uhlenbeck process 
        to establish \Cref{prop:convergence_score_matching}.
          \begin{theorem}    
            \label{prop:convergence_score_matching}
            Assume that there exists $\Mtt \geq 0$ such that for any
            $t \in \ccint{0,T}$ and $x \in \rset^d$
            \begin{equation}
              \label{eq:approx}
              \textstyle{
                \norm{s_{\theta^{\star}}(t,x) - \nabla \log p_{t}(x)} \leq \Mtt ,
                }
              \end{equation}
              with 
              $s_{\theta^\star} \in \rmc(\ccint{0,T} \times \rset^d, \rset^d)$.
              Assume that $\pdata \in \rmc^3(\rset^d, \ooint{0, +\infty})$ is
              bounded and that there exist $d_1, A_1, A_2, A_3 \geq 0$,
              $\beta_1, \beta_2, \beta_3 \in \nset$ and
              $\mtt_1 > 0$ such that for any $x \in \rset^d$ and
              $i \in \{1, 2, 3\}$
            \begin{equation}
              \textstyle{
              \normLigne{\nabla^i \log \pdata(x)} \leq A_i(1 + \normLigne{x}^{\beta_i}) , \quad \langle \nabla \log \pdata(x), x \rangle \leq -\mtt_1 \norm{x}^2 + d_1 \norm{x} ,}
          \end{equation}
          with $\beta_1 = 1$.  Then for any $\alpha \geq 0$, there exist
          $B_\alpha, C_\alpha, D_\alpha \geq0$ such that for any $N \in \nset$
          and $\{\gamma_k\}_{k=1}^N$ with $\gamma_k > 0$ for any
          $k \in \{1, \dots, N\}$, the following bounds on the total variation distance hold:
            \begin{enumerate}[wide, labelwidth=!, labelindent=0pt, label=(\alph*)]
            \item  if $\alpha > 0$,  we have $\tvnormLigne{\mathcal{L}(X_0)-\pdata} \leq   C_\alpha(\Mtt +  \bgamma^{1/2}) \exp[D_\alpha T]+ B_\alpha \exp[-\alpha^{1/2} T];      $
            \item if $\alpha = 0$, we have
              $\tvnormLigne{\mathcal{L}(X_0)-\pdata} \leq  C_0(\Mtt + \bgamma^{1/2}) \exp[D_0 T]+B_0(T^{-1} +  T^{-1/2});$
          \end{enumerate}
          where $T = \sum_{k=1}^N \gamma_k$, $\bgamma = \sup_{k \in \{1, \dots, N\}} \gamma_k$ and
          $\mathcal{L}(X_0)$ is the distribution of $X_0$ given in \eqref{eq:reverse_discrete}.
          \end{theorem}
          
          \begin{proof}
            We provide here a sketch of the proof. The whole proof is detailed in \Cref{prop:convergence_score_matching:proof}. Denote $\Pbb \in \Pens(\mathcal{C})$ the path measure associated with \eqref{eq:forward} and $\Pbb^R$ its time-reversal. Denote $\Qker_N$ the Markov kernel taking us from $Y_0$ to $Y_N$ induced by \eqref{eq:reverse_discrete}. We have
             \begin{align}
      \tvnormLigne{\pprior \Qker_N - \pdata} &= \tvnormLigne{\pprior \Qker_N - \pdata \Pbb_{T|0} (\Pbb^R)_{T|0}} \\
                                         &\leq \tvnormLigne{\pprior \Qker_N - \pprior (\Pbb^R)_{T|0}} + \tvnormLigne{\pprior (\Pbb^R)_{T|0} - \pdata \Pbb_{T|0} (\Pbb^R)_{T|0}} \\
          &\leq  \tvnormLigne{\pprior \Qker_N - \pprior (\Pbb^R)_{T|0}} + \tvnormLigne{\pprior - p_T} .                                  
    \end{align}
    We control the first term by bounding the discretization error of $\Qker_N$ when compared to $(\Pbb^R)_{T|0}$ via the Girsanov theorem. The second term is controlled using the mixing properties of the forward diffusion process. %As this process is a Brownian or Ornstein-Uhlenbeck process, we can obtain sharp bounds for this term.
          \end{proof}
          
          Condition \eqref{eq:approx} ensures that the neural network
          approximates the score with a given precision $\Mtt \geq 0$.  Under
           \eqref{eq:approx} and conditions on $\pdata$,
          \Cref{prop:convergence_score_matching} states how the Markov chain
          defined by \eqref{eq:reverse_discrete} approximates $\pdata$ in the
          total variation norm $\tvnormLigne{\cdot}$.  
          %In both cases,
          %$\alpha =0$ and $\alpha > 0$, the error consists of two terms. The first term stems from the error
          %between the continuous-time process \eqref{eq:time_reversed} with
          %initialization $\bfY_0 \sim \pprior$ and its discrete-time
          %approximation \eqref{eq:reverse_discrete}. The
          %second term decreases with $T \geq 0$ and corresponds to the error
          %between $p_T$ and $\pprior$.  
          The bounds of \Cref{prop:convergence_score_matching} show that
          there is a trade-off between the mixing properties of the forward diffusion
          which increases with $\alpha$, and the quality of the discrete-time
          approximation which deteriorates as $\alpha$ and $T$ increase, since
          $B_{\alpha}, C_{\alpha} D_\alpha \to_{\alpha \to +\infty} +\infty$. Indeed increasing $\alpha$ makes the drift steeper and the continuous-time process converges faster but smaller step sizes are required in order to control the error between the discrete and the continuous-time processes. \Cref{prop:convergence_score_matching} is the first theoretical result assessing the convergence of SGM methods. Indeed while \cite{block2020generative} establish convergence results for a \emph{time-homogeneous} Langevin diffusion targeting a density whose score is approximated by a neural network, all SGM methods used in practice rely on \emph{time-inhomogeneous} processes. Contrary to the time-homogeneous case, this approach does not suffer from poor mixing times as the mixing time dependency in the bounds of \Cref{prop:convergence_score_matching} is entirely determined by the mixing time of the \emph{forward} process, given by a simple Brownian motion or an Ornstein--Ulhenbeck process, and is independent of the dimension. Finally, note that \eqref{eq:approx} is a strong assumption. In practice we expect to obtain such bounds in expectation over $X$ with high probability w.r.t. the data distribution as in \cite[Proposition 9]{block2020generative}. Our results are also related to 
          \citep[Theorem 3.1]{raginsky2019theoretical} which establishes the expressiveness of related generative models using tools from stochastic control.
    \section{Diffusion \schro Bridge and Generative Modeling}
    \label{sec:schro-bridges}
    
\subsection{\schro Bridges}
\label{sec:schro-bridge-problem}
The SB problem is a classical problem appearing in applied
mathematics, optimal control and probability; see
\eg \ \citet{follmer1988random,leonard2014survey,chen2020optimal}.  In the
discrete-time setting, it takes the following (dynamic) form. Consider as
\emph{reference} density $p(x_{0:N})$ given by \eqref{eq:mu_forward}, describing
the process adding noise to the data.  We aim to find $\pi^\star \in \Pens_{N+1}$ such
that
\begin{equation}
  \label{eq:discrete_schro}
  \pi^\star = \argmin \ensemble{\KLLigne{\pi}{p}}{\pi \in \Pens_{N+1}, \  \pi_0 = \pdata,~ \pi_N = \pprior}. 
\end{equation}
Assuming $\pi^\star$ is available, a generative model can be obtained by
sampling $X_N \sim \pprior$, followed by the reverse-time
dynamics $X_k \sim \pi^\star_{k|k+1}(\cdot | X_{k+1})$ for 
$k \in \{N-1, \dots, 0\}$. Before deriving a method to approximate
$\pi^\star$ in \Cref{sec:iterativeproportionalfitting}, we highlight some desirable
features of \schro bridges.
\paragraph{Static \schro bridge problem.} First, we recall that the dynamic
formulation \eqref{eq:discrete_schro} admits a static analogue. Using \eg \
\citet[Theorem 2.4]{leonard2014some}, the following decomposition holds for any
$\pi \in \Pens_{N+1}$,
$
\KLLigne{\pi}{p}=\KLLigne{\pi_{0,N}}{p_{0,N}}+\mathbb{E}_{\pi_{0,N}}[\KLLigne{\pi_{|0,N}}{p_{|0,N}}]$,
where for any $\mu \in \Pens_{N+1}$ we have $\mu=\mu_{0,N}\mu_{|0,N}$ with
$\mu_{|0,N}$ the conditional distribution of $X_{1:N-1}$ given $X_0,X_N$\footnote{See
  \Cref{sec:addit-form-kullb} for a rigorous presentation using the disintegration theorem 
  for probability measures.}.  Hence we have
$\pi^\star(x_{0:N}) = \pi^{\rms, \star}(x_0,x_N) p_{|0,N}(x_{1:N-1}|x_0,x_N) $
where $\pi^{\rms, \star} \in \Pens_2$ with marginals $\pi^{\rms,\star}_0$ and
$\pi^{\rms,\star}_N$ is the solution of the static SB problem
\begin{equation}
  \label{eq:static_sb}
  \textstyle{\pi^{\rms, \star} = \argmin \ensemble{\KLLigne{\pi^\rms}{p_{0,N}}}{\pi^\rms \in \Pens_2,
    \ \pi^\rms_0 = \pdata,~ \pi^\rms_N = \pprior}. }
  \end{equation}
  \smallskip\noindent\textbf{Link with optimal transport.}
  Under mild assumptions, the static SB problem can be seen as an
  entropy-regularized optimal transport problem since \eqref{eq:static_sb} is
  equivalent to
  \begin{equation}\label{eq:link_ot}
    \pi^{\rms, \star} 
    = \textstyle{\argmin  \ensemble{-\mathbb{E}_{\pi^\rms}[\log p_{N|0}(X_N|X_0)]- \Ent(\pi^\rms)} {\pi^\rms \in \Pens_2, \ \pi_0^\rms = \pdata,~ \pi_N^\rms = \pprior}. }
\end{equation}
If $p_{k+1|k}(x_{k+1}|x_k)=\mathcal{N}(x_{k+1};x_k,\sigma^2_{k+1})$ as in
\cite{song2019generative}, then $p_{N|0}(x_N|x_0)=\mathcal{N}(x_N;x_0,\sigma^2)$
with $\sigma^2=\sum_{k=1}^N \sigma_k^2$ which induces a quadratic cost and
\begin{align}
  \label{eq:discrete_schro_ot}
  \textstyle{
  \pi^{\rms, \star} = \argmin  \ensemble{\mathbb{E}_{\pi^\rms}[||X_0-X_N||^2]- 2\sigma^2 \Ent(\pi^\rms)} {\pi^\rms \in \Pens_2, \ \pi_0^\rms = \pdata,~ \pi_N^\rms = \pprior}.
  }
\end{align}
\cite{mikami2004monge} showed that $\pi^{\rms, \star} \to \pi^\star_{\wass}$
weakly and
$2 \sigma^2\KLLigne{\pi^{\rms, \star}}{p_{0,N}} \to \wass_2^2(\pdata, \pprior)$
as $\sigma \to 0$, where $\pi_{\wass}^\star$ is the optimal transport plan
between $\pdata$ and $\pprior$ and $\wass_2$ is the 2-Wasserstein distance. Note
that the transport cost $c(x,x')=-\log p_{N|0}(x'|x)$ is not necessarily
symmetric.
\subsection{Iterative Proportional Fitting and Time
  Reversal} \label{sec:iterativeproportionalfitting} In all but trivial cases,
the SB problem does not admit a closed-form solution. However, it can be solved
using Iterative Proportional Fitting (IPF)
\citep{fortet1940resolution,kullback1968probability,ruschendorf1995convergence}
which is defined by the following recursion for $n\in \nset$ with initialization
$\pi^0=p$ given in \eqref{eq:mu_forward}:
\begin{align}\label{eq:IFPrecursion}
  &\textstyle{\pi^{2n+1} = \argmin \ensemble{\KLLigne{\pi}{\pi^{2n}}}{\pi \in \Pens_{N+1}, \ \pi_N = \pprior},} \\
  &\textstyle{\pi^{2n+2} = \argmin  \ensemble{\KLLigne{\pi}{\pi^{2n+1}}}{\pi \in \Pens_{N+1}, \ \pi_0 = \pdata}. }
\end{align}
This sequence is well-defined if there exists $\tilde{\pi} \in \Pens_{N+1}$ such
that $\tilde{\pi}_0 = \pdata$, $\tilde{\pi}_N = \pprior$ and
$\KLLigne{\tilde{\pi}}{p} < +\infty$. A standard representation of $\pi^{n}$ is obtained by updating the joint density $p$ using potential functions, see \Cref{prop:IPFpotential:proof} for details.
However, this representation of the IPF iterates is difficult to approximate
as it requires approximating the potentials. Our methodology builds upon an alternative
representation that is better suited to numerical approximations for generative modeling where one has access to samples of 
$\pdata$ and $\pprior$.
\begin{proposition}
  \label{prop:IPFrecursion}
  Assume that $\KLLigne{\pdata \otimes \pprior}{p_{0,N}} < +\infty$.  Then for any
  $n \in \nset$, $\pi^{2n}$ and $\pi^{2n+1}$ admit positive densities w.r.t. the Lebesgue measure denoted as $p^n$ resp. $q^n$  and for any $x_{0:N} \in \mcx$, we have $p^0(x_{0:N})=p(x_{0:N})$ and
  \begin{equation}
    \textstyle{
      q^n(x_{0:N}) = \pprior(x_N) \prod_{k=0}^{N-1} p^n_{k|k+1}(x_k|x_{k+1}),}
    \
    \textstyle{p^{n+1}(x_{0:N}) = \pdata(x_0) \prod_{k=0}^{N-1} q^n_{k+1|k}(x_{k+1}|x_{k}).}
  \end{equation} 
\end{proposition}
In practice we have access to $p^n_{k+1|k}$ and
$q^n_{k|k+1}$. Hence, to compute $p^n_{k|k+1}$ and
$q^n_{k+1|k}$ we use
\begin{equation}  
p^n_{k|k+1}(x_k|x_{k+1}) = \frac{p^n_{k+1|k}(x_{k+1}|x_k) p_k^n(x_k)}{p_{k+1}^n(x_{k+1})}, \ q^n_{k+1|k}(x_{k+1}|x_{k}) =\frac{q^n_{k|k+1}(x_{k}|x_{k+1}) q_{k+1}^n(x_{k+1})}{q_k^n(x_{k})}. 
\end{equation}To the best of our knowledge, this representation of the IPF iterates has surprisingly neither been presented nor explored in the literature. 
One may interpret these formulas as follows. At iteration $2n$, we have $\pi^{2n}=p^n$ with $p^0=p$
given by the noising process \eqref{eq:mu_forward}. This forward process initalized with $p^{n}_0=\pdata$ defines reverse-time transitions $p^{n}_{k|k+1}$, which, when combined with an initialization $\pprior$ at step $N$ defines the reverse-time process $\pi^{2n+1}=q^{n}$. The forward transitions $q^{n}_{k+1|k}$ associated to $q^n$ are then used to obtain $\pi^{2n+2}=p^{n+1}$. IPF then iterates this procedure.

\subsection{Diffusion \schro Bridge as Iterative Mean-Matching Proportional Fitting}
\label{sec:iter-score-match}
To approximate the IPF recursion defined in \Cref{prop:IPFrecursion}, we use
similar approximations to \Cref{sec:discr-sett-mark}. If at step $n \in \nset$
we have
$p^{n}_{k+1|k}(x_{k+1}|x_{k})= \mathcal{N}(x_{k+1};x_k+\gamma_{k+1}
f^{n}_k(x_k),2 \gamma_{k+1} \Idbf)$ where $p^0=p$ and $f^0_k=f$, then we can
approximate the reverse-time transitions in \Cref{prop:IPFrecursion} by
\begin{align}
  \label{eq:approx_gauss}
  q^n_{k|k+1}(x_{k}|x_{k+1})&= p^{n}_{k+1|k}(x_{k+1}|x_{k}) \exp[\log p^{n}_{k}(x_{k}) - \log p^{n}_{k+1}(x_{k+1})]\\
  &\approx \mathcal{N}(x_{k};x_{k+1}+ \gamma_{k+1} b^{n}_{k+1}(x_{k+1}), 2\gamma_{k+1} \Idbf),
\end{align}
with
$b^{n}_{k+1}(x_{k+1}) = -f^{n}_{k}(x_{k+1})+2 \nabla \log p^{n}_{k+1}(x_{k+1})$.
We can also approximate the forward transitions in
\Cref{prop:IPFrecursion} by
$ p^{n+1}_{k+1|k}(x_{k+1}|x_{k}) \approx \mathcal{N}(x_{k+1};x_{k}+ \gamma_{k+1}
f^{n+1}_{k}(x_{k}), 2\gamma_{k+1} \Idbf)$ with
$f^{n+1}_{k}(x_{k}) = -b^{n}_{k+1}(x_{k})+2 \nabla \log q^{n}_{k}(x_k)$. Hence
we have
$f^{n+1}_k(x_k)=f^{n}_{k}(x_k)-2 \nabla \log p^{n}_{k+1}(x_k)+2 \nabla \log
q^{n}_{k}(x_k)$. It follows that one could
estimate $f^{n+1}_k, b^{n+1}_k$ by using score-matching to approximate $\{\nabla \log p^{i}_{k+1}(x)\}_{i=0}^{n}$ ,
$\{\nabla \log q^{i}_{k}(x)\}_{i=0}^{n}$. This approach is prohibitively costly in terms of
memory and compute, see \Cref{sec:altern-vari-form}. 
We follow an alternative approach 
which avoids these difficulties.

\begin{proposition}\label{prop:generalizedscorematching} 
  Assume that for any $n \in \nset$ and $k \in \{0, \dots, N-1\}$,
  \begin{equation}
    q_{k|k+1}^n(x_k|x_{k+1}) = \mathcal{N}(x_k;B_{k+1}^n(x_{k+1}), 2\gamma_{k+1}
  \Idbf)  ,\ p_{k+1|k}^n(x_{k+1}|x_{k}) = \mathcal{N}(x_{k+1};F_{k}^n(x_{k}), 2\gamma_{k+1}
  \Idbf) ,
  \end{equation}
 with $B^{n}_{k+1}(x) = x +\gamma_{k+1}b^n_{k+1}(x)$,
  $F^{n}_{k}(x)= x +\gamma_{k+1}f_k^{n}(x)$ for any $x \in \rset^d$. Then we
  have for any $n \in \nset$ and $k\in \{0, \dots, N-1\}$
\begin{align}
&\textstyle{B^n_{k+1}=\argmin_{\mathrm{B}\in \rmL^2(\rset^d, \rset^d)} \expeMarkovLigne{{p^{n}_{k,k+1}}}{\normLigne{\mathrm{B}(X_{k+1})-(X_{k+1} + F^{n}_k(X_{k})-F^{n}_{k}(X_{k+1}))}^2}},\label{eq:regressionb}\\
&\textstyle{F^{n+1}_{k}=\argmin_{\mathrm{F}\in \rmL^2(\rset^d, \rset^d)} \expeMarkovLigne{q^{n}_{k,k+1}}{\normLigne{\mathrm{F}(X_k)-(X_{k} + B^{n}_{k+1}(X_{k+1})-B^n_{k+1}(X_{k}))}^2 }}.\label{eq:regressionf}
 \end{align} 
\end{proposition}
\Cref{prop:generalizedscorematching} shows how one can recursively approximate
$B^n_{k+1}$ and $F^{n+1}_k$. In practice, we use neural networks
$B_{\beta^n}(k,x)\approx B^{n}_k(x)$ and $F_{\alpha^n}(k,x)\approx F^{n}_k(x)$. Note that the networks could also be learned jointly. In this case, at equilibrium, we would obtain a bridge between $\pdata$ and $\pprior$ but not necessarily the Schr\"{o}dinger bridge.

Network parameters $\alpha^n, \beta^n$ are learnt through
   gradient descent to minimize empirical versions of the sum over $k$ of the
   loss functions given by \eqref{eq:regressionb} and \eqref{eq:regressionf}
   computed using $M$ samples and denoted as $\hlb_n(\beta)$ and
   $ \hlf_{n+1}(\alpha)$. The resulting algorithm approximating $L \in \nset$ IPF
   iterations is called Diffusion \schro Bridge (DSB) and is summarized in
   \Cref{algo:ipf_score} with
   $Z^j_k, \tilde{Z}_k^j \overset{\textup{i.i.d.}}\sim \mathcal{N}(0, \Idbf)$,
   see \Cref{fig:schro_bridge} for an illustration. 
   
\begin{minipage}{.57\textwidth}
\begin{algorithm}[H]
    \caption{Diffusion \schro Bridge}
    \label{algo:ipf_score}
    \begin{algorithmic}[1] 
      \FOR{$n \in \{0, \dots,L\}$} \WHILE{not converged}
      \STATE Sample $\{X^j_{k}\}_{k,j=0}^{N,M}$, where  $X^j_0 \sim \pdata$, and \\
      $X^{j}_{k+1} = F_{\alpha^n}(k, X^{j}_{k})+\sqrt{2
        \gamma_{k+1}} Z^{j}_{k+1}$ 
        \STATE Compute $\hlb_n(\beta^n)$ approximating \eqref{eq:regressionb}
        \STATE $\beta^{n} \leftarrow \textrm{Gradient Step}(\hlb_n(\beta^n))$ 
      \ENDWHILE \WHILE{not
        converged}
      \STATE Sample $\{X^j_{k}\}_{k,j=0}^{N,M}$, where $X^j_N \sim \pprior$, and \\
      $X^j_{k-1}=B_{\beta^n}(k, X^{j}_k)+\sqrt{2 \gamma_{k}}
      \tilde{Z}^{j}_{k}$ 
      \STATE Compute $\hlf_{n+1}(\alpha^{n+1})$ approximating \eqref{eq:regressionf}
      \STATE
      $\alpha^{n+1} \leftarrow \textrm{Gradient Step}(\hlf_{n+1}(\alpha^{n+1}))$
      \ENDWHILE \ENDFOR \STATE \textbf{Output: } $(\alpha^{L+1}, \beta^{L})$
    \end{algorithmic}
  \end{algorithm}
\end{minipage}
\hfill
 \begin{minipage}{.40\textwidth}
   \vspace{10pt} The DSB algorithm is initialized using the reference dynamics
   $f_{\alpha^0}(k,x)=f(x)$.  Once $\beta^L$ is learnt we can easily
   approximately sample from $\pdata$ by sampling $X_N \sim \pprior$ and then
   using $X_{k-1} = B_{\beta^L}(k, X_k)+ \sqrt{2 \gamma_{k}} Z_{k}$ with
   $Z_k \overset{\textup{i.i.d.}}\sim \mathcal{N}(0, \Idbf)$.  The resulting
   samples $X_0$ will be approximately distributed from $\pdata$. Although DSB
   requires learning a sequence of network parameters, $\alpha^n, \beta^n$,
   fewer diffusion steps are needed compared to standard SGM. In addition, as
   detailed in \Cref{sec:arch-deta-addit}, $\beta^0$ may be trained efficiently
   in a similar manner to previous SGM methods. Subsequent
   $\alpha^{n+1}, \beta^{n+1}$ are refinements of $\alpha^{n}, \beta^{n}$, hence
   may be fine-tuned from previous iterations.
 \end{minipage}

  \subsection{Convergence of Iterative Proportional Fitting}\label{subsec:convergenceIPF}
  In this section, we investigate the theoretical properties of IPF.  When the
  state-space is discrete and finite
  \citep{franklin1989scaling,peyre2019computational} or in the case where
  $\pdata$ and $\pprior$ are compactly supported \citep{chen2016entropic}, IPF
  converges at a geometric rate w.r.t.  the Hilbert-Birkhoff metric, see
  \cite{lemmens2013birkhoff} for a definition. Other than recent work by \cite{leger2020gradient}, only qualitative results exist in the general case where $\pdata$ or $\pprior$ is not compactly
  supported \citep{ruschendorf1995convergence,ruschendorf1993note}. 
  We establish here quantitative convergence
  of IPF in this non-compact setting as well as novel monotonicity results. We
  require only the following mild assumption.

\begin{assumption}
  \label{assum:existence_spec}
  $p_{N}, \pprior > 0$, $\absLigne{\Ent(\pprior)} < +\infty$, $\int_{\rset^d} \absLigne{\log p_{N|0}(x_N|x_0) } \pdata(x_0) \pprior(x_N) \rmd x_0 \rmd x_N < +\infty$.
\end{assumption}
Assumption \Cref{assum:existence_spec} is satisfied in all of our experimental
settings. We recall that for $\mu, \nu \in \Pens(\mse)$ with $(\mse, \mce)$ a
measurable space, the Jeffrey's divergence is given by
$\JefLigne{\mu}{\nu} = \KLLigne{\mu}{\nu} + \KLLigne{\nu}{\mu}$.
\begin{proposition}
  \label{prop:monotonicity}
  Assume \rref{assum:existence_spec}. Then $(\pi^n)_{n \in \nset}$ is well-defined and for any $n \geq 1$ we have
  \begin{equation}
    \label{eq:fundamental_kl}
    \KLLigne{\pi^{n+1}}{\pi^n} \leq \KLLigne{\pi^{n-1}}{\pi^n}, \qquad \KLLigne{\pi^n}{\pi^{n+1}} \leq \KLLigne{\pi^n}{\pi^{n-1}}.
  \end{equation}
  In addition, $(\tvnormLigne{\pi^{n+1} - \pi^{n}})_{n \in \nset}$ and $(\JefLigne{\pi^{n+1}}{\pi^{n}})_{n \in \nset}$ are non-increasing.
  Finally,  we have $\lim_{n \to +\infty} n \defEns{\KLLigne{\pi_0^n}{\pdata} +
    \KLLigne{\pi_N^n}{\pprior}} = 0$.
\end{proposition}

A more general result with additional monotonicity properties is given in
\Cref{sec:theor-study-schro}.  Under similar assumptions, \citet[Corollary
1]{leger2020gradient} established $\KLLigne{\pi^n_0}{p_0} \leq C/n$
with $C \geq 0$ using a Bregman divergence gradient descent perspective.  In contrast,
our proof relies only on tools from information geometry. In addition, we
improve the convergence rate and show that $(\pi^n)_{n \in \nset}$ converges in total variation
towards $\pi^{\infty}$, \ie \ we not only obtain convergence of the
marginals but also convergence of the joint distribution.  Under restrictive
conditions on $\pdata$ and $\pprior$, \cite{ruschendorf1995convergence} showed
that $\pi^\infty$ is the \schro bridge. In the following proposition, we avoid
this assumption using results on automorphisms of measures
\citep{beurling1960automorphism}.

\begin{proposition}
  \label{prop:convergence_ipf}
  Assume \rref{assum:existence_spec}.  Then there exists a solution
  $\pi^\star \in \Pens_{N+1}$ to the SB problem and we have 
  $\lim_{n \to +\infty} \tvnormLigne{\pi^n - \pi^\infty} = 0$ with
  $\pi^\infty \in \Pens_{N+1}$.  Let
  $h = p_{0,N} / (p_0 \otimes p_N)$ and assume that
  $h \in \rmc((\rset^d)^2, \ooint{0, +\infty})$ and that there exist
  $\Phi_0, \Phi_N \in \rmc(\rset^d, \ooint{0,+\infty})$ such that
  \begin{equation}
    \label{eq:growth}
    \textstyle{
      \int_{\rset^d \times \rset^d} (\abs{\log h(x_0,x_N)} + \abs{\log \Phi_0(x_0)} + \abs{\log \Phi_N (x_N)}) \pdata(x_0) \pprior(x_N) \rmd x_0 \rmd x_N < +\infty, 
      }
    \end{equation}
    with $h(x_0,x_N) \leq \Phi_0(x_0) \Phi_N(x_N)$.  If $p$ is
    absolutely continuous w.r.t. $\pi^\infty$ then $\pi^\infty = \pi^\star$.
\end{proposition}

\Cref{prop:convergence_ipf} extends previous IPF convergence results without the
assumption that the mapping $h$ is lower bounded, see
\cite{ruschendorf1995convergence,chen2016entropic}. Our assumption on $h$ can be
relaxed and replaced by a tighter condition on $\pi^{\infty}$, see
\Cref{prop:convergence_ipf:proof}.  \Cref{prop:monotonicity} suggests a
convergence rate of order $o(n)$ for the IPF in the non-compact
setting. However, in some situations, we recover geometric convergence rates
with explicit dependency w.r.t. the problem constants, see
\Cref{prop:geom_gaussian:proof}. In practice, we do not run IPF for
$\pdata, \pprior$ but using empirical versions of these distributions. Recent
results in \cite{deligiannidis2021quantitative} show that the iterates of IPF
based on empirical distributions remain close to the iterates one
would obtain using the true distributions, uniformly in time. In particular, the
SB computed using the empirical distributions converges to the one computed
using the true distributions as the number of samples goes to infinity.

\subsection{Continuous-time IPF}
\label{sec:continuous-ipf}
We describe an IPF algorithm for solving SB problems in continuous-time. We show that 
DSB proposed in \Cref{algo:ipf_score} can be seen as a discretization of
this  IPF. Given a reference measure $\Pbb \in \Pens(\contspace)$, the
continuous formulation of the SB involves solving the following problem
\begin{equation}
  \label{eq:dynamic_schro}
  \textstyle{
    \Pi^\star = \argmin \ensemble{\KLLigne{\Pi}{\Pbb}}{\Pi \in \Pens(\mathcal{C}), \ \Pi_0 = \pdata, \ \Pi_T = \pprior}, \quad T = \sum_{k=0}^{N-1}\gamma_{k+1}.
    }
\end{equation}
Similarly to
\eqref{eq:IFPrecursion}, we define the IPF $(\Pi^n)_{n \in \nset}$
with $\Pi^0 = \Pbb$ associated with \eqref{eq:forward} and for any $n \in \nset$
\begin{align}
  \textstyle{\Pi^{2n+1}} &= \textstyle{\argmin \ensemble{\KLLigne{\Pi}{\Pi^{2n}}}{\Pi \in \Pens(\mathcal{C}), \ \Pi_T = \pprior}, } \\
  \textstyle{\Pi^{2n+2}} &= \textstyle{\argmin \ensemble{\KLLigne{\Pi}{\Pi^{2n+1}}}{\Pi \in \Pens(\mathcal{C}), \ \Pi_0 = \pdata}.}
\end{align}
One can show that for any $n \in \nset$, $\Pi^n = \pi^{\rms, n} \Pbb_{|0,T}$,
with $(\pi^{\rms, n})_{n \in \nset}$ the IPF for the static SB problem.  In
particular, \Cref{prop:monotonicity} and \Cref{prop:convergence_ipf} extend to
the continuous IPF framework.
In what follows, for any $\Pbb \in \Pens(\contspace)$, we define $\Pbb^R$ as the
reverse-time measure, \ie \ for any $\msa \in \mcb{\contspace}$ we have
$\Pbb^R(\msa) = \Pbb(\msa^R)$ where
$\msa^R = \ensembleLigne{t \mapsto\omega(T-t)}{\omega \in \msa}$.
The following result is the continuous
counterpart of \Cref{prop:IPFrecursion} and states that each IPF iteration is
associated with a diffusion, showing that DSB can be seen as a
discretization of the continuous IPF.
\begin{proposition}
  \label{prop:continuous_schro}
  Assume \rref{assum:existence_spec} and that there exist
  $\Mbb \in \Pens(\contspace)$, $U \in \rmc^1(\rset^d, \rset)$, $C \geq 0$
  such that for any $n \in \nset$, $x \in \rset^d$,
  $\KLLigne{\Pi^n}{\Mbb} < +\infty$,
  $\langle x, \nabla U(x) \rangle \geq - C(1+\normLigne{x}^2)$ and $\Mbb$ is
  associated with
  \begin{equation}
    \label{eq:diff_q}
    \textstyle{
      \rmd \bfX_t = -\nabla U(\bfX_t) \rmd t + \sqrt{2} \rmd \bfB_t, 
      }
    \end{equation}
    with $\bfX_0$ distributed according to the invariant distribution of \eqref{eq:diff_q}.  Then,
    for any $n \in \nset$ we have:
  \begin{enumerate}[wide, labelwidth=!, labelindent=0pt, label=(\alph*)]
  \item $(\Pi^{2n+1})^R$ is associated with $\rmd \bfY_t^{2n+1} = b^n_{T-t}(\bfY_t^{2n+1}) \rmd t  + \sqrt{2} \rmd \bfB_t$ with $\bfY_0^{2n+1} \sim \pprior$;
  \item $\Pi^{2n+2}$ is associated with
    $\rmd \bfX_t^{2n+2} = f^{n+1}_t( \bfX_t^{2n+2}) \rmd t + \sqrt{2} \rmd
    \bfB_t$ with $\bfX_0^{2n+2} \sim \pdata$;
  \end{enumerate}
  \vspace{-.3cm} where for any $n \in \nset$, $t \in \ccint{0,T}$ and 
  $x \in \rset^d$, $b^{n}_t( x) = -f^{n}_t(x) +2 \nabla \log p^{n}_t(x)$, 
  $f^{n+1}_t(x) = -b^n_t(x) +2 \nabla \log q^n_t(x)$, with $f^0_t(x) = f(x)$, see  \eqref{eq:forward}, and $p^n_t$, $q_t^n$
  the densities of $\Pi^{2n}_t$ and  $\Pi_t^{2n+1}$.
\end{proposition}

\section{Experiments}
\label{sec:experiments}

\paragraph{Gaussian example.}

We first confirm that our algorithm recovers the true SB in a Gaussian setting where the ground truth is available. Let
$\pprior = \mathcal{N}(-a, \Idbf)$, $\pdata = \mathcal{N}(a, \Idbf)$ with
$a \in \rset^d$ and consider a Brownian motion as reference dynamics. The analytic expression for the static SB is $\mathcal{N}((-a,a), \Sigma)$ with
$\Sigma \in \rset^{2d \times 2d}$ given in \Cref{sec:convergence_ground}.  We let $a = 0.1 \times \mathbf{1}$ with
$d=50$ or $d=5$.
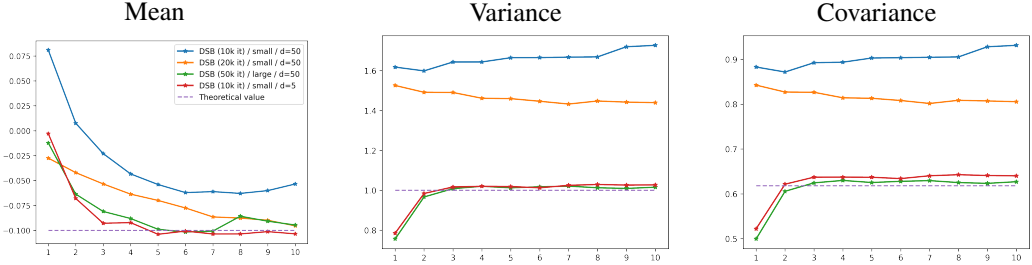
\begin{figure}[h]
  \centering
  \hfill
\begin{tikzpicture}

  \node at  (0,1.85)  {Mean};
  \node at  (\wwwww,1.85)  {Variance};
  \node at  (2*\wwwww,1.85)  {Covariance};
        \node[inner sep=0pt, label={\small},] (scurve_orig) at (0,0)
        {\includegraphics[width=0.3\linewidth, trim=0cm 0cm 0cm 0.7cm, clip]{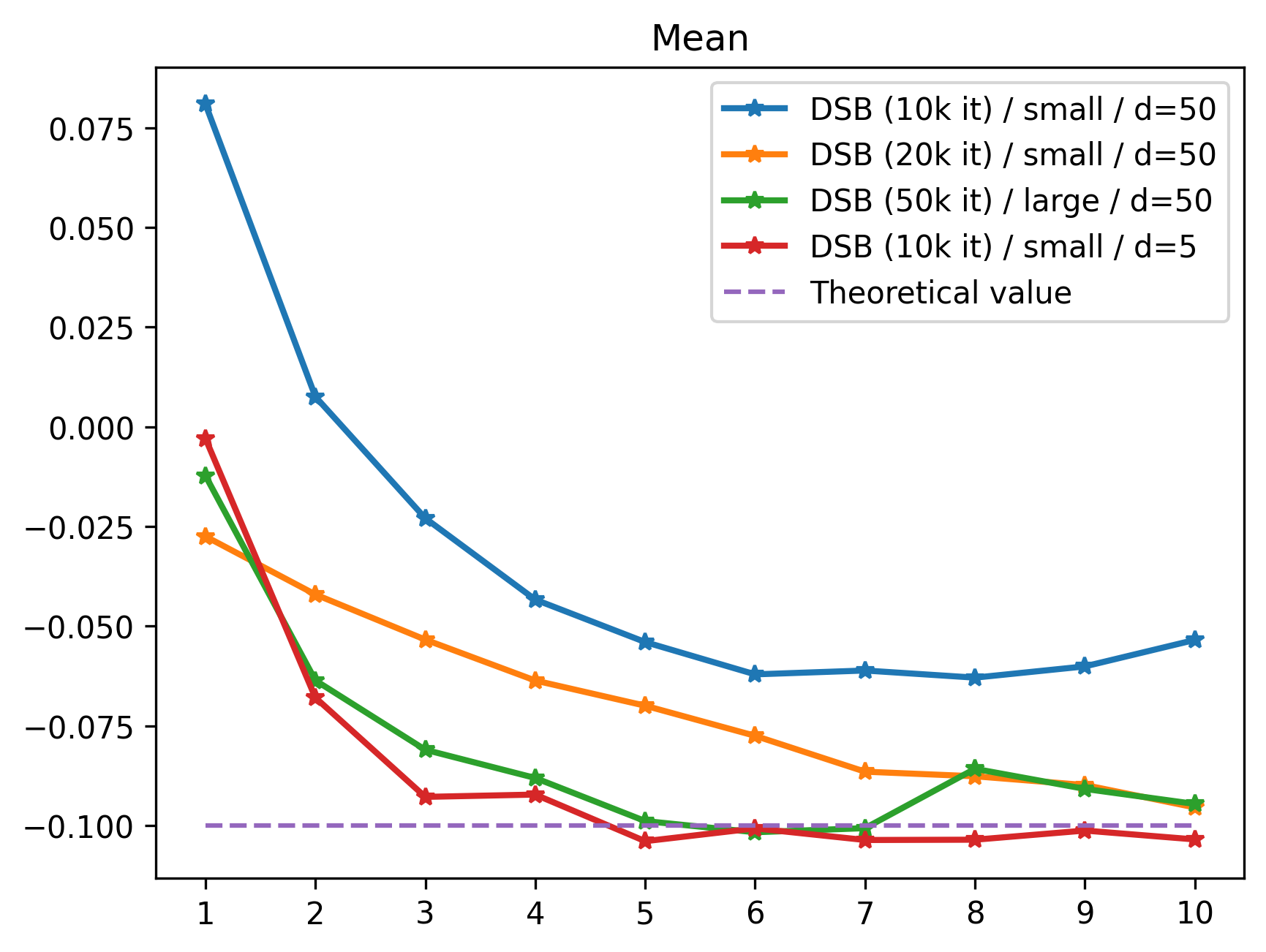}};
        \node[inner sep=0pt, label={\small }] (scurve_uno) at (\wwwww,0)
        {\includegraphics[width=0.3\linewidth, trim=0cm 0cm 0cm 0.7cm, clip]{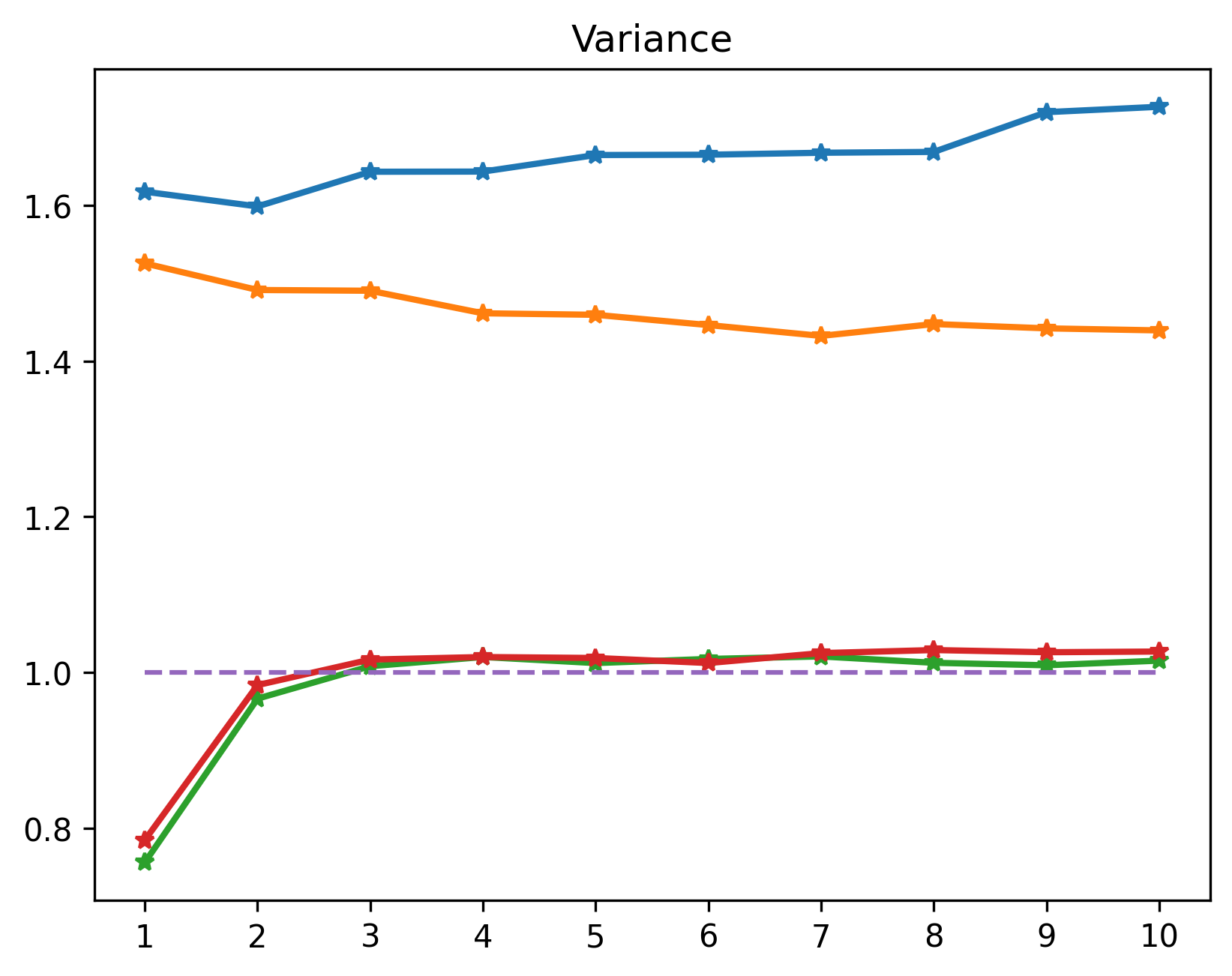}};        
        \node[inner sep=0pt, label={\small }] (middle) at (2*\wwwww,0)
            {\includegraphics[width=0.3\linewidth, trim=0cm 0cm 0cm 0.7cm, clip]{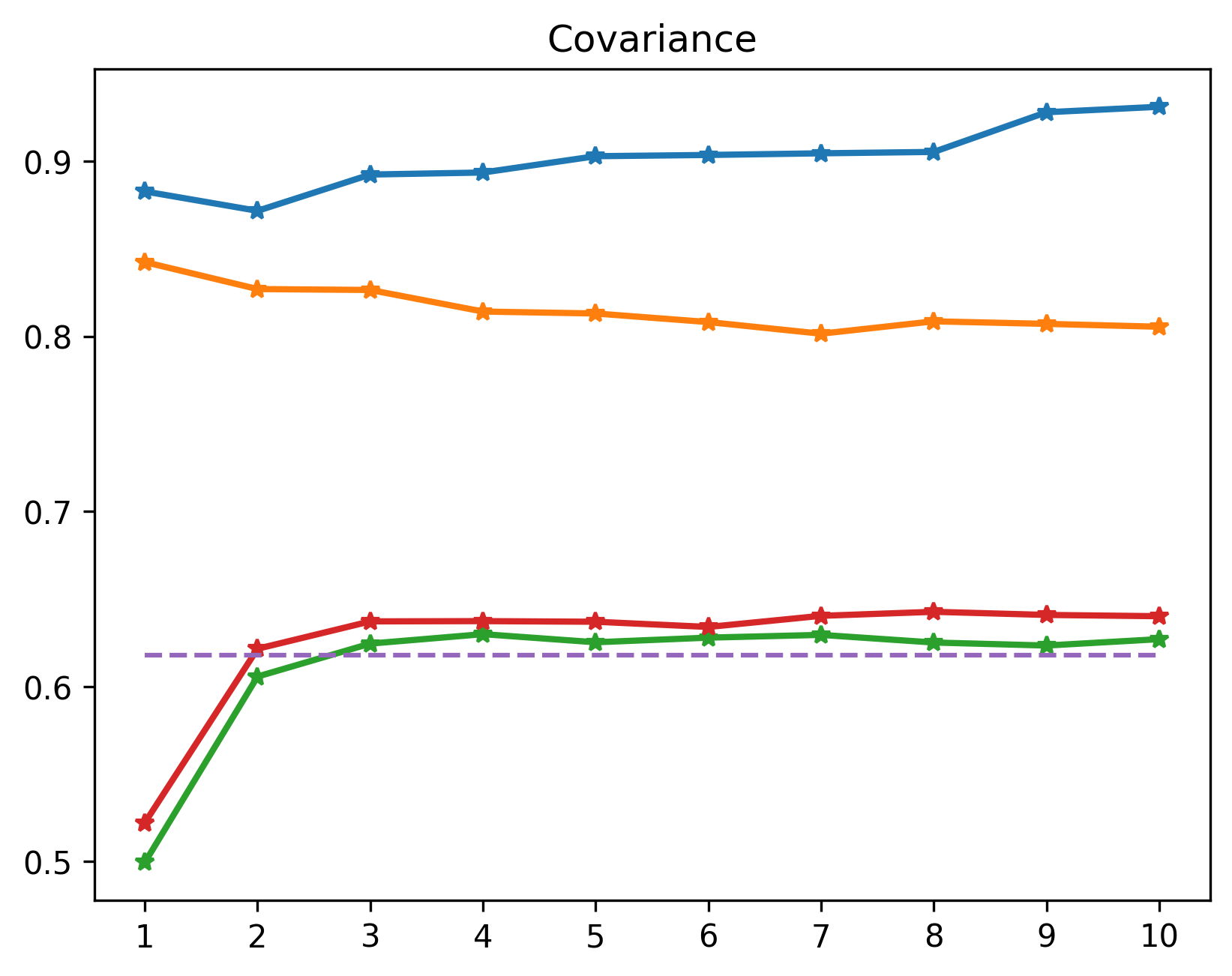}};
          \end{tikzpicture}
  \caption{Convergence of DSB to ground-truth. From left to right: estimated mean, variance and covariance (first component) after each DSB iteration. The ground-truth value is given by the dashed line in each scenario.}
  \label{fig:conv_gt}
  \end{figure}
  In \Cref{fig:conv_gt}, we illustrate the convergence of DSB. We train each DSB
  with a batch size of $128$, $N = 20$ and $\gamma = 1/40$. We compare two
  network configurations: ``small'' where the network is given by
  \Cref{fig:2d_architecture} (30k parameters) whereas ``large'' corresponds to
  the same network but with twice as many latent dimensions (240k
  parameters). The small network recovers the statistics of SB in the
  low-dimensional setting ($d=5$) but is unable to recover the variance and
  covariance for $d=50$. Increasing the size of the network solves this
  problem. %This highlights
  %the crucial role of the network used to perform the drift estimation whose
  %complexity grows with the dimension $d$. We emphasize that in this Gaussian
  %setting, the score is available in closed form and no neural network
  %approximation is needed in this case. In particular, the score mapping is
  %linear in its argument. We deliberately selected a neural network which does not exploit this property, as the true scores will be unknown in real applications.
  %This experiment also highlights another limitation of DSB. % We observe
  % that for small networks the estimation of the mean
  % worsens after many DSB iterations. This phenomenon can be explained by the accumulation of errors
  % during the DSB iterations which are not negligible in this case.

\paragraph{Two dimensional toy experiments.}
\label{sec:two-dimensional-toy}
 \begin{wrapfigure}{R}{0.5\textwidth}
 \vspace{-0.75cm}
        \centering
         \begin{tikzpicture}
        % original ------------------------------------------------------------------
        \node[inner sep=0pt, label={[label distance=0.03cm]87:{\small Data distribution}}] (gaus_orig) at (0+\offset,0)
            {\includegraphics[width=0.24\linewidth]{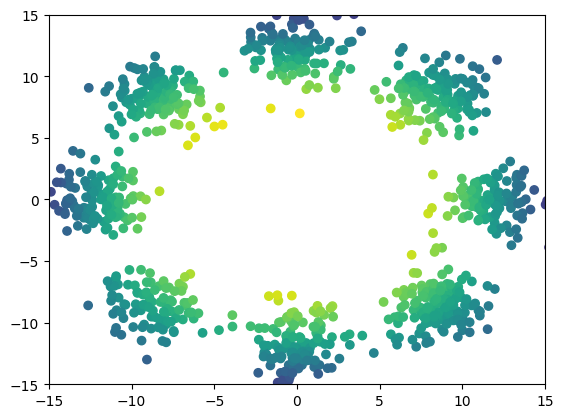}};
        \node[inner sep=0pt, label={\small }] (moon_orig) at (\www+\offset,0)
            {\includegraphics[width=0.24\linewidth]{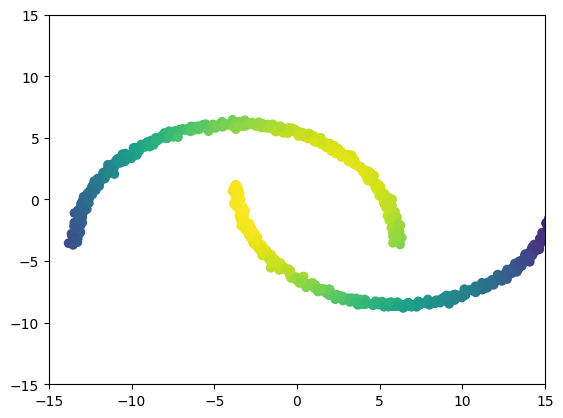}};
        \node[inner sep=0pt, label={\small }] (scurve_orig) at (2*\www+\offset,0)
            {\includegraphics[width=0.24\linewidth]{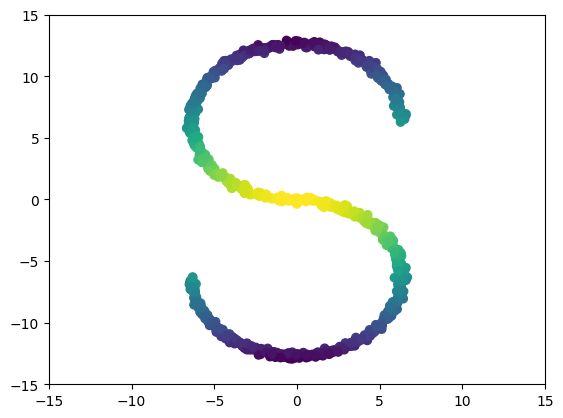}};
        \node[inner sep=0pt, label={\small }] (circle_orig) at (3*\www+\offset,0)
            {\includegraphics[width=0.24\linewidth]{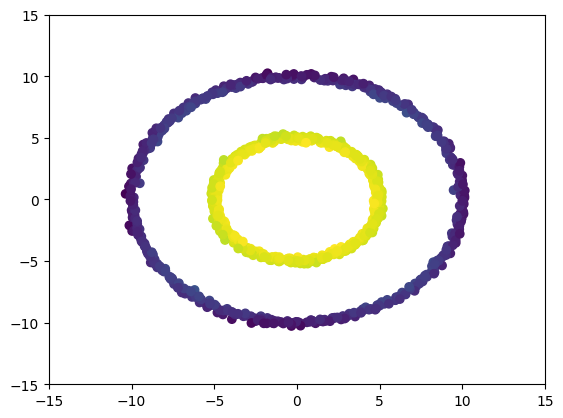}};
     % backward 1 ------------------------------------------------------------------
    \node[inner sep=0pt, label={[label distance=0.03cm]87:{\small DSB Iteration 1}}] (gaus_1) at (0+\offset,-\hsmall)
            {\includegraphics[width=0.24\linewidth]{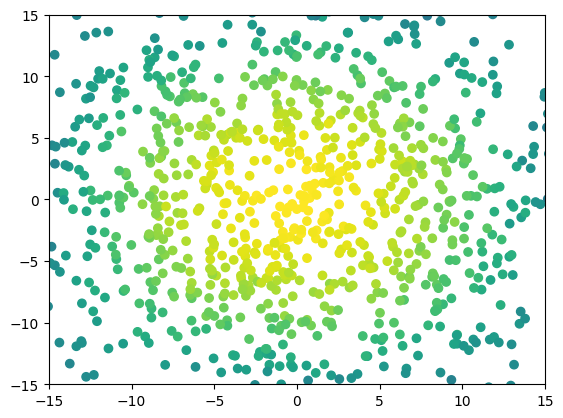}};
        \node[inner sep=0pt, label={\small }] (moon_1) at (\www+\offset,-\hsmall)
            {\includegraphics[width=0.24\linewidth]{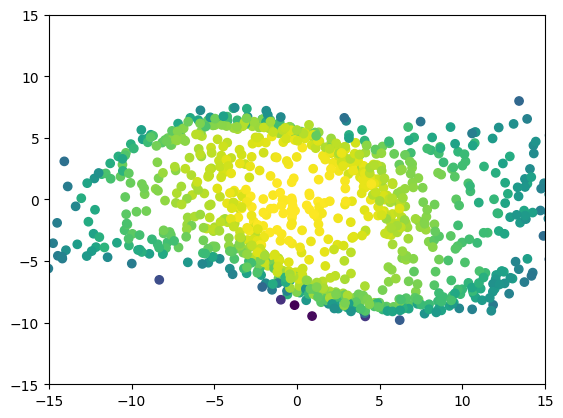}};
        \node[inner sep=0pt, label={\small }] (scurve_1) at (2*\www+\offset,-\hsmall)
            {\includegraphics[width=0.24\linewidth]{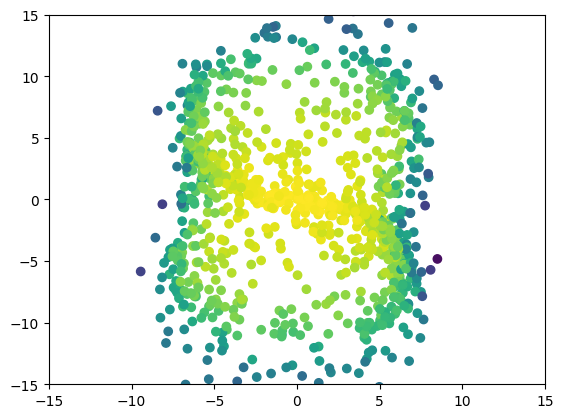}};
        \node[inner sep=0pt, label={\small }] (circle_1) at (3*\www+\offset,-\hsmall)
            {\includegraphics[width=0.24\linewidth]{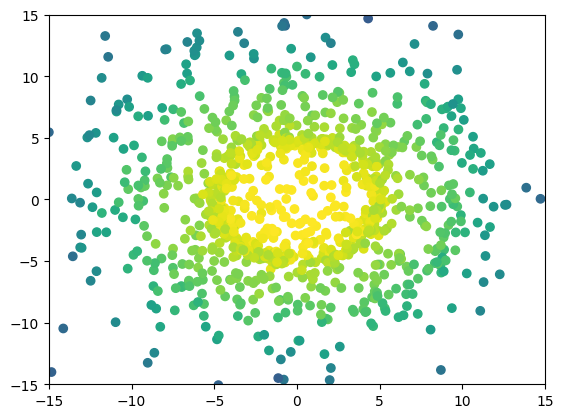}};
    % backward 20 ------------------------------------------------------------------{[label distance=1cm]30:label}
    %
    \node[inner sep=0pt, label={[label distance=0.03cm]87:{\small DSB Iteration 20}}] (gaus_20) at (0+\offset,-2*\hsmall)
            {\includegraphics[width=0.24\linewidth]{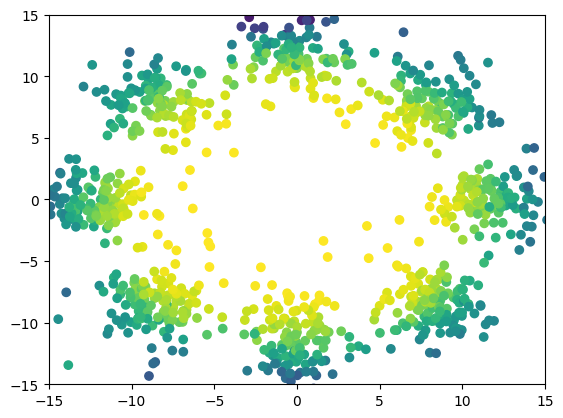}};
        \node[inner sep=0pt, label={\small }] (moon_20) at (\www+\offset,-2*\hsmall)
            {\includegraphics[width=0.24\linewidth]{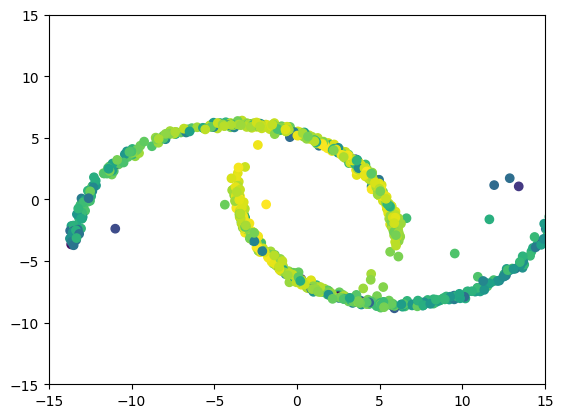}};
        \node[inner sep=0pt, label={\small }] (scurve_20) at (2*\www+\offset,-2*\hsmall)
            {\includegraphics[width=0.24\linewidth]{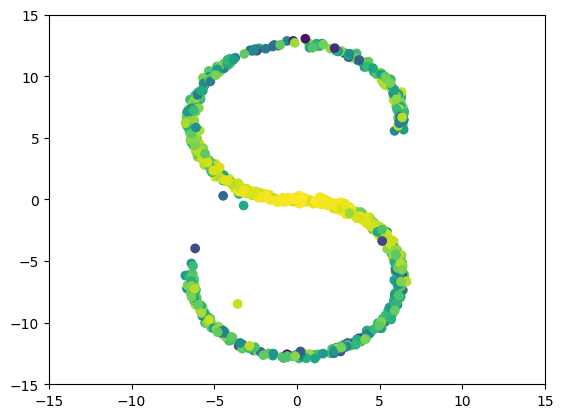}};
        \node[inner sep=0pt, label={\small }] (circle_20) at (3*\www+\offset,-2*\hsmall)
            {\includegraphics[width=0.24\linewidth]{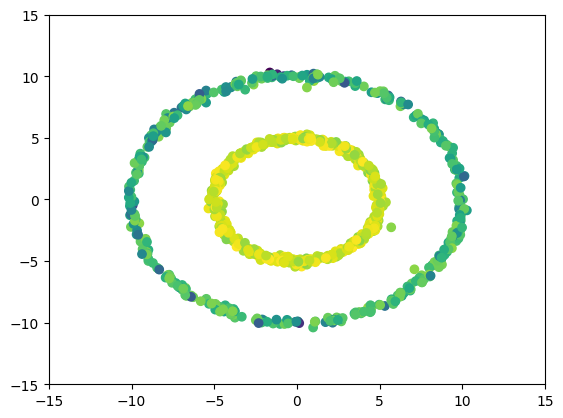}};
      \end{tikzpicture}
    \caption{Data distributions $\pdata$ vs
    distribution at $t=0$ for $T=0.2$ after $1$ and $20$ DSB iterations.}
    \vspace{-10pt}
  \label{fig:target_2d}
      \end{wrapfigure}
      We evaluate the validity of our approach on toy two dimensional examples.
      Contrary to existing SGM approaches we do \emph{not} require that the
      number of steps is large enough for $p_N\approx \pprior$ to hold. We use a
      fully connected network with positional encoding
      \citep{vaswani2017attention} to approximate $B^n_{k}$ and $F^{n}_k$, see
      \Cref{sec:addit-two-dimens} and our \href{https://github.com/JTT94/diffusion_schrodinger_bridge/}{code}\footnote{Code is available here \href{https://github.com/JTT94/diffusion_schrodinger_bridge}{https://github.com/JTT94/diffusion\_schrodinger\_bridge}} for implementation details.
      Animated plots of the DSB iterations may be found online on our project
      \href{https://vdeborto.github.io/publication/schrodinger\_bridge/}{webpage}\footnote{\href{https://vdeborto.github.io/publication/schrodinger\_bridge/}{https://vdeborto.github.io/publication/schrodinger\_bridge/}}.
      In \Cref{fig:target_2d}, we illustrate the benefits of DSB over classical
      SGM. We fix $f(x) = -\alpha x$ and choose
      $\pprior = \mathcal{N}(0, \sigma^2_\mathrm{data}\Idbf)$, hence
      $\alpha=1/\sigma^2_\mathrm{data}$ where $\sigma^2_\mathrm{data}$ is the
      variance of the dataset. We let $N = 20$ and $\gamma_k = 0.01$, \ie \
      $T = 0.2$. Since $T$ is small, we do not have $p_N \approx \pprior$ and the
      reverse-time process obtained after the first DSB iteration (corresponding
      to original SGM methods) does not yield a satisfactory generative
      model. However, multiple iterations of DSB improve the quality of the
      synthesis.
      % We emphasize that even though in theory \schro bridges allow for $T$ to be
      % arbitrary small, we observe that decreasing values of $T$ require an
      % increasing number of DSB iterations to obtain valid generative models. % We
      % discuss this trade-off and present additional experiments in
      % \Cref{sec:addit-two-dimens}.

       \paragraph{Generative
        modeling.}\label{sec:generative-modeling}

      DSB is the first practical
      algorithm for approximating the solution to the SB problem in high
      dimension ($d=3072$ for CelebA). Whilst our implementation does not yet compete with state-of-the-art methods, we
      show promising results with fewer diffusion steps compared to initial
      SGMs \citep{song2019generative} and demonstrate its performance on MNIST \citep{mnist} and
      CelebA \citep{liu2015faceattributes}. 
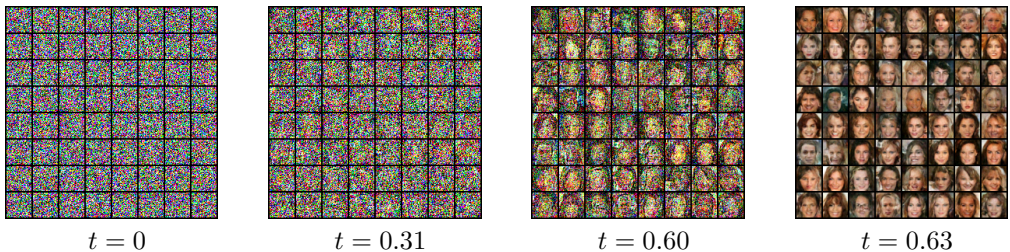
\begin{figure}[h]
              \centering
              \hfill
\begin{tikzpicture}
        % original ------------------------------------------------------------------
        \vspace{-1cm}
        \node at  (0,-1.85)  {$t=0$};
        \node at  (\wwww,-1.85)  {$t=0.31$};
        \node at  (2*\wwww,-1.85)  {$t=0.60$};
        \node at  (3*\wwww,-1.85)  {$t=0.63$};
        \node[inner sep=0pt, label={\small},] (scurve_orig) at (0,0)
            {\includegraphics[width=0.15\linewidth, trim=1.5cm 1.5cm 1.0cm 0.5cm, ]{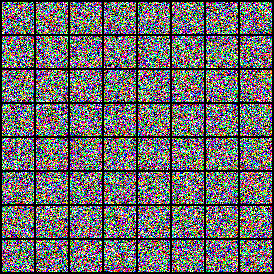}};
        \node[inner sep=0pt, label={\small }] (scurve_uno) at (\wwww,0)
        {\includegraphics[width=0.15\linewidth, trim=1.5cm 1.5cm 1.0cm 0.5cm, ]{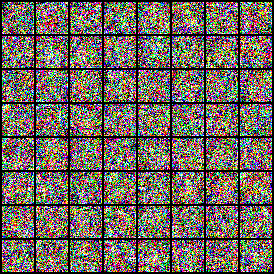}};        
        \node[inner sep=0pt, label={\small }] (middle) at (2*\wwww,0)
            {\includegraphics[width=0.15\linewidth, trim=1.5cm 1.5cm 1.0cm 0.5cm, ]{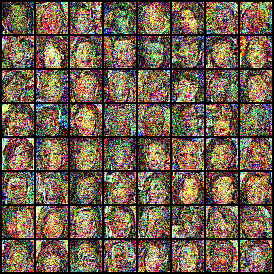}};
        \node[inner sep=0pt, label={\small }] (circle_orig) at (3*\wwww,0)
        {\includegraphics[width=0.15\linewidth, trim=1.5cm 1.5cm 1.0cm 0.5cm ]{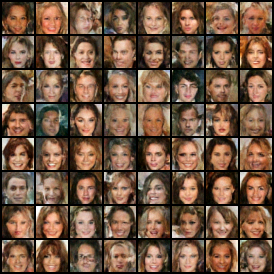}};
          \end{tikzpicture}
          \vspace{-0.25cm}
\caption{Generative model for CelebA $32 \times 32$ after 10 DSB iterations with $N=50$ ($T=0.63$)}
  \label{fig:celeba}
  \end{figure}

A reduced U-net architecture based on
      \citet{nichol2021improved} is used to approximate $B^n_{k}$ and
      $F^{n}_k$. Further details are given in
      \Cref{sec:generative-modeling-1}. Our method is validated on downscaled CelebA in \Cref{fig:celeba}. \Cref{fig:mnist} illustrates qualitative
      improvement over $8$ DSB iterations with as few as $N=12$ diffusion steps. Note, as shown in \Cref{sec:generative-modeling-1}, we obtain better results with higher $N$ yet still significantly fewer steps than in the original SGM
      procedures \citep{song2020improved, song2019generative} which use
      $N=100$. 
      %\Cref{fig:mnist} also shows good diversity of generated samples (red) and coverage of the original dataset (blue) as shown by the two dimensional representation in the latent space of a pre-trained Variational Auto-Encoder (VAE). 
      \Cref{fig:fid} illustrates how the sample
      quality, measured quantitatively in terms of Fr\'{e}chet Inception
      Distance (FID) \citep{heusel2017gans}, improves with the number of DSB
      iterations for various numbers of steps $N$.

% \begin{minipage}{0.45\textwidth}
% \vspace{-0.2cm}
% \begin{figure}[H]
%   \centering
%   \begin{figure}[H]
%     \centering
%          \begin{tikzpicture}
%         % original ------------------------------------------------------------------
%         \node[inner sep=0pt, label={\small }] (mist1) at (1.5*\offset,0)
%             {\includegraphics[width=0.4\linewidth]{./fig/mnist/mnist_12_b1_5000.png}};
%         \node[inner sep=0pt, label={\small }] (mnist8) at (1.3*\ww+\offset,0)
%             {\includegraphics[width=0.4\linewidth]{./fig/mnist/mnist_12_b8_5000.png}};
    
%         \node[] at (1.2*\offset,+0.5*\offset) {DSB 1};
%         \node[] at (1.3*\ww+\offset,+0.5*\offset) {DSB 8};

%       \end{tikzpicture}
%       \caption{Generated samples $(N=12)$}
%   \label{fig:mnist}
%   \end{figure}
%   \includegraphics[width=0.5\textwidth]{./fig/fid.png}
%   \caption{FID vs DSB Iterations.}
%   \label{fig:fid}
%   \end{figure}
% \end{minipage}

\begin{minipage}{0.5\textwidth}
    \begin{figure}[H]
    \centering
         \begin{tikzpicture}
        % original ------------------------------------------------------------------
        \node[inner sep=0pt, label={\small }] (mist1) at (1.2*\offset,0)
            {\includegraphics[width=0.4\linewidth]{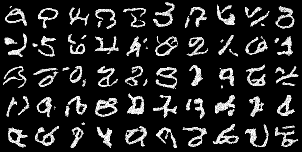}};
        \node[inner sep=0pt, label={\small }] (mnist8) at (1.3*\ww+\offset,0)
            {\includegraphics[width=0.4\linewidth]{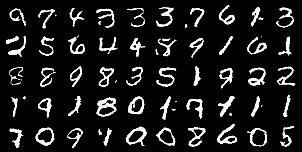}};
    
        \node[] at (1.2*\offset,-0.5*\offset) {DSB 1};
        \node[] at (1.3*\ww+\offset,-0.5*\offset) {DSB 8};

      \end{tikzpicture}
      \caption{Generated samples $(N=12)$}
  \label{fig:mnist}
  \end{figure}
\end{minipage}
\hspace{1cm}
\begin{minipage}{0.4\textwidth}
 \begin{figure}[H]
  \centering
  \includegraphics[width=0.6\textwidth]{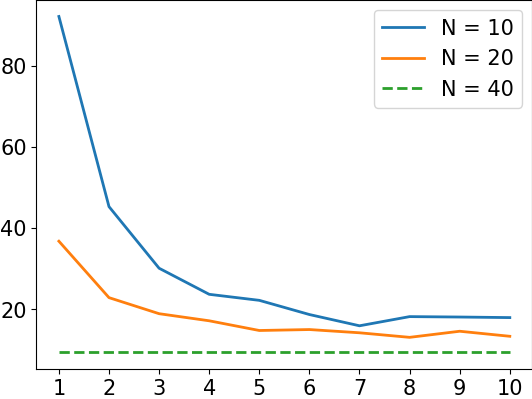}
  \caption{FID vs DSB Iterations.}
  \label{fig:fid}
  \end{figure}
\end{minipage}

\smallskip\noindent\textbf{Dataset interpolation.}
\schro bridges not only reduce the number of steps in SGM
methods but also enable flexibility in the choice of the prior density $\pprior$. Our approach is still valid for non-Gaussian $\pprior$, contrary to previous SGM works, and can be set as any other data distribution
$\pdata'$. In this case DSB converges towards a bridge between $\pdata$ and
$\pdata'$, see \Cref{fig:transport}. These experiments pave the way towards
high-dimensional optimal transport between arbitrary data distributions.
      
 \begin{figure}[H]
        \centering
        \vspace{-0.5cm}
         \begin{tikzpicture}
        % original ------------------------------------------------------------------
        \node[inner sep=0pt, label={\small},] (scurve_orig) at (0+\offset,0)
            {\includegraphics[width=0.18\linewidth, trim=1.5cm 1.5cm 1.0cm 0.5cm, clip]{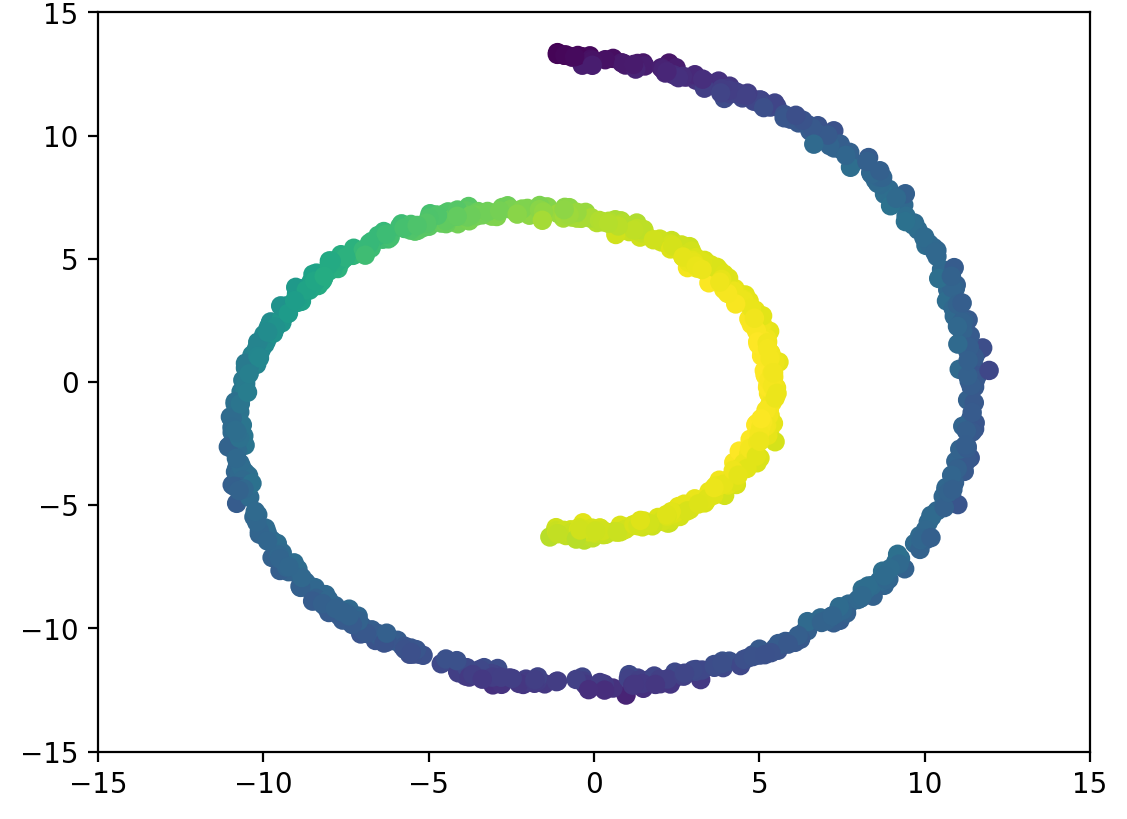}};
        \node[inner sep=0pt, label={\small }] (scurve_uno) at (\ww+\offset,0)
        {\includegraphics[width=0.18\linewidth, trim=1.5cm 1.5cm 1.0cm 0.5cm, clip]{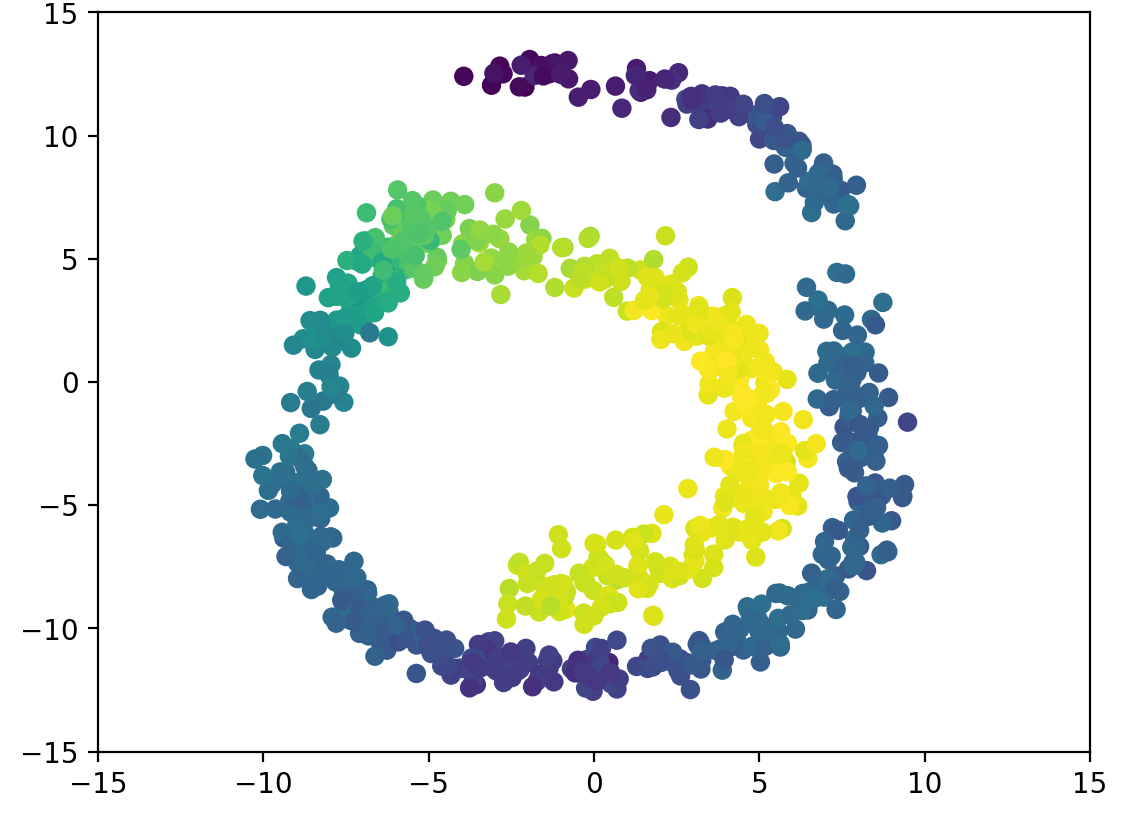}};        
        \node[inner sep=0pt, label={\small }] (middle) at (2*\ww+\offset,0)
            {\includegraphics[width=0.18\linewidth, trim=1.5cm 1.5cm 1.0cm 0.5cm, clip]{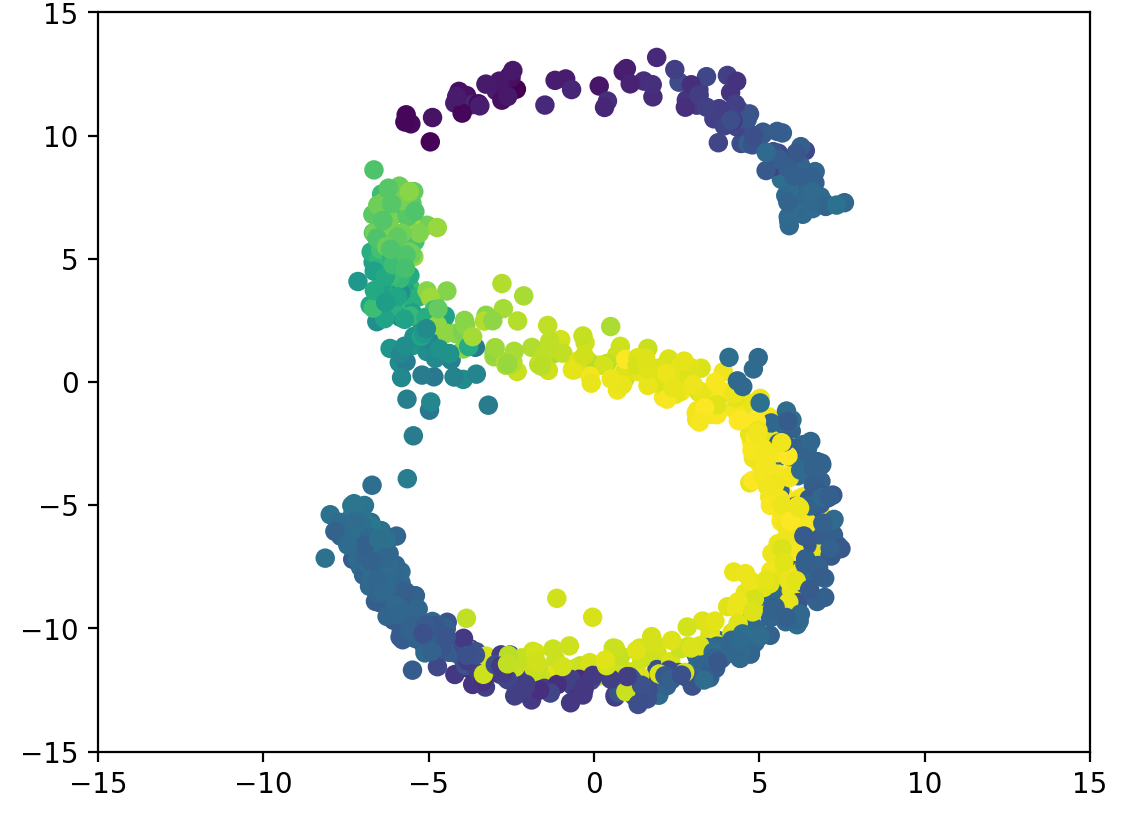}};
        \node[inner sep=0pt, label={\small }] (circle_orig) at (3*\ww+\offset,0)
        {\includegraphics[width=0.18\linewidth, trim=1.5cm 1.5cm 1.0cm 0.5cm, clip]{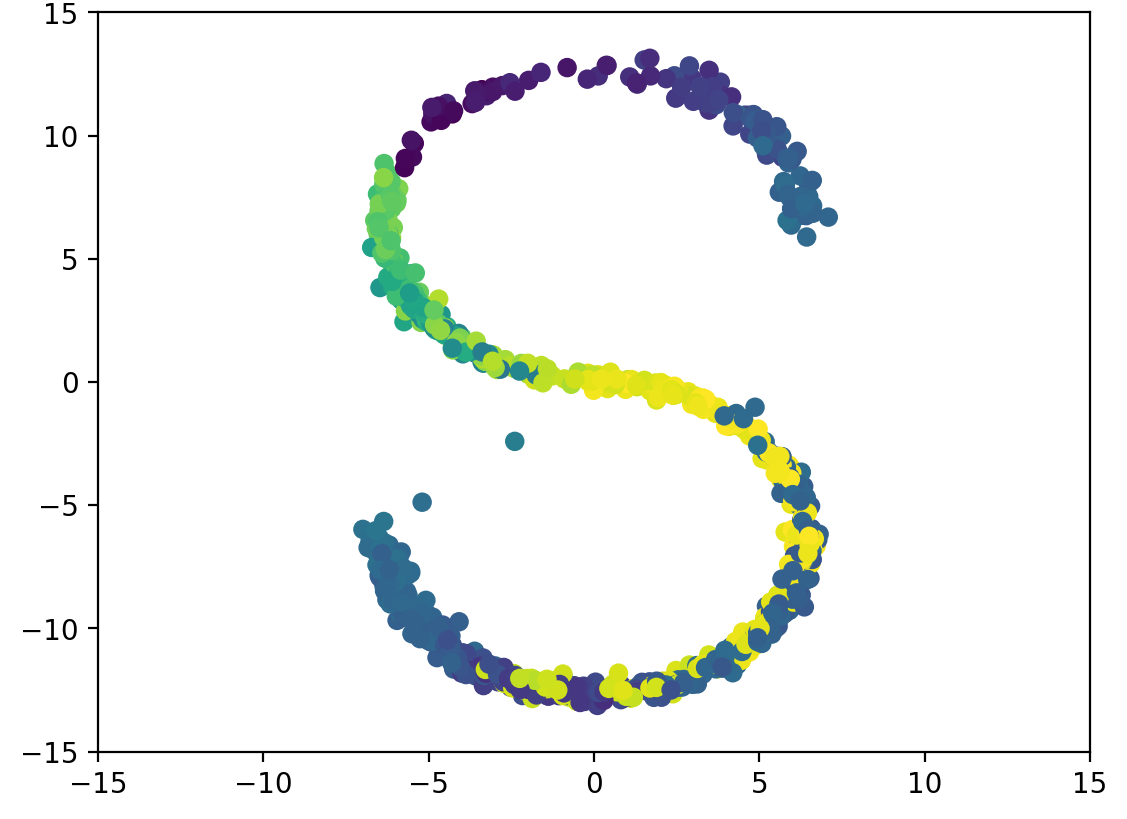}};
        \path[{Latex[length=3mm]}-{Latex[length=3mm]}, dashed]
         (scurve_orig) edge[bend left=\bend] (3*\ww+\offset-0.5,0);

        \node[inner sep=0pt, label={\small }] at (0+\offset,0-\offsety)
            {\includegraphics[width=0.18\linewidth, trim=1.0cm 1.0cm 1.0cm 1.0cm, clip]{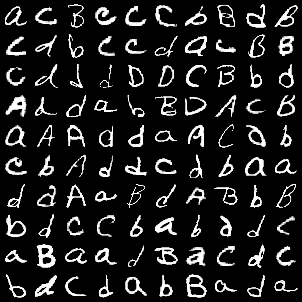}};
        \node[inner sep=0pt, label={\small }]  at (\ww+\offset,0-\offsety)
        {\includegraphics[width=0.18\linewidth, trim=1.0cm 1.0cm 1.0cm 1.0cm, clip]{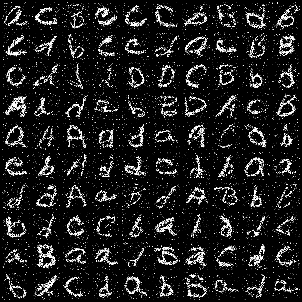}};        
        \node[inner sep=0pt, label={\small }]  at (2*\ww+\offset,0-\offsety)
            {\includegraphics[width=0.18\linewidth, trim=1.0cm 1.0cm 1.0cm 1.0cm, clip]{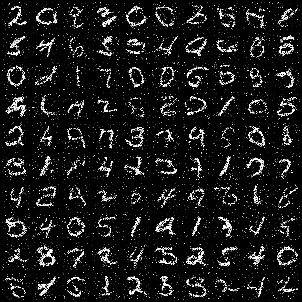}};
        \node[inner sep=0pt, label={\small }] at (3*\ww+\offset,0-\offsety)
            {\includegraphics[width=0.18\linewidth, trim=1.0cm 1.0cm 1.0cm 1.0cm, clip]{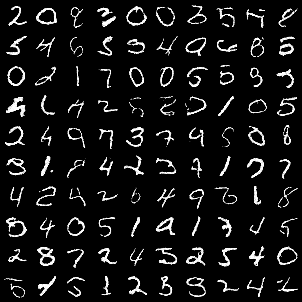}};
          \end{tikzpicture}
          \caption{First row: Swiss-roll to S-curve (2D). Iteration 9 of DSB with
            $T = 1$ ($N = 50$). From left to right:
            $t=0, 0.4, 0.6, 1$. Second row: EMNIST \citep{cohen2017emnist} to
            MNIST. Iteration 10 of DSB with $T=1.5$ ($N=30$).
            From left to right: $t=0, 0.4, 1.25, 1.5$.}
  \label{fig:transport}          
\end{figure}

\vspace{-.75cm}
\section{Discussion}
\label{sec:conclusion}
Score-based generative modeling (SGM) may be viewed as the first stage of
solving a \schro bridge problem.  Building on this interpretation, we develop novel methodology, the Diffusion \schro Bridge (DSB), that extends SGM
approaches and allows one to perform generative modeling with fewer diffusion
steps. DSB complements recent techniques to speed up existing SGM methods that
rely on either different noise schedules
\citep{nichol2021improved,san2021noise,watson2021learning}, alternative
discretizations \citep{jolicoeur2021gotta} or knowledge distillation
\citep{luhman2021knowledge}. Additionally, as the solution of the \schro problem
is a diffusion, it is possible as in \citet[Section 4.3]{song2020score} to
obtain an equivalent neural ordinary differential equation that admits the same
marginals as the diffusion but enables exact likelihood computation, see
\Cref{sec:likel-comp-schr}. Even though the final time $T>0$ within DSB can be
arbitrarily small, we observed that this has limits as choosing $T$ too close to
$0$ decreases the quality of the generative models.  One reason for this
behavior is that if the endpoint of the original forward process is too far from
the target distribution $\pprior$, then learning the score around the support of
$\pprior$ is challenging even for DSB. From a theoretical point of view, we have
provided quantitative convergence results for SGM methods and derived new
state-of-the-art convergence bounds for IPF as well as novel monotonicity
results. We have demonstrated DSB on generative modeling and data interpolation
tasks. Finally, although this work was motivated by generative modeling, DSB is
much more widely applicable as it can be thought of as the continuous
state-space counterpart of the celebrated Sinkhorn algorithm
\citep{cuturi2013sinkhorn,peyre2019computational}. For example, DSB could be
used to solve multi-marginal \schro bridges problems
\citep{dimarino2020optimal}, compute Wasserstein barycenters, find the
minimizers of entropy-regularized Gromov--Wasserstein problems
\citep{memoli2011gromov} or perform domain adaptation in continuous
state-spaces.

\newpage
\begin{ack}
Valentin De Bortoli and Arnaud Doucet are supported by the EPSRC CoSInES (COmputational Statistical INference for Engineering and Security) grant EP/R034710/1, James Thornton by the OxWaSP CDT through grant EP/L016710/1 and Jeremy Heng by the CY Initiative of Excellence (grant “Investissements d’Avenir” ANR-16-IDEX-0008). Computing resources were provided through the Google Cloud research credits programme. Arnaud Doucet also acknowledges support from the UK Defence Science and Technology Laboratory (DSTL) and EPSRC under grant EP/R013616/1. This is part of the collaboration between US DOD, UK MOD and UK EPSRC under the Multidisciplinary University Research Initiative.  
\end{ack}
\bibliographystyle{apalike}
\bibliography{bibliography}

\newpage
\appendix 

\section{Organization of the supplementary}
\label{sec:organ-suppl}

The supplementary is organized as follows. We define our notation in
\Cref{sec:notation-supp}. In \Cref{sec:time-revers-exist}, we prove
\Cref{prop:convergence_score_matching} and draw links between our approach of
SGM and existing works.  We recall the classical formulation of IPF, prove
\Cref{prop:IPFrecursion} and draw links with autoencoders in
\Cref{sec:schro-bridges-with}.  In \Cref{sec:altern-vari-form} we present
alternative variational formulas for \Cref{algo:ipf_score} and prove
\Cref{prop:generalizedscorematching}.  We gather the proofs of our theoretical
study of \schro bridges (\Cref{prop:monotonicity} and
\Cref{prop:convergence_ipf}) in \Cref{sec:theor-study-schro}.  A quantitative
study of IPF with Gaussian targets and reference measure is presented in
\Cref{prop:geom_gaussian:proof}. In particular, we show that the convergence
rate of IPF is geometric in this case.  In \Cref{prop:continuous_schro:proof} we
study the links between continuous-time and discrete-time IPF and prove
\Cref{prop:continuous_schro}. We also provide details on the likelihood
computation of generative models obtained with Schr\"{o}dinger bridges. We
detail training techniques to improve training times in
\Cref{sec:arch-deta-addit} then present architecture details and additional
experiments in \Cref{sec:addn_exp}.

\section{Notation}
\label{sec:notation-supp}

For ease of reading in this section we recall and detail some of the notation
introduced in \Cref{sec:notation}.  For any measurable space $(\mse, \mce)$, we
denote by $\Pens(\mse)$ the space of probability measures over $\mse$. For any
$\ell \in \nset$, we also denote $\Pens_\ell= \Pens((\rset^d)^\ell)$. For any
$\pi \in \Pens(\mse)$ and Markov kernel
$\Kker: \mse \times \mcf \to \ccint{0,1}$ where $(\msf, \mcf)$ is a measurable
space, we define $\pi \Kker \in \Pens(\msf)$ such that for any $\msa \in \mcf$
we have $\textstyle{\pi \Kker(\msa) = \int_{\mse} \Kker(x, \msa) \rmd \pi(x)}$.
If $\mse = \contspace$ then for any $\Pbb \in \Pens(\mse)$ and
$s, t \in \ccint{0, T}$, we denote by $\Pbb_{s,t}$ the marginals of $\Pbb$ at
time $s$ and $t$. In addition, we denote by $\Pbb_{|s, t}$ the disintegration
Markov kernel given by the mapping $\omega \mapsto (\omega(s), \omega(t))$, see
\Cref{sec:addit-form-kullb} for a definition. In particular, we have
$\Pbb = \Pbb_{s,t} \Pbb_{|s,t}$. All defined mappings are considered to be
measurable unless stated otherwise.

For any $\Pbb \in \Pens(\contspace)$ we define $\Pbb^R$ the reverse-time
measure, \ie \ for any $\msa \in \mcb{\contspace}$ we have
$\Pbb^R(\msa) = \Pbb(\msa^R)$ where
$\msa^R = \ensembleLigne{t \mapsto\omega(T-t)}{\omega \in \msa}$. We say that
$\Pbb \in \Pens(\contspace)$ is \emph{associated with a diffusion} if it solves
the corresponding martingale problem. More precisely,
$\Pbb \in \Pens(\contspace)$ is associated with
$\rmd \bfX_t = b(t, \bfX_t) \rmd t + \sqrt{2} \rmd \bfB_t$ for
$b: \ \ccint{0,T} \times \rset^d \to \rset^d$ measurable if for any
$v \in \rmc_c^2(\rset^d, \rset)$, $(M_t^v)_{t \in \ccint{0,T}}$ is a
$\Pbb$-local martingale, where for any $t \in \ccint{0,T}$
\begin{equation}
\label{eq:martingale_pbm}
  \textstyle{M_t^v = v(\bfX_t) - \int_0^t \generator_s(v)(\bfX_s) \rmd s }
\end{equation}
with for any $v \in \rmc^2(\rset^d, \rset)$, $t \in \ccint{0,t}$ and
$x \in \rset^d$
\begin{equation}
  \generator_t(v)(x) = \langle b(t, x) , \nabla v(x) \rangle +  \Delta v(x)  .
\end{equation}
We refer to \cite{revuz1999continuous} for a rigorous treatment of local
martingales. Note that \eqref{eq:martingale_pbm} uniquely defines $\Pbb_{t|s}$
for any $s, t \in \ccint{0,T}$ with $t\geq s$. Hence $\Pbb$ is uniquely defined
up to $\Pbb_0$.

In some cases, we say that $\Pbb \in \Pens(\contspace)$ is \emph{associated with
  a diffusion} if it solves the corresponding martingale problem with initial
condition. More precisely, $\Pbb \in \Pens(\contspace)$ is associated with
$\rmd \bfX_t = b(t, \bfX_t) \rmd t + \sqrt{2} \rmd \bfB_t$ and
$\bfX_0 \sim \mu_0 \in \Pens(\rset^d)$ if it solves the martingale problem and
$\Pbb_0 = \mu_0$. Note that in this case $\Pbb$ is uniquely defined.

Finally, for
any measurable space $(\mse, \mce)$ and $\mu, \nu \in \Pens(\mse)$ we recall
that the Jeffrey's divergence is given by
$\JefLigne{\mu}{\nu} = \KLLigne{\mu}{\nu} + \KLLigne{\nu}{\mu}$.

\section{Time-reversal and existing work}
\label{sec:time-revers-exist}

Before giving the proof of \Cref{prop:convergence_score_matching} we start by
deriving estimates on the logarithmic derivatives of the density of the
Ornstein-Ulhenbeck process given growth conditions on the initial density in
\Cref{sec:estim-logar-deriv}. Note that our estimates are uniform w.r.t. the
time variable.  We give the proof of \Cref{prop:convergence_score_matching} in
\Cref{prop:convergence_score_matching:proof}. Finally, we draw links with
existing works in \Cref{sec:comparison-with-ho}.

\subsection{Estimates for logarithmic derivatives}
\label{sec:estim-logar-deriv}

We start by recalling the following multivariate Faa di Bruno's formula and a
useful technical lemma. Then in \Cref{sec:small-time-gradient} we derive bounds
for the logarithmic derivatives which are non-vacuous for small times. In
\Cref{sec:large-time-gradient} we derive bounds for the logarithmic derivatives
which are non-vacuous for large times. We combine them in \Cref{sec:uniform}.

For any $\alpha \in \nset^d$ we denote $\abs{\alpha} = \sum_{i=1}^d \alpha_i$
and $\alpha ! = \prod_{i=1}^d \alpha_i !$.  If $f: \ \rset^d \to \rset$ is
$m$-differentiable with $m \in \nset$, then for any $\lambda \in \nset^d$ with
$\abs{\lambda} \leq m$ we denote for any $x \in \rset^d$,
$\partial_\lambda f(x) = \partial_{1}^{\lambda_1} \dots \partial_{d}^{\lambda_d}
f(x)$. Similarly to \cite{constantine1996faadibruno}, we define $\prec$ the
order on $\nset^d$ such that for any $\lambda^1, \lambda^2 \in \nset^d$,
$\lambda^1 \prec \lambda^2$ if $\absLigne{\lambda^1} < \absLigne{\lambda^2}$ or
$\absLigne{\lambda^1} = \absLigne{\lambda^2}$ and there exists $j \in \{1, \dots, d\}$
such that $\lambda_j^1 < \lambda_j^2$ and for any $i \in \{1, \dots, j\}$,
$\lambda_i^1 = \lambda_i^2$.

\begin{proposition}
  \label{prop:faadibruno}
  Let $\msu \subset \rset$ open, $N \in \nset$, $f \in \rmc^N(\msu, \rset)$,
  $g \in \rmc^N(\rset^d, \msu)$ and $h = f \circ g$. Then for any
  $\lambda \in \nset^d$ with $\abs{\lambda} \leq N$ and $x \in \rset^d$ we have 
  \begin{equation}
    \textstyle{
      \partial_\lambda h(x) = \sum_{k, s=1}^{\abs{\lambda}} \sum_{p_s(\lambda, k)} f^{(k)}(g(x)) \lambda ! \prod_{j=1}^s \partial_{\ell_j}g(x)^{m_j} / (m_j! \ell_j!^{m_j})  ,
      }
    \end{equation}
    with
    \begin{equation}
      p_s(\lambda, k) =\textstyle{\ensembleLigne{\{\ell_i\}_{i=1}^s\in (\nset^d)^s, \ \{m_i\}_{i=1}^s \in \nset^s}{\ell_1 \prec \dots \prec \ell_s , \ \sum_{i=1}^s m_i = k  , \ \sum_{i=1}^s m_i \ell_i = \lambda}  . }
    \end{equation}
\end{proposition}

\begin{proof}
  The proposition is a direct application of \cite{constantine1996faadibruno}.
\end{proof}

From this multivariate Faa di Bruno formula we derive the following lemma
drawing links between exponential and logarithmic derivatives.

\begin{lemma}
  \label{lemma:faadibruno_spec}
  Let $N \in \nset$, $g_1 \in \rmc^N(\rset^d, \rset)$, 
  $g_2 \in \rmc^N(\rset^d, \ooint{0, +\infty})$, $h_1 = \exp[g_1]$ and 
  $h_2 = \log(g_2)$. Then for any $\lambda \in \nset^d$ with $\abs{\lambda} \leq N$ let
  $c_{d, \lambda} = \sum_{k=1}^{\abs{\lambda}} d^{k}$ and the following hold:
  \begin{enumerate}[label=(\alph*), wide, labelwidth=!, labelindent=0pt]
  \item \label{item:exp} There exists $\rmP_{\lambda, \exp}$ a real polynomial with $c_{d, \lambda}$ variables such that for any $x \in \rset^d$
    \begin{equation}      
      \partial_{\lambda} h_1(x) = \rmP_{\lambda, \exp}((\partial_{\ell} g_1(x))_{\abs{\ell}\leq\abs{\lambda}})h_1(x)   . 
    \end{equation}
  \item \label{item:log} There exists $\rmP_{\lambda, \log}$ a real polynomial with $c_{d, \lambda}$ variables such that for any $x \in \rset^d$
    \begin{equation}
      \partial_{\lambda} h_2(x) = \rmP_{\lambda, \log}((\partial_{\ell} g_2(x)/g_2(x))_{\abs{\ell}\leq\abs{\lambda}})   . 
    \end{equation}
  \end{enumerate}
\end{lemma}

\begin{proof}
  The proof of \ref{item:exp} is a direct application of \Cref{prop:faadibruno}
  upon noting that for any $k \in \nset$, $f^{(k)} = \exp$ if
  $f=\exp$. Similarly, the proof of \ref{item:log} is a direct application of
  \Cref{prop:faadibruno} upon noting that, in the case where $f=\log$, for any
  $k \in \nset$ and $x > 0$, $f^{(k)}(x) = (-1)^{k-1}(k-1)!x^{-k}$ and that for
  any $s \in \{1, \dots, \abs{\lambda}\}$ and
  $(\ell_1, \dots, \ell_s, m_1, \dots, m_s) \in p_s(\lambda, k)$ we have
  $\sum_{i=1}^s m_i =k $.
\end{proof}

We will also make use of  the following technical lemma.

\begin{lemma}
  \label{prop:bound_p_norm}
  Let $p \in \nset$. Then for any $a \geq 0$, $b>0$  and $x \in \rset^d$ we have
  \begin{align}
    &- b \normLigne{x}^{2p} + a \normLigne{x}^{2p-1}  \leq -(b/2) \normLigne{x}^{2p} + a(2a/b)^{2p-1}  , \label{eq:2p-1}\\
    &- b \normLigne{x}^{2p} + a \normLigne{x}^{2p-2}  \leq -(b/2) \normLigne{x}^{2p} + a(2a/b)^{p-1}  . \label{eq:2p-2}
  \end{align}
  In addition  for any $a \geq 0$, $b>0$  and $x \in \rset^d$ we have
  \begin{equation}
    - b \normLigne{x}^{2p} + a \normLigne{x}^{2p-1} \leq (2p-1)^{2p-1}(2p)^{-2p} a^{2p} b^{1-2p}  . 
  \end{equation}
\end{lemma}

\begin{proof}
  For the first part of the proof, we only prove \eqref{eq:2p-1}. The proof of
  \eqref{eq:2p-2} is similar. Let $a \geq 0$, $b>0$.  For any $x \in \rset^d$
  with $\normLigne{x} \leq (b/2a)^{-1}$ we have
  $a \normLigne{x}^{2p-1} \leq a(b/2a)^{-2p+1}$.  For any $x \in \rset^d$ with
  $\normLigne{x} \geq (b/2a)^{-1}$ we have
  $a \normLigne{x}^{2p-1} \leq (b/2) \normLigne{x}^{2p}$. Hence, we get that for
  any $x \in \rset^d$ we have
  \begin{equation}
    a \normLigne{x}^{2p-1} - b \normLigne{x}^{2p} \leq a(b/2a)^{-2p+1} - (b/2) \norm{x}^{2p}  ,
  \end{equation}
  which concludes the first part of the proof.  For the second part of the
  proof, remark that the maximum of $\rmh: \ t \mapsto -bt^{2p} + at^{2p-1}$ is
  attained for $t^\star = (2p-1)/(2p)(a/b)$. We conclude upon noting that
  $h(t^\star) = (2p-1)^{2p-1}(2p)^{-2p} a^{2p} b^{1-2p}$.
\end{proof}

\subsubsection{Small times estimates}
\label{sec:small-time-gradient}

\Cref{lemma:faadibruno_spec} is key in the following proposition which establishes
upper bounds on the logarithmic derivatives of the density of the
Ornstein-Ulhenbeck process. In what follows, we define $(p_t)_{t \in \ccint{0,T}}$ the
density w.r.t. the Lebesgue measure of $\bfX_t$ satisfying
\begin{equation}
  \rmd \bfX_t = -\alpha \bfX_t \rmd t + \sqrt{2} \rmd \bfB_t  , \qquad \bfX_0 \sim \pdata  ,
\end{equation}
with $\alpha \geq 0$. In the rest of this section, $\alpha$ is fixed.

\begin{proposition}
  \label{prop:moment_infty}
  Let $N \in \nset$. Assume that
  $\pdata \in \rmc^N(\rset^d, \ooint{0,+\infty})$ is bounded and that for any
  $\ell \in \{1, \dots, N\}$ there exist $A_\ell \geq 0 $ and $\alpha_\ell \in \nset$
  such that for any $x \in \rset^d$
  \begin{equation}
    \label{eq:growth_p0}
    \normLigne{\nabla^\ell \log \pdata(x)} \leq A_{\ell} (1 + \norm{x}^{\alpha_\ell})   . 
  \end{equation}
  Then for any $t \geq 0$, $p_t \in \rmc^N(\rset^d, \ooint{0,+\infty})$ and for
  any $\ell \in \{1, \dots, N\}$, there exist $B_\ell \geq 0$ and
  $\beta_\ell \in \nset$ such that for any $t \geq 0$
  \begin{equation}
    \textstyle{\normLigne{\nabla^\ell \log p_t (x)} \leq c_t^{-2\beta_{\ell}} B_\ell (1 + \int_{\rset^d}\norm{x_0}^{\beta_\ell} p_{0|t}(x_0|x_t) \rmd x_0), }
  \end{equation}
  with $c_t^2 = \exp[-2\alpha t]$. 
\end{proposition}

\begin{proof}
  First note that for any $t \geq 0$ and $x_t \in \rset^d$ we have 
  \begin{equation}
    \label{eq:int_form}
    \textstyle{
      p_t(x_t) = \int_{\rset^d} \pdata(x_0) \rmg(x_t - c_t x_0) \rmd x_0  ,
      }
  \end{equation}
  with for any $\tilde{x} \in \rset^d$
  \begin{equation}
    c_t = \exp[-\alpha t]  , \qquad \rmg(\tilde{x}) = (2 \uppi \sigma_t^2)^{-d/2} \exp[-\norm{\tilde{x}}^2/(2\sigma_t^2)]  , \quad \sigma_t^2 = (1 - \exp[-2\alpha t])/ \alpha  . 
  \end{equation}
  Let $t \geq 0$. We have that $p_t \in \rmc^N(\rset^d, \ooint{0,+\infty})$ upon
  combining the fact that $\pdata$ is bounded, \eqref{eq:int_form} and the
  dominated convergence theorem.  Let $\ell \in \{1, \dots, N\}$ and
  $\lambda \in \nset^d$ such that $\abs{\lambda} \leq \ell$. Using
  \Cref{lemma:faadibruno_spec}-\ref{item:log} we have for any $x_t \in \rset^d$
    \begin{equation}
          \label{eq:partial_log}
    \partial_{\lambda} \log p_t (x_t) = \rmP_{\lambda, \log}((\partial_m p_t(x_t) / p_t(x_t))_{\abs{m} \leq \abs{\lambda}})  . 
  \end{equation}
  Using \eqref{eq:int_form} and the change of variable $z = x_t - c_t x_0$, we
  have for any $x_t \in \rset^d$
  \begin{equation}
    \textstyle{p_t(x_t) =  c_t^{-1} \int_{\rset^d} \pdata((x_t - z)/c_t) \rmg(z) \rmd z   . } 
  \end{equation}
  Hence, combining this result, the dominated convergence theorem and
  \Cref{lemma:faadibruno_spec}-\ref{item:exp} we get that for any
  $x_t \in \rset^d$ and $m \in \nset^d$ with $\abs{m} \leq \ell$
\begin{align}
  \partial_m p_t(x_t) &= \textstyle{c_t^{-\abs{m}} \int_{\rset^d} \partial_m \pdata(x_0) \rmg(x_t - c_t x_0) \rmd x_0  }\\
  &= \textstyle{c_t^{-\abs{m}} \int_{\rset^d} \rmP_{m, \exp}((\partial_j \log \pdata (x_0))_{\abs{j}\leq\abs{m}}) \pdata(x_0) \rmg(x_t - c_t x_0) \rmd x_0}   . 
\end{align}
We conclude the proof upon combining this result, \eqref{eq:growth_p0},
\eqref{eq:partial_log} and the fact that $c_t \leq 1$.
\end{proof}

For any $t \geq 0$ and $x_t \in \rset^d$ we introduce the infinitesimal
generator
$\generator_{t, x_t}: \ \rmc_2(\rset^d, \rset) \to \rmc_2(\rset^d, \rset)$ given
for any $\varphi \in \rmc^2(\rset^d, \rset)$ and $x_0 \in \rset^d$ by
\begin{align}
  \label{eq:generator}
  \generator_{t, x_t}(\varphi)(x_0) &= \langle \nabla_{x_0} \log p_{0|t}(x_0|x_t), \nabla \varphi(x_0) \rangle + \Delta \varphi(x_0) \\
  &=\langle \nabla \log \pdata (x_0), \nabla \varphi(x_0) \rangle + (c_t /\sigma_t^2) \langle x_t - c_t x_0, \nabla \varphi(x_0) \rangle + \Delta \varphi(x_0)  .
\end{align}
Establishing Foster-Lyapunov drift condition for this infinitesimal generator
will allow us to derive moment bounds for $x_0 \mapsto p_{0|t}(x_0 | x_t)$. We now
introduce the Lyapunov functional which will allow us to control these moments.
For any $p \in \nset$, $t > 0$ and $x_t \in \rset^d$, let
$V_{p, t,x_t}: \ \rset^d \to \coint{1,+\infty}$ given for any $x_0 \in \rset^d$
by
\begin{equation}
  \lyap(x_0) = 1 + \normLigne{x_0 - x_t/c_t}^{2p}  , \qquad c_t = \exp[-\alpha t]  .
\end{equation}

\begin{proposition}
  \label{prop:foster_lyap}
  Assume $\pdata \in \rmc^1(\rset^d, \rset)$ and that there exist $\mtt_0 > 0$, $d_0, C_0 \geq 0$
  such that for any $x_0 \in \rset^d$ we have
  \begin{equation}
    \label{eq:bound_cvx}
    \langle x_0, \nabla \log \pdata (x_0) \rangle \leq -\mtt_0 \normLigne{x_0}^2 + d_0\norm{x_0}  , \quad \normLigne{\nabla \log \pdata (x_0)} \leq C_0(1+ \normLigne{x_0})  . 
  \end{equation}
  Then for any $t > 0$, $x_t \in \rset^d$ and $p \in \nset$ there exist 
  $\beta_p \in \nset$, $a_p > 0$ and $b_p \geq 0$ (independent of $t$ and
  $x_t$) such that for any $x_0 \in \rset^d$ we have
  \begin{equation}
    \generator_{t, x_t}(\lyap)(x_0) \leq -a_p \lyap(x_0) + b_p(1 + \normLigne{x_t/c_t}^{\beta_p})  ,
  \end{equation}
  with $\beta_p = 2p$.
\end{proposition}

\begin{proof}
  Let $t \geq 0$, $x_0, x_t \in \rset^d$ and $p \in \nset$.
First, we have for any $x_0 \in \rset^d$
\begin{align}
  \label{eq:lyap_der}
  &\lyap(x_0) = \normLigne{x_0 -  x_t/c_t}^{2p}  , \quad \nabla \lyap(x_0) = 2p (x_0 - x_t/c_t) \normLigne{x_0 - x_t/c_t}^{2(p-1)}  , \\
  &\Delta \lyap(x_0) = 2p(2p-1) \normLigne{x_0 - x_t/c_t}^{2(p-1)}  . 
\end{align}
Second, using \Cref{prop:bound_p_norm}, the Cauchy-Schwarz inequality and \eqref{eq:bound_cvx}, we have for any
$x_0 \in \rset^d$
\begin{align}
  &\langle \nabla \log \pdata (x_0), x_0 - x_t/c_t \rangle \leq -\mtt_0 \normLigne{x_0}^2 + d_0\norm{x_0}  + \normLigne{\nabla \log \pdata (x_0)}\normLigne{x_t/c_t} \\
  &\quad \leq -\mtt_0 \normLigne{x_0 - x_t/c_t}^2 + 2\mtt_0 \normLigne{x_0}\normLigne{x_t}/c_t  + C_0(1 + \normLigne{x_0})\normLigne{x_t}/c_t \\
  & \qquad + d_0\norm{x_0 - x_t/c_t} + d_0 \norm{x_t}/c_t + \mtt_0 \normLigne{x_t}^2/c_t^2 \\
  &\quad \leq -\mtt_0 \normLigne{x_0 - x_t/c_t}^2 + \{(2\mtt_0 + C_0)\normLigne{x_t}/c_t+ d_0\} \normLigne{x_0 - x_t/c_t}   \\
  & \qquad \qquad + (3\mtt_0 + C_0) \normLigne{x_t}^2/c_t^2 + (C_0+d_0) \normLigne{x_t}/c_t  . 
\end{align}
Combining this result and
\eqref{eq:lyap_der}, we have for any $x_0 \in \rset^d$
  \begin{align}
    &\langle \nabla \log \pdata (x_0), \nabla \lyap(x_0) \rangle \\
    &\quad \leq -2p\mtt_0 \normLigne{x_0 - x_t/c_t}^{2p} + 2p\{(2\mtt_0 + C_0)\normLigne{x_t}/c_t+d_0\} \normLigne{x_0 - x_t/c_t}^{2p-1} \\
                                                                 & \qquad  + 2p(  (3\mtt_0 + C_0) \normLigne{x_t}^2/c_t^2 + (C_0 + d_0) \normLigne{x_t}/c_t )\normLigne{x_0 - x_t/c_t}^{2p-2} . 
  \end{align}
  Combining this result with \eqref{eq:generator} and the fact that for any
  $x_0 \in \rset^d$,
  $(c_t/\sigma_t^2) \langle x_t - c_t x_0, \nabla \lyap(x_0) \rangle \leq 0$, we
  get that for any $x_0 \in \rset^d$
  \begin{align}
    \generator_{t, x_t} (\lyap)(x_0) &\leq -2p\mtt_0 \normLigne{x_0 - x_t/c_t}^{2p} + 2p\{(2\mtt_0 + C_0)\normLigne{x_t}/c_t+d_0\} \normLigne{x_0 - x_t/c_t}^{2p-1} \\
                                                                 & \qquad  + 2p(  (3\mtt_0 + C_0) \normLigne{x_t}^2/c_t^2 + (C_0 + d_0) \normLigne{x_t}/c_t )\normLigne{x_0 - x_t/c_t}^{2p-2} .
  \end{align}
  Using \Cref{prop:bound_p_norm} there exist $\beta_p \in \nset$, $a_p>0$ and
  $b_p \geq 0$ (independent of $x_t$ and $t$) such that for any
  $x_0 \in \rset^d$ we have
  \begin{equation}
    \generator_{t, x_t} (\lyap)(x_0) \leq -a_p \lyap(x_0) + b_p(1 + (\normLigne{x_t}/c_t)^{\beta_p})  ,
  \end{equation}
which concludes the proof.
\end{proof}

Using this Foster-Lyapunov drift we are now ready to bound the moments of
$x_0 \mapsto p_{0|t}(x_0 | x_t)$.

\begin{proposition}
  \label{prop:moment_bound_inf}
  Assume that $\pdata \in \rmc^2(\rset^d, \rset)$ and that there exist $\mtt_0 > 0$, $d_0, C_0 \geq 0$
  such that for any $x_0 \in \rset^d$ we have
  \begin{equation}
    \label{eq:bound_cvx_duo}
    \langle x_0, \nabla \log \pdata (x_0) \rangle \leq -\mtt_0 \normLigne{x_0}^2 + d_0\normLigne{x_0}  , \quad \normLigne{\nabla \log \pdata(x_0)} \leq C_0(1+ \normLigne{x_0})  . 
  \end{equation}
  Then, for any $p \in \nset$ there exist $C_p \geq 0$ and $\beta_p \in \nset$
  such that for any $t \geq 0$ and $x_t \in \rset^d$
  \begin{equation}
    \label{eq:bound_t_small}
    \textstyle{\int_{\rset^d} \norm{x_0}^{p} p(x_0|x_t) \rmd x_0  \leq C_pc_t^{-2\beta_p}(1 + \normLigne{x_t}^{\beta_p})  , }
  \end{equation}
  with $c_t^2 = \exp[-2 \alpha t]$ and $\beta_p = p$.
\end{proposition}

\begin{proof}
  Let $t \geq 0$ and $x_t \in \rset^d$.  Using \cite[Theorem 2.3, Theorem
  3.1]{ikeda1989sto}, \Cref{prop:foster_lyap} and \cite[Theorem 2.1]{meyn1993criteria_iii} for any
  $x \in \rset^d$, there exists a unique strong solution $(\bfX_u^x)_{u \geq 0}$
  such that $\bfX_0^x \sim \updelta_x$ and
  \begin{equation}
    \rmd \bfX_u^x = \nabla \log p_{0|t}(\bfX_u^x|x_t) \rmd u + \sqrt{2} \rmd \bfB_u  . 
  \end{equation}
  Using \cite[Theorem 5.19]{leha1984diffusion} we get that
  $\ensembleLigne{(\bfX_u^x)_{u \geq 0}}{x \in \rset^d}$ is associated with a
  Feller semi-group. In addition, we have that for any
  $f \in \rmc_c^2(\rset^d)$,
  $\int_{\rset^d} \generator_{t, x_t}(f)(x_0) p_{0|t}(x_0|x_t) \rmd x_0 =
  0$. Therefore, using \cite[Proposition 1.5]{revuz1999continuous} and
  \cite[Theorem 9.17]{ethier1986markov} we get that the probability distribution
  with density $x_0 \mapsto p_{0|t}(x_0|x_t)$ is an invariant distribution for the
  semi-group associated with
  $\ensembleLigne{(\bfX_u^x)_{u \geq 0}}{x \in \rset^d}$.  Therefore, using
  \Cref{prop:foster_lyap} and \cite[Theorem 4.6]{meyn1993criteria_iii} we get
  that for any $p \in \nset$
  \begin{equation}
    \textstyle{\int_{\rset^d} (1+ \normLigne{x_0 - c_t^{-1} x_t}^{2p}) p_{0|t}(x_0|x_t) \rmd x_0 \leq b_p(1 + \norm{x_t/c_t}^{\beta_p})/a_p }
    \end{equation}
    which concludes the proof upon using that $c_t \leq 1$ and Jensen's inequality.
  \end{proof}
  
\subsubsection{Large times estimates}
\label{sec:large-time-gradient}

In \Cref{prop:moment_bound_inf}, the bound in \eqref{eq:bound_t_small} goes to
$+\infty$ as $t \to +\infty$ since $\lim_{t\to+\infty}c_t^{-1} = +\infty$ (if
$\alpha > 0$). This does not yield any degeneracy in our setting since we
consider a fixed time horizon $T > 0$. However, we can improve the result by
deriving another bound which is bounded at $t \to +\infty$ but explodes as
$t \to 0$. In this section we assume that $h: \ u \mapsto (\exp[u] - 1)/u$ is
extended to $0$ by continuity with $h(0) = 1$.

The following proposition is the equivalent of \Cref{prop:moment_infty} with a
bound which explodes for $t \to 0$ instead of $t \to +\infty$. Note that
contrary to \Cref{prop:moment_infty} we do not require any differentiability
condition the initial distribution $\pdata$.

\begin{proposition}
  \label{prop:moment_zero}
  Let $N \in \nset$. Assume that
  $\pdata \in \rmc^0(\rset^d, \ooint{0,+\infty})$ is bounded.  Then for any
  $t \geq 0$, $p_t \in \rmc^N(\rset^d, \ooint{0,+\infty})$ and for any
  $\ell \in \{1, \dots, N\}$, there exist $B_\ell \geq 0$ and
  $\beta_\ell \in \nset$ such that for any $t \geq 0$
  \begin{align}
    \textstyle{\normLigne{\nabla^\ell \log p_t (x)}} &\leq \textstyle{\sigma_t^{-\beta_{\ell}} B_\ell (1 + \int_{\rset^d}\norm{x_t - c_t x_0}^{\beta_\ell} p_{0|t}(x_0|x_t) \rmd x_0) } \\
    &  \leq \textstyle{\sigma_t^{-\beta_{\ell}} B_\ell (1 + \int_{\rset^d}\norm{x_t -  x_0}^{\beta_\ell} q_{0|t}(x_0|x_t) \rmd x_0)}  . 
  \end{align}
  with $\sigma_t^2 = (1 - \exp[-2\alpha t])/\alpha$ and for any $\tilde{x} \in \rset^d$
  \begin{align}
    &q_{0|t}(x_0|x_t) = \textstyle{\pdata(x_0/c_t) \rmg(x_t - x_0)  / \int_{\rset^d} \pdata(x_0/c_t) \rmg(x_t - x_0) \rmd x_0 }  , \\
    &\rmg(\tilde{x}) = (2 \uppi \sigma_t^2) \exp[-\norm{\tilde{x}}^2/(2\sigma_t^2)]  . 
  \end{align}

\end{proposition}

\begin{proof}
    First note that for any $t \geq 0$ and $x_t \in \rset^d$ we have 
  \begin{equation}
    \label{eq:int_form_zero}
    \textstyle{
      p_t(x_t) = \int_{\rset^d} \pdata(x_0) \rmg(x_t - c_t x_0) \rmd x_0  ,
      }
  \end{equation}
  with
  \begin{equation}
    c_t = \exp[-\alpha t]  , \qquad \rmg(\tilde{x}) = (2 \uppi \sigma_t^2)^{-d/2} \exp[-\norm{\tilde{x}}^2/(2\sigma_t^2)]  , \quad \sigma_t^2 = (1 - \exp[-2\alpha t])/ \alpha  . 
  \end{equation}
  Let $t \geq 0$. We have $p_t \in \rmc^N(\rset^d, \ooint{0,+\infty})$ upon
  combining the fact that $\pdata$ is bounded, \eqref{eq:int_form_zero} and the
  dominated convergence theorem.  Let $\ell \in \{0, \dots, N\}$ and
  $\lambda \in \nset^d$ such that $\abs{\lambda} \leq \ell$. Using
  \Cref{lemma:faadibruno_spec}-\ref{item:log} we have for any $x_t \in \rset^d$
    \begin{equation}
    \partial_{\lambda} \log p_t(x_t) = \rmP_{\lambda, \log}((\partial_m p_t(x_t) / p_t(x_t))_{\abs{m} \leq \abs{\lambda}})  . 
  \end{equation}
  For any $m \in \nset^d$ with $\abs{m} \leq \abs{\lambda}$, using the dominated
  convergence theorem, there exist $C_m \geq 0$ and $\beta_m \in \nset$ such
  that for any $x_t \in \rset^d$ we have
  \begin{equation}
\textstyle{
    \abs{\partial_m p_t(x_t)} \leq C_m \sigma_t^{-2\beta_m} \int_{\rset^d} (1 + \norm{x_t - c_t x_0}^{\beta_m}) \pdata(x_0) \rmg(x_t - c_t x_0) \rmd x_0  ,}
  \end{equation}
  which concludes the proof.
\end{proof}

For any $t \geq 0$ and $x_t \in \rset^d$ we introduce the infinitesimal
generator
$\generatort_{t, x_t}: \ \rmc_2(\rset^d, \rset) \to \rmc_2(\rset^d, \rset)$ given
for any $\varphi \in \rmc^2(\rset^d, \rset)$ and $x_0 \in \rset^d$ by
\begin{align}
  \label{eq:generatort}
  \generatort_{t, x_t}(f)(x_0) &= \langle \nabla \log q_{0|t}(x_0|x_t), \nabla \varphi(x_0) \rangle + \Delta \varphi(x_0) \\
  &=c_t^{-1} \langle \nabla \log \pdata (x_0/c_t), \nabla \varphi(x_0) \rangle +  \sigma_t^{-2} \langle x_t - x_0, \nabla \varphi(x_0) \rangle + \Delta \varphi(x_0)  . 
\end{align}
For any $p \in \nset$, let $V_{p}: \ \rset^d \to \coint{1,+\infty}$ given for any $x_0 \in \rset^d$ by
\begin{equation}
  \lyapp(x_0) = 1 + \normLigne{x_0}^{2p}  .
\end{equation}

The following proposition is the counterpart to \Cref{prop:foster_lyap}.

\begin{proposition}
  \label{prop:foster_lyap_zero}
  Assume that $\pdata \in \rmc^1(\rset^d, \rset)$ and that there exist $\mtt_0 > 0$, $d_0 \geq 0$
  such that for any $x_0 \in \rset^d$ we have
  \begin{equation}
    \label{eq:bound_cvx_zero}
    \langle x_0, \nabla \log \pdata (x_0) \rangle \leq -\mtt_0 \normLigne{x_0}^2 + d_0\norm{x_0}  .
  \end{equation}
  Then for any $t > 0$, $x_t \in \rset^d$ and $p \in \nset$ there exist 
  $\beta_p \in \nset$, $a_p > 0$ and $b_p \geq 0$ (independent of $t$ and
  $x_t$) such that for any $x_0 \in \rset^d$ we have
  \begin{equation}
    \generatort_{t, x_t}(\lyapp)(x_0) \leq -a_p\sigma_t^{-2} \lyapp(x_0) + b_p(1 + \normLigne{x_t/\sigma_t^2}^{\beta_p})  ,
  \end{equation}
  with $\beta_p = 2p$.
\end{proposition}

\begin{proof}
  Let $t \geq 0$, $x_0, x_t \in \rset^d$ and $p \in \nset$. First, we have for
  any $x_0 \in \rset^d$
  \begin{equation}
    \lyapp(x_0) = 1+ \norm{x_0}^{2p}  , \qquad \nabla \lyapp(x_0) = 2p \norm{x_0}^{2(p-1)} x_0  , \qquad \Delta \lyapp(x_0) = 2p(2p-1)\norm{x_0}^{2(p-1)}  . 
  \end{equation}
  Using this result, \eqref{eq:bound_cvx_zero} and \Cref{prop:bound_p_norm}, we
  get that for any $x_0 \in \rset^d$
  \begin{align}
    2p \langle \nabla \log \pdata (x_0/c_t), x_0/c_t \rangle \norm{x_0}^{2(p-1)} &\leq 2pc_t^{-1} (-\mtt_0 \norm{x_0}^{2p}/c_t + d_0 \norm{x_0}^{2p-1}) \\
    &\leq c_t^{-1} (2p-1)^{2p-1}(2p)^{1-2p}(\mtt_0/c_t)^{1-2p} d_0^{2p}  . 
  \end{align}
  Combining this result and the fact that $c_t \leq 1$, there exists
  $d_p \geq 0$ (independent from $t$ and $x_t$) such that for any
  $x_0 \in \rset^d$
  \begin{equation}
    \label{eq:inq_ineq}
    2p \langle \nabla \log \pdata (x_0/c_t), x_0/c_t \rangle \norm{x_0}^{2(p-1)} \leq d_p  . 
  \end{equation}
  In addition, we have for any $x_0 \in \rset^d$
  \begin{align}
    &(2p/\sigma_t^2) \langle x_0, x_t - x_0 \rangle \norm{x_0}^{2(p-1)} + 2p(2p-1) \norm{x_0}^{2(p-1)} \\
    & \qquad \qquad 
      \leq -(2p/\sigma_t^2) \norm{x_0}^{2p} + (2p/\sigma_t^2) \norm{x_0}^{2p-1}\norm{x_t} + 2p(2p-1) \norm{x_0}^{2p-1} + 2p(2p-1)  .
  \end{align}
Combining this result and \eqref{eq:inq_ineq} we have for any $x_0 \in \rset^d$
\begin{align}
  &\generatort_{t, x_t}(V_p)(x_0) \\
  & \qquad \leq -(2p/\sigma_t^2) \norm{x_0}^{2p} + (2p/\sigma_t^2) \norm{x_0}^{2p-1}\norm{x_t} + 2p(2p-1) \norm{x_0}^{2p-1} + 2p(2p-1) + d_p  . 
\end{align}
We conclude upon using \Cref{prop:bound_p_norm}.
\end{proof}

The next proposition is the counterpart of \Cref{prop:moment_bound_inf}.

\begin{proposition}
\label{prop:moment_bound_zero}
    Assume that $\pdata \in \rmc^2(\rset^d, \rset)$ and that there exist $\mtt_0 > 0$, $d_0 \geq 0$
  such that for any $x_0 \in \rset^d$ we have
  \begin{equation}
    \langle x_0, \nabla \log \pdata (x_0) \rangle \leq -\mtt_0 \normLigne{x_0}^2 + d_0\normLigne{x_0}  .
  \end{equation}
  Then, for any $p \in \nset$ there exist $C_p \geq 0$ and $\beta_p \in \nset$
  such that for any $t \in \geq 0$ and $x_t \in \rset^d$
  \begin{equation}
    \label{eq:bound_t}
    \textstyle{\int_{\rset^d} \norm{x_t - x_0}^{p} q_{0|t}(x_0|x_t) \rmd x_0  \leq C_p\sigma_t^{-2\beta_p}(1 + \normLigne{x_t}^{\beta_p})  , }
  \end{equation}
  with $\sigma_t^2 = (1 - \exp[-2\alpha t])/\alpha$ and $\beta_p = p$.
\end{proposition}

\begin{proof}
  The proof is similar to the one of \Cref{prop:moment_bound_inf}.
\end{proof}

\subsubsection{Uniform in time logarithmic derivatives estimates}
\label{sec:uniform}

In this section we combine the results of \Cref{sec:large-time-gradient} and
\Cref{sec:small-time-gradient} to establish uniform in time estimates for the
logarithmic derivatives of the density of the Ornstein-Ulhenbeck diffusion.

\begin{theorem}
  \label{thm:uniform_bounds}
  Let $N \in \nset$ with $N \geq 2$.  Assume that
  $\pdata \in \rmc^N(\rset^d, \rset)$ and that there exist $\mtt_0 > 0$,
  $d_0, C_0 \geq 0$ such that for any $x_0 \in \rset^d$ we have
  \begin{equation}
    \label{eq:bound_cvx_duo}
    \langle x_0, \nabla \log \pdata (x_0) \rangle \leq -\mtt_0 \normLigne{x_0}^2 + d_0\normLigne{x_0}  , \quad \normLigne{\nabla \log \pdata (x_0)} \leq C_0(1+ \normLigne{x_0})  . 
  \end{equation}
  In addition, assume that
  $\pdata$ is bounded and that for any
  $\ell \in \{1, \dots, N\}$ there exist $A_\ell \geq 0 $ and $\alpha_\ell \in \nset$
  such that for any $x_0 \in \rset^d$
  \begin{equation}
    \label{eq:log_bound_uno}
    \normLigne{\nabla^\ell \log \pdata (x_0)} \leq A_{\ell} (1 + \norm{x_0}^{\alpha_\ell})   . 
  \end{equation}
  Then for any $t \geq 0$, $p_t \in \rmc^N(\rset^d, \ooint{0,+\infty})$ and for
  any $\ell \in \{1, \dots, N\}$, there exist $D_\ell \geq 0$ and
  $\beta_\ell \in \nset$ such that for any $t \geq 0$
  \begin{equation}
    \textstyle{\normLigne{\nabla^\ell \log p_t (x_t)} \leq D_{\ell}(1 + \norm{x_t}^{\beta_\ell})  .}
  \end{equation}
  In particular if $\alpha_1 = 1$ then $\beta_1 =1$.
\end{theorem}

\begin{proof}
  Let $t \geq0$ and $\ell \in \{1, \dots, N\}$. Using \Cref{prop:moment_infty} and \Cref{prop:moment_bound_inf} there exist
  $D_{\ell}^1 \geq 0 $ and $\beta_\ell^1 \in \nset$ such that for any $x_t \in \rset^d$ we have
  \begin{equation}
    \textstyle{\normLigne{\nabla^\ell \log p_t (x_t)} \leq D_\ell^1c_t^{-2\beta_\ell^1} (1 + \norm{x_t}^{\beta_\ell^1}) . }
  \end{equation}
  Similarly, using \Cref{prop:moment_zero} and \Cref{prop:moment_bound_zero} there exist $D_{\ell}^2 \geq 0$ and
  $\beta_\ell^2 \in \nset$ such that for any $x_t \in \rset^d$ we have
  \begin{equation}
    \textstyle{\normLigne{\nabla^\ell \log p_t (x_t)} \leq D_\ell^2 (\alpha^{1/2}\sigma_t)^{-2\beta_\ell^2} (1 + \norm{x_t}^{\beta_\ell^2}) . }
  \end{equation}
  Therefore, there exist $\tilde{D}_\ell \geq 0$ and $\beta_\ell \in \nset$ such that for any $x_t \in \rset^d$ we have
  \begin{equation}
    \textstyle{\normLigne{\nabla^\ell \log p_t (x_t)} \leq \tilde{D}_\ell \min(\alpha^{-1} \sigma_t^{-2}, c_t^{-2})^{\beta_\ell} (1 + \norm{x_t}^{\beta_\ell}) . }
  \end{equation}
  Since for any $c_t^{-2} = \exp[2\alpha t]$ and
  $\alpha^{-1} \sigma_t^{-2} = (1 - \exp[-2\alpha t])^{-1}$. Hence we have
  \begin{equation}
   \min(\alpha^{-1} \sigma_t^{-2}, c_t^{-2})^{\beta_\ell} \leq \max \ensembleLigne{\min(1/u, 1/(1-u))}{u \in \ccint{0,1}} \leq 2^{\beta_\ell}  , 
  \end{equation}
  which concludes the first part proof.  We now show that if $\alpha_1 = 1$ then
  $\beta_1=1$. Recall that for any $t \geq 0$ and $x_t \in \rset^d$ we have
  \begin{equation}
    \textstyle{
      p_t(x_t) = \int_{\rset^d} \pdata(x_0) \rmg(x_t - c_t x_0) \rmd x_0  ,
      }
    \end{equation}
      with for any $\tilde{x} \in \rset^d$
  \begin{equation}
    c_t = \exp[-\alpha t]  , \qquad \rmg(\tilde{x}) = (2 \uppi \sigma_t^2)^{-d/2} \exp[-\norm{\tilde{x}}^2/(2\sigma_t^2)]  , \quad \sigma_t^2 = (1 - \exp[-2\alpha t])/ \alpha  . 
  \end{equation}
  Therefore, using the dominated convergence theorem we get that for any $x_t \in \rset^d$
  \begin{equation}
    \label{eq:uno_udo}
    \textstyle{
    \nabla \log p_t (x_t) = \sigma_t^{-2} \int_{\rset^d} (x_t - c_t x_0) p_{0|t}(x_0|x_t) \rmd x_0 = \sigma_t^{-2} \int_{\rset^d} (x_t - c_t x_0) q_{0|t}(x_0|x_t) \rmd x_0  .}
  \end{equation}
  Similarly, using the dominate convergence theorem and  change of variable $z = x_t - c_t x_0$, we have for any $x_t \in \rset^d$
  \begin{equation}
    \textstyle{
    \nabla \log p_t(x_t) = c_t^{-1} \int_{\rset^d} \nabla \log \pdata (x_0) p_{0|t}(x_0|x_t) \rmd x_0  .}
  \end{equation}
  We conclude the proof upon combining this result, \eqref{eq:uno_udo},
  \eqref{eq:log_bound_uno} with $\alpha_1 = 1$, \Cref{prop:moment_bound_zero}
  and \Cref{prop:moment_bound_inf}. In particular, we use that $\beta_1 = 1$.
\end{proof}

\subsection{Proof of \Cref{prop:convergence_score_matching}}
\label{prop:convergence_score_matching:proof}

We start by recalling the following basic lemma.

\begin{lemma}
  \label{lemma:data_processing_tv}
  Let $(\mse, \mce)$ and $(\msf, \mcf)$ be two measurable spaces and
  $\Kker: \ \mse \times \mcf \to \ccint{0,1}$ be a Markov kernel. Then for any
  $\mu_0, \mu_1 \in \Pens(\mse)$ we have
  \begin{equation}
    \tvnormLigne{\mu_0 \Kker - \mu_1 \Kker} \leq \tvnormLigne{\mu_0 - \mu_1}  . 
  \end{equation}
  In addition, for any $\varphi: \mse \to \msf$ measurable we get that
  \begin{equation}
    \tvnormLigne{\varphi_\# \mu_0  - \varphi_\# \mu_1 } \leq \tvnormLigne{\mu_0 - \mu_1}  ,
  \end{equation}
  with equality if $\varphi$ is injective.
\end{lemma}

\begin{proof}
  We divide the proof into two parts.
  \begin{enumerate}[label=(\alph*), wide, labelwidth=!, labelindent=0pt]
  \item Note that for any $f: \ \msf \to \rset$ such that $\norm{f}_{\infty} \leq 1$ we have $\norm{\Kker f}_{\infty} \leq 1$.
  Using this result we get 
  \begin{align}
    \tvnormLigne{\mu_0 \Kker - \mu_1 \Kker} &= \textstyle{\sup \ensembleLigne{\int_{\msf}  f(y) \rmd (\mu_0 \Kker)(y)-   \int_{\msf}  f(y) \rmd (\mu_1 \Kker)(y)}{\norm{f}_{\infty} \leq 1} }\\
                                       &= \textstyle{\sup \ensembleLigne{\int_{\mse}  \Kker f(x) \rmd \mu_0(x) -  \int_{\mse}  \Kker f(x) \rmd \mu_0(x)}{ \norm{f}_{\infty} \leq 1} \leq \tvnormLigne{\mu_0 - \mu_1}  . }
  \end{align}
\item We have
  \begin{align}
    \tvnormLigne{\varphi_\# \mu_0  - \varphi_\# \mu_1 } &= \textstyle{\sup \ensembleLigne{\int_{\mse}  f(\varphi(x)) \rmd \mu_0 (x)-   \int_{\mse}  f(\varphi(x)) \rmd \mu_1 (x)}{\norm{f}_{\infty} \leq 1} } \\
    &\leq \textstyle{\sup \ensembleLigne{\int_{\mse}  f(x) \rmd \mu_0(x)-   \int_{\mse}  f(x) \rmd \mu_1 (x)}{\norm{f}_{\infty} \leq 1}  \leq \tvnormLigne{\mu_0 - \mu_1}  . }
  \end{align}
  If $\varphi$ is injective then there exists $\varphi^{-1}: \ \msf \to \msf$
  (measurable) such that $\varphi^{-1} \circ \varphi = \Id$. Therefore, for any
  $f: \mse \to \rset$ with $\normLigne{f}_{\infty} \leq 1$ we have
  $f = (f \circ \varphi^{-1}) \circ \varphi$ and
  $\normLigne{f \circ \varphi^{-1}}_{\infty} \leq 1$. Hence we have
  \begin{align}
    &\tvnormLigne{\mu_0  - \mu_1 } = \textstyle{\sup \ensembleLigne{\int_{\mse}  f(x) \rmd \mu_0 (x)-   \int_{\mse}  f(x) \rmd \mu_1 (x)}{\norm{f}_{\infty} \leq 1} } \\
    &\qquad \leq \textstyle{\sup \ensembleLigne{\int_{\mse}  f(\varphi(x)) \rmd \mu_0(x)-   \int_{\mse}  f(\varphi(x)) \rmd \mu_1 (x)}{\norm{f}_{\infty} \leq 1}  \leq \tvnormLigne{\varphi_\#\mu_0 - \varphi_\#\mu_1}  , }
  \end{align}
  which concludes the proof.
  \end{enumerate}
\end{proof}

We will also make use of the following inequality.

\begin{lemma}
  \label{lemma:ineq_varphi}
  Let $\vareps> 0$, $x,y \in \rset^d$, $t > 2/\vareps$ and $\varphi: \ \ccint{0,1} \to \rset$ such that for any
  $s \in \ccint{0,1}$, $\varphi(s) = \exp[-\norm{x - sy}^2/(4t)]$. Then
  $\varphi \in \rmc^1(\ccint{0,1}, \rset)$ and we have for any $s \in \ccint{0,1}$
  \begin{equation}
    \abs{\varphi'(s)} \leq 2(1+\vareps^{-1}) (1 + \norm{x}) \exp[-\norm{x}^2/(8t)]\exp[\vareps\norm{y}^2]/t  .
  \end{equation}
\end{lemma}

\begin{proof}
  Let $s \in \ccint{0,1}$, we have
  \begin{equation}
    \varphi'(s) = \parentheseLigne{\langle x,y \rangle - s \norm{y}^2}\exp[-\norm{x - sy}^2/(4t)]/(2t)  . 
  \end{equation}
  Using the Cauchy-Schwarz inequality and that for any $a,b \in \rset^d$,
  $-\norm{a+b}^2 \leq -\norm{a}^2/2 +\norm{b}^2$ we get
\begin{equation}
      \label{eq:first_ineq_l}
  \abs{\varphi'(s)} \leq (\norm{x} \norm{y} + \norm{y}^2) \exp[-\norm{x}^2/(8t) + \norm{y}^2/(4t)]/(2t)  .
\end{equation}
In addition, we have
\begin{align}
  \label{eq:second_ineq_l}
  \norm{y} \exp[\normLigne{y}^2/(4t)] \leq \norm{y} \exp[\vareps \normLigne{y}^2/2] \leq (1 + \norm{y}^2)\exp[\vareps \norm{y}^2/2] \leq 2(1+\vareps^{-1})\exp[\vareps \norm{y}^2]  . 
\end{align}
Finally we also have $\norm{y}^2 \exp[\normLigne{y}^2/(4t)] \leq (1+\vareps^{-1})\exp[\vareps \norm{y}^2]$.
Combining this result, \eqref{eq:first_ineq_l} and \eqref{eq:second_ineq_l} concludes the proof.
\end{proof}

Finally we show the following lemma which is a straightforward consequence of
Girsanov's theorem \cite[Theorem 7.7]{lipster2001statistics}. A similar version
of this lemma can be found in the proof of \cite[Proposition
2]{durmus2017nonasymp} and in \cite[Lemma 26]{laumont2021bayesian} (version
where the dependence of the drift in $w \in \rmc(\ccint{0,T}, \rset^d)$ is
replaced by a (simpler) dependence in $x \in \rset^d$). We refer to \cite[Section
4]{lipster2001statistics} for the definitions of semi-group, non-anticipative
processes and diffusion type processes.
\begin{lemma}
  \label{lemma:girsanov}
  Let $T > 0$,
  $b_1, b_2: \ \coint{0, +\infty} \times \rmc(\ccint{0,T}, \rset^d) \to
  \rset^{\dim}$ measurable such that for any $i \in \{1, 2\}$ and
  $x \in \rset^{\dim}$,
  $\rmd \bfX_t^{(i)} = b_i(t, (\bfX_s^{(i)})_{s \in \ccint{0,T}}) \rmd t +
  \sqrt{2} \rmd \bfB_t$ admits a unique strong solution with $\bfX_0^{(i)} = x$
  and $(b_i(t, (\bfX_s^{(i)}))_{t \in \ccint{0,T}}$ is non-anticipative, with
  Markov semi-group $(\Pker_t^{(i)})_{t \geq 0}$. In addition, assume that for
  any $x \in \rset^{\dim}$ and $i \in \{1, 2\}$,
  $\probaLigne{\int_0^T \defEnsLigne{ \normLigne{b_i(t, (\bfX_s^{(i)})_{s \in
          \ccint{0,T}})}^2 + \normLigne{b_i(t, (\bfB_s)_{s \in \ccint{0,T}})}^2
    } \rmd t < + \infty} = 1$.  Then for any $x \in \rset^{\dim}$ we have
  \begin{equation}
    \textstyle{
    \tvnormLigne{\updelta_x \Pker_T^{(1)} - \updelta_x \Pker_T^{(2)}}^2  \leq  (1/2) \int_0^T \expeLigne{\normLigne{b_1(t, (\bfX_s^{(1)})_{s \in \ccint{0,T}}) - b_2(t, (\bfX_s^{(1)})_{s \in \ccint{0,T}})}^2} \rmd t  .}
  \end{equation}
\end{lemma}

\begin{proof}
  Let $T > 0$ and $x \in \rset^{\dim}$. For any $i \in \{1, 2\}$, denote
  $\mu_{(i)}^x$ the distribution of $(\bfX_t^{(i)})_{t \in \ccint{0,T}}$ on the
  Wiener space $(\contspace, \mcb{\contspace})$ with $\bfX_0^{(i)} =
  x$. Similarly denote $\mu_B^x$ the distribution of
  $(\bfB_t)_{t \in \ccint{0,T}}$ with $\bfB_0 = x$, where we recall that
  $(\bfB_t)_{t \in \ccint{0,T}}$ is a $d$-dimensional Brownian motion. Using
  Pinsker's inequality \cite[Equation 5.2.2]{bakry:gentil:ledoux:2014} and the
  transfer theorem \cite[Theorem 4.1]{kullback1997information} we get that
  \begin{equation}
    \tvnormLigne{\updelta_x \Pker_T^{(1)} - \updelta_x \Pker_T^{(2)}}^2 \leq 2  \KLLigne{\mu_{(1)}}{\mu_{(2)}}  .
  \end{equation}
  Since for any $i \in \{1, 2\}$,
  $\probaLigne{\int_0^T \defEnsLigne{ \normLigne{b_i(t, (\bfX_s^{(i)})_{s \in
          \ccint{0,T}})}^2 + \normLigne{b_i(t, (\bfB_s)_{s \in \ccint{0,T}})}^2
    } \rmd t < + \infty} = 1$ and the processes
  $(\bfX_t^{(i)})_{t \in \ccint{0,T}}$ are of diffusion type for $i \in \{1,2\}$
  we can apply Girsanov's theorem \cite[Theorem 7.7]{lipster2001statistics} and
  $\mu_B$-almost surely for any $w \in \rmc(\ccint{0,T}, \rset)$ we get
  \begin{align}
    &(\rmd \mu_{(1)}^x / \rmd \mu_B^x)((w_t)_{t \in \ccint{0,T}}) \\
    & \qquad \textstyle{= \exp \parentheseDeuxLigne{(1/2) \int_0^T \langle b_1(t, (w_s)_{s \in \ccint{0,T}}), \rmd w_t \rangle - (1/4) \int_0^T  \normLigne{b_1(t, (w_s)_{s \in \ccint{0,T}})}^2 \rmd t} } \\
    &(\rmd \mu_B^x / \rmd \mu_{(2)}^x)((w_t)_{t \in \ccint{0,T}}) \\
    & \qquad \textstyle{= \exp \parentheseDeuxLigne{-(1/2) \int_0^T \langle b_2(t, (w_s)_{s \in \ccint{0,T}})), \rmd w_t \rangle + (1/4) \int_0^T  \normLigne{b_2(t, (w_s)_{s \in \ccint{0,T}}))}^2 \rmd t}  .}
  \end{align}
  Hence, we obtain that
  \begin{align}
    \KLLigne{\mu_{(1)}^x}{\mu_{(2)}^x} &= \expeLigne{\log((\rmd \mu_{(1)}^x / \rmd \mu_{(2)}^x)((\bfX_t^{(1)})_{t \in \ccint{0,T}}))} \\
    &\textstyle{= (1/4) \int_0^T \expeLigne{\normLigne{b_1(t, (\bfX_s^{(1)})_{s \in \ccint{0,T}}) - b_2(t, (\bfX_s^{(1)})_{s \in \ccint{0,T}})}^2} \rmd t }
  \end{align}
which concludes the proof.
\end{proof}

We study distributions satisfying some curvature assumption and
show that they are sub-Gaussian. More precisely, we show the following proposition.

\begin{lemma}
  \label{prop:subgaussian}
  Let $q \in \rmc^1(\rset^d, \ooint{0, +\infty})$ and $\mtt > 0$ and $c \geq 0$
  such that for any $x \in \rset^d$ we have
  $\langle \nabla \log q (x), x \rangle \leq -\mtt \norm{x}^2 + c 
  \norm{x}$. Then for any $\vareps \in \coint{0, \mtt/2}$ we have 
  \begin{equation}
    \textstyle{\int_{\rset^d} \exp[\vareps \norm{x}^2] q(x)\rmd x < +\infty  . }
  \end{equation}

\end{lemma}

\begin{proof}
For any $x \in \rset^d$ we have
\begin{align}
  \log q (x) &= \textstyle{\log q(0) + \int_0^1 \langle \nabla \log q (tx), x \rangle \rmd t } \\
  &\leq \textstyle{\log q (0) - \mtt \int_0^1 t \norm{x}^2 \rmd t + c\normLigne{x} } \leq \textstyle{\log q (0) + c \normLigne{x} - \mtt  \norm{x}^2 }  .
\end{align}
which concludes the proof.
\end{proof}

Finally, we will use the following basic lemma.

\begin{lemma}
  \label{lemma:basic_ornstein}
  Let $\mu \in \Pens(\rset^d)$, $\alpha_1 \in \rset$, $\beta_1 > 0$ and $(\bfX_t)_{t \geq 0}$ such that
  $\bfX_0$ has distribution $\mu$ and
  \begin{equation}
    \rmd \bfX_t = \alpha_1 \bfX_t \rmd t + \beta_1^{1/2} \rmd \bfB_t  , 
  \end{equation}
  where $(\bfB_t)_{t \geq 0}$ is a Brownian motion. Then for any
  $\alpha_2 \in \rset$ and $\beta_2 > 0$ we have that $(\bfY_t)_{t \geq 0}$
  given for any $t \geq 0$ by $\bfY_t = \alpha_2 \bfX_{\beta_2 t}$ satisfies
  \begin{equation}
    \rmd \bfY_t = \beta_2 \alpha_1 \bfY_t \rmd t + \alpha_2 (\beta_2 \beta_1)^{1/2} \rmd \tilde{\bfB}_t  ,
  \end{equation}
  where $(\tilde{\bfB}_t)_{t \geq 0}$ is a Brownian motion, and $\bfY_0$ has
  distribution $(\tau_{\alpha_2})_\# \mu$, where for any $x \in \rset^d$,
  $\tau_{\alpha_2}(x) = \alpha_2 x$.
\end{lemma}

\begin{proof}
  Let $t \geq 0$. Using the change of variable $u \mapsto \beta_2 u$ the following equalities hold in distribution 
  \begin{align}
    \bfY_t &= \textstyle{\alpha_2 \alpha_1 \int_0^{\beta_2 t} \bfX_s \rmd s + \alpha_2 \beta_1^{1/2} \bfB_{\beta_2 t}}\\
    &= \textstyle{\beta_2 \alpha_2 \alpha_1 \int_0^{t} \bfX_{\beta_2 s} \rmd s + \alpha_2 (\beta_1\beta_2)^{1/2} \bfB_{t}  }
      = \textstyle{\beta_2  \alpha_1 \int_0^{t} \bfY_{s} \rmd s + \alpha_2 (\beta_1\beta_2)^{1/2} \bfB_{t}  , }
  \end{align}
  which concludes the proof.
\end{proof}

We now turn to the proof of \Cref{prop:convergence_score_matching}
\begin{proof}
Let $\alpha \geq 0$. For any $k \in \{1, \dots, N\}$, denote $\Rker_k$ the
Markov kernel such that for any $x \in \rset^d$, $\msa \in \mcb{\rset^d}$ and
$k \in \{0, \dots, N-1\}$ we have
    \begin{equation}
      \textstyle{\Rker_{k+1}(x, \msa) = (4 \uppi \gamma_{k+1})^{-1/2} \int_{\msa} \exp[-\norm{\tilde{x} - \Tnplusun(x)}^2/(4 \gamma_{k+1})] \rmd \tilde{x}  , }
    \end{equation}
    where for any $x \in \rset^d$,
    $\Tnplusun(x) = x + \gamma_{k+1} \defEns{\alpha x + 2s_{\theta}(t_k, x)}$,
    where $t_k =\sum_{\ell=0}^{k-1} \gamma_{\ell}$. Define for any
    $k_0, k_1 \in \{1, \dots, N\}$ with $k_1 \geq k_0$
    $\Qker_{k_0, k_1} = \prod_{\ell = k_0}^{k_1} \Rker_{\ell}$. Finally, for
    ease of notation, we also define for any $k \in \{1, \dots, N\}$,
    $\Qker_k = \Qker_{1,k}$. Note that for any $k \in \{1, \dots, N\}$, $Y_k$
    has distribution $\pi_{\infty} \Qker_k$, where $\pi_{\infty}\in \Pens(\rset^d)$ with
    density w.r.t. the Lebesgue measure $\pdata$. Let
    $\Pbb \in \Pens(\contspace)$ be the probability measure associated with the diffusion
    \begin{equation}
      \rmd \bfX_t = -\alpha \bfX_t \rmd t + \sqrt{2} \rmd \bfB_t  , \quad \bfX_0 \sim \pi_0  , 
    \end{equation}
    where $\pi_0 \in \Pens(\rset^d)$ admits a density w.r.t. the Lebesgue
    measure given by $\pdata$.  First note that using that $\Pbb_0 = \pi_0$ we
    have for any $\msa \in \mcb{\rset^d}$
    \begin{equation}
      \pi_0 \Pbb_{T|0} (\Pbb^R)_{T|0} (\msa) = \Pbb_T (\Pbb^R)_{T|0} (\msa) = (\Pbb^R)_0 (\Pbb^R)_{T|0} (\msa)  = (\Pbb^R)_T(\msa) = \pi_0(\msa)  . 
    \end{equation}
    Hence $\pi_0 = \pi_0 \Pbb_{T|0} (\Pbb^R)_{T|0}$.  Using this result and
    \Cref{lemma:data_processing_tv}, we have
    \begin{align}
      \tvnormLigne{\pi_0 - \pi_{\infty} \Qker_N} &= \tvnormLigne{\pi_0 \Pbb_{T|0} (\Pbb^R)_{T|0} - \pi_{\infty} \Qker_N} \\
                                         &\leq \tvnormLigne{\pi_0 \Pbb_{T|0} (\Pbb^R)_{T|0} - \pi_{\infty} (\Pbb^R)_{T|0}} +\tvnormLigne{\pi_{\infty} (\Pbb^R)_{T|0} - \pi_{\infty} \Qker_N} \\
          &\leq \tvnormLigne{\pi_0 \Pbb_{T|0}- \pi_{\infty}} +\tvnormLigne{\pi_{\infty} (\Pbb^R)_{T|0} - \pi_{\infty} \Qker_N}  .                                  
    \end{align}
    Note that $\mathcal{L}(X_0) = \mathcal{L}(Y_N) = \pi_{\infty} \Qker_N$ and therefore
    \begin{equation}
      \tvnormLigne{\mathcal{L}(X_0) - \pi_0} \leq \tvnormLigne{\pi_0 \Pbb_{T|0}- \pi_{\infty}} +\tvnormLigne{\pi_{\infty} (\Pbb^R)_{T|0} - \pi_{\infty} \Qker_N}  .
    \end{equation}
    We now bound each one of these terms.
    \begin{enumerate}[wide, labelwidth=!, labelindent=0pt, label=(\alph*)]
    \item First, assume that $\alpha > 0$. Let $T_\alpha = \alpha T$ and 
      $\tilde{\Pbb} \in \Pens(\rmc(\ccint{0,T_{\alpha}}, \rset^d))$ be associated with
      $(\bfZ_t)_{t \in \ccintLigne{0,T_{\alpha}}}$ the classical
      Ornstein-Ulhenbeck process with $\bfZ_0 \sim (\tau_\alpha)_{\#} \pi_0$,
      where for any $x \in \rset^d$ we have $\tau_{\alpha}(x) = \alpha^{1/2} x$,
      satisfying the following SDE:
      $\rmd \bfZ_t = -\bfZ_t \rmd t + \sqrt{2} \rmd \bfB_t$. We denote
      $\pi_0^\alpha = (\tau_\alpha)_{\#} \pi_0$,
      $\mu = (\tau_\alpha)_{\#} \pi_{\infty}$. Note
      that since $\pprior$ is the Gaussian density with zero mean and covariance
      matrix $(1/\alpha) \Id$, $\mu$ is the Gaussian distribution with zero mean
      and identity covariance matrix.

      First, using \cite[Proposition 4.1.1, Proposition 4.3.1, Theorem
      4.2.5]{bakry:gentil:ledoux:2014}, we get that for any
      $t \in \ccintLigne{0,T_{\alpha}}$, $f \in \rmL^1(\mu)$ and $x \in \rset^d$
    \begin{equation}
      \label{eq:variance_decay}
      \textstyle{
        \int_{\rset^d} (\tilde{\Pbb}_{t|0}g(x))^2 \rmd \mu(x) \leq \exp[-2t]  \int_{\rset^d} g^2(x) \rmd \mu(x)  , \quad \text{with} \ g(x) = f(x) - \int_{\rset^d} f(\tilde{x}) \rmd \mu(\tilde{x})  .
        }
    \end{equation}
    Recall that $(\bfX_t)_{t \geq 0}$ satisfies
    $\rmd \bfX_t = -\alpha \bfX_t + \rmd \bfB_t$. Using
    \Cref{lemma:basic_ornstein} we have that for any $t \in \ccint{0,T}$,
    $\bfZ_t$ and $\alpha^{1/2} \bfX_{\alpha^{-1} t}$ have the same distribution.
    Hence for any $t \in \ccintLigne{0,T}$ we have
    $\Pbb_t = (\tau_\alpha^{-1})_{\#} \tilde{\Pbb}_{\alpha t}$.
    Therefore, using that $(\tau_{\alpha})_\# \pi_{\infty} = \mu$, that $\tilde{\Pbb}$
    is Markov and \Cref{lemma:data_processing_tv}, we get that
    \begin{align}
      \tvnormLigne{\pi_0 \Pbb_{t|0} - \pi_{\infty}} &=  \tvnormLigne{\Pbb_{t} - \pi_{\infty}} = \tvnormLigne{(\tau_{\alpha})_{\#}\Pbb_{t} - (\tau_{\alpha})_{\#}\pi_{\infty}} \\
      &= \tvnormLigne{\tilde{\Pbb}_{\alpha t} - \mu} = \tvnormLigne{\tilde{\Pbb}_{\alpha t_0} \tilde{\Pbb}_{\alpha(t-t_0)|0} - \mu} . 
    \end{align}
    Finally, note that we have for any $t \geq t_0 \in \ccintLigne{0,T}$ and $x \in \rset^d$
    \begin{equation}
      \label{eq:reversible}
      (\rmd (\tilde{\Pbb}_{\alpha  t_0} \tilde{\Pbb}_{\alpha (t-t_0)|0}) / \rmd \mu)(x) = \tilde{\Pbb}_{\alpha (t-t_0)|0} f(x)  , \quad \text{with} \ f(x) = (\rmd \tilde{\Pbb}_{\alpha  t_0} / \rmd \mu)(x)  .  
    \end{equation}
    Let $g = f - 1$.  Using \eqref{eq:reversible}, \eqref{eq:variance_decay} and
    that $(\tau_{\alpha})_\# \pi_{\infty} = \mu$, we get that for any $t \geq t_0$ with
    $t \in \ccintLigne{0,T}$
    \begin{align}
      \label{eq:tv_ineq_int}
      \tvnormLigne{\pi_0 \Pbb_{t|0} - \pi_{\infty}} &\leq \tvnormLigne{\tilde{\Pbb}_{\alpha t_0} \tilde{\Pbb}_{\alpha (t-t_0)|0} - \mu} \\
                                        &\leq \textstyle{\int_{\rset^d} \absLigne{\tilde{\Pbb}_{\alpha (t-t_0)|0} f(x) - 1} \rmd \mu(x)} \\
      &\leq \textstyle{\parentheseLigne{\int_{\rset^d} \parentheseLigne{\tilde{\Pbb}_{\alpha (t-t_0)|0} g(x)}^2 \rmd \mu(x)}^{1/2}} \\
      &\leq \textstyle{\exp[-\alpha (t - t_0)]\parentheseLigne{\int_{\rset^d} g^2(x) \rmd \mu(x)}^{1/2} }\\
                                        &\leq \textstyle{\exp[-\alpha (t - t_0)]\parentheseLigne{\int_{\rset^d} g^2(\alpha^{1/2}x) \rmd \pi_{\infty}(x)}^{1/2} } .     \end{align}
In addition, we have for any $\varphi \in \rmc_c(\rset^d, \rset)$
\begin{align}
  \textstyle{\int_{\rset^d} \varphi(x) f(\alpha^{1/2}x) \rmd \pi_{\infty}(x)} &= \textstyle{\int_{\rset^d} \varphi(\alpha^{-1/2}  x) f(x) \rmd \mu(x)} \\
    &= \textstyle{\int_{\rset^d} \varphi(\alpha^{-1/2}  x) \rmd \tilde{\Pbb}_{\alpha t_0}(x) =  \int_{\rset^d} \varphi(x) \rmd \Pbb_{t_0}(x)  .}
\end{align}
Hence, for any $x \in \rset^d$, $g(\alpha^{1/2}x) = (\rmd \Pbb_{t_0} / \rmd \pi_{\infty})(x) - 1$. Combining this result and \eqref{eq:tv_ineq_int} we get that for any $t \geq t_0$ with
    $t \in \ccintLigne{0,T}$
    \begin{equation}
      \label{eq:tv_ineq}
      \textstyle{
      \tvnormLigne{\pi_0 \Pbb_{t|0} - \pi_{\infty}} \leq \sqrt{2} \exp[-\alpha (t - t_0)]\parenthese{1 + \int_{\rset^d}  (\rmd \Pbb_{t_0} / \rmd \pi_{\infty})(x)^2 \rmd \pi_{\infty}(x)}^{1/2}  . }
  \end{equation}
  Let $t_0 \in \ccintLigne{0,T}$. 
    We now derive an upper bound for
    $\int_{\rset^d} (\rmd \Pbb_{t_0} / \rmd \pi_{\infty})(x)^2 \rmd \pi_{\infty}(x)$. We
    recall that $\Pbb_{t_0}$ and $\pi_{\infty}$ admit density w.r.t. the Lebesgue
    measure denoted $p_{t_0}$ and $p_{\infty}$ such that for any $x \in \rset^d$
    \begin{equation}
      \textstyle{
      p_{t_0}(x) = \int_{\rset^d} \rmG_{t_0}(x, \tilde{x}) \rmd \pi_0(\tilde{x})  , \quad p_{\infty}(x) = (2 \uppi/  \alpha)^{-d/2} \exp[-\alpha \norm{x}^2/ 2] }  ,
    \end{equation}
    where for any $x, \tilde{x} \in \rset^d$
    \begin{align}
      \label{eq:def_G}
      &\rmG_{t_0}(x, \tilde{x}) = \parentheseLigne{2 \uppi \sigma_{t_0}^2}^{-d/2} \exp[-\norm{x - m_{t_0}(\tilde{x})}^2/(2 \sigma_{t_0}^2)] , \\ &\sigma_{t_0}^2 = (1 - \exp[-2 \alpha t_0]) / \alpha   , \qquad m_{t_0}(\tilde{x}) = \exp[-\alpha t_0] \tilde{x}  . 
    \end{align}
Combining this result and Jensen's inequality we get 
\begin{equation}
  \label{eq:intermediate}
  \textstyle{
  \int_{\rset^d} p_{t_0}^2(x) p_{\infty}^{-1}(x) \rmd x \leq \alpha^{-d/2}(2 \uppi)^{-d/2} \sigma_{t_0}^{-2d} \int_{\rset^d} \exp\parentheseDeuxLigne{-\norm{x - m_{t_0}(\tilde{x})}^2/\sigma_{t_0}^2  + \alpha\norm{x}^2/2} \rmd x \rmd \pi_0(\tilde{x})  . }
\end{equation}
For any $x, \tilde{x} \in \rset^d$ we have
\begin{align}
  \norm{x - m_{t_0}(\tilde{x})}^2/\sigma_{t_0}^2  - \alpha\norm{x}^2/2 &= \norm{x - m_{t_0}(\tilde{x})( 2\tilde{\sigma}_{t_0}^2/\sigma^2_{t_0})}^2/(2 \tilde{\sigma}_{t_0}^2) -\norm{\tilde{x}}^2 \phi(\alpha, t_0)/\sigma_{t_0}^2  , 
\end{align}
with
$\tilde{\sigma}_{t_0}^2 = (\sigma_{t_0}^2/2)(1 - \alpha \sigma_{t_0}^2/2)^{-1}$
and
$\phi(\alpha,t_0) = \alpha \sigma_{t_0}^2 (1 - \sigma_{t_0}^2 \alpha)/(2 - \sigma_{t_0}^2 \alpha)$.
Using this result, we get that
\begin{equation}
  \textstyle{
  \int_{\rset^d} \exp\parentheseDeuxLigne{-\normLigne{x - m_{t_0}(\tilde{x})}^2/\sigma_{t_0}^2
    + \alpha \normLigne{x}^2/2} \rmd x \rmd \pi_0(\tilde{x})  \leq (2 \uppi
  \tilde{\sigma}_{t_0}^2)^{d/2}\int_{\rset^d} \exp[\phi(\alpha,t_0)\normLigne{\tilde{x}}^2] \rmd \pi_0(\tilde{x})}  , 
\end{equation}
Let $\vareps = \mtt/4$ and $t_0 \geq 0$ such that
$\phi(\alpha, t_0) \leq \vareps$.  Using \Cref{prop:subgaussian}, we get that
\begin{equation}
  \textstyle{
  \int_{\rset^d} \exp\parentheseDeuxLigne{-\norm{x - m_{t_0}(\tilde{x})}^2/\sigma_{t_0}^2
    + \alpha \norm{x}^2/2} \rmd x \rmd \pi_0(\tilde{x})  \leq (2 \uppi
  \tilde{\sigma}_{t_0}^2)^{d/2}\int_{\rset^d} \exp[\vareps \norm{\tilde{x}}^2] \rmd \pi_0(\tilde{x})  .}
\end{equation}
Combining this result, the fact that $\sigma_{t_0}^2 \leq \alpha^{-1}$,
\eqref{eq:intermediate} and that for any $t \geq 0$,
$(1 - \rme^{-t})^{-1} \leq 1 + 1/t$, we obtain
\begin{align}
  \textstyle{\int_{\rset^d} p_{t_0}^2(x) p_{\infty}^{-1}(x) \rmd x }& \textstyle{\leq (\alpha^{-1} 
                                                                \tilde{\sigma}_{t_0}^2 \sigma_{t_0}^{-4})^{d/2} \int_{\rset^d} \exp[\vareps \norm{\tilde{x}}^2] \rmd \pi_0(\tilde{x}) }\\
                                                              &\textstyle{\leq  (1- \exp[- 2\alpha t_0])^{-d/2}  \int_{\rset^d} \exp[\vareps \norm{\tilde{x}}^2] \rmd \pi_0(\tilde{x}) }\\
  &\textstyle{\leq (1 + 1/(2\alpha t_0))^{d/2}  \int_{\rset^d} \exp[\vareps \norm{\tilde{x}}^2] \rmd \pi_0(\tilde{x})  . }
\end{align}
Combining this result and \eqref{eq:tv_ineq}, we get that for any $t > t_{0}$
\begin{equation}
  \tvnormLigne{\pi_0\Pbb_{t|0} - \pi_{\infty}} \leq C_1^a \exp[- \alpha t]  ,
\end{equation}
with
\begin{equation}
  \textstyle{
  C_1^a = \sqrt{2} (1 + 1/(2 \alpha t_{0}))^{d/2}\parentheseLigne{1 + \parentheseLigne{\int_{\rset^d} \exp[\vareps \normLigne{\tilde{x}}^2] \rmd \pi_0(\tilde{x})}^{1/2}} \exp[\alpha t_{0}]   .}
\end{equation}
For $t \leq t_{0}$, using that $\tvnormLigne{\pi_0\Pbb_{t|0} - \pi_{\infty}} \leq 1$  we have
\begin{equation}
  \tvnormLigne{\pi_0\Pbb_{t|0} - \pi_{\infty}} \leq C_1^b \exp[- \alpha t]  , \qquad \text{with } C_1^b = \exp[ \alpha t_{0}]  .
\end{equation}
Let $C_1 = C_1^a + C_1^b$ and we have that for any $t \in \ccintLigne{0,T}$
\begin{equation}
  \label{eq:tv_ineq_alpha_pos}
  \tvnormLigne{\pi_0\Pbb_{t|0} - \pi_{\infty}} \leq C_1 \exp[- \alpha t]  . 
\end{equation}

\item Second assume that $\alpha = 0$. 
\begin{equation}
  \label{eq:no_pinsker}
  \textstyle{
  \tvnormLigne{\pi_0\Pbb_{T|0} - \pi_{\infty}} \leq \int_{\rset^d} \int_{\rset^d} (4 \uppi T)^{-d/2} \absLigne{\exp[-\norm{x - \tilde{x}}^2/(4T)] - \exp[-\norm{x}^2/(4T)]} \rmd x \rmd \pi_0(\tilde{x})  .}
\end{equation}
For any $x, \tilde{x} \in \rset^d$, let
$\varphi \in \rmc^1(\ccintLigne{0,1}, \rset)$ with  for any $s \in \ccintLigne{0,1}$,
$\varphi(s) = \exp[-\norm{x - s\tilde{x}}^2/ (4T)]$. First, assume that
$T \geq 2/\vareps$. Using \Cref{lemma:ineq_varphi}, we get that for any
$s \in \ccintLigne{0,1}$
\begin{equation}
  \abs{\varphi'(s)} \leq (1+\vareps^{-1}) (1 + \norm{x}) \exp[-\norm{x}^2/(8T)]\exp[\vareps\norm{y}^2]/T  .
\end{equation}
Using this result we get that
\begin{align}
  &\tvnormLigne{\pi_0\Pbb_{T|0} - \pi_{\infty}} \leq \textstyle{\int_{\rset^d} \int_{\rset^d} (4 \uppi T)^{-d/2} \absLigne{\exp[-\norm{x - \tilde{x}}^2/(4T)] - \exp[-\norm{x}^2/(4T)]} \rmd x \rmd \pi_0(\tilde{x}) } \\
  &\quad \leq \textstyle{\int_{\rset^d} \int_{\rset^d} (4 \uppi T)^{-d/2} (1+\vareps^{-1}) (1 + \norm{x}) \exp[-\norm{x}^2/(8T)]\exp[\vareps\norm{\tilde{x}}^2]/T \rmd x \rmd \pi_0(\tilde{x}) } \\
  &\quad \leq 2^{d/2}(1+\vareps^{-1}) \textstyle{\int_{\rset^d} (8 \uppi T)^{-d/2}  (1 + \norm{x}) \exp[-\norm{x}^2/(8T)]\rmd x\int_{\rset^d} \exp[\vareps\norm{\tilde{x}}^2]/T  \rmd \pi_0(\tilde{x}) } \\
  &\quad \leq 2^{d/2}(1+\vareps^{-1}) (1 + 2\sqrt{2}d^{1/2}T^{1/2})\textstyle{\int_{\rset^d} \exp[\vareps\norm{\tilde{x}}^2]/T  \rmd \pi_0(\tilde{x}) }  .   
\end{align}
In addition, if $T \leq 2/\vareps$ then
\begin{equation}
  \tvnormLigne{\pi_0\Pbb_{T|0} - \pi_{\infty}} \leq (\vareps/2 + (\vareps/2)^{1/2})^{-1}(T^{-1} + T^{-1/2})  . 
  \end{equation}
Hence, we get that there exists $C_2 \geq 0$ such that
\begin{equation}
  \label{eq:tv_ineq_alpha_zero}
  \tvnormLigne{\pi_0\Pbb_{T|0} - \pi_{\infty}} \leq C_2 (T^{-1} + T^{-1/2})  , 
\end{equation}
with
\begin{equation}
  C_2 = (\vareps/2 + (\vareps/2)^{1/2})^{-1} + 2^{d/2}(1+\vareps^{-1}) (1 + 2\sqrt{2}d^{1/2}) \textstyle{\int_{\rset^d} \exp[\vareps\norm{\tilde{x}}^2]  \rmd \pi_0(\tilde{x}) }  . 
\end{equation}

\item Recall that $\Pbb^R$ is associated with the diffusion
  $(\bfY_t)_{t \geq 0}$ such that for any $t \in \ccintLigne{0,T}$ and $x \in \rset^d$
  \begin{equation}
    \rmd \bfY_t = b_1(t, \bfY_t) \rmd t + \sqrt{2} \bfB_t  , \quad b_1(t, x) = \alpha x + 2 \nabla \log p_{T-t}(x)  . 
  \end{equation}
  Similarly, for any $k \in \{1, \dots, N\}$ we have $\Qker_k = \Qbb_{t_k}$
  where $\Qbb$ is associated with the
  diffusion $(\bbfY_t)_{t \in \ccintLigne{0,T}}$ such that for any
  $(w_t)_{t \in \ccintLigne{0,T}} \in \rmc(\ccintLigne{0,T}, \rset^d)$ we have 
  \begin{align}
    &\rmd \bbfY_t = b_2(t, (\bbfY_s)_{s \in \ccintLigne{0,T}}) \rmd t + \sqrt{2} \bfB_t  , \\ &\textstyle{b_2(t, (w_t)_{t \in \ccintLigne{0,T}}) =   \sum_{k=0}^{N-1} \1_{\coint{t_k, t_{k+1}}}(t) \defEns{2 \alpha w_{t_k} + s_{\theta}(t_k, w_{t_k})} }
  \end{align}
  where  for any $k \in \{0, \dots, N\}$,
  $t_k = \sum_{\ell=0}^{k-1} \gamma_{\ell + 1}$.  Recall that for any
  $i \in \{1, 2, 3\}$ there exist $A_i \geq 0$ and $\alpha_i \in \nset$ such
  that for any $x_0 \in \rset^d$
  \begin{equation}
    \normLigne{\nabla^i \log p_0(x)} \leq A_i (1 + \normLigne{x_0}^{\alpha_i})  ,
  \end{equation}
  with $\alpha_1 = 1$. Using this result and \Cref{thm:uniform_bounds} we get
  that for any $i \in \{1, 2, 3\}$ there exist $B_i \geq 0$ and $\beta_i \in \nset$
  with $\beta_1 = 1$ such that for any $x_t \in \rset^d$ and $t \in \ccintLigne{0,T}$ 
  \begin{equation}
    \label{eq:unif_bound_log}
    \normLigne{\nabla^i \log p_t (x_t)} \leq B_i (1 + \norm{x_t}^{\beta_i} )  . 
  \end{equation}
  In addition, for any $t \in \ccintLigne{0,T}$ and $x \in \rset^d$ we have 
  \begin{equation}
    \partial_t p_t(x) = - \mathrm{div}(b p_t)(x) + \Delta p_t(x)  , 
  \end{equation}
  with $b(x) = -\alpha x$. Therefore, since $\log p \in \rmc^\infty(\ocint{0,T} \times \rset^d, \rset)$ we obtain that
  for any $t \in \ocint{0,T}$ and $x_t \in \rset^d$
  \begin{equation}
    \partial_t \log p_t (x_t) = - \mathrm{div}(b \log p_t )(x_t) + \Delta \log p_t (x_t) + \norm{\nabla \log p_t(x_t)}^2  .
  \end{equation}
  Finally, we get that for any $t \in \ocint{0,T}$ and $x_t \in \rset^d$
    \begin{equation}
    \partial_t \nabla \log p_t (x_t) = - \nabla \mathrm{div}(b \log p_t )(x_t) + \nabla \Delta \log p_t(x_t) + \nabla\norm{\nabla \log p_t }^2(x_t)  .
  \end{equation}
  Therefore combining this result and \eqref{eq:unif_bound_log} there exist
  $\tilde{A} \geq 0$ and $\upbeta \in \nset$ such that for any $x_t \in \rset^d$
  and $t \in \ocint{0,T}$,
  $\norm{\partial_t \nabla \log p_t (x_t)} \leq \tilde{A}(1 + \norm{x_t}^\upbeta)$.  Hence, for any
  $t_1, t_2 \in \ccintLigne{0,T}$ and $x \in \rset^d$
  \begin{equation}
    \label{eq:bound_der}
    \norm{\nabla \log p_{t_2} (x) - \nabla \log p_{t_1} (x)} \leq \tilde{A} \abs{t_2 - t_1} (1 + \norm{x}^{\upbeta})  . 
  \end{equation}
  In addition, using \eqref{eq:unif_bound_log}, we have for any $t \in \ccintLigne{0,T}$ and $x_1, x_2 \in \rset^d$
  \begin{align}
    \label{eq:bound_loc_lip}
    \normLigne{\nabla \log p_t (x_1) - \nabla \log p_t (x_2)} &\leq \textstyle{\int_0^1 \normLigne{\nabla^2 \log p_t ((1-s)x_1+sx_2)}\rmd s \normLigne{x_1 - x_2}} \\
                                                              &\leq \textstyle{B_2(1 + \int_0^1\norm{(1-s)x_1+sx_2}^{\beta_2} \rmd s)\normLigne{x_1 - x_2} }\\
    &\leq B_2(1 + \norm{x_1}^{\beta_2} + \norm{x_2}^{\beta_2})\normLigne{x_1 - x_2}  . 
  \end{align}
  Since $s_{\theta} \in \rmc(\ccintLigne{0,T} \times \rset^d, \rset^d)$ and
  $\nabla \log p \in \rmc(\ccintLigne{0,T} \times \rset^d, \rset^d)$ we have
  using \Cref{lemma:girsanov}, \eqref{eq:bound_der}, \eqref{eq:bound_loc_lip}
  and the Cauchy-Schwarz inequality
  \begin{align}
    \label{eq:bigbound}
    &\tvnormLigne{\pi_{\infty} (\Pbb^R)_{T|0} - \pi_{\infty} \Qker_N}^2 \leq (1/2)\textstyle{\int_0^T \expeLigne{\normLigne{b_1(t, \bfY_t) - b_2(t, (\bfY_t)_{t \in \ccintLigne{0,T}})}^2} \rmd t}   \\
    &\qquad \leq 2 \textstyle{\sum_{k=0}^{N-1} \int_{t_k}^{t_{k+1}} \expeLigne{\normLigne{\nabla \log p_{T-t} (\bfY_t) - s_{\theta}(\bfY_{t_k})}^2} \rmd t} 
    \\ & \qquad \qquad +  \textstyle{\sum_{k=0}^{N-1}  \int_{t_k}^{t_{k+1}} \alpha^2 \expeLigne{\normLigne{\bfY_t - \bfY_{t_k}}^2}}\rmd t \\
    &\qquad \leq 6 \textstyle{\sum_{k=0}^{N-1} \int_{t_k}^{t_{k+1}} \expeLigne{\normLigne{\nabla \log p_{T-t} (\bfY_t) - \nabla \log p_{T-t} (\bfY_{t_k})}^2}} \rmd t \\
    &\qquad \qquad + 6 \textstyle{\sum_{k=0}^{N-1} \int_{t_k}^{t_{k+1}} \expeLigne{\normLigne{\nabla \log p_{T-t} (\bfY_{t_k}) - \nabla \log p_{T-{t_k}} (\bfY_{t_k})}^2}} \rmd t \\    
    & \qquad \qquad + 6 \textstyle{\sum_{k=0}^{N-1} \int_{t_k}^{t_{k+1}} \expeLigne{\normLigne{\nabla \log p_{T-t_k} (\bfY_{t_k}) - s_{\theta}(t_k, \bfY_{t_k})}^2} \rmd t }\\
    & \qquad \qquad +  \textstyle{\sum_{k=0}^{N-1} \int_{t_k}^{t_{k+1}} \alpha^2 \expeLigne{\normLigne{\bfY_t - \bfY_{t_k}}^2} \rmd t }\\
    &\qquad \leq 18 \sqrt{2} \textstyle{B_2^2(1 + 2 N_T(4 \beta_2))^{1/2} \sum_{k=0}^{N-1} \int_{t_k}^{t_{k+1}} \expeLigne{\normLigne{\bfY_t - \bfY_{t_k}}^4}^{1/2} \rmd t   } \\
    & \qquad \qquad +  12\tilde{A}^2(1 + N_T(2 \upbeta)) \textstyle{\sum_{k=0}^{N-1} \int_{t_k}^{t_{k+1}} (t - t_k)^2  \rmd t  } + 6 T \Mtt^2 \\
    & \qquad \qquad +  \textstyle{\sum_{k=0}^{N-1} \int_{t_k}^{t_{k+1}} \alpha^2 \expeLigne{\normLigne{\bfY_t - \bfY_{t_k}}^2} \rmd t }\\
    & \qquad \leq \{18 \sqrt{2} \textstyle{B_2^2(1 + 2 N_T(4 \beta_2))^{1/2} + \alpha^2\}  \sum_{k=0}^{N-1} \int_{t_k}^{t_{k+1}} \expeLigne{\normLigne{\bfY_t - \bfY_{t_k}}^4}^{1/2} \rmd t} \\
    & \qquad \qquad +  4\tilde{A}^2(1 + N_T(2 \upbeta)) \textstyle{\sum_{k=0}^{N-1}  (t_{k+1} - t_k)^3    } + 6 T \Mtt^2 \\
    & \qquad \leq \{18 \sqrt{2} \textstyle{B_2^2(1 + 2 N_T(4 \beta_2))^{1/2} + \alpha^2\}  \sum_{k=0}^{N-1} \int_{t_k}^{t_{k+1}} \expeLigne{\normLigne{\bfY_t - \bfY_{t_k}}^4}^{1/2} \rmd t} \\
    & \qquad \qquad +  4\tilde{A}^2(1 + N_T(2 \upbeta)) T \bgamma^2  + 6 T \Mtt^2  ,
  \end{align}
  where for any $\ell \in \nset$,
  $N_T(\ell) = \sup_{t \in \ccintLigne{0,T}}
  \expeLigne{\normLigne{\bfY_t}^\ell}$.  For any $t \in \ccintLigne{0,T}$, let
  $\generator_t: \ \rmc^2(\rset^d) \to \rmc^2(\rset^d, \rset)$ the generator
  given for any $t \geq 0$, $\varphi \in \rmc^2(\rset^d, \rset)$ and
  $x \in \rset^d$ by
  \begin{equation}
    \generator_t(\varphi)(x) = \langle \alpha x + 2\nabla \log p_{T-t} (x), \nabla \varphi(x) \rangle + \Delta \varphi(x)  . 
  \end{equation}
  For any $\ell \in \nset$, let $V_\ell(x) = \norm{x}^{2\ell}$. Hence, for any $\ell \in \nset$, 
  $x \in \rset^d$ and $t \in \ccintLigne{0,T}$ we have using \eqref{eq:unif_bound_log}
  \begin{equation}
    \generator_t (V_\ell)(x) = 2\ell \alpha \norm{x}^{2\ell} + 2\ell B_1 \norm{x}^{2\ell-1} + 2\ell B_1 \norm{x}^{2\ell} + 2\ell (2 \ell-1) \norm{x}^{2(\ell-1)}  . 
  \end{equation}
  Hence, for any $\ell \in \nset$ there exist $\tilde{B}_\ell$ such that
  $x \in \rset^d$ and $t \in \ccintLigne{0,T}$
  \begin{equation}
        \label{eq:lyap_last}
    \absLigne{\generator_t(V_\ell)(x)} \leq \tilde{B}_\ell (1 + V_\ell(x))  .
  \end{equation}
  For any $\ell \in \nset$,
  $(M_{\ell, t})_{t \in \ccintLigne{0,T}} = (\V_\ell(\bfY_t) - V_\ell(\bfY_0) -
  \int_{0}^t \generator_t(V_\ell)(\bfY_s) \rmd s)_{t \in \ccintLigne{0,T}}$ is a
  local martingale. For any $\ell \in \nset$, there exists $(\tau_{\ell, k})_{k \in \nset}$ a sequence of stopping
  times such that $\lim_{k \to +\infty} \tau_{\ell, k} = T$ and
  $(M_{\ell, t \wedge \tau_{\ell, k}})_{t \in \ccintLigne{0,T}}$ is a martingale.
  Using \eqref{eq:lyap_last}, we have for any $t \in \ccintLigne{0,T}$, $\ell \in \nset$ and $k \in \nset$  
  \begin{equation}    
    \expeLigne{\V_\ell(\bfY_{t \wedge \tau_{\ell, k}})} \leq  \expeLigne{\V_\ell(\bfY_0)} + \tilde{B}_\ell \textstyle{\int_0^t (1 + \expeLigne{\V_\ell(\bfY_{s\wedge \tau_{\ell, k}})}) \rmd s  } .
  \end{equation}
  Hence, using Gr\"{o}nwall's lemma we get that for any $\ell \in \nset$,
  $\sup_{k \in \nset} \expeLigne{V_\ell(\bfY_{t \wedge \tau_{\ell, k}})} <
  +\infty$. Therefore for any $\ell \in \nset$,
  $((M_{\ell, t \wedge \tau_k})_{t \in \ccintLigne{0,T}})_{k \in \nset}$ is uniformly
  integrable and we have that for any $\ell \in \nset$,
  $(M_{\ell, t})_{t \in \ccintLigne{0,T}}$ is a martingale. Therefore we get that for
  any $t \in \ccintLigne{0,T}$, $\ell \in \nset$
  \begin{equation}
    \expeLigne{\V_\ell(\bfY_{t})} \leq  \expeLigne{\V_\ell(\bfY_0)} + \tilde{B}_\ell \textstyle{\int_0^t (1 + \expeLigne{\V_\ell(\bfY_{s})} \rmd s ) }  . 
  \end{equation}
  Using Gr\"{o}nwall's lemma we get that for any $\ell \in \nset$ there exist
  $\tilde{C}_\ell \geq 0$ such that
  \begin{equation}
    \label{eq:NT_bound}
    N_T(\ell) = \sup_{t \in \ccintLigne{0,T}} \expeLigne{\normLigne{\bfY_t}^{2\ell}} \leq \tilde{C}_\ell \exp[\tilde{B}_\ell T]  .
  \end{equation}
  We have that for any $s, t \in \ccintLigne{0,T}$
  \begin{equation}
    \textstyle{\bfY_t = \bfY_s + \int_s^t \{\alpha \bfY_u + 2 \nabla \log p_{T-t} (\bfY_u) \}\rmd u + \sqrt{2} \int_s^t \rmd \bfB_u }   . 
  \end{equation}
  Using \eqref{eq:bound_der} and Cauchy-Schwarz inequality we have for any $s,t \in \ccintLigne{0,T}$
  \begin{align}
    \expeLigne{\normLigne{\bfY_t - \bfY_s}^4}  &\leq 64 (t-s)^3 \textstyle{\int_s^t \{\alpha^4 \expeLigne{\normLigne{\bfY_u}^4} + 16 \expeLigne{\normLigne{\nabla \log p_{T-t} (\bfY_u)}^4} \} \rmd u + 48 \sqrt{2} (t-s)^2} \\
                                               &\leq 64 (t-s)^3 \textstyle{\int_s^t \{\alpha^4 \expeLigne{\normLigne{\bfY_u}^4} + 128 B_1^4 (1 + \expeLigne{\normLigne{\bfY_u}^4}) \} \rmd u + 48 \sqrt{2} (t-s)^2} \\
    &\leq  64(\alpha^4 + 128B_1^4)(1 + N_T(4))(t-s)^4 + 48 \sqrt{2} (t-s)^2  .     \label{eq:diff_bound}
  \end{align}
  Combining \eqref{eq:NT_bound} and \eqref{eq:diff_bound} in \eqref{eq:bigbound} we get that there exist $C_3 \geq 0$ such that 
  \begin{equation}
    \label{eq:error_disc}
    \tvnormLigne{\pi_{\infty} (\Pbb^R)_{T|0} - \pi_{\infty} \Qker_N}^2 \leq C_3 \exp[C_3 T] (\bgamma + \Mtt^2)  ,
  \end{equation}

\end{enumerate}
We conclude the proof upon combining \eqref{eq:tv_ineq_alpha_pos} and
\eqref{eq:error_disc} if $\alpha > 0$ and \eqref{eq:tv_ineq_alpha_zero} and
\eqref{eq:error_disc} if $\alpha = 0$.
\end{proof}

\subsection{General SGM and links with existing works}
\label{sec:comparison-with-ho}

In this section we describe a general algorithm for SGM in
\Cref{sec:general-algorithm} and show that the formulation \eqref{eq:forward}
encompasses the ones of \citep{song2020score,ho2020denoising} in
\Cref{sec:links-with-existing}.

\subsubsection{General SGM algorithm}
\label{sec:general-algorithm}

We first present a general algorithm to compute approximate reverse dynamics,
\ie \ to compute the reverse-time Markov chain associated with the forward process
\begin{equation}
  \label{eq:forward_app}
  \rmd \bfX_t = f_t(\bfX_t) \rmd t + \sqrt{2} \rmd \bfB_t  , \qquad \bfX_0 \sim \pdata  . 
\end{equation}
We use the Euler-Maruyama
discretization of \eqref{eq:forward_app}, \ie \ let $X_0 \sim \pdata$ and for
any $k \in \{0, \dots, N-1\}$
\begin{equation}
  X_{k+1} = X_k + \gamma_{k+1} f_k(X_k) + \sqrt{2 \gamma_{k+1}} Z_{k+1}  . 
\end{equation}
In general, we do not have that $p(x_k|x_0)$ is a Gaussian density contrary to
\cite{song2019generative,ho2020denoising}. However, in this case, we obtain that
for any $x \in \rset^d$,
\begin{equation}
\textstyle{p_{k+1}(x) = (4 \uppi \gamma_{k+1})^{-d/2}\int_{\rset^d} p_k(\tilde{x}) \exp[-\norm{\Tnplusun(\tilde{x})
  - x}^2/(4 \gamma_{k+1})] \rmd \tilde{x}  ,   }
\end{equation}
with $\Tnplusun(x) = \tilde{x} + \gamma_{k+1} f_k(\tilde{x})$. 
Therefore, we get that for any $x \in \rset^d$
\begin{equation}
  \textstyle{  (2 \gamma_{k+1} p_{k+1}(x))\nabla \log p_{k+1} (x)= \int_{\rset^d} (\Tnplusun(\tilde{x})- x) p_k(\tilde{x}) \exp[-\norm{\Tnplusun(\tilde{x})
  - x}^2/(4 \gamma_{k+1})] \rmd \tilde{x} }  . 
\end{equation}
Hence, we get that for any $x \in \rset^d$
\begin{equation}
  \label{eq:regression_em}
  \nabla \log p_{k+1} (x) = \expeLigne{\Tnplusun(X_{k})- X_{k+1}|X_{k+1}=x} / (2 \gamma_{k+1}) = -(2\gamma_{k+1})^{1/2} \expeLigne{Z_{k+1}|X_{k+1}=x}  .
\end{equation}
From this formula we derive a regression problem similar to the one of
\Cref{sec:discr-sett-mark}. We obtain 
\Cref{algo:generalized_score_matching}. We highlight a few differences between
our approach and the ones of \cite{song2019generative,ho2020denoising}:
\begin{enumerate}[label=(\alph*), wide, labelwidth=!, labelindent=0pt]
\item As emphasized in \eqref{eq:regression_em}, the regression problem in
  \Cref{algo:generalized_score_matching} is different from the one usually
  considered in SGM which restrict themselves to the setting $f_k(x) = \alpha x$
  with $\alpha = 0$ \citep{song2019generative} or $\alpha > 0$
  \citep{ho2020denoising}.
\item In the present algorithm we do not use any corrector step
  \citep{song2020score} at sampling time. Note that the use of a corrector step
  is only justified in the context of classical SGM algorithms and not the DSB
  method introduced in \Cref{sec:iter-score-match}. This is because, we do not
  have access to the marginal of the time-reverse density during the IPF
  iterations contrary to classical SGMs.
\item Finally, we do not present the Exponential Moving Average (EMA) procedure
  \cite{song2020improved} which is key to prevent the network from
  oscillating. Contrary to the corrector step, this technique can easily be
  incorporated in \Cref{algo:generalized_score_matching}. 
\end{enumerate}
Further comments and additional techniques are presented in \Cref{sec:implementation_details}.

\begin{algorithm}[h]
    \caption{Generalized score-matching}
    \label{algo:generalized_score_matching}
    \begin{algorithmic}[1]
      \STATE \textbf{Inputs: } $(b_k)_{k \in \{0,\dots, N-1\}}$ , $N \in \nset$
      (nb. of iterations), $M \in \nset$ (batch size), $N_{\textup{epochs}}$
      (nb. of epochs), $(\gamma_k)_{k \in \{0, \dots, N-1\}}$ (stepsizes),
      $\ensembleLigne{s_{\theta}}{\theta \in \Theta}$ (neural network),
      $\mathrm{opt}$ (optimizer), $\pprior$ (prior distribution), $\lambda(k)$ (weights)
      \FOR{$n_{\textup{epoch}} = 0, \dots, N_{\textup{epoch}} -1$}
      \FOR{$j\in \{1, \dots, M\}$}
      \STATE $X_0^j \sim \pdata$
      \FOR{$k \in \{0, \dots, N-1\}$}
      \STATE $X_{k+1}^{j} = X_k^j + \gamma_{k+1} f_k(X_k^j) + \sqrt{2 \gamma_{k+1}} Z_{k+1}^j$
      \ENDFOR
      \ENDFOR
      \STATE $\widehat{\ell}(\theta) = M^{-1}\sum_{j=1}^M\sum_{k=0}^{N-1} \lambda(k)/(2 \gamma_{k+1}) \sum_{j=1}^M \normLigne{\sqrt{2 \gamma_{k+1}} s_{\theta}(k+1,X_{k+1}^j) + Z_{k+1}^j}^2$
      \STATE $\theta_{n_{\textup{epoch}} +1} = \mathrm{opt}(\ell, \theta_{n_{\textup{epoch}}})$
      \ENDFOR
      \STATE $X_N \sim \pprior$      
      \FOR{$k \in \{N-1, \dots, 0\}$} \STATE
      $X_{k} = X_{k+1} + \gamma_{k+1} \defEnsLigne{-f_k(X_{k+1}) + 2 s_{
          \theta_{N_{\textup{epoch}}}}(k+1, X_{k+1})} + \sqrt{2 \gamma_{k+1}} Z_{k+1}$ \ENDFOR
      \STATE \textbf{Output: } $X_0$
    \end{algorithmic}
  \end{algorithm}

  \subsubsection{Links with existing work}
\label{sec:links-with-existing}

In this section, we show that we can recover the training and sampling algorithm
of \cite{song2019generative} and \cite{ho2020denoising} by reversing homogeneous
diffusions. Note that \cite{song2020score} identified links with non-homogeneous
SDEs. We explicitly characterize the fundamental difference between
the approaches of \cite{song2019generative,ho2020denoising} by identifying the
two corresponding forward homogeneous processes (Brownian motion or
Ornstein-Ulhenbeck).

\paragraph{Brownian motion}
First, we show that we can recover the sampling procedure and the loss function of
\cite{song2019generative} by reversing a Brownian motion. Assume that we have
\begin{equation}
  \label{eq:brownian}
  \rmd \bfX_t = \sqrt{2} \rmd \bfB_t  , \qquad \bfX_0 \sim \pdata  .
\end{equation}
In what follows we define
$\{Y_k\}_{k=0}^{N-1}$ such that $\{Y_k\}_{k=0}^{N-1}$ approximates
$\{\bfX_{T-t_k}\}_{k=0}^{N-1}$ for a specific sequence of times
$\{t_k\}_{k=0}^{N-1} \in \ccint{0,T}^N$. We recall that the time-reversal of
\eqref{eq:brownian} is associated with the following SDE
\begin{equation}
  \label{eq:brownian_reverse}
  \rmd \bfY_t = 2 \nabla \log p_{T-t}(\bfY_t) + \sqrt{2} \rmd \bfB_t  . 
\end{equation}
The Euler-Maruyama discretization of \eqref{eq:brownian_reverse} yields for any
$k \in \{0, \dots, N-1\}$
\begin{align}
  \label{eq:sampling_perfect_song}
  \tilde{Y}_{k+1} = \tilde{Y}_k + 2 \gamma_{k+1} \nabla \log p_{T - t_k}(\tilde{Y}_k)  + \sqrt{2 \gamma_{k+1}} Z_{k+1}  .
\end{align}
where $\{\gamma_{k+1}\}_{k=0}^{N-1}$ is a sequence of stepsizes and for any
$k \in \{0, \dots, N\}$, $t_k = \sum_{j=0}^{k-1} \gamma_{j+1}$.  A close form
for $\{\nabla \log p_{T-t_k}\}_{k=0}^{N-1}$ is not available and in practice we
consider
\begin{equation}
  \label{eq:sampling_song}
  Y_{k+1} = Y_k + 2 \gamma_{k+1} s_{\theta^\star}(T-t_k, Y_k) + \sqrt{2 \gamma_{k+1}} Z_{k+1}  ,
\end{equation}
where for any $k \in \{0, \dots, N-1\}$, $s_{\theta^\star}(T-t_k, \cdot)$ is an
approximation of $\nabla \log p_{T-t_k} $. The sampling procedure
\eqref{eq:sampling_song} is similar to the one of \cite{song2019generative} upon
setting (with the notations of \cite{song2019generative}) $T \leftarrow 1$ in
\cite[Algorithm 1]{song2019generative} (no corrector step),
$\alpha_k/2 \leftarrow \gamma_k $ and
$\bfs_{\pmb{\theta}}(\cdot, \sigma_{k+1}) \leftarrow 2 s_{\theta^\star}(T-t_k,
\cdot)$. It remains to show that $2 s_{\theta^\star}$ is the solution to the
same regression problem as $\bfs_{\pmb{\theta}}$ in \cite[Equation
6]{song2019generative}. First, note that for any $t > 0$ and $x_t \in \rset^d$
we have
\begin{equation}
  \textstyle{p_t(x_t) = (4 \uppi t)^{-d/2} \int_{\rset^d} \pdata(x_0) \exp[-\norm{x_t - x_0}^2/(4t)] \rmd x_0  .}
\end{equation}
Therefore, we get that for any $t > 0$ and $x_t \in \rset^d$
\begin{equation}
  \textstyle{\nabla \log p_t(x_t)} \textstyle{= \int_{\rset^d} (x_0 - x_t)/(2t) p_{0|t}(x_0|x_t) \rmd x_0}
  = \CPELigne{\bfX_0 - \bfX_t}{\bfX_t=x_t}/(2t)  . 
\end{equation}
Hence, we have that $\theta^\star$ satisfies the following regression problem
\begin{equation}
  \textstyle{\theta^\star = \argmin_{\theta} \sum_{k=0}^{N-1} \lambda(k) \expeLigne{\normLigne{(\bfX_0 - \bfX_{T - t_k})/(T-t_k) - 2 s_{\theta}(T-t_k, \bfX_{T-t_k})}}}  . 
\end{equation}
Note that this loss function is similar to the one of \cite[Equation 6]{song2019generative}
upon letting $\sigma_{k+1}^2 \leftarrow 2 (T - t_k)$ and $L \leftarrow N$. Hence,
the two recursions approximately define the same scheme if for any
$k \in \{0, \dots, N-1\}$,
$\sigma_1^2 - \sigma_{k+1}^2 \approx (1/2)\sum_{j=0}^{k-1} \alpha_{j+1}$ since $t_0 = 0$
implies $T=(1/2)\sigma_1^2$. In \cite{song2019generative} we have for any
$k \in \{0, \dots, N-1\}$, $ \sigma_k^2 = \kappa^{N-k} \sigma_N^2$ (recall that
$N=L$) with $\kappa > 1$. In addition, we have for any $k \in \{0, \dots, N-1\}$,
$\alpha_k = \vareps \sigma_k^2 / \sigma_N^2$ for some $\vareps > 0$. We get that
\begin{align}
  (1/2)\textstyle{\sum_{j=0}^{k-1} \alpha_{j+1}} &= \textstyle{(\vareps/2)\kappa^{N-1} \sum_{j=0}^{k-1} \kappa^{-j}}\\
                                             &= (\vareps/2 ) (\kappa^{N-1} - \kappa^{N-k-1})/(1 - \kappa^{-1}) \\
  &= \vareps/(2(1 - \kappa^{-1}) \sigma_N^2) (\sigma_1^2 - \sigma_{k+1}^2)  .
\end{align}
Hence, the two schemes are identical if
$\vareps = 2(1 - \kappa^{-1}) \sigma_N^2$. In practice in
\cite{song2019generative} the authors choose $N=10$, $\sigma_N = 10^{-2}$,
$\sigma_1 = 1$ (hence $\kappa = 10^{4/9}$) and $\vareps = 2 \times 10^{-5}$.  We
have $2(1 - \kappa^{-1}) \sigma_N^2 \approx 1.3 \times 10^{-4}$ which has one
order of difference with $\vareps$.

\paragraph{Ornstein-Ulhenbeck}
Second, we show that we can recover the sampling procedure and the loss function
of \cite{ho2020denoising} by reversing an Ornstein-Ulhenbeck process.  Contrary
to the previous analysis we do not show a strict equivalence between the two
recursions but instead that our algorithm can be seen as a first order
approximation of the one of \cite{ho2020denoising}.

In this section, we consider the following diffusion
\begin{equation}
  \label{eq:ornstein}
  \rmd \bfX_t = -\alpha \bfX_t \rmd t + \sqrt{2} \rmd \bfB_t  , \qquad \bfX_0 \sim \pdata  . 
\end{equation}
In what follows we define
$\{Y_k\}_{k=0}^{N-1}$ such that $\{Y_k\}_{k=0}^{N-1}$ approximates
$\{\bfX_{T-t_k}\}_{k=0}^{N-1}$ for a specific sequence of times
$\{t_k\}_{k=0}^{N-1} \in \ccint{0,T}^N$. We recall that the time-reversal of
\eqref{eq:ornstein} is associated with the following SDE
\begin{equation}
  \label{eq:ornstein_reverse}
  \rmd \bfY_t = \{ \alpha \bfY_t + 2 \nabla \log p_{T-t}(\bfY_t)\} \rmd t  + \sqrt{2} \rmd \bfB_t  . 
\end{equation}
In what follows, we fix $\alpha = 1$.  The Euler-Maruyama discretization of
\eqref{eq:ornstein_reverse} yields for any $k \in \{0, \dots, N-1\}$
\begin{align}
  \tilde{Y}_{k+1} = (1 + \gamma_{k+1}) \tilde{Y}_k + 2 \gamma_{k+1} \nabla \log p_{T - t_k}(\tilde{Y}_k)  + \sqrt{2 \gamma_{k+1}} Z_{k+1}  .
\end{align}
where $\{\gamma_{k+1}\}_{k=0}^{N-1}$ is a sequence of stepsizes and for any
$k \in \{0, \dots, N-1\}$, $t_k = \sum_{j=0}^{k-1} \gamma_{j+1}$.  A close
form for $\{\nabla \log p_{T-t_k}\}_{k=0}^{N-1}$ is not available and in practice
we consider
\begin{equation}
  \label{eq:backward_em_ornstein}
  Y_{k+1} = (1 + \gamma_{k+1}) Y_k + 2 \gamma_{k+1} s_{\theta^\star}(T-t_k, Y_k) \rmd t + \sqrt{2 \gamma_{k+1}} Z_{k+1}  .
\end{equation}
In \cite[Equation 11]{ho2020denoising}  the backward
recursion is given for any $k \in \{0, \dots, N-1\}$
\begin{equation}
  \label{eq:backward_ho}
  Y_{k+1} = \alpha_{N-k}^{-1/2} (Y_k - \beta_{N-k}/(1 - \bar{\alpha}_{N-k})^{1/2} \pmb{\epsilon}_{\theta}(Y_k, T-t_k) ) + \sigma_{N-k} Z_{k+1}  . 
\end{equation}
In \eqref{eq:backward_ho} we set $\sigma_k^2 = \beta_k$ as suggested in
\cite{ho2020denoising} where for any $k \in \{0, \dots, N-1\}$
\begin{align}
  \sigma_{k+1}^2 = \beta_{k+1}  , 
  \quad 
    \alpha_{k+1} = 1 - \beta_{k+1}  , \quad \textstyle{\bar{\alpha}_{k+1} = \prod_{i=1}^{k+1} \alpha_i  .}
\end{align}
We consider a first-order expansion of \eqref{eq:backward_ho} with respect to
$\{\beta_{k+1}\}_{k=0}^{N-1}$. We obtain the following recursion for any
$k \in \{0, \dots, N-1\}$
\begin{equation}
  Y_{k+1} = (1 + \beta_{N-k}/2) Y_k - \beta_{N-k}/(1 - \bar{\alpha}_{N-k})^{1/2} \pmb{\epsilon}_{\theta}(Y_k, T-t_k) + \sqrt{\beta_{N-k}} Z_{k+1}  . 
\end{equation}
This last recursion is equivalent to \eqref{eq:backward_em_ornstein} upon
setting $\beta_{N-k} \leftarrow 2 \gamma_{k+1}$ and
$-\pmb{\epsilon}_{\theta}(\cdot, T-t_k)/(1 - \bar{\alpha}_{N-k})^{1/2} \leftarrow -s_{\theta^\star}(T-t_k, \cdot)$.
It remains to show that $s_{\theta^\star}$ is the solution to the same regression
problem as $\pmb{\epsilon}_{\theta}/(1 - \bar{\alpha}_{N-\cdot})$ in \cite[Equation 12]{ho2020denoising}. First, note that for any
$t > 0$ and $x_t \in \rset^d$ we have
\begin{equation}
  \textstyle{p_t(x_t) = (2 \uppi \bar{\sigma}_t^2)^{-d/2} \int_{\rset^d} \pdata(x_0) \exp[-\norm{x_t - c_t x_0}^2/(2\bar{\sigma}_t^2)] \rmd x_0 }  , 
\end{equation}
with
\begin{equation}
  c_t^2 = \exp[-2 t]  , \qquad \bar{\sigma}_t^2 = 1 - \exp[-2t]  . 
\end{equation}
Therefore we get that for any $t \in \ccint{0,T}$ and $x_t \in \rset^d$
\begin{align}
  \nabla \log p_t (x_t) &= \textstyle{\int_{\rset^d} (c_t x_0 - x_t) \pdata(x_0) \exp[-\norm{x_t - c_t x_0}^2/(2\bar{\sigma}_t^2)] \rmd x_0} \\
  &=\CPE{c_t \bfX_0 - \bfX_t }{\bfX_t = x_t}/\bar{\sigma}_t^2 = - \CPE{\bfZ}{\bfX_t = x_t}/\bar{\sigma}_t  ,
\end{align}
where we recall that $\bfX_t$ has the same distribution as
$c_t \bfX_0 + \bar{\sigma}_t \bfZ$, with $\bfZ$ a $d$-dimensional Gaussian random
variable with zero mean and identity covariance matrix. Hence, we have that
$\theta^\star$ satisfies the following regression problem
\begin{equation}
  \textstyle{\theta^\star = \argmin_{\theta} \sum_{k=0}^{N-1} \lambda(k) \expeLigne{\normLigne{\bfZ/\sigma_{T-t_k} + s_{\theta}(T-t_k, \bfX_{T-t_k})}}}  .
\end{equation}
Note that we have
\begin{equation}
  \textstyle{
  \sum_{i=1}^{N-k} \beta_i = \sum_{i=k}^{N-1} \beta_{N-i} = 2 \sum_{i=k}^{N-1} \gamma_{i+1} = 2(T - t_k)  . }
\end{equation}
Using this result we have for any $k \in \{0, \dots, N-1\}$ 
\begin{equation}
  \textstyle{
  1 - \bar{\alpha}_{N-k} = 1 - \exp[-\sum_{i=1}^{N-k} \log(1 - \beta_i)] \approx 1 - \exp[-\sum_{i=1}^{N-k} \beta_i] \approx  \bar{\sigma}_{T-t_k}^2  . }
\end{equation}
Let $\tilde{\theta}^\star$ the minimizer of \cite[Equation
12]{ho2020denoising} we have 
\begin{align}
  \tilde{\theta}^\star &\approx \textstyle{\argmin_{\theta}  \sum_{k=0}^{N-1} (2\alpha_{N-k}(1 - \alpha_{N-k}))^{-1}\expeLigne{\normLigne{\bfZ - \pmb{\epsilon}_{\theta}(\bfX_{T-t_k}, T-t_k)}^2}} \\
               &\approx \textstyle{\argmin_{\theta}  \sum_{k=0}^{N-1} (2\alpha_{N-k})^{-1}\expeLigne{\normLigne{\bfZ/(1 - \alpha_{N-k} )^{1/2}- \pmb{\epsilon}_{\theta}(\bfX_{T - t_k}, T-t_k)/(1 - \alpha_{N-k})^{1/2}}^2}} \\
   &\approx \textstyle{\argmin_{\theta}  \sum_{k=0}^{N-1} (2\alpha_{N-k})^{-1}\expeLigne{\normLigne{\bfZ/\bar{\sigma}_{T-t_k}+ s_{\theta}(T-t_k, \bfX_{T-t_k})}^2}}  . 
\end{align}
Hence the two regression problems are approximately the same (for small values
of $\{\beta_{k+1}\}_{k=0}^{N-1}$) if we set $\lambda(k) = (2\alpha_{N-k})^{-1}$.

\section{\schro bridges with potentials and DSB recursion}
\label{sec:schro-bridges-with}

In this section, we start by proving an additive formula for the
Kullback--Leibler divergence in \Cref{sec:addit-form-kullb} following
\cite{leonard2014some}.  We recall the classical IPF formulation using
potentials in \Cref{prop:IPFpotential:proof}. Then, \Cref{prop:IPFrecursion} is
proved in \Cref{prop:IPFrecursion:proof}. Finally, we highlight a link between
our formulation and autoencoders in \Cref{sec:autoencoder}.

\subsection{Additive formula for the Kullback--Leibler divergence}
\label{sec:addit-form-kullb}

In this section, we prove a formula for the Kullback--Leibler divergence
following the proof of \cite{leonard2014some} which extends the result to
unbounded measures defined on the space of right-continuous left-limited
functions from $\ccint{0,T}$. We recall that a Polish space is a complete metric
separable space.

We start with the following disintegration theorem for probability measures.

\begin{theorem}
  \label{thm:dis}
  Let $(\msx, \mcx)$ and $(\msy, \mcy)$ be two Polish spaces. Let
  $\pi \in \Pens(\msx)$ and $\varphi:\ \msx \to \msy$ measurable. Then there
  exists a Markov kernel $\Kker^\pi_{\varphi}: \ \msy \times \mcx \to \ccint{0,1}$
  such that the following hold:
  \begin{enumerate}[wide, labelwidth=!, labelindent=0pt, label=(\alph*)]
  \item \label{item:a_dis} For any $y \in \msy$, $\Kker^\pi_\varphi(y, \varphi^{-1}(\{y\})) = 1$.
  \item \label{item:b_dis} For any $f: \ \msx \to \coint{0,+\infty}$ measurable we have
    $\textstyle{\int_{\msx} f(x) \rmd \pi(x) = \int_{\msy} \Kker^\pi_{\varphi}(y, f)
      \rmd \pi_{\varphi}(y)}$,
  \end{enumerate}
  where $\pi_{\varphi} = \varphi_\# \pi$.
\end{theorem}
\begin{proof}
  See \cite[III-70]{dellacherie1988} for instance.
\end{proof}

$\Kker^\pi_\varphi$ is called the disintegration of $\pi$ w.r.t. $\varphi$ and
is unique, see \cite[III-70]{dellacherie1988}.  In particular, for any
$\msx$-valued random variable $X$ with distribution $\pi$ we have
$\CPELigne{f(X)}{\varphi(X)} = \Kker^\pi_{\varphi}(\varphi(X), f)$. Next we
prove the following proposition, see \cite[Proposition A.13]{leonard2014some}
for an extension to unbounded measures. In what follows, for any
$\varphi: \ \msx to \rset$ measurable we denote
$\pi_{\varphi} = \varphi_\# \pi$.

\begin{proposition}
  \label{prop:decompo}
  Let $(\msx, \mcx)$ and $(\msy, \mcy)$ be two Polish spaces. Let
  $\pi, \mu \in \Pens(\msx)$ and $\varphi:\ \msx \to \msy$ measurable. Assume
  that $\pi \ll \mu$. Then the following holds:
  \begin{enumerate}[wide, labelwidth=!, labelindent=0pt, label=(\alph*)]
  \item $\pi_{\varphi} \ll \mu_{\varphi}$
  \item There exists $\msa \in \mcy$ with $\pi_{\varphi}(\msa) = 1$ such that for any $y \in \msa$, $\Kker_{\varphi}^\pi(y, \cdot) \ll \Kker_{\varphi}^\mu(y, \cdot)$.
  \end{enumerate}
  In addition, we have for any $y \in \msy$, $y' \in \msa$ and $x \in \msx$
  \begin{equation}
    (\rmd \pi_\varphi / \rmd \mu_\varphi)(y) = \Kker^\mu_\varphi(y, (\rmd \pi / \rmd \mu))  , \quad (\rmd \Kker_{\varphi}^\pi(y', \cdot) / \rmd \Kker_{\varphi}^\mu(y', \cdot))(x) = (\rmd \pi / \rmd \mu)(x) / (\rmd \pi_\varphi / \rmd \mu_\varphi)(y') . 
  \end{equation}
  Finally, there exists $\msc \in \mcx$ with $\pi(\msc) = 1$ such that for any $x \in \msc$ we have
  \begin{equation}
    (\rmd \pi / \rmd \mu)(x) = (\rmd \pi_\varphi / \rmd \mu_\varphi)(\varphi(x)) (\rmd \Kker_{\varphi}^\pi(\varphi(x), \cdot) / \rmd \Kker_{\varphi}^\mu(\varphi(x), \cdot))(x)  .
  \end{equation}

\end{proposition}

\begin{proof}
  Let $f : \ \msx \to \coint{0,+\infty}$ measurable. Using \Cref{thm:dis} we have 
  \begin{align}
    \pi_{\varphi}[f] &= \textstyle{\int_{\msx} f(\varphi(x)) \rmd \pi(x) }= \textstyle{\int_{\msx} f(\varphi(x)) (\rmd \pi / \rmd \mu)(x) \rmd \mu(x) } = \textstyle{\int_{\msx} f(y) \Kker^\mu_\varphi(y, (\rmd \pi / \rmd \mu)) \rmd \mu_\varphi(y)  , }
  \end{align}
  which concludes the first part of the proof. For the second part of the proof,
  let
  $\msb = \ensembleLigne{y \in \msy}{(\rmd \pi_\varphi / \rmd \mu_\varphi)(y)
    =0}$. We have
  \begin{equation}
    \textstyle{0 = \int_{\msy} \1_{\msb}(y) (\rmd \pi_\varphi / \rmd \mu_\varphi)(y) \rmd \mu_\varphi(y) = \pi_\varphi(\msb)  .}
  \end{equation}
  Therefore, there exists $\msa_1 \in \mcy$ such that
  $\pi_{\varphi}(\msa_1) = 1$ and for any $y \in \msa_1$,
  $(\rmd \pi_\varphi / \rmd \mu_\varphi)(y) > 0$. Let
  $g: \ \msy \to \coint{0,+\infty}$. Using \Cref{thm:dis} we have
  \begin{equation}
    \textstyle{ \int_{\msx} g(\varphi(x)) f(x) \rmd \pi(x) = \int_{\msx} g(\varphi(x)) f(x) (\rmd \pi/\rmd \mu)(x) \rmd \mu(x) = \int_{\msy} g(y) \Kker_\varphi^\mu(y,f \times (\rmd \pi/\rmd \mu)) \rmd \mu_\varphi(y)  .}
  \end{equation}
  Similarly, using \Cref{thm:dis} we have
  \begin{equation}
    \textstyle{\int_\msx g(\varphi(x)) f(x) \rmd \pi(x) = \int_\msy g(y) \Kker_\varphi^\pi(y,f) \rmd \pi_\varphi(y) = \int_\msy g(y) \Kker_\varphi^\pi(y,f) (\rmd \pi_\varphi / \rmd \mu_\varphi)(y) \rmd \pi_\varphi(y)  . }
  \end{equation}
  Hence, we get that there exists $\msa_2 \in \mcy$ with
  $\mu_\varphi(\msa_2) = 1$ (hence $\pi_\varphi(\msa_2) = 1$) such that for any
  $y \in \msa_2$ we have
  \begin{equation}
    \Kker_\varphi^\pi(y,f) (\rmd \pi_\varphi / \rmd \mu_\varphi)(y) = \Kker_\varphi^\mu(y, f \times (\rmd \pi/\rmd \mu))  . 
  \end{equation}
  We conclude upon letting $\msa = \msa_1 \cap \msa_2$ and using the fact that
  for any $y \in \msa$, $(\rmd \pi_\varphi / \rmd \mu_\varphi)(y) > 0$.
  Finally, since $\pi_{\varphi}(\msa) = 1$ if and only if $\pi(\varphi^{-1}(\msa))=1$, we have for any $x \in \varphi^{-1}(\msa)$
  \begin{equation}
    (\rmd \pi /  \rmd \mu)(x) = (\rmd \pi_\varphi / \rmd \mu_\varphi)(\varphi(x)) (\rmd \Kker_{\varphi}^\pi(\varphi(x), \cdot) / \rmd \Kker_{\varphi}^\mu(\varphi(x), \cdot))(x)  ,
  \end{equation}
which concludes the proof.
\end{proof}

We are now ready to state the additive formula.

\begin{proposition}
  \label{prop:additive}
  Let $(\msx, \mcx)$ and $(\msy, \mcy)$ be two Polish spaces and
  $\pi, \mu \in \Pens(\msx)$ with $\pi \ll \mu$. Then for any $\varphi:\ \msx \to \msy$ we have 
  \begin{equation}
    \textstyle{\KLLigne{\pi}{\mu} = \KLLigne{\pi_\varphi}{\mu_\varphi} + \int_{\msy} \KLLigne{\Kker^\pi_\varphi(y, \cdot)}{\Kker^\mu_\varphi(y, \cdot)} \rmd \pi_\varphi(y)  . }
  \end{equation}
\end{proposition}

\begin{proof}
  First assume that
  $\int_{\msx} \abs{\log((\rmd \pi / \rmd \mu)(x))} \rmd \pi(x) =
  +\infty$. Then, using \Cref{prop:decompo} we have
  $\int_{\msx} \abs{\log((\rmd \pi_\varphi / \rmd \mu_\varphi)(\varphi(x)))}
  \rmd \pi(x) = +\infty$ or
  $\int_{\msx} \abs{\log((\rmd \Kker^\pi_\varphi(\varphi(x), \cdot)/ \rmd
    \Kker^\mu_\varphi(\varphi(x), \cdot))(x))} \rmd \pi(x) = +\infty$, \ie \
  either $\KLLigne{\pi_\varphi}{\mu_\varphi} = +\infty$ or
  $\int_{\msx} \KLLigne{\Kker^\pi_\varphi(\varphi(x),
    \cdot)}{\Kker^\mu_\varphi(\varphi(x), \cdot)} \rmd \pi(x) = +\infty$ using
  \Cref{thm:dis}, which concludes the first part of the proof. Second, assume
  that $\int_{\msx} \abs{\log((\rmd \pi / \rmd \mu)(x))} \rmd \pi(x) <
  +\infty$. Using Pinsker's inequality \cite[Equation
  5.2.2]{bakry:gentil:ledoux:2014} we get that
  $\KLLigne{\pi_\varphi}{\mu_\varphi}<+\infty$, \ie \
  $\int_{\msx} \abs{\log((\rmd \pi_\varphi / \rmd \mu_\varphi)(\varphi(x)))}
  \rmd \pi(x) < +\infty$. Hence, we get that
  $\int_{\msx} \abs{\log((\rmd \Kker^\pi_\varphi(\varphi(x), \cdot)/ \rmd
    \Kker^\mu_\varphi(\varphi(x), \cdot))(x))} \rmd \pi(x) < +\infty$. Therefore
  we have
  \begin{equation}
    \textstyle{\KLLigne{\pi}{\mu} = \KLLigne{\pi_\varphi}{\mu_\varphi} + \int_{\msy} \KLLigne{\Kker^\pi_\varphi(y, \cdot)}{\Kker^\mu_\varphi(y, \cdot)} \rmd \pi_\varphi(y) }
  \end{equation}
which concludes the proof
\end{proof}

We emphasize that in the case where $\msx = \rset^d \times \rset^d$,
$\varphi = \mathrm{proj}_0$ the projection on the first variable and $\pi$,
$\mu$ admit densities w.r.t. the Lebesgue measure denoted $p$ and $q$ such that
for any $x, y \in \rset^d$, $p(x,y) = p_0(x)p_{1|0}(y|x)$ and
$q(x,y) = q_0(x)q_{1|0}(y|x)$ then one can avoid using disintegration theory and
\Cref{prop:additive} can be proved directly.

\subsection{Iterative Proportional Fitting  via potentials}
\label{prop:IPFpotential:proof}

In this section, before recalling the usual definition of the IPF via potentials
we provide a condition under which the IPF sequence is well-defined which is
used throughout \Cref{sec:iterativeproportionalfitting}.

\begin{proposition}
  \label{prop:existence_IPF}
  Assume that there exists $\tilde{\pi} \in \Pens_{N+1}$ such that
  $\tilde{\pi}_0 = \pdata$, $\tilde{\pi}_N = \pprior$ and
  $\KLLigne{\tilde{\pi}}{\pi^0} < +\infty$. Then the IPF sequence is well-defined.
\end{proposition}

\begin{proof}
  We prove the existence of the IPF sequence by recursion. First, note that
  $\pi^1$ is well-defined since $\tilde{\pi} \in \Pens_{N+1}$ with
  $\tilde{\pi}_N = \pprior$ and $\KLLigne{\tilde{\pi}}{\pi^0}< +\infty$. Second,
  assume that the sequence is well-defined up to $n$ with $n \in \nset$. Using
  \cite[Theorem 2.2]{csiszar1975divergence} we have
  \begin{equation}
    \textstyle{
      \KLLigne{\tilde{\pi}}{\pi^0} = \KLLigne{\tilde{\pi}}{\pi^{n}} + \sum_{j=0}^{n-1} \KLLigne{\pi^{j+1}}{\pi^j}  .
      }
  \end{equation}
  Hence $\KLLigne{\tilde{\pi}}{\pi^n} < +\infty$. Using that $\tilde{\pi}_0 = \pdata$
  if $n$ is odd and that $\tilde{\pi}_N = \pprior$ if $n$ is even, we get that
  $\pi^{n+1}$ is well-defined, which concludes the proof.
\end{proof}

We now introduce the IPF using potentials. This construction is not new and can
be found in
\cite{bernton2019schr,chen2016entropic,chen2020optimal,pavon2018data,peyre2019computational}
for instance (in continuous state spaces). In discrete settings the recursion
can be found in the following earlier works
\cite{kruithof1937telefoonverkeersrekening,deming1940least,fortet1940resolution,sinkhorn1967concerning,kullback1968probability,ruschendorf1995convergence}. The
IPF is defined by the following recursion $\pi^0=p$ given in
\eqref{eq:mu_forward} and for $n\geq0$
\begin{align}\label{eq:IFPrecursion_pot}
  &\pi^{2n+1} = \argmin \ensemble{\KLLigne{\pi}{\pi^{2n}}}{\pi \in \Pens_{N+1}, \ \pi_N = \pprior}  , \\
  &\pi^{2n+2} = \argmin  \ensemble{\KLLigne{\pi}{\pi^{2n+1}}}{\pi \in \Pens_{N+1}, \ \pi_0 = \pdata}. 
\end{align}
In the classical IPF presentation we obtain under mild assumptions that
$\pi^{2n+1}$ admits a density $q^n$ w.r.t the Lebesgue measure and that
$\pi^{2n}$ admits a density $p^n$ w.r.t the Lebesgue measure, given by the
following expressions
\begin{align}
  \label{eq:forward_bad}
  &\textstyle{q^n(x_{0:N}) = \pdata^n(x_0) \prod_{k=0}^{N-1} p^{n+1}_{k+1|k}(x_{k+1}|x_k) }  , \\
  &\textstyle{p^{n+1}(x_{0:N}) = \pdata(x_0) \prod_{k=0}^{N-1} p^{n+1}_{k+1|k}(x_{k+1}|x_k) }  , 
\end{align}
where $(\pdata^n(x_0))_{n \in \nset}$ and $(p^n_{k+1|k}(x_{k+1}|x_k))_{n \in \nset}$ are
densities which are iteratively computed, with $p^0_{k+1|k} = p_{k+1|k}$.

In the context of generative modelling the derivation \eqref{eq:forward_bad} is
not useful because it does not provide a generative model, \ie \ a probabilistic
transition from $\pprior$ to $\pdata$ but instead defines a transition from
$\pdata$ to $\pprior$.  Therefore, in this section only, we reverse the roles of
$\pprior$ and $\pdata$ and consider a reference density $\bar{p}$ such that for
any $x_{0:N} \in \mcx$ we have
\begin{equation}
  \label{eq:mu_forward_pot}
  \textstyle{\bar{p}(x_{0:N}) = \pprior(x_0) \prod_{k=0}^{N-1} \bar{p}_{k+1|k}(x_{k+1}|x_k)  . }
\end{equation}
Then, we consider the following recursion $\pi^0= \bar{p}$ given in
\eqref{eq:mu_forward_pot} and for $n\in \nset$
\begin{align}\label{eq:IFPrecursion_pot_pot}
  &\pi^{2n+1} = \argmin \ensemble{\KLLigne{\pi}{\pi^{2n}}}{\pi \in \Pens_{N+1}, \ \pi_N = \pdata}  , \\
  &\pi^{2n+2} = \argmin  \ensemble{\KLLigne{\pi}{\pi^{2n+1}}}{\pi \in \Pens_{N+1}, \ \pi_0 = \pprior}  .    
\end{align}
Again, we emphasize that the roles of $\pprior$ and $\pdata$ are exchanged in this formulation.
Using the classical IPF presentation we obtain the following expressions under mild assumptions
\begin{align}
  \label{eq:forward_better_before}
  &\textstyle{\bar{q}^n(x_{0:N}) = \pprior^n(x_0) \prod_{k=0}^{N-1} \bar{p}^{n+1}(x_{k+1}|x_k)  , } \\
  &\textstyle{\bar{p}^{n+1}(x_{0:N}) = \pprior(x_0) \prod_{k=0}^{N-1} \bar{p}^{n+1}(x_{k+1}|x_k)  .}
\end{align}
In this case, we get that $\pi^{2n+1}$ (approximately) defines a generative model
for large values of $n \in \nset$ since it provides a transition from to
$\pprior$ to (approximately) $\pdata$. In the following proposition we give the
precise statement corresponding to \eqref{eq:forward_better_before}. We assume that
$\bar{p}^0 = \bar{p}$.

\begin{proposition}
\label{prop:existence_potential}
  Assume that $\KLLigne{\pprior \otimes \pdata}{\bar{p}_{0,N}} < +\infty$. Then
  $(\pi^n)_{n \in \nset}$ given by \eqref{eq:IFPrecursion_pot_pot} is
  well-defined and for any $n \in \nset$ we have that $\pi^{2n+1}$ and
  $\pi^{2n+2}$ admit a density w.r.t. the Lebesgue measures denoted $\bar{q}^n$
  and $\bar{p}^{n+1}$. In addition, we have for any $n \in \nset$ and
  $x_{0:N} \in \mcx$
\begin{align}
  \label{eq:forward_better}
  &\textstyle{\bar{q}^{n}(x_{0:N}) = \pprior^n(x_0) \prod_{k=0}^{N-1} \bar{p}^{n+1}(x_{k+1}|x_k) }  , \\
  &\textstyle{\bar{p}^{n+1}(x_{0:N}) = \pprior(x_0) \prod_{k=0}^{N-1} \bar{p}^{n+1}(x_{k+1}|x_k) }  , 
\end{align}
where for any $n \in \nset$ we have for any $x_{0:N} \in \mcx$ and $k \in \{0, \dots, N-1\}$
\begin{equation}
  \pprior^n(x_0) = \psi_0^n(x_0) \pprior(x_0)  , \quad \bar{p}^{n+1}(x_{k+1}|x_k) = \bar{p}^n(x_{k+1}|x_k) \psi_{k+1}^{n}(x_{k+1}) / \psi_{k}^{n}(x_{k})  ,
\end{equation}
with 
\begin{equation}
  \psi_N^{n}(x_N) = \pdata(x_N) / \bar{p}_N^{n}(x_N)  , \quad \textstyle{\psi_k^n(x_k) = \int_{\rset^d} \psi_{k+1}^n(x_{k+1}) \bar{p}^{n}(x_{k+1}|x_k) \rmd x_{k+1}  . }
\end{equation}
\end{proposition}

\begin{proof}
  Let $\tilde{\pi} = (\pprior \otimes \pdata)\bar{p}_{|0,N}$. Using
  \Cref{prop:additive} we get that
  $\KLLigne{\tilde{\pi}}{\bar{p}} = \KLLigne{\pprior \otimes \pdata}{\bar{p}_{0,N}} < +\infty$.  Using
  \Cref{prop:existence_IPF} the IPF sequence is well-defined. In addition, using
  \cite[Theorem 3.1]{csiszar1975divergence} for any $n \in \nset$ there exists
  $\psi_N^n: \ \rset^d \to \coint{0,+\infty}$ such that for any
  $x_{0:N} \in \msa$ with $\tilde{\pi}(\msa) = 1$ we have
  \begin{equation}
    \bar{q}^{n}(x_{0:N}) = \bar{p}^{n}(x_{0:N}) \psi_N^n(x_N)  . 
  \end{equation}
  Since $\tilde{\pi} $ is equivalent to the Lebesgue measure we get that for any
  $x_{0:N} \in \rset^d$
  \begin{equation}
    \bar{q}^n(x_{0:N}) = \bar{p}^n(x_{0:N}) \psi_N^n(x_N)  . 
  \end{equation}
  Let $n \in \nset$.  We have for any $x_N \in \rset^d$,
  $\pdata(x_N) = \bar{q}^n(x_N) = \bar{p}^n_N(x_{N}) \psi_N^n(x_N)$. Hence, we get
  that for any $N \in \nset$, $\psi_N^n(x_N) = \pdata(x_N) / \bar{p}^{n}_N(x_N)$. For
  any $x_{0:N} \in \mcx$ and $k \in \{0, \dots, N-1\}$ let
  \begin{equation}
    \textstyle{\psi_k^n(x_k) = \int_{\rset^d} \psi_{k+1}^n(x_{k+1}) \bar{p}^n(x_{k+1}|x_k) \rmd x_{k+1}  . }
  \end{equation}
  We obtain that for any $x_{0:N} \in \mcx$
  \begin{equation}
\textstyle{    \bar{q}^n(x_{0:N}) = \pprior(x_0) \psi_0(x_0) \prod_{k=0}^{N-1}(\bar{p}^n(x_{k+1}|x_k) \psi_{k+1}(x_{k+1}) / \psi_k(x_k))  .}
  \end{equation}
  Hence, we get that for any $x_{0:N} \in \mcx$,
  $\bar{q}^{n}(x_0) = \pprior^n(x_0) \prod_{k=0}^{N-1} \bar{p}^{n+1}(x_{k+1}|x_k)$.  Using
  \Cref{prop:additive} we get that for any $x_{0:N} \in \mcx$,
  $\bar{p}^{n+1}(x_0) = \pprior(x_0) \prod_{k=0}^{N-1} \bar{p}^{n+1}(x_{k+1}|x_k)$, which
  concludes the proof.
\end{proof}

The previous expression is not symmetric and the IPF iterations appear as a
policy refinement of the original forward dynamic $\bar{p}$. In the next proposition
we present another potential formulation of the IPF iterations which is symmetric.

\begin{proposition}
  Assume that $\KLLigne{\pprior \otimes \pdata}{q_{0,N}} < +\infty$. Then
  $(\pi^n)_{n \in \nset}$ given by \eqref{eq:IFPrecursion_pot_pot} is
  well-defined and for any $n \in \nset$ we have that $\pi^{2n+1}$ and
  $\pi^{2n+2}$ admit a density w.r.t. the Lebesgue measures denoted $\bar{q}^n$
  and $\bar{p}^{n+1}$. In addition, we have for any $n \in \nset$ and
  $x_{0:N} \in \mcx$
\begin{align}
  \label{eq:forward_sym}
  &\textstyle{\bar{q}^n(x_{0:N}) = \varphi_0^n(x_0) \prod_{k=0}^{N-1} \bar{p}(x_{k+1}|x_k) \psi_N^{n}(x_N) }  , \\
  &\textstyle{\bar{p}^{n+1}(x_{0:N}) = \varphi_0^{n+1}(x_0)  \prod_{k=0}^{N-1} \bar{p}(x_{k+1}|x_k) \psi_N^n(x_N) }  , 
\end{align}
where for any $n \in \nset$ we have for any $x_{0:N} \in \mcx$ and $k \in \{0, \dots, N-1\}$
\begin{align}
  &\psi_N^{n}(x_N) =  \pdata(x_N) / \varphi_N^{n}(x_N)  , \quad \textstyle{\psi_k^n(x_k) = \int_{\rset^d} \psi_{k+1}^n(x_{k+1}) \bar{p}(x_{k+1}|x_k) \rmd x_{k+1}  , } \\
  &\varphi_0^{n+1}(x_0) =  \pprior(x_0) / \psi_0^{n}(x_0)  , \quad \textstyle{\varphi_{k+1}^{n+1}(x_{k+1}) = \int_{\rset^d} \varphi_{k}^{n+1}(x_{k}) \bar{p}(x_{k+1}|x_k) \rmd x_{k}  , } 
\end{align}
and $\varphi_0^0 = \pprior$ and $\psi_N^{-1} = 1$.
\end{proposition}

\begin{proof}
  Let $\tilde{\pi} = (\pprior \otimes \pdata)q_{|0,N}$. Using
  \Cref{prop:additive} we get that
  $\KLLigne{\tilde{\pi}}{q} = \KLLigne{\pprior \otimes \pdata}{p_{0,N}} < +\infty$.  Using
  \Cref{prop:existence_IPF} the IPF sequence is well-defined. In addition, using
  \cite[Theorem 3.1]{csiszar1975divergence} for any $n \in \nset$ there exists
  $\psi_N^n: \ \rset^d \to \coint{0,+\infty}$ such that for any
  $x_{0:N} \in \msa$ with $\tilde{\pi}(\msa) = 1$ we have
  \begin{equation}
    \bar{q}^n(x_{0:N}) = \bar{p}^n(x_{0:N}) \tilde{\psi}_N^n(x_N)  , \qquad \bar{p}^{n+1}(x_{0:N}) = \bar{q}^n(x_{0:N}) \tilde{\varphi}_0^n(x_0)  . 
  \end{equation}
  Since $\tilde{\pi} $ is equivalent to the Lebesgue measure we get that for any
  $x_{0:N} \in \rset^d$
  \begin{equation}
    \bar{q}^n(x_{0:N}) = \bar{p}^n(x_{0:N}) \tilde{\psi}_N^n(x_N)  ,  \qquad \bar{p}^{n+1}(x_{0:N}) = \bar{q}^n(x_{0:N}) \tilde{\varphi}_0^n(x_0)  .  
  \end{equation}
  For any $n \in \nset$, let $\psi_N^{n}= \psi_N^{n-1} \tilde{\psi}^n_N$ and
  $\varphi_0^{n+1}= \varphi_0^{n} \tilde{\varphi}_0^n$. By recursion, we get
  that for any $n \in \nset$ and $x_{0:N} \in \mcx$
\begin{align}
  \label{eq:forward_sym}
  &\textstyle{\bar{q}^n(x_{0:N}) = \varphi_0^n(x_0) \prod_{k=0}^{N-1} \bar{p}(x_{k+1}|x_k) \psi_N^{n}(x_N) }  , \\
  &\textstyle{\bar{p}^{n+1}(x_{0:N}) = \varphi_0^{n+1}(x_0)  \prod_{k=0}^{N-1} \bar{p}(x_{k+1}|x_k) \psi_N^n(x_N)  .}
\end{align}
Let $n \in \nset$. For any $x_N \in \rset^d$ we have
\begin{equation}
  \label{eq:equ_fin}
  \bar{q}^n_N(x_N) = \pdata(x_N) = \bar{p}^n_N(x_N) \tilde{\psi}_N^n(x_n)  . 
\end{equation}
In addition, for any $k \in \{0, \dots, N-1\}$ and $x_{0:N} \in \mcx$ we define
$\varphi_{k+1}^{n+1}(x_{k+1}) = \int_{\rset^d} \varphi_{k}^{n+1}(x_{k})
\bar{p}(x_{k+1}|x_k) \rmd x_{k}$. We have for any $x_{N} \in \rset^d$,
$\bar{p}_N^{n}(x_N) = \varphi_N^n(x_N)\psi_N^{n-1}(x_n) $. Combining this result
with \eqref{eq:equ_fin} we get that for any $x_N \in \rset^d$
\begin{equation}
\psi_N^n(x_N) = \pdata(x_N) / \varphi^n_N(x_N)  . 
\end{equation}
Similarly, we get that for any $x_0 \in \rset^d$,
$\varphi_0^{n+1}(x_0) = \pprior(x_0) / \psi_0^{n}(x_0)$, which concludes the
proof.
\end{proof}

\subsection{Proof of \Cref{prop:IPFrecursion}}\label{prop:IPFrecursion:proof}

Let $\tilde{\pi} = (\pprior \otimes \pdata)p_{|0,N}$. Using
\Cref{prop:additive} we get that
$\KLLigne{\tilde{\pi}}{p} = \KLLigne{\pprior \otimes \pdata}{p_{0,N}} <
+\infty$.  Using \Cref{prop:existence_IPF} the IPF sequence is
well-defined. Note that $\pi^0$ admits a density w.r.t. the Lebesgue measure
given by $p>0$. Let $n \in \nset$ and assume that $p^n > 0$ is given for any $x_{0:N} \in \mcx$ by
\begin{equation}
  \label{eq:markov_2n}
  \textstyle{p^n(x_{0:N}) = \pdata(x_0) \prod_{k=0}^{N-1} q^{n-1}(x_{k+1}|x_{k})  . }
\end{equation}
 Using
\Cref{prop:additive} we get that for any $\pi \in \Pens_{N+1}$ such
that $\pi_N = \pprior$ we have
\begin{equation}
  \textstyle{\KLLigne{\pi}{\pi^{2n}} = \KLLigne{\pprior}{\pi^{2n}_0} + \int_{\rset^d} \KLLigne{\pi_{|N}}{\pi^{2n}_{|N}} \pprior(x_N) \rmd x_N  . }
\end{equation}
Hence, we have that $\pi^{2n+1} = \pprior \pi^{2n}_{|N}$. Since $p^n > 0$ we
get that for any $\pi^{2n}_{|N}$ satisfies for any $\msa \in \mcb{\mcx}$ and $x_N \in \rset^d$
\begin{equation}
  \textstyle{\pi^{2n}_{|N}(\msa|x_N) = \int_{\msa} p^n(x_{0:N})/p^n(x_N) \rmd x_{0:N} \updelta_{x_N}(\msa_N)}  . 
\end{equation}
Therefore, $\pi^{2n+1}$ admits a density w.r.t. the Lebesgue measure denoted
$q^n$ and given for any $x_{0:N} \in \mcx$ by
\begin{align}
  q^n(x_{0:N}) &= p^n(x_{0:N}) \pprior(x_N)/p^n(x_N) \\
               &\textstyle{= \pprior(x_N) \prod_{k=0}^{N-1} p^n(x_{k+1}|x_k) p^n(x_k) / p^n(x_{k+1}) = \pprior(x_N) \prod_{k=0}^{N-1} p^n(x_{k}|x_{k+1}) }  , \end{align}
             where we have used \eqref{eq:markov_2n}. 
             Note that $q^n > 0$. Similarly, we get that for any $x_{0:N} \in \mcx$
\begin{equation}
  \textstyle{p^{n+1}(x_{0:N}) = \pdata(x_0) \prod_{k=0}^{N-1} q^n(x_{k+1}|x_{k})  . }
\end{equation}
Note that again that $p^{n+1} > 0$. We conclude by recursion.

\subsection{Link with autoencoders}
\label{sec:autoencoder}

Consider the maximum likelihood problem
\begin{equation}
  \textstyle{
    q^{\star} = \argmax \ensembleLigne{\mathbb{E}_{\pdata}[\log q_0(X_0)]}{q \in \Pens_d(\mcx), \ q_N = \pprior}  ,
    }
\end{equation}
where $\Pens_d(\mcx)$ is the subset of the probability distribution over $\mcx$
which admit a density w.r.t. the Lebesgue measure.
Using Jensen's inequality we have for any $q \in \Pens_d(\mcx)$
\begin{align}  
  \mathbb{E}_{\pdata}[\log q_0(X_0)] &= \textstyle{\int_{\rset^d}\log(\int_{(\rset^d)^{N-1}} q(x_{0:N}) p(x_{1:N}|x_0) / p(x_{1:N}|x_0) \rmd x_{1:N}) p_0(x_0)\rmd x_0} \\
                                      &\geq \textstyle{\int_{\mcx} \log(q(x_{0:N})/ p(x_{1:N}|x_0)) p(x_{0:N}) \rmd x_{0:N} } \geq -\KLLigne{p}{q} - \Ent(p_0)  . 
\end{align}
This Evidence Lower Bound (ELBO) is similar to the one identified in
\cite{ho2020denoising}. Maximizing this ELBO is equivalent to solving the
following problem
\begin{equation}
    \textstyle{
    q^0 = \argmin \ensembleLigne{\KLLigne{q}{p}}{q \in \Pens_d(\mcx), \ q_N = \pprior}  ,
    }
\end{equation}
which is the first step of IPF. Hence subsequent steps can be obtained by
maximizing ELBOs associated with the following maximum likelihood problems
for any $n \in \nset$
\begin{align}
    &\textstyle{
    q^{\star} = \argmax \ensembleLigne{\mathbb{E}_{\pdata}[\log q_0(X_0)]}{q \in \Pens_d(\mcx), \ q_N = \pprior}  ,
    }  \\
    &\textstyle{
    p^{\star} = \argmax \ensembleLigne{\mathbb{E}_{\pprior}[\log p_N(X_N)]}{p \in \Pens_d(\mcx), \ p_0 = \pdata} 
    }  .
\end{align}

\section{Alternative variational formulations}
\label{sec:altern-vari-form}

In this section, we draw links between IPF and score-matching techniques. We
start by proving \Cref{prop:generalizedscorematching} in
\Cref{prop:generalizedscorematching:proof}. We then present alternative
variational formulations in \Cref{sec:variational-formulas}.

\subsection{Proof of \Cref{prop:generalizedscorematching}}
\label{prop:generalizedscorematching:proof}

We only prove \eqref{eq:regressionb} since the proof \eqref{eq:regressionf} is
similar. Let $n \in \nset$ and $k \in \{0, \dots, N-1\}$. For any
$x_{k+1} \in \rset^d$ we have
\begin{equation}
  \textstyle{
    p^n_{k+1}(x_{k+1}) = (4 \uppi \gamma_{k+1})^{-d/2} \int_{\rset^d} p^n(x_k) \exp[-\normLigne{F_k^n(x_k) - x_{k+1}}^2/(4\gamma_{k+1})] \rmd x_k  ,
    }
\end{equation}
with $F_k^n(x_k) = x_k + \gamma_{k+1} f_k^n(x_k)$. Since $p^n_k>0$ is bounded
using the dominated convergence theorem we have for any $x_{k+1} \in \rset^d$
\begin{equation}
  \textstyle{\nabla \log p^n_{k+1} (x_{k+1}) = \int_{\rset^d} (F_k^n(x_k) - x_{k+1})/(2 \gamma_{k+1}) p_{k|k+1}(x_k | x_{k+1}) \rmd x_{k}  . }
\end{equation}
Therefore we get that for any $x_{k+1} \in \rset^d$
\begin{equation}
  \textstyle{b_{k+1}^n(x_{k+1}) = \int_{\rset^d} (F_k^n(x_k) - F_{k}^n(x_{k+1}))/\gamma_{k+1} p_{k|k+1}(x_k | x_{k+1}) \rmd x_{k}  . }
\end{equation}
This is equivalent to
\begin{equation}
  \textstyle{B_{k+1}^n(x_{k+1})  = \CPELigne{X_{k+1} + F_k^n(X_k) - F_{k}^n(X_{k+1})}{X_{k+1} = x_{k+1}}}  ,
\end{equation}
with $(X_k, X_{k+1}) \sim p_{k, k+1}(x_k, x_{k+1})$. Hence, we get that
\begin{equation}
\textstyle{B^n_{k+1}=\argmin_{\mathrm{B}\in \rmL^2(\rset^d, \rset^d)} \expeMarkovLigne{{p^{n}_{k,k+1}}}{\normLigne{\mathrm{B}(X_{k+1})-(X_{k+1} + F^{n}_k(X_{k})-F^{n}_{k}(X_{k+1}))}^2}}  ,
\end{equation}
which concludes the proof.

 \subsection{Variational formulas}
 \label{sec:variational-formulas}

 In \Cref{prop:generalizedscorematching} and \Cref{sec:iter-score-match} we
 present a variational formula for $B_{k+1}^n$ and $F_k^{n+1}$ for any
 $n \in \nset$ and $k \in \{0, \dots, N-1\}$, where we recall that for any
 $x \in \rset^d$ we have
 \begin{equation}
   B_{k+1}^n(x) = x + \gamma_{k+1} b_{k+1}^n(x)  , \qquad F_{k}^{n+1} = x + \gamma_{k+1} f_{k}^{n+1}(x)  ,
 \end{equation}
 where we have
 \begin{equation}
   \label{eq:drift_recursion}
   b_{k+1}^n(x) = -f_k^n(x)  + 2 \nabla \log p_{k+1}^n (x)  , \qquad f_{k}^{n+1}(x) = -b_{k+1}^n(x)  + 2 \nabla \log q_{k}^n (x)  .
 \end{equation}
 In the rest of this section we assume that  for any $n \in \nset$,
 $k \in \{0, \dots, N-1\}$ and $x \in \rset^d$ we have
 \begin{align}
   \textstyle{q_{k|k+1}^n(x_k|x_{k+1}) = (4 \uppi \gamma_{k+1})^{-d/2} \exp[-\normLigne{x_k - B_{k+1}^n(x_{k+1})}^2/(4\gamma_{k+1})] } , \\
   \textstyle{p_{k+1|k}^{n+1}(x_{k+1}|x_{k}) = (4 \uppi \gamma_{k+1})^{-d/2} \exp[-\normLigne{x_{k+1} - F_{k}^{n+1}(x_{k})}^2/(4\gamma_{k+1})] } . 
 \end{align}
 We recall that in this case \Cref{prop:generalizedscorematching} ensures that
 for any $n \in \nset$ and $k\in \{0, \dots, N-1\}$
\begin{align}
&\textstyle{B^n_{k+1}=\argmin_{\mathrm{B}\in \rmL^2(\rset^d, \rset^d)} \expeMarkovLigne{{p^{n}_{k,k+1}}}{\normLigne{\mathrm{B}(X_{k+1})-(X_{k+1} + F^{n}_k(X_{k})-F^{n}_{k}(X_{k+1}))}^2}},\\
&\textstyle{F^{n+1}_{k}=\argmin_{\mathrm{F}\in \rmL^2(\rset^d, \rset^d)} \expeMarkovLigne{q^{n}_{k,k+1}}{\normLigne{\mathrm{F}(X_k)-(X_{k} + B^{n}_{k+1}(X_{k+1})-B^n_{k+1}(X_{k}))}^2 }}.
\end{align}
In the rest of this section we derive other variational formulas and
discuss their practical limitations/advantages.

\subsubsection{Score-matching formula and sum of networks}
\label{sec:score-match-form_sum}
First, using \eqref{eq:drift_recursion} we have for any $n \in \nset$,
$k \in \{0, \dots, N-1\}$ and $x \in \rset^d$
\begin{align}
  &b_{k+1}^n(x) = \textstyle{\alpha x + 2 \sum_{j=0}^n \nabla \log p_{k+1}^j(x) - 2 \sum_{j=0}^{n-1} \nabla \log q_k^j(x)  ,} \label{eq:sumscore}\\
  &f_k^n(x) = \textstyle{-\alpha x + 2 \sum_{j=0}^{n-1} \nabla \log q_k^j(x) - 2 \sum_{j=0}^{n-1} \nabla \log p_{k+1}^{j}(x)  . } \label{eq:sumscore2}
\end{align}
In the following proposition we derive a variational formula for
$\nabla \log p_{k+1}^n$ and $\nabla \log q_{k}^n(x)$ for any $n \in \nset$ and
$k \in \{0, \dots, N-1\}$.

\begin{proposition}\label{prop:generalizedscorematching_sum} 
  For any $n \in \nset$ and $k\in \{0, \dots, N-1\}$ we have
\begin{align}
&\textstyle{\nabla \log p_{k+1}^n =\argmin_{u \in \rmL^2(\rset^d, \rset^d)} \expeMarkovLigne{{p^{n}_{k,k+1}}}{\normLigne{u(X_{k+1})-(F_k^n(X_k) - X_{k+1})/(2\gamma_{k+1})}^2}},\label{eq:regressionb_score}\\
&\textstyle{\nabla \log q_{k}^{n} =\argmin_{v \in \rmL^2(\rset^d, \rset^d)} \expeMarkovLigne{q^{n}_{k,k+1}}{\normLigne{v(X_k)-(B_{k+1}^{n}(X_{k+1})-X_{k})/(2 \gamma_{k+1})}^2 }}.\label{eq:regressionf_score}
 \end{align} 
\end{proposition}

\begin{proof}
  The proof is similar to the one of \Cref{prop:generalizedscorematching} but is
  provided for completeness.
  We only prove \eqref{eq:regressionb_score} since the proof \eqref{eq:regressionf_score} is
similar. Let $n \in \nset$ and $k \in \{0, \dots, N-1\}$. For any
$x_{k+1} \in \rset^d$ we have
\begin{equation}
  \textstyle{
    p^n_{k+1}(x_{k+1}) = (4 \uppi \gamma_{k+1})^{-d/2} \int_{\rset^d} p^n(x_k) \exp[-\normLigne{F_k^n(x_k) - x_{k+1}}^2/(4\gamma_{k+1})] \rmd x_k  ,
    }
\end{equation}
with $F_k^n(x_k) = x_k + \gamma_{k+1} f_k^n(x_k)$. Since $p^n_k>0$ is bounded
using the dominated convergence theorem we have for any $x_{k+1} \in \rset^d$
\begin{equation}
  \textstyle{\nabla \log p^n_{k+1} (x_{k+1}) = \int_{\rset^d} (F_k^n(x_k) - x_{k+1})/(2 \gamma_{k+1}) p_{k|k+1}(x_k | x_{k+1}) \rmd x_{k}  . }
\end{equation}
This is equivalent to
\begin{equation}
  \textstyle{\nabla \log p^n_{k+1} (x_{k+1})  = \CPELigne{(F_k^n(X_k) - X_{k+1})/(2 \gamma_{k+1})}{X_{k+1} = x_{k+1}}}  ,
\end{equation}
with $(X_k, X_{k+1}) \sim p_{k,k+1}(x_k, x_{k+1})$. Hence, we get that
\begin{equation}
\textstyle{\nabla \log p^n_{k+1} =\argmin_{u\in \rmL^2(\rset^d, \rset^d)} \expeMarkovLigne{{p^{n}_{k,k+1}}}{\normLigne{u(X_{k+1})-(F^{n}_k(X_{k})-X_{k+1})/(2 \gamma_{k+1})}^2}}  ,
\end{equation}
 which concludes the proof.
\end{proof}

Note that \eqref{eq:regressionb_score} and \eqref{eq:regressionf_score} can be
simplified upon remarking that for any $n \in \nset$ and $k \in \{0, \dots, N-1\}$
\begin{equation}
  \textstyle{
    X_{k+1}^n = F_k^n(X_k^n) + \sqrt{2 \gamma_{k+1}} Z_{k+1}^n  , \quad \tilde{X}_{k}^n = F_k^n(\tilde{X}_{k+1}^n) + \sqrt{2 \gamma_{k+1}} \tilde{Z}_{k+1}^n  ,
    }
\end{equation}
with $\{X_k^n\}_{k=0}^{N} \sim p^n$, $\{\tilde{X}_k^n\}_{k=0}^{N} \sim q^n$ and
$\ensembleLigne{(Z_{k+1}^n, \tilde{Z}_{k+1}^n)}{n \in \nset, \ k \in \{0, \dots,
  N-1\}}$ a family of independent Gaussian random variables with zero mean an
identity covariance matrix. Using this result we get that
for
  any $n \in \nset$ and $k\in \{0, \dots, N-1\}$
\begin{align}
&\textstyle{\nabla \log p_{k+1}^n =\argmin_{u \in \rmL^2(\rset^d, \rset^d)} \expeMarkovLigne{{p^{n}_{k,k+1}}}{\normLigne{u(X_{k+1})-Z_{k+1}^n/\sqrt{2\gamma_{k+1}}}^2}},\label{eq:regressionb_score}\\
&\textstyle{\nabla \log q_{k}^{n} =\argmin_{v \in \rmL^2(\rset^d, \rset^d)} \expeMarkovLigne{q^{n}_{k,k+1}}{\normLigne{v(X_k)-\tilde{Z}_{k+1}^n/\sqrt{2\gamma_{k+1}}}^2 }}.\label{eq:regressionf_score}
\end{align}
In practice, neural networks $u_{\alpha^n}(k,x)\approx \nabla \log p_{k}^n(x)$,
and $v_{\beta^n}(k,x)\approx \nabla \log q_k^n(x)$ are used.  Hence, we sample
approximately from $q^n$ and $p^n$ for any $n \in \nset$ using the following
recursion:
\begin{align}
    &\textstyle{\tilde{X}_{k}^n  =  \tilde{\tau}_{k+1} \tilde{X}_{k+1}^n + 2 \gamma_{k+1} \defEnsLigne{\sum_{j=0}^n u_{\alpha^j}(k+1, \tilde{X}_{k+1}^n) - \sum_{j=0}^{n-1} v_{\beta^j}(k, \tilde{X}_{k+1}^n)} + \sqrt{2 \gamma_{k+1}} \tilde{Z}_{k+1}^n  ,} \\
  &\textstyle{X_{k+1}^n =  \tau_{k+1} X_k^n + 2 \gamma_{k+1} \defEnsLigne{\sum_{j=0}^n u_{\alpha^j}(k+1, X_k^n) - \sum_{j=0}^{n} v_{\beta^j}(k, X_k^n)} + \sqrt{2 \gamma_{k+1}} Z_{k+1}^n  ,}   \label{eq:recursion_sampling}
\end{align}
where $\tilde{\tau}_{k+1} = 1 + \alpha \gamma_{k+1}$, $\tau_{k+1} = 1 - \alpha \gamma_{k+1}$ and $X_0^n \sim \pdata$, $\tilde{X}_N^n \sim \pprior$.

\subsubsection{Drift-matching formula}
\label{sec:drift-match-formulation}
In \Cref{prop:generalizedscorematching} we have given a variational formula for
$B_{k+1}^n$ and $F_{k}^{n+1}$ for any $n \in \nset$ and
$k \in \{0, \dots, N-1\}$. In \Cref{prop:generalizedscorematching_sum} we have
given a variational formula for $\nabla \log p_{k+1}^n$ and $\nabla \log q_k^n$
for any $n \in \nset$ and $k \in \{0, \dots, N-1\}$. In the following
proposition we give a variational formula for the drifts $b_{k+1}^n$ and $f_{k}^{n+1}$.

\begin{proposition}\label{prop:generalizedscorematching_drift} 
  For any $n \in \nset$ and $k\in \{0, \dots, N-1\}$ we have 
\begin{align}
&\textstyle{b^n_{k+1}=\argmin_{b\in \rmL^2(\rset^d, \rset^d)} \expeMarkovLigne{{p^{n}_{k,k+1}}}{\normLigne{b(X_{k+1})-(F^{n}_k(X_{k})-F^{n}_{k}(X_{k+1}))/\gamma_{k+1}}^2}}\label{eq:regressionb_drift}\\
&\textstyle{f^{n+1}_{k}=\argmin_{f\in \rmL^2(\rset^d, \rset^d)} \expeMarkovLigne{q^{n}_{k,k+1}}{\normLigne{f(X_k)-(B^{n}_{k+1}(X_{k+1})-B^n_{k+1}(X_k))/\gamma_{k+1}}^2 }}\label{eq:regressionf_drift}
\end{align}
\end{proposition}

\begin{proof}
  The proof is similar to the one of \Cref{prop:generalizedscorematching} but is
  provided for completeness.
We only prove \eqref{eq:regressionb_drift} since the proof \eqref{eq:regressionf_drift} is
similar. Let $n \in \nset$ and $k \in \{0, \dots, N-1\}$. For any
$x_{k+1} \in \rset^d$ we have
\begin{equation}
  \textstyle{
    p^n_{k+1}(x_{k+1}) = (4 \uppi \gamma_{k+1})^{-d/2} \int_{\rset^d} p^n(x_k) \exp[-\normLigne{F_k^n(x_k) - x_{k+1}}^2/(4\gamma_{k+1})] \rmd x_k  ,
    }
\end{equation}
with $F_k^n(x_k) = x_k + \gamma_{k+1} f_k^n(x_k)$. Since $p^n_k>0$ is bounded
using the dominated convergence theorem we have for any $x_{k+1} \in \rset^d$
\begin{equation}
  \textstyle{\nabla \log p^n_{k+1} (x_{k+1}) = \int_{\rset^d} (F_k^n(x_k) - x_{k+1})/(2 \gamma_{k+1}) p_{k|k+1}(x_k | x_{k+1}) \rmd x_{k}  . }
\end{equation}
Therefore we get that for any $x_{k+1} \in \rset^d$
\begin{equation}
  \textstyle{b_{k+1}^n(x_{k+1}) = \int_{\rset^d} (F_k^n(x_k) - F_{k}^n(x_{k+1}))/ \gamma_{k+1} p_{k|k+1}(x_k | x_{k+1}) \rmd x_{k}  . }
\end{equation}
This is equivalent to
\begin{equation}
  \textstyle{b_{k+1}^n(x_{k+1})  = \CPELigne{(F_k^n(X_k) - F_{k}^n(X_{k+1}))/ \gamma_{k+1}}{X_{k+1} = x_{k+1}}}  ,
\end{equation}
with $(X_k, X_{k+1}) \sim p(x_k, x_{k+1})$. Hence, we get that
\begin{equation}
\textstyle{b^n_{k+1}=\argmin_{b\in \rmL^2(\rset^d, \rset^d)} \expeMarkovLigne{{p^{n}_{k,k+1}}}{\normLigne{b(X_{k+1})-(F^{n}_k(X_{k})-F^{n}_{k}(X_{k+1}))/\gamma_{k+1}}^2}}  ,
\end{equation}
 which concludes the proof.
\end{proof}

In practice, neural networks
   $b_{\beta^n}(k,x)\approx b^{n}_k(x)$, and
   $f_{\alpha^n}(k,x)\approx f^{n}_k(x)$ are used.
   Hence, we sample
approximately from $q^n$ and $p^n$ for any $n \in \nset$ using the following
recursion:
\begin{align}
  \label{eq:recursion_sampling_drift}
  &\textstyle{\tilde{X}_{k}^n  =  \tilde{X}_{k+1}^n + \gamma_{k+1} b_{\beta^n}(k+1, \tilde{X}_{k+1}^n) + \sqrt{2 \gamma_{k+1}} \tilde{Z}_{k+1}^n  ,} \\
  &\textstyle{X_{k+1}^n  =  X_{k}^n + \gamma_{k+1} f_{\alpha^n}(k, X_{k}^n) + \sqrt{2 \gamma_{k+1}} Z_{k+1}^n  ,} 
\end{align}
with $X_0^n \sim \pdata$, $\tilde{X}_N^n \sim \pprior$.

\subsubsection{Discussion}

We identify three variational formulas associated with
\Cref{prop:generalizedscorematching}, \Cref{prop:generalizedscorematching_sum}
and \Cref{prop:generalizedscorematching_drift}. In practice we discard the
approach of \Cref{sec:score-match-form_sum} because it requires storing $2n$
neural networks to sample from $p^n$, see \eqref{eq:recursion_sampling}. Hence
the algorithm requires more memory as $n$ increases and the sampling procedure
requires $\bigO(nN)$ passes through a neural network. The approaches described
in \Cref{prop:generalizedscorematching} and
\Cref{prop:generalizedscorematching_drift} yield sampling procedures which only
require $\bigO(N)$ passes through a neural network and have fixed memory cost
for any $n \in \nset$. In practice we observed that the approach of
\Cref{prop:generalizedscorematching} yields better results. We conjecture that
this favorable behavior is mainly due to the architecture of the neural networks
used to approximate $B_{k+1}^n$ and $F_{k}^{n+1}$ which have residual
connections and therefore are better suited at representing functions of the
$x \mapsto x + \Phi(x)$ where $\Phi$ is a perturbation.
 
\section{Theoretical study of \schro bridges and the IPF}
\label{sec:theor-study-schro}

In this section, we explore some of the theoretical properties of \schro bridges
and the IPF procedure.  \Cref{prop:monotonicity} and \Cref{prop:convergence_ipf}
are proved in \Cref{prop:monotonicity:proof} and \Cref{prop:convergence_ipf:proof}
respectively.

\subsection{Proof of \Cref{prop:monotonicity}}
\label{prop:monotonicity:proof}
In this section, we prove \Cref{prop:monotonicity}. First we gather novel
monotonicity results for the IPF in \Cref{prop:monotonicity_appendix}, see
\Cref{prop:monotonicity_appendix:sec}. Then we prove our quantitative
convergence bounds in \Cref{prop:schro_quantitative_appendix}, see
\Cref{prop:schro_quantitative_appendix:sec}.

\subsubsection{Monotonicity results}
\label{prop:monotonicity_appendix:sec}

We consider the static  IPF recursion: $\pi^0 = \mu \in \Pens_2$ and 
\begin{align}
  &\pi^{2n+1} = \argmin \ensemble{\KLLigne{\pi}{\pi^{2n}}}{\pi \in \Pens_{2}, \ \pi_1 = \nu_1}  , \\
  &\pi^{2n+2} = \argmin  \ensemble{\KLLigne{\pi}{\pi^{2n+1}}}{\pi \in \Pens_{2}, \ \pi_0 = \nu_0}  ,
\end{align}
where $\nu_0, \nu_1 \in \Pens(\rset^d)$.  We also consider the following assumption.
\begin{assumptionB}
  \label{assum:existence}
  $\mu$ is absolutely continuous w.r.t. $\mu_0 \otimes \mu_1$ and 
  $\KLLigne{\nu_0 \otimes \nu_1}{\mu} < +\infty$. In addition, $\nu_i$ and
  $\mu_i$ are equivalent for $i \in \{0, 1\}$.
\end{assumptionB}

First we draw links between \Cref{assum:existence_spec} and \Cref{assum:existence}.

\begin{proposition}
   \rref{assum:existence_spec} implies \rref{assum:existence} with $\mu = p_{0,N}$.
 \end{proposition}

 \begin{proof}
   Since $p_N> 0$ we get that $p_N$ and $\pprior$ are equivalent. Hence $\mu_1$
   and $\nu_1$ are equivalent and $\mu_0 = \nu_0$.  Let us show that $\mu$ is
   absolutely continuous w.r.t. $\mu_0 \otimes \mu_1$, \ie \ that $p_{0,N}$ is
   absolutely continuous w.r.t. $\pdata \otimes p_N$.  Since $p_N > 0$ we get
   that $p_{0,N}$ is absolutely continuous w.r.t. $\pdata \otimes p_N$ with
   density $p_{N|0}/p_N$. Finally we have
   \begin{align}
     &\textstyle{\int_{(\rset^d)^2} \log(\pdata(x_0) \pprior(x_N)/(\pdata(x_0) p_{N|0}(x_N|x_0))) \pdata(x_0) \pprior(x_N) \rmd x_0 \rmd x_N} \\ & \qquad = \textstyle{\int_{(\rset^d)^2} \log(\pprior(x_N)/p_{N|0}(x_N|x_0)) \pdata(x_0) \pprior(x_N) \rmd x_0 \rmd x_N} \\
     & \qquad \textstyle{\leq \absLigne{\Ent(\pprior)} + \int_{\rset^d} \absLigne{\log p_{N|0}(x_N|x_0)} \pdata(x_0) \pprior(x_N) \rmd x_0 \rmd x_N < +\infty }
   \end{align}
   which concludes the proof.
 \end{proof}

In this section we prove the following proposition.

\begin{proposition}
  \label{prop:monotonicity_appendix}
  Assume \rref{assum:existence}. Then, the  IPF sequence is well-defined and for any
  $n \in \nset$ with $n \geq 1$ we have
  \begin{equation}
    \label{eq:fundamental_kl_appendix}
    \KLLigne{\pi^{n+1}}{\pi^n} \leq \KLLigne{\pi^{n-1}}{\pi^n}  , \qquad \KLLigne{\pi^n}{\pi^{n+1}} \leq \KLLigne{\pi^n}{\pi^{n-1}}  . 
  \end{equation}
  In addition, the following results hold:
  \begin{enumerate}[wide, labelwidth=!, labelindent=0pt, label=(\alph*)]
  \item \label{item:mono_a} $(\tvnormLigne{\pi^{n+1} - \pi^{n}})_{n \in \nset}$ and $(\JefLigne{\pi^{n+1}}{\pi^{n}})_{n \in \nset}$ are non-increasing.
  \item \label{item:mono_b} $(\KLLigne{\pi^{2n+1}}{\pi^{2n}})_{n \in \nset}$ and $(\KLLigne{\pi^{2n+2}}{\pi^{2n+1}})_{n \in \nset}$ are non-increasing.
  \item \label{item:mono_c} $(\KLLigne{\pi^{2n+1}_1}{\nu_1})_{n \in \nset}$ and $(\KLLigne{\pi^{2n}_0}{\nu_0})_{n \in \nset}$ are non-increasing.
  \item \label{item:mono_d} $(\tvnormLigne{\pi^{2n+1}_1-\nu_1})_{n \in \nset}$ and $(\tvnormLigne{\pi^{2n}_0-\nu_0})_{n \in \nset}$ are non-increasing.     
  \end{enumerate}
\end{proposition}

First, we show that under \Cref{assum:existence}, the IPF sequence is
well-defined and is associated with a sequence of potentials.

\begin{proposition}
  \label{prop:well_def}
  Assume \rref{assum:existence}. Then, the IPF sequence is well-defined and
  there exist $(a_n)_{n \in \nset}$ and $(b_n)_{n \in \nset}$ such that for any
  $n \in \nset$, $a_n, b_n :\ \rset^d \to \ooint{0,+\infty}$ and for any
  $x, y \in \rset^d$
\begin{align}
  \label{eq:pi_form}
  &\textstyle{
    (\rmd \pi^{2n+1}/ \rmd (\mu_0 \otimes \mu_1))(x,y) = a_{n}(x) h(x,y) b_n(y) } \\
  &\textstyle{(\rmd \pi^{2n+2}/ \rmd (\mu_0 \otimes \mu_1))(x,y) = a_{n+1}(x) h(x,y) b_n(y)  , }
\end{align}
and
\begin{align}
  \label{eq:iterative_schro}
  \textstyle{
  v_0(x) = a_{n+1}(x) \int_{\rset^d} h(x,y) b_n(y) \rmd \mu_1(y)  , \quad v_1(y) = b_{n}(y) \int_{\rset^d} h(x,y) a_{n}(x) \rmd \mu_0(x)  ,
  }
\end{align}
where $v_i = \rmd \nu_i / \rmd \mu_i$ for $i \in \{0,1\}$.
\end{proposition}

\begin{proof}
  First, we show that the IPF sequence is well-defined.  Note that $\pi^1$ is
  well-defined since $\KLLigne{\nu_0 \otimes \nu_1}{\mu} < + \infty$. Assume that
  $\{\pi^\ell\}_{\ell=1}^n$ is well-defined.  Using \cite[Theorem
  2.2]{csiszar1975divergence} we have
  \begin{equation}
    \textstyle{\KLLigne{\nu_0 \otimes \nu_1}{\mu} = \KLLigne{\nu_0 \otimes \nu_1}{\pi^{n}} + \sum_{\ell=0}^{n-1} \KLLigne{\pi^{\ell+1}}{\pi^{\ell}}  . }
  \end{equation}
  In particular, $\KLLigne{\nu_0 \otimes \nu_1}{\pi^n} < +\infty$ and
  $\pi^{n+1}$ is well-defined. We conclude by recursion.

  Using \cite[Theorem 3.1]{csiszar1975divergence} and \Cref{assum:existence},
  there exists $(\tilde{b}_n)_{n \in \nset}$ such that for any $n \in \nset$,
  $\tilde{b}_n: \ \rset^d \to \coint{0,+\infty}$ and for any $x, y \in \msa_n$,
  $(\rmd \pi^{2n+1} / \rmd \pi^{2n})(x, y) = \tilde{b}_n(y)$ with
  $\msa_n \in \mcb{\rset^d}$, $\tilde{\pi}(\msa_n)= 0 $ for any $\tilde{\pi}$
  such that $\tilde{\pi}_1 = \nu_1$ and $\KLLigne{\tilde{\pi}}{\pi^{2n}}<+\infty$. In
  particular we have $(\nu_0 \otimes \nu_1)(\msa_n)=0$. Since $\nu_i$ is
  equivalent to $\mu_i$ for any $i \in \{0,1\}$ we have
  $(\mu_0 \otimes \mu_1)(\msa_n) =0$. Similarly, there exists
  $(\tilde{a}_n)_{n \in \nset}$ such that for any $n \in \nset$, 
  $\tilde{a}_n: \ \rset^d \to \coint{0,+\infty}$ and for any $x, y \in \msb_n$,
  $(\rmd \pi^{2n+2} / \rmd \pi^{2n+1})(x, y) = \tilde{a}_{n+1}(x)$ with
  $\msb_n \in \mcb{\rset^d}$ and $(\mu_0 \otimes \mu_1)(\msb_n) = 0$. As a
  result, there exist $(a_n)_{n \in \nset}$ and $(b_n)_{n \in \nset}$ with
  $a_n: \ \rset^d \to \coint{0,+\infty}$ and
  $b_n: \ \rset^d \to \coint{0,+\infty}$ such that for any $n \in \nset$ and
  $x, y \in \rset^d$
\begin{align}
  &\textstyle{
    (\rmd \pi^{2n+1}/ \rmd (\mu_0 \otimes \mu_1))(x,y) = a_{n}(x) h(x,y) b_n(y) } \\
  &\textstyle{(\rmd \pi^{2n+2}/ \rmd (\mu_0 \otimes \mu_1))(x,y) = a_{n+1}(x) h(x,y) b_n(y)  , }
\end{align}
where $h = \rmd \mu / \rmd (\mu_0 \otimes \mu_1)$ and $a_0 = 1$.
In addition, setting $b_{-1} = 1$, we have for any $x, y \in \rset^d$,
\begin{equation}
  (\rmd \pi^{0}/ \rmd (\mu_0 \otimes \mu_1))(x,y) = a_{0}(x) h(x,y) b_{-1}(y)  . 
\end{equation}
Using 
that $\nu_i$ is absolutely continuous w.r.t. $\mu_i$
for $i \in \{0,1\}$ with density $v_i : \ \rset^d \to \ooint{0,+\infty}$ we get
that for any $x, y \in \rset^d$ and $n \in \nset$
\begin{align}
  \textstyle{
  v_0(x) = a_{n+1}(x) \int_{\rset^d} h(x,y) b_n(y) \rmd \mu_1(y)  , \quad v_1(y) = b_{n}(y) \int_{\rset^d} h(x,y) a_{n}(x) \rmd \mu_0(x)  .
  }
\end{align}
Since $v_0, v_1 > 0$ for  any $n \in \nset$, $a_n, b_n > 0$.
\end{proof}

Note that the system of equations \eqref{eq:iterative_schro} corresponds to
iteratively solving the \schro system, see \cite{leonard2014survey} for a
survey. In addition, \eqref{eq:iterative_schro} has connections with Fortet's
mapping
\citep{leonard2019revisiting,fortet1940resolution}. 

In the rest of the section we detail the proof of \Cref{prop:monotonicity}.  We
start by deriving identities between the marginals of the IPF and its joint
distribution both w.r.t. the Kullback-Leibler divergence and the total
variation norm in \Cref{lemma:identities_joint_marginal}. Second, we establish
that $(\tvnormLigne{\pi^{n+1} - \pi^{n}})_{n \in \nset}$ is non-increasing in
\Cref{lemma:tv_non_inc}. Then, we prove \eqref{eq:fundamental_kl_appendix} in
\Cref{lemma:kl_fundamental}. We conclude with the proof of 
\Cref{prop:monotonicity_appendix}.

\begin{lemma}
  \label{lemma:identities_joint_marginal}
  Assume \rref{assum:existence}. 
  Then, for any $n \in \nset$ we have
  \begin{equation}
    \label{eq:tv_eq}
    \tvnormLigne{\pi^{2n+1} - \pi^{2n}} = \tvnormLigne{\pi^{2n}_1 - \nu_1}  , \qquad \tvnormLigne{\pi^{2n+2} - \pi^{2n+1}} = \tvnormLigne{\pi^{2n+1}_0 - \nu_0}  .
  \end{equation}
  In addition, we have
  \begin{equation}
    \label{eq:kl_eq}
    \KLLigne{\pi^{2n}}{\pi^{2n+1}} = \KLLigne{\pi^{2n}_1}{\nu_1}  , \qquad \KLLigne{\pi^{2n+1}}{\pi^{2n+2}} = \KLLigne{\pi^{2n+1}_0}{\nu_0}  .
  \end{equation}
\end{lemma}

\begin{proof}
  We divide the proof into two parts. First, we prove \eqref{eq:tv_eq}. Second,
  we show that \eqref{eq:kl_eq} holds.
    \begin{enumerate}[label=(\alph*), wide, labelwidth=!, labelindent=0pt]
    \item We only show that for any $n \in \nset$ we have
      $\tvnormLigne{\pi^{2n+1} - \pi^{2n}} = \tvnormLigne{\pi^{2n}_1 -
        \nu_1}$. The proof that for any $n \in \nset$,
      $\tvnormLigne{\pi^{2n+2} - \pi^{2n+1}} = \tvnormLigne{\pi^{2n+1}_0 -
        \nu_0}$ is similar. Let $n \in \nset$. Using \eqref{eq:pi_form} and
      \eqref{eq:iterative_schro} we have
  \begin{align}
    \label{eq:pi_eq}
    \tvnormLigne{\pi^{2n+1} - \pi^{2n}} &= \textstyle{\int_{(\rset^d)^2}  \abs{b_{n}(y) - b_{n-1}(y)} a_n(x) h(x,y) \rmd \mu_0(x) \rmd \mu_1(y)} \\
                                   &= \textstyle{\int_{\rset^d}  \abs{1 - b_{n-1}(x)/b_{n}(x)} \rmd \nu_1(y)   . }
  \end{align}
  In addition, we have that for any $\msa \in \mcb{\rset^d}$
  \begin{equation}
    \textstyle{
    \pi^{2n}_1(\msa) = \int_{\rset^d \times \msa} a_{n}(x) b_{n-1}(y) h(x,y) \rmd \mu_0(x) \rmd \mu_1(y) = \int_{\msa} (b_{n-1}/b_{n})(y) \rmd \nu_1(y)  .}
  \end{equation}
  We get that for any $y \in \rset^d$,
  $(\rmd \pi_{1}^{2n} / \rmd \nu_1)(y) = (b_{n-1}/b_{n})(y)$. Hence, using
  \eqref{eq:pi_eq} we get that
  \begin{equation}
    \textstyle{
      \tvnormLigne{\pi^{2n}_1 - \nu_1} = \int_{\rset^d}  \abs{1 - a_n(x)/a_{n+1}(x)} \rmd \nu_0(x) = \tvnormLigne{\pi^{2n+1} - \pi^{2n}}  .
      }
  \end{equation}
\item We only show that for any $n \in \nset$ we have
  $\KLLigne{\pi^{2n}}{\pi^{2n+1}} = \KLLigne{\pi^{2n}_1}{\nu_1}$. The proof that for any
  $n \in \nset$, $\KLLigne{\pi^{2n+1}}{\pi^{2n+2}} = \KLLigne{\pi^{2n+1}_0}{\nu_0}$ is
  similar. Let $n \in \nset$. Using that for any $x, y \in \rset^d$,
  $(\rmd \pi^{2n}_1 / \rmd \nu_1)(y) = b_{n-1}(y)/b_{n}(y)$ and that
  $(\rmd \pi^{2n+1} / \rmd \pi^{2n})(x,y) = b_{n}(y) / b_{n-1}(y)$ we have
  \begin{equation}
    \KLLigne{\pi^{2n}}{\pi^{2n+1}} = -\textstyle{\int_{\rset^d} \log(b_{n}(y) / b_{n-1}(y)) \rmd \pi^{2n}_1(y)} = \KLLigne{\pi^{2n}_1}{\nu_1}  .
  \end{equation}
 This concludes the proof.
  \end{enumerate}
\end{proof}

\begin{lemma}
  \label{lemma:tv_non_inc}
  Assume \rref{assum:existence}. Then $(\tvnormLigne{\pi^{n+1} - \pi^{n}})_{n \in \nset}$ is
  non-increasing.
\end{lemma}

\begin{proof}
  We only prove that for any $n \in \nset$ with $n \geq 1$,
  $\tvnormLigne{\pi^{2n+1} - \pi^{2n}} \leq \tvnormLigne{\pi^{2n} - \pi^{2n-1}}$. The
  proof that for any $n \in \nset$,
  $\tvnormLigne{\pi^{2n+2} - \pi^{2n+1}} \leq \tvnormLigne{\pi^{2n+1} - \pi^{2n}}$ is
  similar. Let $n \in \nset$ with $n \geq 1$. 
  Similarly to the proof of \Cref{lemma:identities_joint_marginal} we have that
  \begin{equation}
    \label{eq:tv_norm_joint}
    \textstyle{
      \tvnormLigne{\pi^{2n+1} - \pi^{2n}} = \int_{\rset^d}  \abs{1 - b_{n-1}(y)/b_{n}(y)} \rmd \nu_1(y) = \int_{\rset^d}  \abs{b_{n}^{-1}(y) - b_{n-1}^{-1}(y)}b_{n-1}(y) \rmd \nu_1(y)  .
      }
  \end{equation}
  In addition, we have that for any $y \in \rset^d$
  \begin{equation}
    \textstyle{
      \abs{b_{n-1}^{-1}(y) - b_n^{-1}(y)} \leq v_1^{-1}(y) \int_{\rset^d} h(x,y) \abs{a_{n-1}(x) - a_{n}(x)} \rmd \mu_0(x)  .
      }
\end{equation}
Combining this result and \eqref{eq:tv_norm_joint} we get that
\begin{align}
  \tvnormLigne{\pi^{2n+1} - \pi^{2n}} &\leq \textstyle{\int_{\rset^d}  \abs{b_{n-1}^{-1}(y) - b_n^{-1}(y)}b_{n-1}(y) \rmd \nu_1(y)} \\
  &\leq \textstyle{\int_{(\rset^d)^2} \abs{a_n(x) - a_{n-1}(x)} h(x,y) b_{n-1}(y) \rmd \mu_0(x) \rmd \mu_1(y)} \\
  &\leq \textstyle{\int_{\rset^d} \abs{1 - a_{n-1}(x)/a_{n}(x)} \rmd \nu_0(x) \leq \tvnormLigne{\pi^{2n} - \pi^{2n-1}}}  ,
\end{align}
which concludes the proof.
\end{proof}

\begin{lemma}
  \label{lemma:kl_fundamental}
  Assume \rref{assum:existence}. Then for any $n \in \nset$ with $n \geq 1$ we have
  \begin{equation}
    \KLLigne{\pi^{n+1}}{\pi^n} \leq \KLLigne{\pi^{n-1}}{\pi^n}  , \qquad \KLLigne{\pi^n}{\pi^{n+1}} \leq \KLLigne{\pi^n}{\pi^{n-1}}  . 
  \end{equation}
\end{lemma}

\begin{proof}
  Using \Cref{lemma:identities_joint_marginal} and the data
  processing theorem \cite[Lemma 9.4.5]{ambrosio200gradient} we get that for any $n \in \nset$
  \begin{equation}
    \KLLigne{\pi^{2n}}{\pi^{2n+1}} = \KLLigne{\pi^{2n}_1}{\nu_1} \leq \KLLigne{\pi^{2n}}{\pi^{2n+1}}  . 
  \end{equation}
  Similarly, we get that for any $n \in \nset$, 
  $\KLLigne{\pi^{2n+1}}{\pi^{2n+2}} \leq \KLLigne{\pi^{2n+1}}{\pi^{2n}}$. Hence, we get that
  for any $n \in \nset$, $\KLLigne{\pi^n}{\pi^{n+1}} \leq \KLLigne{\pi^n}{\pi^{n-1}}$.

  In addition, using that for any $n \in \nset$ with $n \geq 1$ and
  $x, y \in \rset^d$, we have that $\pi^{2n+1}_1 = \nu_1$ and
  $(\rmd \pi^{2n+1} / \rmd \pi^{2n})(x, y) = b_{n}(y) / b_{n-1}(y)$ we get for
  any $n \in \nset$ with $n \geq 1$
  \begin{align}
    \label{eq:kl_def_joint}
    \textstyle{\KLLigne{\pi^{2n+1}}{\pi^{2n}} = - \int_{\rset^d} \log(b_{n-1}(y)/b_{n}(y)) \rmd \nu_1(y)  . }
  \end{align}
  Using Jensen's inequality we have for any $n \in \nset$
  \begin{align}
    &\textstyle{- \log(b_{n-1}(y)/b_{n}(y)) \leq -\log \parenthese{\left. \int_{\rset^d} h(x,y) a_n(x) \rmd \mu_0(x) \middle/ \int_{\rset^d} h(x,y) a_{n-1}(x) \rmd \mu_0(x) \right.}} \\
                                & \qquad \leq \textstyle{-\log \parenthese{\left. \int_{\rset^d}  (a_n(x)/a_{n-1}(x)) h(x,y) a_{n-1}(x) \rmd \mu_0(y) \middle/ \int_{\rset^d} h(x,y) a_{n-1}(x) \rmd \mu_0(x) \right.}} \\
    & \qquad \leq \textstyle{- \int_{\rset^d} \log(a_n(x)/a_{n-1}(x)) b_{n-1}(y) h(x,y) a_{n-1}(x) / v_1(y) \rmd \mu_0(x)  . }
  \end{align}
  Combining this result, \eqref{eq:kl_def_joint}, Fubini's theorem and that for
  any $n \in \nset$ with $n \geq 1$ and $x \in \rset^d$,
  $(\rmd \pi^{2n-1}_0 / \rmd \nu_0)(x) = a_{n-1}(x) / a_n(x)$ we get that for any
  $n \in \nset$ with $n \geq 1$
\begin{align}
  \KLLigne{\pi^{2n+1}}{\pi^{2n}} &\textstyle{\leq\int_{(\rset^d)^2} \log(a_{n-1}(x)/a_{n}(x)) a_{n-1}(x) h(x,y) b_{n-1}(y) \rmd \mu_1(y) \rmd \mu_0(x) }\\
  &\textstyle{\leq\int_{(\rset^d)^2} \log(a_{n-1}(x)/a_{n}(x))  (a_{n-1}(x)/a_n(x)) \rmd \nu_0(x) \leq \KLLigne{\pi^{2n-1}_0}{\nu_0}  .} 
\end{align}
Using \Cref{lemma:identities_joint_marginal} (or the data processing theorem)
we get that for any $n \in \nset$ with $n \geq 1$,
$\KLLigne{\pi^{2n+1}}{\pi^{2n}} \leq \KLLigne{\pi^{2n-1}}{\pi^{2n}}$. Similarly, we get
that for any $n \in \nset$,
$\KLLigne{\pi^{2n+2}}{\pi^{2n+1}}\leq \KLLigne{\pi^{2n}}{\pi^{2n+1}}$, which concludes the
proof.
\end{proof}

We now turn to the proof of \Cref{prop:monotonicity_appendix}

\begin{proof}
  First, \eqref{eq:fundamental_kl_appendix} is a direct consequence of
  \Cref{lemma:kl_fundamental}. Using \Cref{lemma:tv_non_inc} we get that
  $(\tvnormLigne{\pi^{n+1} - \pi^{n}})_{n \in \nset}$ is non-increasing.  Since for
  any $\eta_0, \eta_1 \in \Pens(\rset^d)$ we have
  $\JefLigne{\eta_0}{\eta_1} = (1/2)\defEnsLigne{\KLLigne{\eta_0}{\eta_1} +
    \KLLigne{\eta_0}{\eta_1}}$ and using \eqref{eq:fundamental_kl_appendix}, we get that
  $(\JefLigne{\pi_{n+1}}{\pi_{n}})_{n \in \nset}$ is non-increasing which proves
  \Cref{prop:monotonicity_appendix}-\ref{item:mono_a}. \Cref{prop:monotonicity_appendix}-\ref{item:mono_b}
  is a straightforward consequence of \eqref{eq:fundamental_kl_appendix}.
  \Cref{prop:monotonicity_appendix}-\ref{item:mono_c} is a consequence of
  \Cref{lemma:identities_joint_marginal} and
  \Cref{prop:monotonicity_appendix}-\ref{item:mono_a}. Finally,
  \Cref{prop:monotonicity_appendix}-\ref{item:mono_c} is a consequence of
  \Cref{lemma:identities_joint_marginal} and \eqref{eq:fundamental_kl_appendix}.
\end{proof}

Note that we also have that for any $n \in \nset$,
$(\KLLigne{\pi^{2n}}{\pi^{2n+1}})_{n \in \nset}$ and
$(\KLLigne{\pi^{2n+1}}{\pi^{2n+2}})_{n \in \nset}$ are non-increasing.

\subsubsection{Quantitative convergence bounds}
\label{prop:schro_quantitative_appendix:sec}

% Similarly to \Cref{prop:monotonicity:proof}, we consider the static IPF
% recursion: $\pi^0 = \mu \in \Pens_2$ and
% \begin{align}
%   &\pi^{2n+1} = \argmin \ensemble{\KLLigne{\pi}{\pi^{2n}}}{\pi \in \Pens_{2}, \ \pi_1 = \nu_1}  , \\
%   &\pi^{2n+2} = \argmin  \ensemble{\KLLigne{\pi}{\pi^{2n+1}}}{\pi \in \Pens_{2}, \ \pi_0 = \nu_0}  ,
% \end{align}
% where $\nu_0, \nu_1 \in \Pens(\rset^d)$.
In this section we prove the following theorem.

\begin{theorem}
  \label{prop:schro_quantitative_appendix}
  Assume \rref{assum:existence}. Then, the IPF sequence $(\pi^n)_{n \in \nset}$ is well-defined and there
  exists a probability measure $\pi^{\infty}$ such that
  $\lim_{n \to +\infty} \tvnormLigne{\pi^n - \pi^{\infty}} = 0$ and the following
  hold:
\begin{enumerate}[wide, labelwidth=!, labelindent=0pt, label=(\alph*)]
\item $
    \lim_{n \to +\infty} n^{1/2} \defEns{\tvnormLigne{\pi_0^n - \nu_0} + \tvnormLigne{\pi_1^n -\nu_1}} = 0  . $
\item $
    \lim_{n \to +\infty} n \defEns{\KLLigne{\pi_0^n}{\nu_0} + \KLLigne{\pi_1^n}{\nu_1}} = 0  . 
$
  \end{enumerate}
\end{theorem}

We begin with \Cref{lemma:finite_sum} which is an adaption of \cite[Proposition
2.1]{ruschendorf1995convergence}.  Then we state and prove \Cref{lemma:sum_cv}
which is a classical lemma from real analysis.  Combining these two lemmas and
the monotonicity results from \Cref{prop:monotonicity_appendix} conclude the
proof.

\begin{lemma}
  \label{lemma:finite_sum}
  Assume \rref{assum:existence}. 
 Then,  $(\pi^n)_{n \in \nset}$ is well-defined and
  we have $\sum_{n \in \nset} \KLLigne{\pi^{n+1}}{\pi^{n}} < +\infty$.
\end{lemma}
\begin{proof}
  The sequence is well-defined using \Cref{prop:well_def}. 
  In addition, using \cite[Theorem 2.2]{csiszar1975divergence} we have for any
  $n \in \nset$
  \begin{equation}
    \textstyle{
    \KLLigne{\mu^\star}{\pi^0} = \KLLigne{\pi^\star}{\pi^n} + \sum_{k=0}^{n-1} \KLLigne{\pi^{k+1}}{\pi^k}, }
  \end{equation}
  which concludes the proof.
\end{proof}
\begin{lemma}
  \label{lemma:sum_cv}
  Let $(c_n)_{n \in \nset} \in \coint{0,+\infty}^\nset$ a non-increasing
  sequence such that $\sum_{n \in \nset} c_n < +\infty$. Then
  $\lim_{n \to +\infty} c_n n = 0$.
\end{lemma}

\begin{proof}
  Let $\vareps > 0$ and $n_0 \in \nset$ such that for any $n \geq n_0$,
  $\sum_{k =n}^{+\infty} c_k \leq \vareps $. Let $n\in \nset$ with
  $n \geq 2 n_0$. Note that $n - n_0 \geq n/2 \geq n_0$. Therefore we have
  $\vareps \geq (n-n_0) c_n \geq (n/2) c_n$. Hence, for any $n \in \nset$ with
  $n \geq 2n_0$, $c_n n \leq 2 \vareps$, which concludes the proof.
\end{proof}

We now conclude with the proof of \Cref{prop:schro_quantitative_appendix}.

\begin{proof}
  Using \Cref{lemma:finite_sum} and Pinsker's inequality \cite[Equation
  5.2.2]{bakry:gentil:ledoux:2014} we have
  $\sum_{n \in \nset} \tvnormLigne{\pi^{n+1} - \pi^{n}} < +\infty$. For any
  $N \in \nset$, let $S_N = \sum_{n=0}^N \pi^{n+1} - \pi^n = \pi^{N+1} -
  \mu$. Since the space of finite signed measures endowed with
  $\tvnormLigne{\cdot}$ is a Banach space \cite[Theorem
  D.2.7]{douc:moulines:priouret:soulier:2018} we have that $(S_N)_{N \in \nset}$
  converges. Hence there exists a finite signed measure $\pi^{\infty}$ such that
  $\lim_{n \to +\infty} \tvnormLigne{\pi^n - \pi^{\infty}} = 0$. $\pi^\infty$ is
  a probability measure since for any $n \in \nset$, $\pi^n$ is a probability
  measure.

  In addition, since $(\KLLigne{\pi^{2n+1}}{\pi^{2n}})_{n \in \nset}$ and
  $(\KLLigne{\pi^{2n+2}}{\pi^{2n+1}})_{n \in \nset}$ are non-increasing by
   \Cref{prop:monotonicity_appendix}, using \Cref{lemma:sum_cv}, we get that
   \begin{equation}
     \lim_{n \to +\infty} n \defEns{\KLLigne{\pi_0^n}{\nu_0} + \KLLigne{\pi_1^n}{\nu_1}} = 0 .    
   \end{equation}
   We conclude upon using Pinsker's inequality \cite[Equation
  5.2.2]{bakry:gentil:ledoux:2014}.
\end{proof}

\subsection{Proof of \Cref{prop:convergence_ipf}}
\label{prop:convergence_ipf:proof}

Similarly to \Cref{prop:monotonicity:proof}, we consider the static  IPF recursion: $\pi^0 = \mu \in \Pens_2$ and 
\begin{align}
  &\pi^{2n+1} = \argmin \ensemble{\KLLigne{\pi}{\pi^{2n}}}{\pi \in \Pens_{2}, \ \pi_1 = \nu_1}  , \\
  &\pi^{2n+2} = \argmin  \ensemble{\KLLigne{\pi}{\pi^{2n+1}}}{\pi \in \Pens_{2}, \ \pi_0 = \nu_0}  ,
\end{align}
where $\nu_0, \nu_1 \in \Pens(\rset^d)$.
We recall that in this context that if the \schro bridge $\pi^\star$ exists it is given by
\begin{equation}
  \label{eq:schro_appendix}
  \pi^\star = \argmin \ensembleLigne{\KLLigne{\pi}{\mu}}{\pi \in \Pens_2,\ \pi_0 = \nu_0,\ \pi_1 = \nu_1}  .
\end{equation}
In this section, we prove the following proposition which directly implies
\Cref{prop:convergence_ipf}.

\begin{proposition}
  \label{prop:convergence_ipf_appendix}
  Assume \rref{assum:existence} and denote
  $h = \rmd \mu / (\rmd \mu_0 \otimes \mu_1)$. Assume that
  $h \in \rmc(\rset^d \times \rset^d, \ocint{0, +\infty})$ and that there exist
  $\Phi_0, \Phi_1 \in \rmc(\rset^d, \ooint{0,+\infty})$ such that for any $x, y \in \rset^d$
  \begin{align}    
    &h(x,y) \leq \Phi_0(x) \Phi_1(y)  , \text{~and} \\
    &\textstyle{
      \int_{\rset^d \times \rset^d} (\abs{\log h(x_0,x_1)}+ \abs{\log \Phi_0(x_0) } + \abs{\log \Phi_1(x_1)}) \rmd \mu_0(x_0) \rmd \mu_1(x_1) < +\infty  .
      } \label{eq:condition_beurling}
    \end{align}
    Then there exists a solution
  $\pi^\star$ to the \schro bridge and the IPF sequence satisfies
  $\lim_{n \to +\infty} \tvnormLigne{\pi^n - \pi^\infty} = 0$ with
  $\pi^\infty \in \Pens_2$. If $\mu$ is absolutely continuous w.r.t.
  $\pi^\infty$ then $\pi^\infty = \pi^\star$.
\end{proposition}

We begin with  an adaptation of \cite[Proposition 2]{ruschendorf1993note}.
\begin{proposition}
  \label{prop:ruschendorf_prod}
  Let $\mu \in \Pens_2$ and assume that $\mu$ is absolutely continuous
  w.r.t. $\mu_0 \otimes \mu_1$.  Let $(a_n)_{n \in \nset}$ and
  $(b_n)_{n \in \nset}$ such that for any $n \in \nset$,
  $a_n: \ \rset^d \to \ooint{0,+\infty}$ and $b_n: \ \rset^d \to
  \ooint{0,+\infty}$. Assume that there exists
  $\Phi : \ (\rset^d)^2 \to \coint{0, +\infty}$ and $\msa \in \mcb{\rset^d} \otimes \mcb{\rset^d}$ with
  $\mu(\msa) = 1$ such that for any $(x,y) \in \msa$
  \begin{equation}
    \lim_{n \to +\infty} a_n(x) b_n(y) = \Phi(x,y)  . 
  \end{equation}
  Then, there exist $a: \ \rset^d \to \coint{0,+\infty}$,
  $b: \ \rset^d \to \coint{0,+\infty}$ and $\msb \in \mcb{\rset^d} \otimes \mcb{\rset^d}$ with
  $\mu(\msb) = 1$ such that for any $x, y \in \msb$
  \begin{equation}
    \Phi(x,y) = a(x) b(y)  , \qquad \text{or  } \quad  \Phi(x,y) =0  . 
  \end{equation}
\end{proposition}

\begin{proof}
  Let $\tilde{\msa} = \ensembleLigne{(x,y) \in (\rset^d)^2}{\Phi(x,y) =0}$ and
  $\msa_a = \tilde{\msa} \cap \msa$ and
  $\msa_b = \tilde{\msa}^\complementary \cap \msa$. If $\msa_b = \emptyset$, we
  conclude the proof. Otherwise, let $(x_0, y_0) \in \msa_b$. Let
  $\msc_0, \msc_1 \in \mcb{\rset^d} \otimes \mcb{\rset^d}$ be given by
  \begin{align}
    \label{eq:construction_0}
    \msc_0^0 &= \ensembleLigne{x \in \rset^d}{\lim_{n \to +\infty} a_n^0(x) = a^0(x) \ \text{exists and $a^0(x) > 0$}}  , \\
    \msc_1^0 &= \ensembleLigne{y \in \rset^d}{\lim_{n \to +\infty} b_n^0(y) = b^0(y) \ \text{exists and $b^0(y) > 0$}}  ,
  \end{align}
  where for any $n \in \nset$ and $x,y \in \rset^d$, $a_n^0(x) = a_n(x) / a_n(x_0)$
  and $b_n^0(y) = b_n(y) a_n(x_0)$, which is well-defined since for any $n \in \nset$, $a_n(x_0) >
  0$. Note that $x_0 \in \msc_0^0$ and that $y_0 \in \msc_1^0$. If
  $\msa_b \subset \msc_0^0 \times \msc_1^0$, we conclude the proof. Otherwise,
  let $(x_1, y_1) \in \msa_b \cap (\msc_0^0 \times \msc_1^0)^\complementary$ and define
  \begin{align}
    \msc_0^1 &= \ensembleLigne{x \in \rset^d}{\lim_{n \to +\infty} a_n^1(x) = a^1(x) \ \text{exists and $a^1(x) > 0$}}  , \\
    \msc_1^1 &= \ensembleLigne{y \in \rset^d}{\lim_{n \to +\infty} b_n^1(y) = b^1(y) \ \text{exists and $b^1(y) > 0$}}  ,
  \end{align}
  where for any $n \in \nset$ and $x,y \in \rset^d$, $a_n^1(x) = a_n(x) / a_n(x_1)$
  and $b_n^1(y) = b_n(y) a_n(x_1)$, which is well-defined since for any
  $n \in \nset$, $a_n(x_1) > 0$. Note that $\msc_0^0 \cap \msc_0^1 = \emptyset$
  and $\msc_1^0 \cap \msc_1^1 = \emptyset$. Indeed, if there exists
  $x \in \msc_0^0 \cap \msc_0^1$, then
  $a^0(x) = \lim_{n \to +\infty} a_n(x) / a_n(x_0) > 0$ and
  $a^1(x) = \lim_{n \to +\infty} a_n(x) / a_n(x_1) > 0$ exists. Therefore
  $\lim_{n \to +\infty} a_n(x_1) / a_n(x_0) > 0$ exists and
  $\lim_{n \to +\infty} b_n(y_1) a_n(x_0) > 0$ exists. Hence
  $(x_1, y_1) \in \msc_0^0 \times \msc_1^0$ which is absurd. Similarly, if there
  exists $y \in \msc_1^0 \cap \msc_1^1$ then
  $(x_1, y_1) \in \msc_0^0 \times \msc_1^0$ which is absurd.  Hence, we consider
  $T: \msa_b \to 2^{(\rset^d)^2}$ such that for any $(x,y) \in \msa_b$,
  $T(x,y) = \msc_0^{(x,y)} \times \msc_1^{(x,y)}$, where
  $\msc_0^{(x,y)} \times \msc_1^{(x,y)}$ is constructed as in
  \eqref{eq:construction_0} replacing $(x_0, y_0)$ by $(x,y)$.

  Consider a well order on $(\msa_b, \leq)$, which is possible by the
  well-ordering principle \cite[p. 196]{enderton1977elements}. For any
  $(x,y) \in \rset^d$, let
  $\msa_b^{(x,y)} = \ensembleLigne{(x',y') \in (\rset^d)^2}{(x',y') < (x,y)}$. Using
  the transfinite recursion theorem \cite[p. 175]{enderton1977elements} there
  exists $f: \ \msa_b \to \{0, 1\}$ such that for any $(x,y) \in \msa_b$ if
  there exists $(x', y') \in (\rset^d)^2$ such that $(x',y') < (x,y)$, $f(x',y') = 1$
  and $(x,y) \in T(x',y')$ then $f(x,y) = 0$ and $f(x,y) = 1$ otherwise. Let
  $I = f^{-1}(\{1\})$. Let $(x,y), (x',y) \in I$ with $(x,y) \neq (x',y')$ then
  for $(x,y) < (x',y')$ for instance. Since $f(x,y) = f(x',y') = 1$ we have that
  $(\msc_0^{(x,y)} \times \msc_1^{(x,y)}) \cap (\msc_0^{(x',y')} \times
  \msc_1^{(x',y')}) = \emptyset$. Let $(x,y) \in \msa_b$. If $f(x,y) = 1$ then
  $(x,y) \in \msc^{(x,y)}_0 \times \msc^{(x,y)}_1$. If $f(x,y) = 0$ then there
  exists $(x',y') < (x,y)$ such that
  $(x,y) \in \msc^{(x',y')}_0 \times \msc^{(x',y')}_1$. Therefore, we get that
  $\ensembleLigne{\msc^{(x,y)} = (\msc_0^{(x,y)} \times \msc_1^{(x,y)}) \cap
    \msa_b}{(x,y) \in I}$ is a partition of $\msa_b$.

  Since $\mu(\msa_b) \leq 1$, and
  $\ensembleLigne{\msc^{(x,y)} = (\msc_0^{(x,y)} \times \msc_1^{(x,y)}) \cap
    \msa_b}{(x,y) \in I}$ is a partition of $\msa_b$, we get that
  $J = \ensembleLigne{\msc^{(x,y)}}{(x,y) \in I,
    \mu_0(\msc^{(x,y)}_0)\mu_1(\msc^{(x,y)}_1) > 0}$ is countable. Denote
  $\msa_c = \cup_{(x,y) \in J} \msc^{(x,y)}$. Let us show that
  $\mu(\msa_c^\complementary \cap \msa_b) = \mu(\cup_{(x,y) \in I \cap
    J^\complementary} \msc^{(x,y)}) = 0$. Let $x \in \rset^d$ and define
  $\msd_x = \ensembleLigne{y \in \rset^d}{(x,y) \in \msa_b \cap
    \msa_c^\complementary}$. If $\msd_x$ is not empty, then there exists
  $(x',y') \in I$ such that $x \in \msc_0^{(x',y')}$. Then, for any
  $y \in \msd_x$, $y \in \msc_1^{(x',y')}$. Hence,
  $(x', y') \in I \cap J^{\complementary}$ by definition of $\msd_x$ and
  $\mu_1(\msd_x) = 0$.  We get that
  \begin{equation}
    \textstyle{
      \mu(\msa_b \cap \msa_c^\complementary) = \int_{\rset^d} \parenthese{\int_{\msd_x} h(x,y) \rmd \mu_1(y)} \rmd \mu_0(x)  = 0 ,
      }
  \end{equation}
  where $h$ is the density of $\mu$ w.r.t. $\mu_0 \otimes \mu_1$. Note that this
  is the only instance in the proof, where we use that $\mu$ is absolutely
  continuous w.r.t. $\mu_0 \otimes \mu_1$. For any $(x,y) \in \msa_c$ define for
  any $n \in \nset$
  \begin{equation}
    \textstyle{
      \hat{a}_n(x) = \sum_{(x',y') \in J} \1_{\msc_0^{(x',y')}}(x) a_n^{(x',y')}(x)  , \quad \hat{b}_n(y) = \sum_{(x',y') \in J} \1_{\msc_1^{(x',y')}}(x) b_n^{(x',y')}(y)  .
      }
  \end{equation}
  There exist $\hat{a}, \hat{b}: \ \rset^d \to \ooint{0, +\infty}$ such that for
  any $(x,y) \in \msa_c$, $\lim_{n \to +\infty} \hat{a}_n(x) = \hat{a}(x)$ and
  $\lim_{n \to +\infty} \hat{b}_n(y) = \hat{b}(y)$. In addition, for any
  $(x,y) \in \msa_c$, $a_n(x) b_n(y) = \hat{a}_n(x) \hat{b}_n(y)$. Hence, for
  any $(x,y) \in \msa_c$, $\Phi(x,y) = \hat{a}(x) \hat{b}(y)$. Since
  $\msa_a \cap \msa_c = \emptyset$ and $\mu(\msa_c) = \mu(\msa_b)$, we have
  \begin{equation}
    \mu(\msa_a) + \mu(\msa_c) = \mu(\msa_a) + \mu(\msa_b) = \mu(\msa) = 1  . 
  \end{equation}
  We conclude the proof upon remarking that for any $(x,y) \in \msa_a$,
  $\Phi(x,y) =0$ and for any $(x,y) \in \msa_c$,
  $\Phi(x) = \hat{a}(x) \hat{b}(y)$.
\end{proof}

In what follows we prove \Cref{prop:convergence_ipf_appendix}.

\begin{proof}
  Since $\lim_{n \to +\infty} \tvnormLigne{\pi^n - \pi^{\infty}} = 0$ by
  \Cref{prop:schro_quantitative_appendix} and $\KLLigne{\pi^\infty}{\mu} < +\infty$,
  there exist $\msa$ with $\mu(\msa) = 1$ and
  $\Phi: \ (\rset^d)^2 \to \coint{0,+\infty}$ such that, up to extraction, for any
  $x, y \in \msa$
  \begin{equation}
    \lim_{n \to +\infty} a_n(x)b_n(y) = \Phi(x,y)  ,
  \end{equation}
  and $(\rmd \pi^\infty / \rmd \mu) = \Phi$.  Using
  \Cref{prop:ruschendorf_prod}, there exist
  $a, b: \ \rset^d \to \coint{0,+\infty}$ and $\msb$ with $\pi^\infty(\msb) = 1$
  such that for any $x,y \in \msb$,
  $(\rmd \pi^\infty / \rmd \mu)(x,y) = a(x)b(y)$. Since $\mu$ is absolutely
  continuous w.r.t. $\pi^\infty$, we get that for any $x, y \in \rset^d$,
  $(\rmd \pi^\infty / \rmd (\mu_0 \otimes \mu_1))(x,y) = a(x)b(y)h(x,y)$.
  In addition, the
  \schro bridge $\pi^\star \in \Pens((\rset^d)^2)$ exists, see \cite[Theorem
  3]{ruschendorf1993note}, and there exist
  $a', b': \ \rset^d \to \coint{0, +\infty}$ and $\msb'$ with $\mu(\msb') = 1$
  such that for any $x,y \in \msb'$
    \begin{equation}
      (\rmd \pi^\star / \rmd (\mu_0 \otimes \mu_1))(x,y) = a'(x) b'(y) h(x,y) . 
    \end{equation}
    Let $\Mens_{+, \times}$ be the space of non-negative product measures over
    $\mcb{\rset^d} \otimes \mcb{\rset^d}$. Let
    $\Psibf_{\bar{h}}: \ \Mens_{+, \times} \to \Mens_{+, \times}$ be given for any
    $\lambda = \lambda_0 \otimes \lambda_1 \in \Mens_{+, \times}$ by
    $\Psibf_{\bar{h}}(\lambda) = \Psi_h^\lambda$ where for any
    $\msa, \msb \in \mcb{\rset^d}$
    \begin{equation}
      \textstyle{\Psi_h^\lambda(\msa \times \msb) = \parentheseLigne{\int_{\msa \times \rset^d}\bar{h}(x,y) \rmd \lambda_0(x) \rmd \lambda_1(x)} \parentheseLigne{\int_{\rset^d \times \msb}\bar{h}(x,y) \rmd \lambda_0(x) \rmd \lambda_1(y)}}
    \end{equation}
    where for any $x,y \in \rset^d$,
    $\bar{h}(x,y) = h(x,y) \Phi_0^{-1}(x) \Phi_1^{-1}(y)$.  Note that
    $\bar{h} \in \rmc(\rset^d \times \rset^d, \coint{0, +\infty})$ and is
    bounded.  Hence, using \cite[Theorem 2]{beurling1960automorphism} and
    \eqref{eq:condition_beurling} we get that $\Psibf_{\bar{h}}$ is a bijection.
    Let $\lambda = (a \Phi_0\mu_0, b \Phi_1\mu_1)$ and
    $\lambda'=(a' \Phi_0\mu_0, b'\Phi_1\mu_1)$. Then, since
    $\pi^\star_i = \pi^\infty_i = \nu_i$ for $i \in \{0,1\}$ we get that
    $\Psibf_h(\lambda) = \Psibf_h(\lambda')$. Hence $\lambda = \lambda'$ and
    $\pi^\infty = \pi^\star$ which concludes the proof.
\end{proof}

In \Cref{prop:alternative} we derive an alternative proposition to
\Cref{prop:convergence_ipf_appendix}. We start with the following lemma.

\begin{lemma}
  \label{lemma:schro_bridge_existence}
  Let $\pi^\star \in \Pens_2$ with $\pi^\star_i = \nu_i$ for
  $i \in \{0,1\}$. Assume that $\KLLigne{\pi^\star}{\mu} < +\infty$ and that
  $\rmL^1(\nu_0) \oplus \rmL^1(\nu_1)$ is closed in $\rmL^1(\pi^\star)$. In
  addition, assume that there exist $a, b: \ \rset^d \to \coint{0,+\infty}$ and
  $\msa$ with $\pi^\star(\msa) = 1$ such that for any $(x,y) \in \msa$,
  \begin{equation}
    (\rmd \pi^\star/ \rmd \mu)(x,y) = a(x) b(y)  . 
  \end{equation}
  Then $\pi^\star$ is the \schro bridge.
\end{lemma}

\begin{proof}
  Since $\KLLigne{\pi^\star}{\mu} < +\infty$ we have that
  \begin{equation}
    \textstyle{
      \int_{(\rset^d)^2} \abs{\log(a(x)b(y))} \rmd \pi^\star(x,y) < +\infty  .
      }
  \end{equation}
  Using \cite[Theorem 1]{kober1939} and that $\pi^\star_i = \nu_i$ for
  $i \in \{0,1\}$, we get that
  \begin{equation}
    \label{eq:finite_integral}
    \textstyle{
      \int_{\rset^d} \abs{\log a(x)} \rmd \nu_0(x) + \int_{\rset^d} \abs{\log b(y)} \rmd \nu_1(y)  < +\infty  .
      }
    \end{equation}
  Let $\pi \in \Pens_2$ such that $\pi_i = \nu_i$ for $i \in \{1,2\}$ and
  $\KLLigne{\pi}{\mu} < +\infty$.  Using \eqref{eq:finite_integral}, we have that
  $\int_{(\rset^d)^2} \abs{\log((\rmd \pi^\star / \rmd \mu)(x,y))} \rmd \pi(x,y) <
  +\infty$. Hence, $(\rmd \pi^\star /\rmd \mu) > 0$, $\pi$-almost surely.
  Using this result we have for any $\msa \in \mcb{\rset^d}$
  \begin{align}
    \pi[\msa] &= \textstyle{\int_{\rset^d} \1_{\msa}(x) (\rmd \pi^\star / \rmd \mu)(x) (\rmd \pi^\star / \rmd \mu)(x)^{-1} \rmd \pi(x)} \\
              &= \textstyle{\int_{\rset^d} \1_{\msa}(x) (\rmd \pi^\star / \rmd \mu)(x) (\rmd \pi^\star / \rmd \mu)(x)^{-1} (\rmd \pi / \rmd \mu) (x) \rmd \mu(x)} \\
    &= \textstyle{\int_{\rset^d} \1_{\msa}(x) (\rmd \pi^\star / \rmd \mu)(x)^{-1} (\rmd \pi / \rmd \mu) (x) \rmd \pi^\star(x)}  . 
  \end{align}
  Hence we get that 
  $\rmd \pi / \rmd \pi^\star = (\rmd \pi / \rmd \mu) (\rmd \pi^\star / \rmd
  \mu)^{-1}$.
  In addition, we have that
  \begin{equation}
    \textstyle{\KLLigne{\pi^\star}{\mu} = \int_{\rset^d} \log(a(x)) \rmd \nu_0(x) + \int_{\rset^d} \log(b(y)) \rmd \nu_1(y) = \int_{(\rset^d)^2} \log((\rmd \pi^\star / \rmd \mu)(x,y)) \rmd \pi(x,y) .}
  \end{equation}
We get that
  \begin{equation}
    \textstyle{
    \KLLigne{\pi}{\pi^\star} = \int_{\rset^d} \log((\rmd \pi / \rmd \mu) (\rmd \pi^\star / \rmd
    \mu)(x,y)^{-1}) \rmd \pi(x,y) = \KLLigne{\pi}{\mu} - \KLLigne{\pi^\star}{\mu}  .
    }
\end{equation}
Hence, $\KLLigne{\pi}{\mu} \geq \KLLigne{\pi^\star}{\mu}$ with equality if and only if
$\pi^\star = \pi$. Therefore, $\pi^\star$ is the \schro bridge.
\end{proof}

The following proposition is an alternative to \Cref{prop:convergence_ipf_appendix}.

\begin{proposition}
\label{prop:alternative}
  Assume \rref{assum:existence}. Then there exists a solution $\pi^\star$ to the
  \schro bridge and the IPF sequence $(\pi^n)_{n \in \nset}$ satisfies
  $\lim_{n \to +\infty} \tvnormLigne{\pi^n - \pi^\infty} = 0$ with
  $\pi^\infty \in \Pens_2$. If $\KLLigne{\pi^\infty}{\mu} < +\infty$
  and $\rmL^1(\nu_0) \oplus \rmL^1(\nu_1)$ is closed in $\rmL^1(\pi^\infty)$
  then $\pi^\infty = \pi^\star$.
\end{proposition}

\begin{proof}
  Since $\lim_{n \to +\infty} \tvnormLigne{\pi^n - \pi^{\infty}} = 0$ by
  \Cref{prop:schro_quantitative_appendix} and $\KLLigne{\pi^\infty}{\mu} < +\infty$,
  there exist $\msa$ with $\mu(\msa) = 1$ and
  $\Phi: \ (\rset^d)^2 \to \coint{0,+\infty}$ such that, up to extraction, for any
  $x, y \in \msa$
  \begin{equation}
    \lim_{n \to +\infty} a_n(x)b_n(y) = \Phi(x,y)  ,
  \end{equation}
  and $(\rmd \pi^\infty / \rmd \mu) = \Phi$.  Using
  \Cref{prop:ruschendorf_prod}, there exist $a, b: \ \rset^d \to \coint{0,+\infty}$
  and $\msb$ with $\pi^\infty(\msb) = 1$ such that for any $x,y \in \msb$,
  $(\rmd \pi^\infty / \rmd \mu)(x,y) = a(x)b(y)$. We conclude upon using
  \Cref{lemma:schro_bridge_existence}.
\end{proof}

\section{Geometric convergence rates and convergence to ground-truth}
\label{prop:geom_gaussian:proof}

In this section, we derive geometric convergence rates in
\Cref{sec:geom-conv-rates} in a Gaussian setting. In particular, we provide an
explicit upper-bound on the convergence rate that depends only on the covariance
of the reference measure and the target. In \Cref{sec:convergence_ground}, we
show that DSB (with Brownian reference measure) converges towards the \schro
bridge in a Gaussian setting where the ground-truth is available. In
\Cref{sec:experiments} we show that our implementation actually recovers the
Schr\"{o}dinger bridge in this setting.

\subsection{Geometric convergence rates}
\label{sec:geom-conv-rates}

In the following proposition we show that we recover a geometric convergence
rate in a Gaussian setting and derive intuition from this case study. We set
$N=1$ and assume that for any $x_0, x_N \in \rset^d$ we have
\begin{equation}
  p(x_0, x_N) \propto \exp[-\norm{x_0}^2 +2\alpha \langle x_0, x_N \rangle - \norm{x_N}^2]  ,
\end{equation}
with $\alpha \in \coint{0,1}$. In this case assume that there exists $\beta> 0$ such that the target marginals are
given for any $x_0, x_N \in \rset^d$ by
\begin{equation}
  \pdata(x_0) \propto \exp[-\beta \norm{x_0}^2]  , \qquad \pprior(x_N) \propto \exp[-\beta \norm{x_N}^2]  .
\end{equation}
\begin{proposition}
  \label{prop:geom_gaussian}
  Let $\alpha \in \ooint{0,1}$ and $\beta > 0$. Then the \schro bridge
  $\pi^\star$ exists and there exists $C \geq 0$ (explicit in the proof) such
  that for any $n \in \nset$, $\KLLigne{\pi^\star}{\pi^n} \leq C \kappa^{2n}$,
  with $\kappa < 1$ given by $\kappa = \rho / (1 + \rho)$ and
  $\rho = 2 \alpha / \beta^2$.   In addition, $\pi^\star$ admits a density w.r.t. the Lebesgue measure denoted
  $p^\star$ and given for any $x, y \in \rset^d$ by
  \begin{equation}
    \textstyle{
      p^\star(x,y) = \exp[-\gamma^\star\normLigne{x}^2 + 2\alpha\langle x,y \rangle - \gamma^\star\normLigne{y}^2] / \int_{\rset^d} \exp[-\gamma^\star\normLigne{x}^2 + 2\alpha\langle x,y \rangle - \gamma^\star\normLigne{y}^2] \rmd x \rmd y  ,
      }
  \end{equation}
  with $\gamma^\star = (\beta^2/2)(1 + (1 + 4\alpha^2/\beta^2)^{1/2})$.
\end{proposition}
Remark that if $\beta^2 = 1 -\alpha^2$ then $\gamma^\star$ and $p^\star = p$, \ie \
the IPF leaves $\mu$ invariant.  Note that the performance of the IPF improves
if $\kappa$ is close to $0$, \ie \ if $\rho= 2\alpha / \beta^2$ is close to
$0$. This is the case if $\alpha \approx 0$ (the marginals are almost
independent) or if $\beta \approx +\infty$ (the target distribution is close to
$\updelta_0$), see \Cref{fig:kappa_cv}.  This behavior is in accordance with the limit case where the
marginals are independent or one of the target distribution is a Dirac mass in
which case the IPF converges in two iterations.

\begin{figure}[h]
  \centering
  \includegraphics[width=0.5\linewidth]{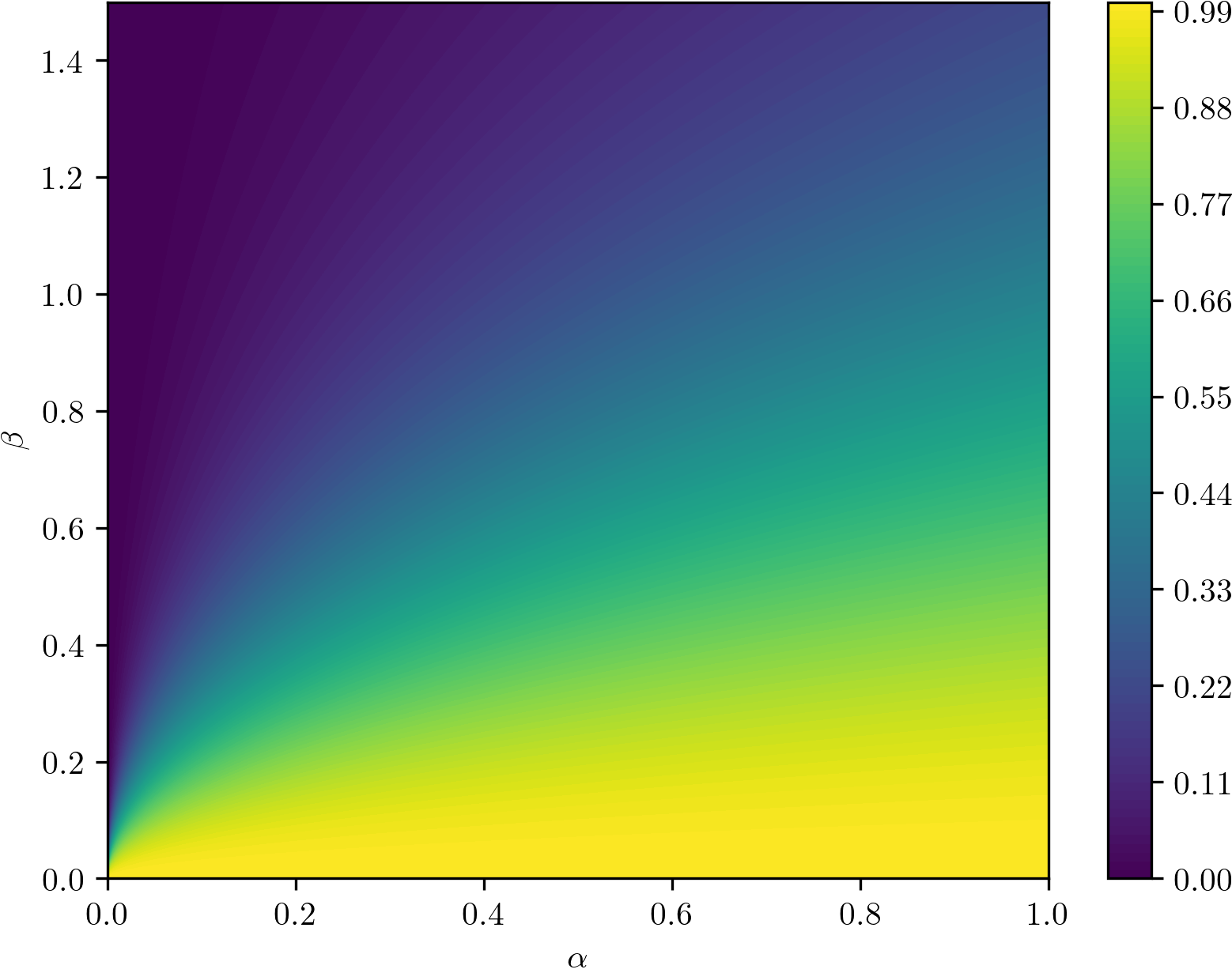}
\caption{Evolution of $\kappa^2$ depending on $\alpha$ and $\beta$.}
\label{fig:kappa_cv}
\end{figure}

Also, note that the convergence rate does not depend on the dimension but only
on the constants of the problem.  In what follows we first derive the IPF
sequence for this Gaussian problem and establish that $\alpha$ controls the
amount of information shared by the marginals. Then we prove
\Cref{prop:geom_gaussian}.  In the rest of this section, we let
$\mu \in \Pens_2$ with density $p$ w.r.t. the Lebesgue measure such that for any
$x_0, x_1 \in \rset^d$
\begin{equation}
    \textstyle{p(x_0, x_1) = \exp[-\norm{x_0}^2 +2\alpha \langle x_0, x_1 \rangle - \norm{x_1}^2] / \int_{\rset^d} \exp[-\norm{x_0}^2 +2\alpha \langle x_0, x_1 \rangle - \norm{x_1}^2] \rmd x_0 \rmd x_1  .}
  \end{equation}
We have that $\mu$ is the Gaussian distribution with zero mean and covariance matrix $\Sigma$ such that
\begin{equation}
  \textstyle{
  \Sigma = (2(1-\alpha^2))^{-1}
  \parenthese{
  \begin{matrix}
    \Id &\alpha \Id \\
    \alpha \Id &\Id
  \end{matrix}}  .
}
\end{equation}
We have that $\det(\Sigma) = 2^{2d}(1 - \alpha^2)^{-d}$ using Schur complement
\cite[Section 9.1.2]{petersen2008matrix}.  Hence we get that for any
$x_0, x_1 \in \rset^d$
\begin{equation}
  \textstyle{p(x_0, x_1) = \uppi^{-d} (1 - \alpha^2)^{d/2} \exp[-\norm{x_0}^2 +2\alpha \langle x_0, x_1 \rangle - \norm{x_1}^2]  . }
\end{equation}
In what follows, we denote $C = \uppi^d (1 - \alpha^2)^{-d/2}$.  Similarly, we
get that $\mu_0 = \mu_1$ and that they admit the density $p_0$ w.r.t. the
Lebesgue measure given for any $x \in \rset^d$ by
\begin{equation}
  p_0(x) = \uppi^{-d/2} (1 - \alpha^2)^{d/2} \exp[-\norm{x}^2(1 - \alpha^2)]  .
\end{equation}
In what follows, we denote $C_0 = \uppi^{d/2}(1 - \alpha^2)^{-d/2}$.  In this
case note that $\mu$ admits a density w.r.t. $\mu_0 \otimes \mu_1$ given for any
$x_0, x_1 \in \rset^d$ by
\begin{equation}
  h(x_0, x_1) = (\rmd \mu/ \rmd (\mu_0 \otimes \mu_1))(x_0,x_1) = (1-\alpha^2)^{-d/2} \exp[-\alpha^2\norm{x_0}^2 -2\alpha\langle x_0, x_1 \rangle -\alpha^2 \norm{x_1}^2]  . 
\end{equation}
Remark that 
$\pprior = \pdata =q$ with for any $x \in \rset^d$,
$q(x) = \uppi^{-d/2} \beta^{d/2} \exp[-\beta\norm{x}^2]$.
We have for any $x_1, x_0 \in \rset^d$
\begin{equation}
  p_{1|0}(x_1|x_0) = p(x_0, x_1) / p_0(x_0) = \uppi^{-d/2} (1 - \alpha^2)^{d/2} \exp[-\alpha^2\norm{x_0}^2 +2\alpha \langle x_0, x_1 \rangle - \norm{x_1}^2]  . 
\end{equation}
Hence, we have that \Cref{assum:existence_spec} holds and the IPF sequence is
well-defined and converges using \Cref{prop:convergence_ipf}.  In what follows
we start to show that $\alpha$ controls the amount of information shared by the
two marginals $\mu_0$ and $\mu_1$, \ie \ the mutual information. More precisely
we have the following result.

\begin{proposition}
  For any $\alpha \in \ooint{0,1}$ we have $\KLLigne{\mu}{\mu_0 \otimes \mu_1}= -(d/2)\log(1- \alpha^2)$.
\end{proposition}

\begin{proof}
  For any $x, y \in \rset^d$ we have
  \begin{equation}
    (\rmd \mu / (\rmd \mu_0 \otimes \rmd \mu_1))(x,y) = \exp[-\alpha^2\normLigne{x}^2 + 2\alpha \langle x,y \rangle -\alpha^2\normLigne{y}^2](1-\alpha^2)^{-d/2}  .
  \end{equation}
We have that
\begin{equation}
  \textstyle{
\int_{\rset^d \times \rset^d} (-\alpha^2 \norm{x}^2   -\alpha^2 \norm{y}^2 + 2\alpha \langle x,y\rangle) \rmd \mu(x,y) = 0  .} 
\end{equation}
Hence, $\KLLigne{\mu}{\mu_0 \otimes \mu_1} = -(d/2)\log(1- \alpha^2)$, which concludes the proof.
\end{proof}

In what follows,
we denote by $(\pi^n)_{n \in \nset}$ the IPFP sequence, defined for any
$n \in \nset$ we have for any $x,y \in \rset^d$
\begin{equation}
  \textstyle{
    (\rmd \pi^{2n} / \rmd \mu)(x,y) = a_n(x) b_n(y) h(x,y)  , \qquad (\rmd \pi^{2n+1} / \rmd \mu)(x,y) = a_{n+1}(x) b_n(y) h(x,y)  ,
    }
\end{equation}
where for any $x, y \in \rset^d$
\begin{align} 
  &\textstyle{a_{n+1}(x) = (\rmd \nu_0 / \rmd \mu_0)(x) \parenthese{\int_{\rset^d} h(x,y) b_n(y) \rmd \mu_1(y)}^{-1}  , }\\
  &\textstyle{b_{n+1}(x) = (\rmd \nu_1 / \rmd \mu_1)(y) \parenthese{\int_{\rset^d} h(x,y) a_{n+1}(x) \rmd \mu_0(x)}^{-1} . }
\end{align}
We now turn to the proof of the \Cref{prop:geom_gaussian}.

\begin{proof}  Let $\alpha \in \ooint{0,1}$ and $\beta > 1$. We have for any $x, y \in \rset^d$
  \begin{equation}
    (\rmd \nu_0 / \rmd \mu_0)(x) = \exp[(1 -\beta^2 - \alpha^2) \normLigne{x}^2]/C_2  , \qquad (\rmd \nu_1 / \rmd \mu_1)(y) = \exp[(1 -\beta^2 - \alpha^2) \normLigne{y}^2]/C_2  ,
  \end{equation}
  with $C_2 = C_1 / C_0$ with
  $C_1 = \uppi^{d/2}\beta^{d/2}$.  For any
  $x \in \rset^d$ and $\gamma \geq 0$ we have
\begin{align}
  &\textstyle{(\rmd \nu_0 / \rmd \mu_0)(x) \parenthese{\int_{\rset^d} \exp[-\gamma \norm{y}^2] h(x,y) \rmd \mu_1(y)}^{-1}} \\
  & \qquad \textstyle{=  (C_0C_2)^{-1}C \exp[(1 - \beta^2 - \alpha^2)\norm{x}^2] \parenthese{\int_{\rset^d} \exp[-\gamma\norm{y}^2 - \norm{y - \alpha x}^2] \rmd y }}^{-1}\\
  & \qquad =  (C_0C_2)^{-1}C \exp[(1 - \beta^2 - \alpha^2)\norm{x}^2] \\
  & \qquad \qquad \textstyle{\times \parenthese{\int_{\rset^d} \exp[-(\gamma+1)\norm{y - \alpha /(\gamma +1)x}^2 -\alpha^2(1 - 1/(\gamma+1)) \norm{x}^2] \rmd y}^{-1} }\\
  & \qquad =  (C_0C_2)^{-1}C \exp[(1 - \beta^2 - \alpha^2+ \alpha^2\gamma/(\gamma+1))\norm{x}^2]\\
  & \qquad \qquad \textstyle{ \times \parenthese{\int_{\rset^d} \exp[-(\gamma+1)\norm{y - \alpha /(\gamma +1)x}^2] \rmd y}^{-1}} \\
  & \qquad = (C_0C_2\tilde{C}_{\gamma})^{-1} C \exp[(1 - \beta^2 - \alpha^2/(\gamma+1))\norm{x}^2]  ,
\end{align}
with $\tilde{C}_{\gamma} = \pi^{d/2} (1 + \gamma)^{-d/2}$.
Note that $a_0 = b_0 =1$. Let $n \in \nset$ and assume that for any
$y \in \rset^d$ $b_{n}(y) = \exp[-\gamma_{2n} \norm{y}^2] /C_{2n}$ with
$\gamma_{2n} \geq 0$ and $C_{2n} > 0$ then we have for any $x \in \rset^d$
\begin{equation}
  a_{n+1}(x) = (C_0 C_2 \tilde{C}_{\gamma_{2n}} )^{-1}C C_{2n} \exp[-(1-\beta^2 -\alpha^2/(\gamma_{2n}+1))\norm{x}^2] = \exp[-\gamma_{2n+1} \norm{x}^2] / C_{2n+1}  ,
\end{equation}
with
\begin{equation}
  \label{eq:recu_1}
  \gamma_{2n+1} = \beta^2 -1 + \alpha^2/(\gamma_{2n}+1)  , \qquad (C_0 C_2 \tilde{C}_{\gamma_{2n}}) / (CC_{2n}) = C_{2n+1}  . 
\end{equation}
Similarly, if we assume that for any
$x \in \rset^d$ $a_{n+1}(x) = \exp[-\gamma_{2n+1} \norm{x}^2] /C_{2n+1}$ with
$\gamma_{2n+1} \geq 0$ and $C_{2n+1} > 0$ then we have for any $y \in \rset^d$
\begin{align}
  b_{n+1}(y) &= (C_0 C_2 \tilde{C}_{\gamma_{2n+1}} )^{-1}(CC_{2n+1}) \exp[-(1-\beta^2 -\alpha^2/(\gamma_{2n+1}+1))\norm{y}^2] \\
  &= \exp[-\gamma_{2n+2} \norm{y}^2] / C_{2n+2}  ,
\end{align}
with
\begin{equation}
  \label{eq:recu_2}
  \gamma_{2n+2} = \beta^2 -1 +  \alpha^2/(\gamma_{2n+1}+1)  , \qquad (C_0 C_2 \tilde{C}_{\gamma_{2n+1}} ) / (CC_{2n+1}) = C_{2n+2}  . 
\end{equation}
Combining this result, \eqref{eq:recu_1} and using the recursion principle we
get that for any $n \in \nset$
\begin{equation}
  \label{eq:expression_potential}
  a_{n+1}(x) = \exp[-\gamma_{2n+1}\norm{x}^2] / C_{2n+1}  , \qquad b_{n+1}(y) = \exp[-\gamma_{2n+2}\norm{y}^2] / C_{2n+2}  .
\end{equation}
The recursion can be extended to $a_0$ and $b_0$ by setting
$\gamma_{-1} = \gamma_0 = 0$ and $C_{-1} = C_0 = 1$. Therefore, for any
$n \in \nset$ we have
\begin{equation}
  \label{eq:recu_gamma}
  \gamma_{n+1} = \beta^2 -1 +  \alpha^2/(\gamma_{n}+1)  .
\end{equation}
We now study the convergence of the sequence $(\gamma_n)_{n \in \nset}$. By
recursion, we have that for any $k, \ell \in \nset$, if
$\gamma_k \geq \gamma_{\ell}$ then for any $m \in \nset$ with $m$ even we have
$\gamma_{m+k} \geq \gamma_{m + \ell}$ and for any $m \in \nset$ with $m$ odd we
have $\gamma_{m+k} \leq \gamma_{m + \ell}$.     We have $\gamma_0 = 0$ and
\begin{equation}
  \label{eq:init}
      \gamma_1 = \beta^2 + \alpha^2 - 1  , \qquad \gamma_2 = \beta^2 - 1 + \alpha^2 / (\beta^2 + \alpha^2)  .
    \end{equation}
    We divide the rest of the proof into three parts.

  \begin{enumerate}[label=(\alph*), wide, labelwidth=!, labelindent=0pt]
  \item First assume that $\beta^2 > 1 - \alpha^2$.  Using \eqref{eq:init} we
    have that $\gamma_1 > \gamma_0$ and $\gamma_2 > \gamma_0$.  Therefore, we
    obtain that $(\gamma_{2n})_{n \in \nset}$ is non-decreasing, that
    $(\gamma_{2n+1})_{n \in \nset}$ is non-increasing and that for any
    $n \in \nset$, $0 \leq \gamma_{2n} \leq \gamma_{2n+1} \leq
    \gamma_1$. Therefore, $(\gamma_n)_{n \in \nset}$ converges and we denote
    $\gamma^\star$ its limit. We have
    $\gamma^\star = \beta^2 -1 + \alpha^2/(\gamma^\star +1)$. Hence,
    $\gamma^\star$ is a root of $X^2 + (2 - \beta^2)X + 1 -\alpha^2
    -\beta^2$. We get that $\gamma^\star = \gamma_0^\star$ or
    $\gamma^\star = \gamma_1^\star$ with
\begin{equation}
\gamma_0^\star = \beta^2/2 -1  - (1/2)\parentheseLigne{\beta^4 + 4\alpha^2}^{1/2}  , \qquad \gamma_1^\star = \beta^2/2 -1 + (1/2)\parentheseLigne{\beta^4 + 4\alpha^2}^{1/2}  , 
\end{equation}
$\gamma_0^\star, \gamma_1^\star$ are non-decreasing function of $\beta$.  We get
that for any $\beta \geq 0$ such that $\beta^2 \geq 1 - \alpha^2$,
$\gamma_0^\star \leq 0$. In addition, we have $\gamma_1^\star = 0$ for
$\beta^2 = 1 -\alpha^2$, hence for any $\beta \geq 0$ such that
$\beta^2 \geq 1- \alpha^2$, $\gamma_1^\star \geq 0$. Since $\gamma^\star \geq 0$
we have
\begin{equation}
  \label{eq:gammastar}
\gamma^\star = -1 + \beta^2/2 + (1/2)\parentheseLigne{\beta^4 + 4\alpha^2}^{1/2}   . 
\end{equation}
For any $n \in \nset$, denote $\xi_n = \gamma_n - \gamma^\star$ and
$\tau = \gamma^\star + 1$. Let $\vareps > 0$.  Since
$\lim_{n \to + \infty} \xi_n = 0$, there exists $n_0 \in \nset$ such that
$\abs{\xi_n}/\tau \leq \vareps$.  Using \eqref{eq:recu_gamma}, we obtain that
for any $n \in \nset$
\begin{align}
  \textstyle{\absLigne{\xi_{n+1}} = \alpha^2\absLigne{1/(\gamma_n+1) - \tau^{-1}} = (\alpha^2/\tau)\absLigne{1 - (\xi_n/\tau +1)^{-1}} \leq (\alpha/\tau)^2\absLigne{\xi_n}/(1 - \vareps)  . }
\end{align}
Hence, we get
that for any $\vareps \in \ooint{0,1}$, there exists $C_{\vareps} > 0$ such that
for any $n \in \nset$
\begin{equation}
  \label{eq:geom}
\abs{\xi_n} \leq C_{ \vareps} \kappa^n , \qquad \kappa = (\alpha/(\tau(1-\vareps)^{1/2}))^{2}  .
\end{equation}
Note that $\tau > \alpha$ using \eqref{eq:gammastar} and
$\kappa \in \ooint{0,1}$ if $\vareps < 1 - \alpha/ \tau$.

For any $n \in \nset$ and $x,y \in \rset^d$ we have
\begin{align}
  \Phi_n(x,y) &= a_{n+1}(x) b_{n+1}(y) = \exp[-\gamma_{2n+1} \norm{x}^2 - \gamma_{2n+2}\norm{y}^2] / (C_{2n+1} C_{2n+2}) \\
  &= \exp[-\gamma_{2n+1} \norm{x}^2 - \gamma_{2n+2}\norm{y}^2] / (\tilde{C} \tilde{C}_{\gamma_{2n+1}})  ,
\end{align}
with $\tilde{C} = C_0 C_2 / C$.
Therefore we obtain that for any $x, y \in \rset^d$,
$\Phi^\star(x,y) = \lim_{n \to +\infty} \Phi_n(x,y)$ exists and we have
\begin{equation}
  \Phi^\star(x,y) = \exp[-\gamma^\star \norm{x}^2 - \gamma^\star\norm{y}^2] / (\tilde{C} \tilde{C}_{\gamma^\star})  . 
\end{equation}
Using this result we get that for any $x,y \in \rset^d$
\begin{align}
  (\rmd \pi^{2n} / \rmd \pi^\star)(x,y) &= \exp[-\xi_{2n+1}\norm{x}^2 - \xi_{2n+2}\norm{y}^2] C_{\gamma^\star} / C_{\gamma_{2n+1}} \\
                                     &= \exp[-\xi_{2n+1}\norm{x}^2 - \xi_{2n+2}\norm{y}^2] \defEns{(1+\gamma_{2n+1})/(1 + \gamma^\star)}^{-d/2} \\
&= \exp[-\xi_{2n+1}\norm{x}^2 - \xi_{2n+2}\norm{y}^2] \defEns{1 + \xi_{2n+1}/(1 + \gamma^\star)}^{-d/2}  .
\end{align}
Therefore we have for any $x,y \in \rset^d$
\begin{align}
  \log\parenthese{(\rmd \pi^{2n} / \rmd \pi^\star)(x,y)} &\leq \abs{\xi_{2n+1}} \norm{x}^2 + \abs{\xi_{2n+2}} \norm{y}^2 + (d/2) \abs{\log\parenthese{1 + \xi_{2n+1}/(1 + \gamma^\star)}} \\
  &\leq \abs{\xi_{2n+1}} \norm{x}^2 + \abs{\xi_{2n+2}} \norm{y}^2 + (d/2)\abs{\xi_{2n+1}}  . 
\end{align}
Therefore we obtain that for any $n \in \nset$
\begin{equation}
  \KLLigne{\pi^{\star}}{\pi^n} \leq (d/2)(\beta^{-2} \abs{\xi_{2n+1}} + \beta^{-2} \abs{\xi_{2n+2}} + \abs{\xi_{2n+1}})  .
\end{equation}
A similar inequality holds for $\KLLigne{\pi^{\star}}{\pi^n}$. Therefore we get
that for any $\vareps \in \ooint{0, 1 - \alpha/\tau}$ there exists
$C_{\vareps} \geq 0$ such that for any $n \in \nset$ we have
\begin{equation}
  \KLLigne{\pi^\star}{\pi^n} \leq C_\vareps \kappa_{\vareps}^{2n}  ,
\end{equation}
with
\begin{align}
  \kappa_{\vareps} &= \alpha / (\tau(1-\vareps)^{1/2}) = (2\alpha) / ((\beta^2 + (\beta^4 + 4 \alpha^2)^{1/2})(1-\vareps)^{1/2}) \\
  &\leq \rho / ((1+(1 +\rho^2)^{1/2})(1- \vareps)^{1/2})  . 
\end{align}
Let $\vareps < 1 - (1 + \rho)/(1 + (1 +\rho^2)^{1/2})$. Then we get that
$\kappa_{\vareps} \leq \kappa$ which concludes the first part of the proof.
\item If $\beta^2 = 1 -\alpha^2$ then the IPF is stationary since
  the IPF leaves $\mu$ invariant.
\item Finally we assume that $\beta^2 < 1 - \alpha^2$. Using \eqref{eq:init} we
  have that $\gamma_1 < \gamma_0$ and $\gamma_2 < \gamma_0$ since
  $\beta^2 < 1 - \alpha^2$.  Therefore, we obtain that
  $(\gamma_{2n})_{n \in \nset}$ is non-increasing, that
  $(\gamma_{2n+1})_{n \in \nset}$ is non-decreasing and that for any
  $n \in \nset$, $0 \geq \gamma_{2n} \geq \gamma_{2n+1} \geq
  \gamma_1$. Therefore, $(\gamma_n)_{n \in \nset}$ converges and we denote
  $\gamma^\star$ its limit. We have
  $\gamma^\star = \beta^2 -1 + \alpha^2/(\gamma^\star +1)$. Hence,
  $\gamma^\star$ is a root of $X^2 + (2 - \beta^2)X + 1 -\alpha^2 -\beta^2$. We
  recall that the two roots of this polynomial are given by
\begin{equation}
\gamma_0^\star = \beta^2/2 -1  - (1/2)\parentheseLigne{\beta^4 + 4\alpha^2}^{1/2}  , \qquad \gamma_1^\star = \beta^2/2 -1 + (1/2)\parentheseLigne{\beta^4 + 4\alpha^2}^{1/2}  .
\end{equation}
We have
\begin{align}  
    \gamma_1 - \gamma_0^\star &= \textstyle{\beta^2 + \alpha^2 - 1 - \beta^2/2 + 1 -(1/2)(\beta^4 + 4\alpha^2)^{1/2} }\\
  &= (1/2)\parentheseLigne{\beta^2 + 2 \alpha^2 - (\beta^4 + 4\alpha^2)^{1/2}} \geq 0  .    
\end{align}
Since $\gamma_3 > \gamma_1$ we get that for any $n \in \nset$ with $n \geq 3$,
$\gamma_n \geq \gamma_3 > \gamma_0^\star$. Therefore
$\gamma^\star > \gamma_0^\star$ and then $\gamma^\star = \gamma_1^\star$. The
rest of the proof is similar to the case where $\beta^2 > 1 - \alpha^2$.

\end{enumerate}
\end{proof}

\subsection{Convergence to ground-truth}
\label{sec:convergence_ground}

In this section, we provide an analytic form for the \schro bridge in a Gaussian
context. Let $\nu_0$ be the $d$ dimensional Gaussian distribution with mean $-a$
(with $a \in \rset^d$) and covariance matrix $\Idbf \in \rset^{d \times d}$.
Similarly, let $\nu_1$ be the one-dimensional Gaussian distribution with mean
$a$ and covariance matrix $\Idbf$.  We consider the reference distribution
$\pi^0$ such that $\pi^0_0 = \nu_0$ and for any $x, y \in \rset^d$
\begin{equation}
  \textstyle{(\rmd \pi_{1|0}^0 / \rmd \Leb)(x,y) = (2 \uppi)^{-d/2} \exp[-\normLigne{x-y}^2/2]  , }
\end{equation} where $\lambda$ denotes the Lebesgue measure on $\rset$.
Note that $\pi_{1|0}^0$ can be obtained by running a $d$-dimensional Brownian motion up to time
$1$. We consider the following \schro bridge problem
\begin{equation}
  \label{eq:schro_gaussian}
  \pi^\star = \argmin \ensembleLigne{\KLLigne{\pi}{\pi^0}}{\pi \in \Pens(\rset^{2d}), \ \pi_0 = \nu_0  , \pi_1 = \nu_1}  . 
\end{equation}
Before giving the analytic solution of the SB problem we consider the following algebraic lemma.

\begin{lemma}
  \label{lemma:linear_algebra}
  Let $A \in \rset^{d \times d}$ and
  \begin{equation}
      M = \parenthese{\begin{matrix}
    \Idbf & A \\
    A^\top  & \Idbf
  \end{matrix}} , \qquad  M^S = \parenthese{\begin{matrix}
    \Idbf & (A+A^\top)/2 \\
    (A+A^\top)/2  & \Idbf
  \end{matrix}}
 ,
  \end{equation}
such that $M$ is symmetric and positive semi-definite. Then $\det(M) \leq \det(M^S)$.
\end{lemma}

\begin{proof}
  Let $M^{\mathrm{up}} = M$ and $M^{\mathrm{down}} = \parenthese{\begin{matrix}
    \Idbf & A^\top \\
    A  & \Idbf
  \end{matrix}}$. Since $M^{\mathrm{up}}$ is symmetric and real-valued, $M^{\mathrm{up}}$ is
diagonalizable. Let $x, y \in \rset^d$ and $\theta \geq 0$ such that
$M^{\mathrm{up}} X = \theta X$ with $X = (x,y)$. Let $Y = (y, x)$. We have
$M^{\mathrm{down}} Y = \theta Y$. Hence $M^{\mathrm{down}}$ is symmetric,
positive semi-definite and
$\det(M^{\mathrm{up}}) = \detLigne{M^{\mathrm{down}}}$. Hence using that
$M \mapsto \log(\detLigne{M})$ is concave on the space of symmetric positive
semi-definite matrices we get that
$\detLigne{M^{\mathrm{up}}}\leq \detLigne{(M^{\mathrm{up}}+
  M^{\mathrm{down}})/2} = \detLigne{M^S}$, which concludes the proof.
\end{proof}

\begin{proposition}
  The solution to \eqref{eq:schro_gaussian} exists and $\pi^\star$ is a Gaussian
  distribution with mean $m \in \rset^{2d}$ and covariance matrix
  $\Sigma \in \rset^{2d \times 2d}$ where
  \begin{equation}
  \textstyle{
    m = (-a, a)  , \qquad   \Sigma = \parenthese{\begin{matrix}
    \Idbf & \beta \Idbf \\
    \beta \Idbf  & \Idbf
  \end{matrix}}  ,}
\end{equation}
where $\beta = (-1 + \sqrt{5}) / 2$ and $\Idbf$ is the $d$-dimensional identity matrix.
\end{proposition}

\begin{proof}
  The fact that $\pi^\star$ exists and is Gaussian is similar to
  \Cref{prop:geom_gaussian}. $\pi^\star$ has mean $m$ since
  $\pi^\star_i = \nu_i$ for $i \in \{0,1\}$. Similarly, we have that
  $\Sigma_{00} = \Sigma_{11} =\Idbf$ since $\pi^\star_i = \nu_i$ for
  $i \in \{0,1\}$.   We have that $\pi^0$ admits a density $p^0$ with
  respect to the Lebesgue measure such that for any $x,y \in \rset$ we have
  \begin{equation}
    p^0(x,y) \propto \exp[-(1/2)\{2\norm{x}^2 + \norm{y}^2 + 2\langle a, x \rangle - 2 \langle x, y \rangle +\norm{a}^2\}]  . 
  \end{equation}
  Hence $\pi^0$ is a Gaussian distribution with mean $m^0$ and covariance matrix
  $\Sigma^0$ where
  \begin{equation}
  \textstyle{
    m^0 = (-a, -a)  , \qquad \Sigma^0  = \parenthese{\begin{matrix}
    \Idbf & \Idbf \\
    \Idbf  & 2\Idbf
  \end{matrix}}  .}
\end{equation}
The Kullback--Leibler divergence between a Gaussian distribution $\pi$, with mean
$\tilde{m}$ and covariance matrix $\tilde{\Sigma}$, and $\pi^0$, with mean $m^0$ and covariance $\Sigma^0$ is given by
\begin{equation}
  \textstyle{ \KLLigne{\pi}{\pi^0} = (1/2)\{ \log(\detLigne{\Sigma^0}/\detLigne{\tilde{\Sigma}}) - d + \trace(\left(\Sigma^0\right)^{-1} \tilde{\Sigma}) + (\tilde{m} - m^0)^\top \left(\Sigma^0\right)^{-1}(\tilde{m}-m^0)\}  .}
\end{equation}
Assume that $\tilde{m} = (-a, a)$ and
$\tilde{\Sigma} = \parenthese{\begin{matrix}
    \Idbf & S \\
    S^\top & \Idbf
  \end{matrix}}$ with $S \in \rset^{d \times d}$ such that $\tilde{\Sigma}$ is
positive semi-definite . Then we have
\begin{equation}
  \KLLigne{\pi}{\pi^0} = (1/2)\{-\log(\detLigne{\tilde{\Sigma}}) -2 \trace(S) + C\}   ,
\end{equation}
where $C \geq 0$ is a constant which does not depend on $\Sigma$.  In what follows, let $\tilde{\Sigma'} = \parenthese{\begin{matrix}
    \Idbf & (S + S^\top)/2 \\
    (S + S^\top)/2 & \Idbf
  \end{matrix}}$ and denote $\pi$ the distribution with mean
$\tilde{m}$ and covariance matrix $\tilde{\Sigma}'$. Using \Cref{lemma:linear_algebra} we have 
\begin{align}
  \KLLigne{\pi'}{\pi^0} &= (1/2)\{-\log(\detLigne{\tilde{\Sigma}'}) -2 \trace(S) + C\} \\
  &\leq (1/2)\{-\log(\detLigne{\tilde{\Sigma}}) -2 \trace(S) + C\} = \KLLigne{\pi}{\pi^0}   .
\end{align}
Hence, we can assume that $S = S^\top$ and therefore (since $S$ is real-valued),
$S$ is diagonalizable. Let $\{\lambda_i\}_{i=1}^d$ the eigenvalues of $S$.
Using Schur complements \cite[Section 9.1.2]{petersen2008matrix} we have 
\begin{equation}
\textstyle{\detLigne{\tilde{\Sigma}} = \detLigne{\Idbf - S^2} = \detLigne{\Idbf - S}\detLigne{\Idbf + S} =  \prod_{i=1}^d(1 - \lambda_i^2)  .  }
\end{equation}
Therefore we have that for any $\lambda \in \ooint{0,1}$
\begin{equation}
  \textstyle{
    \KLLigne{\pi}{\pi^0} = (1/2) \sum_{i=1}^d f(\beta_i) + C  , \qquad f(\lambda) = -\log(1-\lambda^2) - 2\lambda  . 
  }
\end{equation}
Hence we get that $\Sigma_{0,1} = \beta \Idbf$ with $\beta = \argmin_I f$, where
$I = \ooint{-1,0} \cup \ooint{0,1}$.  We have that $f'(\beta) =0$ if and only if
$\beta = (-1 + \sqrt{5})/2$ or $\beta = -(1 + \sqrt{5})/2$. We conclude the
proof using that $\beta \in I$.
\end{proof}

\section{Continuous-time \schro bridges}
\label{prop:continuous_schro:proof}

In this section, we prove \Cref{prop:continuous_schro} in
\Cref{sec:proof-crefpr_continuous} and draw a link between the potential
approach to \schro bridges and DSB in continuous time in
\Cref{sec:dsb-continuous-time}.

\subsection{Proof of \Cref{prop:continuous_schro}}
\label{sec:proof-crefpr_continuous}

We recall the continuous \schro problem is given by
\begin{equation}
  \label{eq:dynamic_schro}
  \textstyle{
    \Pi^\star = \argmin \ensemble{\KLLigne{\Pi}{\Pbb}}{\Pi \in \Pens(\mathcal{C}), \ \Pi_0 = \pdata, \ \Pi_T = \pprior}, \quad T = \sum_{k=0}^{N-1}\gamma_{k+1}.
    }
\end{equation}

In this section, we prove \Cref{prop:continuous_schro}.
We start with the following property which 
can be found in \cite[Proposition 2.3, Proposition 2.10]{leonard2014survey} and
establishes basic properties of dynamic continuous \schro bridges.

\begin{proposition}
  The solution to \eqref{eq:dynamic_schro} exists if and only if the solution to
  the static \schro bridge
  exists.  In addition, if the solution exists
  and $\Pbb$ is Markov then the \schro bridge is Markov.
\end{proposition}

We now turn to the proof of \Cref{prop:continuous_schro}. 
First we highlight that $(\Pi^n)_{n \in \nset}$ is well-defined since its static
counterpart $(\pi_n)_{n \in \nset}$ is well-defined using \Cref{prop:well_def}.
We only prove that for any $n \in \nset$, $(\Pi^{2n+1})^R$ is the path measure
associated with the process $(\bfY_t^{2n+1})_{t \in \ccint{0,T}}$ such that
$\bfY_0^{2n+1}$ has distribution $\pprior$ and satisfies
\begin{equation}
  \textstyle{
    \rmd \bfY_t^{2n+1} = b^{n}_{T-t}(\bfX_t^{2n+1}) \rmd t  + \sqrt{2} \rmd \bfB_t  .
    }
    \end{equation}
    The proof for $\Pi^{2n+2}$ is similar. Let $n \in \nset$ and assume that
    $\Pi^{2n}$ is the path measure associated with the process
    $(\bfX_t^{2n})_{t \in \ccint{0,T}}$ such that $\bfX_0^{2n}$ has
    distribution $\pdata$ and satisfies
    \begin{equation}
      \label{eq:weak_2n}
      \textstyle{
        \rmd \bfX_t^{2n} = f^{n}_t (\bfX_t^{2n}) \rmd t  + \sqrt{2} \rmd \bfB_t  .
        }
    \end{equation}
    We have that
    \begin{equation}
      \textstyle{
        \Pi^{2n+1} = \argmin \ensemble{\KLLigne{\Pi}{\Pi^{2n}}}{\Pi \in \Pens(\mathcal{C}), \ \Pi_T = \pprior}  .
        }
    \end{equation}    
    Let $\phi = \mathrm{proj}_T$ such that for any $\omega \in \mathcal{C}$,
    $\mathrm{proj}_T(\omega) = \omega_T$. Using \Cref{prop:additive} we
    get that for any $\Pi \in \Pens(\mathcal{C})$ we have
    \begin{equation}
      \textstyle{
      \KLLigne{\Pi}{\Pi^{2n}} = \KLLigne{\Pi_T}{\Pi_T^{2n}} + \int_{\rset^d} \KLLigne{\Kker(x, \cdot)}{\Kker^{2n}(x, \cdot)} \rmd \Pi_T(x)  , }
    \end{equation}
    where $\Kker$ and $\Kker^{2n}$ are the disintegrations of $\Pi$ and
    $\Pi^{2n}$ with respect to $\phi$. Therefore, we get that
    $\Pi^{2n+1} = \pprior \Kker^{2n}$. 
    Since $\KLLigne{\Pi^{2n}}{\Qbb}< +\infty$ and $\Pi^{2n}$ is Markov, Using
    \cite[Theorem 4.9]{cattiaux2021time} we get that $(\Pi^{2n})^R = \Pi_T \Kker^{2n}$ satisfies the
    martingale problem associated with the diffusion
    \begin{equation}
      \label{eq:reverse_diff}
      \textstyle{
        \rmd \bfY^{2n}_t = \defEns{-f^{n}_{T-t}(\bfY^{2n}_t) + 2 \nabla \log p_{T-t}^n(\bfY^{2n}_t)} \rmd t  + \sqrt{2} \rmd \bfB_t  .
        }
    \end{equation}
    Since $\Pi^{2n+1} = \pprior \Kker^{2n}$ we get that $\Pi^{2n+1}$ also
    satisfies the martingale problem associated with \eqref{eq:reverse_diff} and
    is Markov which concludes the proof by recursion.

\subsection{IPF in continuous time and potentials}
\label{sec:dsb-continuous-time}

First, we recall that the  IPF $(\Pi^n)_{n \in \nset}$
with $\Pi^0 = \Pbb$ associated with \eqref{eq:forward} and for any $n \in \nset$
\begin{align}
  \textstyle{\Pi^{2n+1}} &= \textstyle{\argmin \ensemble{\KLLigne{\Pi}{\Pi^{2n}}}{\Pi \in \Pens(\mathcal{C}), \ \Pi_T = \pprior}, } \\
  \textstyle{\Pi^{2n+2}} &= \textstyle{\argmin \ensemble{\KLLigne{\Pi}{\Pi^{2n+1}}}{\Pi \in \Pens(\mathcal{C}), \ \Pi_0 = \pdata}.}
\end{align}

In this section, we draw a link between our time-reversal approach and the
potential approach in continuous time. More precisely, we explicit an identity
between the two in \Cref{prop:continuous_schro_potential}.

\begin{proposition}
  \label{prop:continuous_schro_potential}
  Assume \rref{assum:existence_spec} and that there exist
  $\Mbb \in \Pens(\contspace)$, $U \in \rmc^1(\rset^d, \rset)$, $C \geq 0$
  such that for any $n \in \nset$, $x \in \rset^d$,
  $\KLLigne{\Pi^n}{\Mbb} < +\infty$,
  $\langle x, \nabla U(x) \rangle \geq - C(1+\normLigne{x}^2)$ and $\Mbb$ is
  associated with
  \begin{equation}
    \label{eq:diff_q_app}
    \textstyle{
      \rmd \bfX_t = -\nabla U(\bfX_t) \rmd t + \sqrt{2} \rmd \bfB_t, 
      }
    \end{equation}
    with $\bfX_0$ distributed according to the invariant distribution of
    \eqref{eq:diff_q}.  For any $n \in \nset$, let 
    $\{\varphi_t^{n,\star}, \varphi_t^{n,\circ}\}_{t=0}^T$ such that for any
    $t \in \ccint{0,T}$, $\varphi_T^{n,\star}: \ \rset^d \to \rset$, 
    $\varphi_0^{n,\circ}: \ \rset^d \to \rset$, for any $x_0, x_T \in \rset^d$
    \begin{equation}
      \varphi_T^{\star, n}(x_T) = \pprior(x_T) / p_T^n(x_T)  , \qquad  \varphi_0^{\circ, n}(x_0) = \pdata(x_0) / p_0^{n+1}(x_0)  ,
    \end{equation}
    and for any $t \in \ooint{0,T}$ and $x_t \in \rset^d$
    \begin{equation}
      \textstyle{\varphi_t^{\star, n}(x_t) = \int \varphi_T^{\star, n}(x_T) p_{T|t}^n(x_T|x_t) \rmd x_T  , \qquad \varphi_t^{\circ, n+1}(x) = \int \varphi_0^{\circ, n+1}(x_0)  q^{n}_{0|t}(x_0|x_t) \rmd x_0 . }
    \end{equation}
    We have for any $n \in \nset$, $t \in \ccint{0,T}$ and $x_t \in \rset^d$
    \begin{equation}
      \label{eq:potential_backward}
      q_t^{n}(x_t) = p_t^{n}(x_t) \varphi_t^{\star, n}(x_t)  , \qquad p_t^{n+1}(x_t) = q_t^{n}(x_t) \varphi_t^{\circ, n}(x_t)  .
    \end{equation}
    In particular, for any $n \in \nset$ we have
    \begin{enumerate}[wide, labelwidth=!, labelindent=0pt, label=(\alph*)]
      \item $(\Pi^{2n+1})^R$ is associated with $\rmd \bfY_t^{2n+1} = b^{n}_{T-t}(\bfY_t^{2n+1}) \rmd t  + \sqrt{2} \rmd \bfB_t$ with $\bfY_0^{2n+1} \sim \pprior$;
  \item $\Pi^{2n+2}$ is associated with
    $\rmd \bfX_t^{2n+2} = f^{n+1}_t( \bfX_t^{2n+2}) \rmd t + \sqrt{2} \rmd
    \bfB_t$ with $\bfX_0^{2n+2} \sim \pdata$;
      \end{enumerate}
      with for any $x \in \rset^d$ and $t \in \ooint{0,T}$
      \begin{align}
        \label{eq:dynamic_link}
        \textstyle{
        f_t^n(x) = f(x) + 2 \sum_{k=1}^n \nabla \log \varphi_t^{\star,n}(x)  , \ b_t^n(x) = -f(x) + \nabla \log p_t^0(x) + 2 \sum_{k=1}^n \nabla \log \varphi_t^{\circ,n}(x)  .}
      \end{align}
\end{proposition}

\begin{proof}
  We only prove that \eqref{eq:potential_backward} holds. Then
  \eqref{eq:dynamic_link} is a direct consequence of
  \eqref{eq:potential_backward} and \Cref{prop:continuous_schro}. Let
  $n \in \nset$. Similarly to the proof of \Cref{prop:existence_potential},
  there exists $\varphi_T^{\star, n}: \ \rset^d \to \rset_+$ such that for any
  $\{\omega_t\}_{t=0}^T \in \contspace$ we have
  \begin{equation}
    \label{eq:density_cont}
    (\rmd \Pi^{2n+1} / \rmd \Pi^{2n})(\{\omega_t\}_{t=0}^T) = \varphi_T^{\star, n}(\omega_T)  . 
  \end{equation}
  Note that as in \Cref{prop:continuous_schro}, that for any
  $s, t \in \ccint{0,T}$, $\Pi^{2n+1}_{s,t}$ admits a positive density w.r.t the Lebesgue
  measure denoted $q^n_{s,t}$ and $\Pi^{2n}_{s,t}$ admits a positive density w.r.t the Lebesgue
  measure denoted $p^n_{s,t}$. Combining this result and \eqref{eq:density_cont}, we
  get that for any $t \in \ccint{0,T}$ and $x_t, x_T \in \rset^d$ we have
  \begin{equation}
    q_{t,T}^n(x_t, x_T) = p_{t,T}^n(x_t, x_T)\varphi_T^{\star, n}(x_T)  . 
  \end{equation}
  We have that for any $t \in \ccint{0,T}$
  \begin{equation}
    \textstyle{q_t(x_t) = p_t^{n}(x_t)  \int \varphi_T^{\star, n}(x_T) p_{T|t}^n(x_T|x_t) \rmd x_T = p_t^{n}(x_t) \varphi_t^{\star, n}(x_t) .}
  \end{equation}
  The proof for that for any $n \in \nset$, $t \in \ccint{0,T}$ and
  $x_t \in \rset^d$, $p_t^{n+1}(x_t) = q_t^{n}(x_t) \varphi_t^{\circ, n}(x_t)$,
  is similar.
\end{proof}

The link between the two formulations is explicit in
\eqref{eq:potential_backward}. Then, \eqref{eq:dynamic_link} is a
straightforward consequence of \eqref{eq:potential_backward} and should be
compared with \Cref{sec:proof-crefpr_continuous}. Another proof of
\Cref{prop:continuous_schro_potential} make use of a generalization of
\eqref{eq:potential_backward} to joint densities and use the fact that for any $n \in \nset$, $\Pi^{n+1}$ is a Doob $h$-transform of $\Pi^n$ (see \cite[Paragraph 39.1]{rogers2000diffusions} for a definition). Note that this relationship
between the potential and the density of the half-bridge is not new. In
particular, a similar version of this equation can be found in
\cite{bernton2019schr}. In \cite{finlay2020learning}, the authors establish a
similar relationship in the case of the full \schro bridge.

\subsection{Likelihood computation for Schr\"{o}dinger bridges}
\label{sec:likel-comp-schr}

We provide here details on the likelihood computation of generative
models obtained with Schr\"{o}dinger bridges. Under the conditions of
\cite[Theorem 4.12]{leonard2011stochastic}, we define $(\bfX_t^\star)_{t \in \ccint{0,T}}$ the
diffusion associated with $\Pi^\star$, see \eqref{eq:dynamic_schro} as well as
its time reversal, $(\bfY_t^\star)_{t \in \ccint{0,T}}$. There exist
$f^\star, b^\star: \ \ccint{0,T} \times \rset^d \to \rset^d$ such that 
$(\bfX_t^\star)_{t \in \ccint{0,T}}$ and $(\bfY_t^\star)_{t \in \ccint{0,T}}$
are weak solutions to the following SDEs
\begin{equation}
  \rmd \bfX_t^\star = f^\star_t(\bfX_t^\star) \rmd t + \sqrt{2} \rmd \bfB_t  , \qquad  \rmd \bfY_t^\star = b^\star_{T-t}(\bfY_t^\star) \rmd t + \sqrt{2} \rmd \bfB_t  .
\end{equation}
We assume that for any $t \in \ccint{0,T}$ there exists
$p_t^\star : \ \rset^d \to \rset_+$ such that for any $x \in \rset^d$,
$(\rmd \Pi^\star_t / \rmd \Leb)(x) = p_t^\star(x)$. In addition, we assume that
$p^\star \in \rmc^\infty(\ccint{0,T} \times \rset^d, \rset_+)$. In this case, we
have that $\Pi^\star$ is also associated with the process
$(\tbfX_t^\star)_{t \in \ccint{0,T}}$ associated with the ODE
\begin{equation}
  \rmd \tbfX^\star_t= \{f^\star_t(\tbfX_t^\star) - \nabla \log p_t^\star(\tbfX_t^\star) \} \rmd t  ,
\end{equation}
and $\tbfX_T^\star$ has distribution $\pprior$; see e.g. \cite[Section A]{song2020score}. 
Since $(\bfY_t^\star)_{t \in \ccint{0,T}}$ is the time-reversal of
$(\bfX_t^\star)_{t \in \ccint{0,T}}$ we have that for any $t \in \ccint{0,T}$
and $x \in \rset^d$
\begin{equation}
  b^\star_t(x) = -f^\star_t(x) + 2 \nabla \log p_t^\star(x)  . 
\end{equation}
Therefore, we get that $(\tbfX_t^\star)_{t \in \ccint{0,T}}$ is associated with the ODE
\begin{equation}
\label{eq:ode_forward}
  \rmd \tbfX^\star_t = \frac{1}{2}\left(f^\star_t(\tbfX_t^\star) - b^\star_t(\tbfX_t^\star) \right) \rmd t  . 
\end{equation}
Using this result we can compute the log-likelihood of the model using the
instantaneous change of variable formula \citep{chen2018neural}, see also
\cite[Appendix D.2]{song2020score}
\begin{equation}
  \label{eq:neural_ode}
  \textstyle{\log \pdata(\tbfX_0^\star) = \log \pprior(\tbfX_T^\star) + \tfrac{1}{2}\int_0^T \mathrm{div}(f^\star_t - b^\star_t)(\tbfX_t^\star) \rmd t \eqsp . }
\end{equation}
As in \cite{song2020score}, we can use the Skilling--Hutchinson trace estimator
to compute the divergence operator \citep{skilling1989eigenvalues,hutchinson1989stochastic}. In practice, we
discretize the dynamics of $(\tbfX_t^\star)_{t \in \ccint{0,T}}$ and use the
network $B_{\beta^n}$ obtained with the last iterate of \Cref{algo:ipf_score}
and solve the ODE backward in time, recalling that $\tbfX_T^\star$ has
distribution $\pprior$. Similarly, we can define
\begin{equation}
  \rmd \tbfY_t = \{b^\star_{T-t}(\tbfY_t^\star) - \nabla \log p_{T-t}^\star(\tbfY_t^\star) \} \rmd t  ,
\end{equation}
and $\bfY_0^\star$ has distribution $\pprior$. Similarly to \eqref{eq:ode_forward}, we get that $(\tbfY_t^\star)_{t \in \ccint{0,T}}$ is associated with the ODE
\begin{equation}
  \rmd \tbfY_t = \frac{1}{2}\left(b^\star_{T-t}(\tbfY_t^\star)-f^\star_{T-t}(\tbfY_t^\star)\right) \rmd t  . 
\end{equation}
Similarly to \eqref{eq:neural_ode}, we have
\begin{equation}
  \label{eq:neural_ode}
  \textstyle{\log \pdata(\tbfY_T^\star) = \log \pprior(\tbfY_0^\star) + \tfrac{1}{2} \int_0^T \mathrm{div}(b^\star_{T-t} - f^\star_{T-t})(\tbfY_t^\star) \rmd t \eqsp . }
\end{equation}
In practice, we discretize the dynamics of $(\tbfY_t^\star)_{t \in \ccint{0,T}}$
and use the networks $F_{\alpha^n}$, $B_{\beta^n}$ obtained with the last iterate of
\Cref{algo:ipf_score} and solve the ODE forward in time, recalling that $\tbfY_0^\star$ has distribution $\pprior$. Note that in this case, we solve the ODE forward in time contrary to  \cite{durkan2021maximum}.

\section{Training Techniques}
\label{sec:implementation_details}
    \label{sec:arch-deta-addit}

    In this section we present some practical guidelines for the implementation of DSB, based on \Cref{algo:ipf_score}. We emphasize that, contrarily to previous approaches  \cite{song2020score,song2020improved,ho2020denoising,nichol2021beatgans}, we do not weight the loss functions as we do not notice any improvement.  Let
$I \subset \{0, N-1\} \times \{1, M\}$. We define the generalized losses
$\hat{\ell}_{n,I}^b$ and $\hat{\ell}_{n,I}^f$ given by
\begin{align}
\label{eq:scorematchforward_gen}
     &\textstyle{\hlb_{n,I}(\beta) = M^{-1}\sum_{(k,j) \in I} \normLigne{B_\beta(k+1,X^{j}_{k+1})-(X^j_{k+1}+F^{n}_k(X^{j}_{k+1})-F^{n}_{k}(X^{j}_{k}))}^2}  , \\
\label{eq:scorematchbackward_gen}
    &\textstyle{\hlf_{n+1, I}(\alpha) = M^{-1} \sum_{(k,j) \in I} \normLigne{F_\alpha(k,X^j_k)-(X^j_{k}+B^{n}_{k+1}(X^j_{k+1})-B^n_{k+1}(X^j_k))}^2}  . 
\end{align}
We first describe three techniques to compute these losses, then further methods to improve performance.
\vspace{0.2cm}
\begin{technique} Simulated Trajectory \label{tech_1} \end{technique}
The losses \eqref{eq:scorematchforward_gen} and
\eqref{eq:scorematchbackward_gen} may be computed by simulating diffusion
trajectories as described in \Cref{algo:ipf_score}. For each sample $j \in \{1,\ldots, M\}$ the skeleton of points in the sampled trajectory, $\{X_k^j\}_k$, will be correlated hence only a single uniformly sampled time-step per sample is used to compute the loss per gradient step. In addition, after the initial DSB iteration, simulating the diffusion trajectory involves computationally heavy neural network operations per diffusion step.

\vspace{0.2cm}
\begin{technique}Closed Form Sampling \label{tech_2} \end{technique}
  Since $f_\alpha^0(x) = -\alpha x$, with fixed $\alpha$, it is
  not necessary to compute full trajectories for the first IPF iteration and one
  may sample points along the trajectory in closed-form by sampling from a Gaussian distribution with appropriate mean and covariance. This technique also improves the computational speed of the first DSB iteration.

\vspace{0.2cm}
\begin{technique}Cached Trajectory \end{technique}
After the initial DSB iterations it is not possible perform closed form sampling as per \Cref{tech_2}. Simulating the full diffusion trajectory is both wasteful and expensive as described in \Cref{tech_1}. In order to obtain a speed-up we consider a cached-version of \Cref{algo:ipf_score} given by \Cref{algo:ipf_score_cached} which entails storing and then resampling diffusion trajectories. Resampled trajectories are then used to compute losses \eqref{eq:scorematchforward_gen} and \eqref{eq:scorematchbackward_gen}. The cache may be refreshed at a certain frequency by once again simulating the diffusion. One may tune the cache-size and refresh frequency to available memory. This modification allows for significant speed-up as the trajectories are not
  simulated at each training iteration. 
  
  %Let $n_{\mathrm{cache}}$ be the number of times the loss is simulater per each cache and denote batch size by $\abs{I}$. As a rule of thumb we find that, in order to avoid overfitting the cache, it should be refreshed once $n_{\mathrm{cache}} \times \abs{I} \geq m \times N \times M$ where $m \leq 2$.

\begin{algorithm}[h]
    \caption{Cached Diffusion \schro Bridge}
    \label{algo:ipf_score_cached}
    \begin{algorithmic}[1] 
      \FOR{$n \in \{0, \dots,L\}$} 
        \WHILE{not converged}
        \STATE Sample and store $\{X^j_{k}\}_{k,j=0}^{N,M}$ where  $X^j_0 \sim \pdata$ and \\
      $X^{j}_{k+1} =X^{j}_{k} +\gamma_{k+1}
      f_{\alpha^{n}}(k,X^{j}_{k})+\sqrt{2 \gamma_{k+1}} Z^{j}_{k+1}$
      \WHILE{not refreshed}
      \STATE Sample $I$ (uniform in $\{0, N-1\} \times \{1, M\}$)
      \STATE Compute $\hlb_{n, I}(\beta^n)$ using \eqref{eq:scorematchforward_gen}
      \STATE $\beta^{n} = \textrm{Gradient Step}(\hlb_{n, I}(\beta^n))$
      \ENDWHILE
      \ENDWHILE
      \WHILE{not converged}
      \STATE Sample $\{X^j_{k}\}_{k,j=0}^{N,M}$, where $X^j_N \sim \pprior$, and \\
      $X^j_{k}=X^j_{k+1}+\gamma_{k} b_{\beta^n}(k,X^j_{k})+\sqrt{2 \gamma_{k+1}} Z^{j}_{k}$      
      \WHILE{not refreshed}
      \STATE Sample $I$ (uniform in $\{0, N-1\} \times \{1, M\}$)
      \STATE Compute $\hlf_{n+1, I}(\alpha^{n+1})$ using \eqref{eq:scorematchbackward_gen}
      \STATE $\alpha^{n+1} = \textrm{Gradient Step}(\hlf_{n+1, I}(\alpha^{n+1}))$
      \ENDWHILE
      \ENDWHILE
      \ENDFOR
      \STATE \textbf{Output: } $(\alpha^{L+1}, \beta^{L})$
    \end{algorithmic}
  \end{algorithm}
 
 \vspace{0.2cm}
  \begin{technique} \label{sec:gaussian_variance}Tune Gaussian Prior mean/ variance \end{technique} The
  convergence of the IPF is affected by the mean and covariance matrix of
  the target Gaussian. In \Cref{sec:addit-exper_2d} we investigate
  possible choices for these values. In practice we recommend to choose the
  variance of the Gaussian prior $\pprior$ to be slightly larger than the one of
  the target dataset and to choose the mean of $\pprior$ to be equal to the one
  of the target dataset. This remark is in accordance with \cite[Technique
  1]{song2020improved}.

\vspace{0.2cm}
\begin{technique} Network Refinement / Fine Tuning \end{technique}
Training large networks from scratch, per DSB iteration, is very expensive. However, from \eqref{eq:sumscore}-\eqref{eq:sumscore2},
\begin{align}
  &b_{k+1}^n(x) = \textstyle{b_{k+1}^{n-1}(x) + 2\nabla \log p_{k+1}^{n}(x) - 2 \nabla \log q_k^{n-1}(x)  ,} \\
  &f_k^n(x) = \textstyle{f_k^{n-1}(x) + 2 \nabla \log q_k^{n-1}(x) - 2 \nabla \log p_{k+1}^{n-1}(x)  . }
\end{align}
One may therefore initialize networks at DBS iteration $n$ from $n-1$ in order
to reduce training time. In future work, we plan to investigate more
sophisticated warm-start approaches through meta-learning.

\vspace{0.2cm}
\begin{technique} Exponential Moving Average \end{technique} Similar to
\cite[Technique 5]{song2020improved}, we found taking the exponential moving
average of network parameters across training iterations, with rate $0.999$,
improved performance.
 
\section{ Additional Experimental Results and Details} \label{sec:addn_exp} 
We provide additional examples for the two-dimensional setting in \Cref{sec:addit-two-dimens}. We then turn to higher dimensional generative modeling in
\Cref{sec:generative-modeling-1}. Finally, we detail our dataset interpolation
experiments in \Cref{sec:dataset-morphing}. Code is available here:
\href{https://github.com/JTT94/diffusion\_schrodinger\_bridge}{https://github.com/JTT94/diffusion\_schrodinger\_bridge}.

\subsection{Two-dimensional experiments}
  \label{sec:addit-exper_2d}
  \label{sec:addit-two-dimens}
  
  In the case of two-dimensional distributions we use a simple
  architecture for the networks $f_{\alpha}$ and $b_{\beta}$, see
  \Cref{fig:2d_architecture}. We use the variational formulation
  \Cref{sec:drift-match-formulation} because our network architecture does not
  have a residual structure.  To optimize our networks we use ADAM
  \cite{kingma:ba:2014} with momentum $0.9$ and learning rate $10^{-4}$.
  % refresh our optimizers and reinitialize the networks $f_\alpha$ and $b_\beta$
  % at each step of the IPF.
  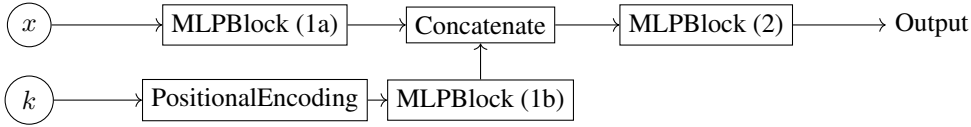
\begin{figure}[H]
  \centering
  \begin{tikzpicture}
    \node (Output) at (0,0) {Output};

    \node[shape=rectangle,draw=black] (MLP2) at ($(Output)+(\xNodeMoins,0)$) {MLPBlock (2)};
    \node[shape=rectangle,draw=black] (Concatenate) at ($(MLP2)+(\xNodeMoins,0)$) {Concatenate};
    \node[shape=rectangle,draw=black] (MLP1a) at ($(Concatenate)+(\xNodeMoins,0)$) {MLPBlock (1a)};
    \node[shape=rectangle,draw=black] (MLP1b) at ($(Concatenate)+(0,\xNodemoinstiny)$) {MLPBlock (1b)};
    \node[shape=rectangle,draw=black] (PosBlock) at ($(MLP1b)+(\xNodeMoins,0)$) {PositionalEncoding};
    \node[shape=circle,draw=black] (Input) at ($(MLP1a)+(\xNodeMoins,0)$) {$x$};
    \node[shape=circle,draw=black] (Time) at ($(PosBlock)+(\xNodeMoins,0)$) {$k$};    

    \path [->] (MLP2) edge (Output);
    \path [->] (Concatenate) edge (MLP2);
    \path [->] (MLP1b) edge (Concatenate);
    \path [->] (MLP1a) edge (Concatenate);
    \path [->] (Input) edge (MLP1a);
    \path [->] (PosBlock) edge (MLP1b);
    \path [->] (Time) edge (PosBlock);
  \end{tikzpicture}
  \caption{Architecture of the networks used in the two-dimensional
    setting. Each MLP Block is a Multilayer perceptron network. The
    ``PositionalEncoding'' block applies the sine transform described in
    \cite{vaswani2017attention}. MLPBlock (1a) has shape $(2,16,32)$, MLPBlock
    (1b) has shape $(1,16,32)$ and MLPBlock has shape $(64,128,128,2)$. The
    total number of parameters is 26498.}
  \label{fig:2d_architecture}
\end{figure}

 In all two-dimensional experiments we fix $\gamma_k = 10^{-2}$ and use a batch
  size of $512$. The mean and variance of $\pprior$ are matched to those of
  $\pdata$. The cache contains $10^4$ samples and is refreshed every $10^3$
  iterations. We train each DSB step for $10^4$ iterations.  All two-dimensional
  experiments are run on Intel(R) Core(TM) i7-10850H CPU @ 2.70GHz CPUs.

  In \Cref{fig:ipf_iteration} we present additional two-dimensional experiments.

\captionsetup[subfigure]{labelformat=empty}
\begin{figure}[h]
  \centering   
  \subfloat{\includegraphics[width=0.12\linewidth]{./fig/circle_gaussian_adaptive/im/9999_backward_1_registration_19.png}}
  \subfloat{\includegraphics[width=0.12\linewidth]{./fig/scurve/im/9999_backward_1_registration_19.png}}
  \subfloat{\includegraphics[width=0.12\linewidth]{./fig/8gaussians/im/9999_backward_1_registration_19.png}}
  \subfloat{\includegraphics[width=0.12\linewidth]{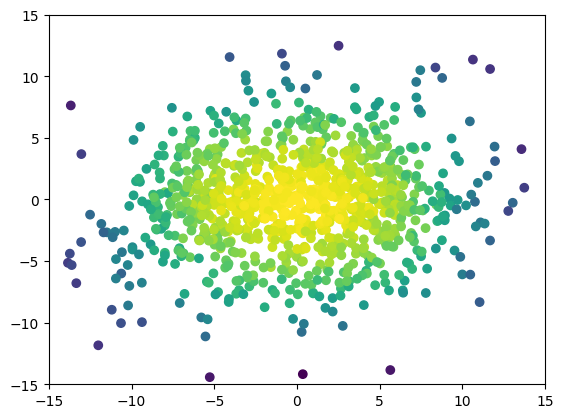}}
  \subfloat{\includegraphics[width=0.12\linewidth]{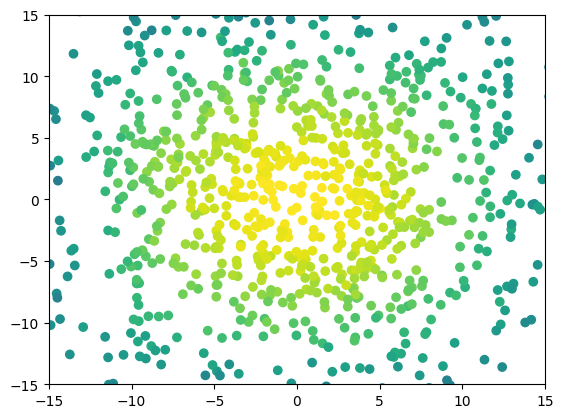}}
  \subfloat{\includegraphics[width=0.12\linewidth]{./fig/moon/im/9999_backward_1_registration_19.png}}
  \subfloat{\includegraphics[width=0.12\linewidth]{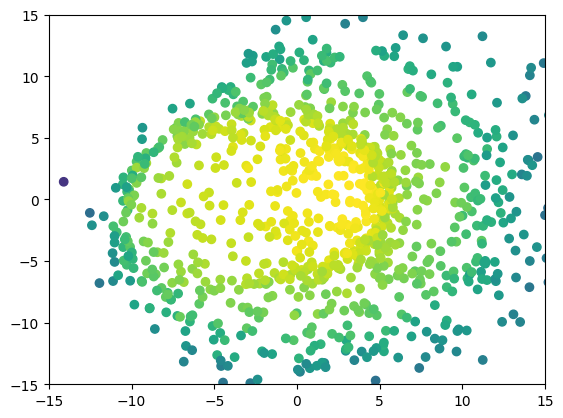}}
  \subfloat{\includegraphics[width=0.12\linewidth]{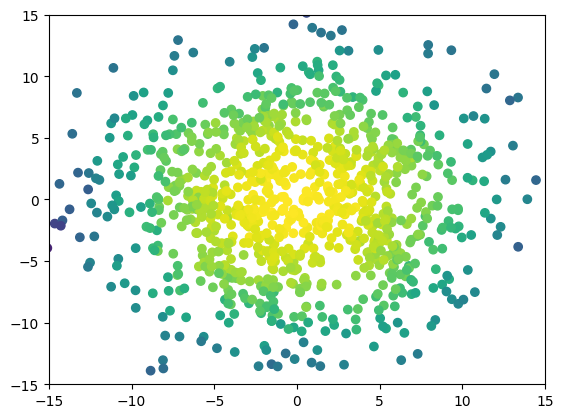}}\\
  \subfloat{\includegraphics[width=0.12\linewidth]{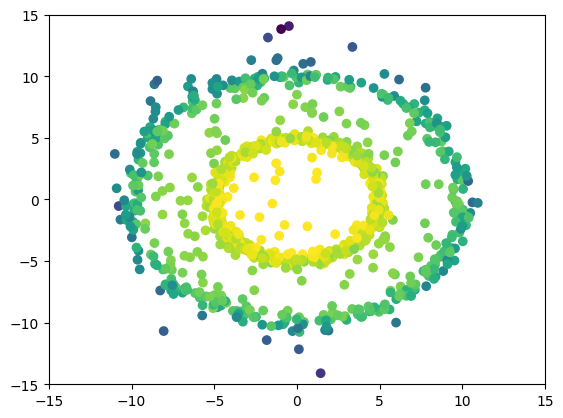}}
  \subfloat{\includegraphics[width=0.12\linewidth]{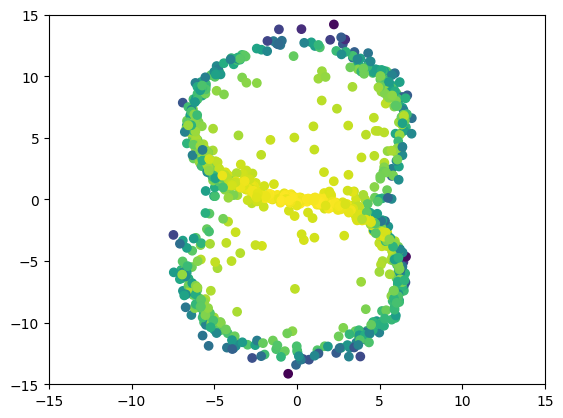}}
  \subfloat{\includegraphics[width=0.12\linewidth]{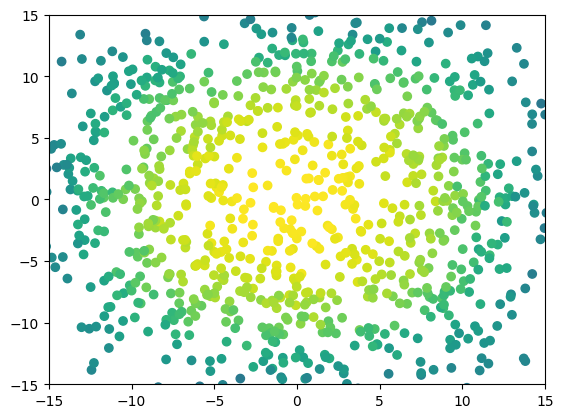}}
  \subfloat{\includegraphics[width=0.12\linewidth]{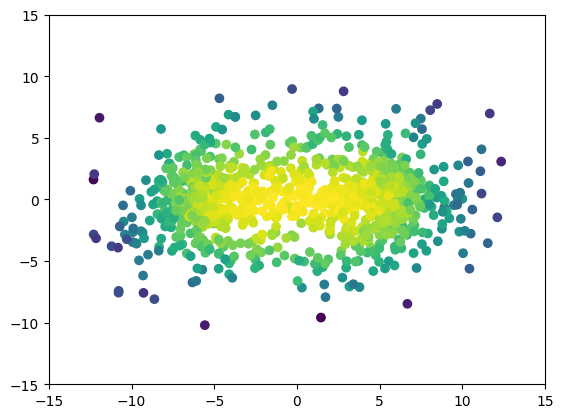}}
  \subfloat{\includegraphics[width=0.12\linewidth]{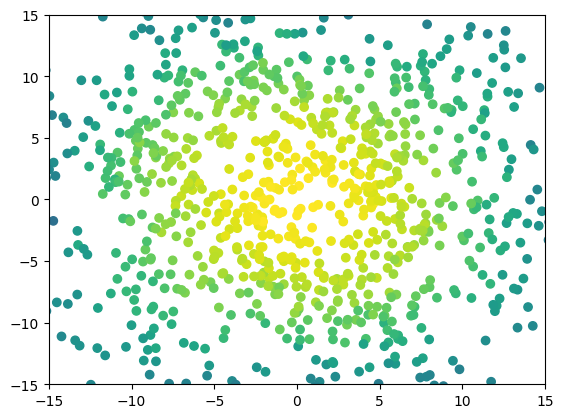}}
  \subfloat{\includegraphics[width=0.12\linewidth]{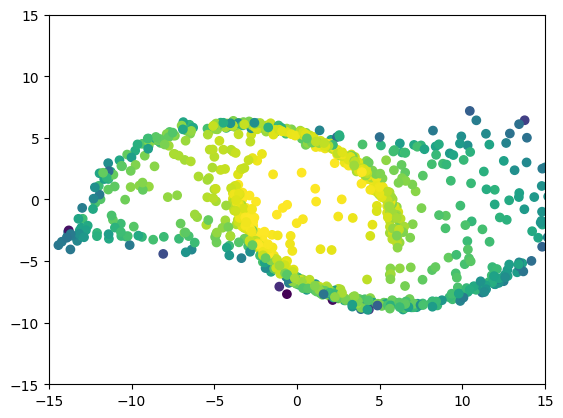}}
  \subfloat{\includegraphics[width=0.12\linewidth]{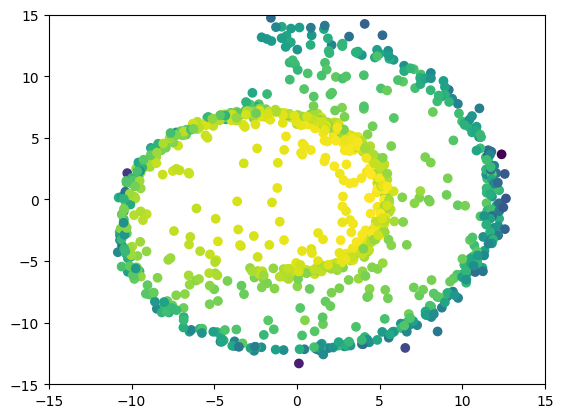}}
  \subfloat{\includegraphics[width=0.12\linewidth]{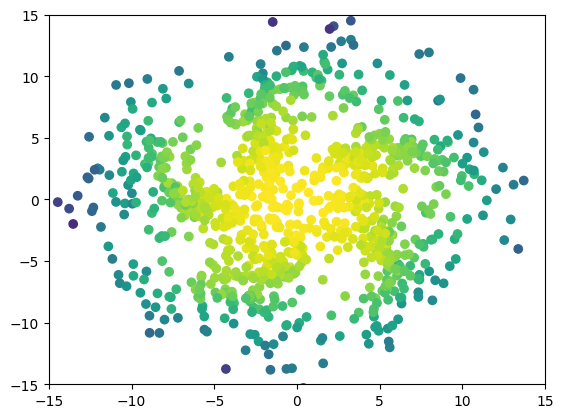}}\\
  \subfloat{\includegraphics[width=0.12\linewidth]{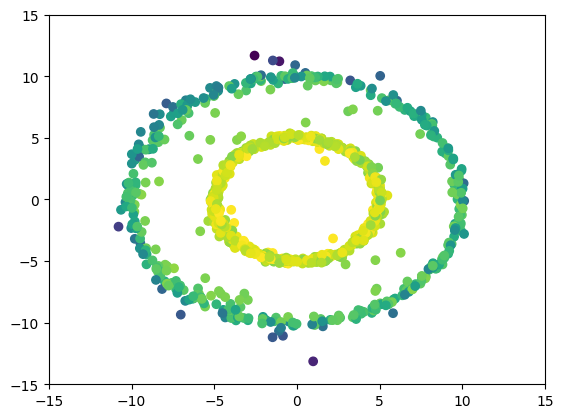}}
  \subfloat{\includegraphics[width=0.12\linewidth]{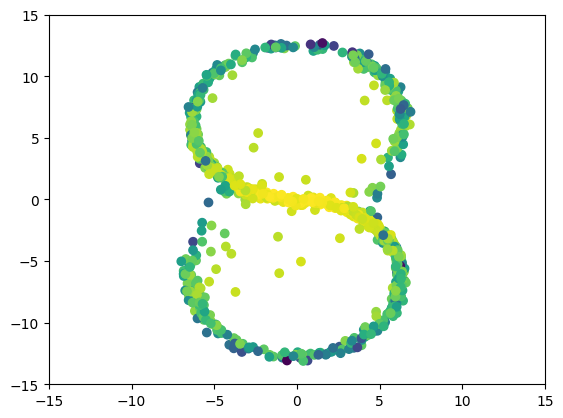}}
  \subfloat{\includegraphics[width=0.12\linewidth]{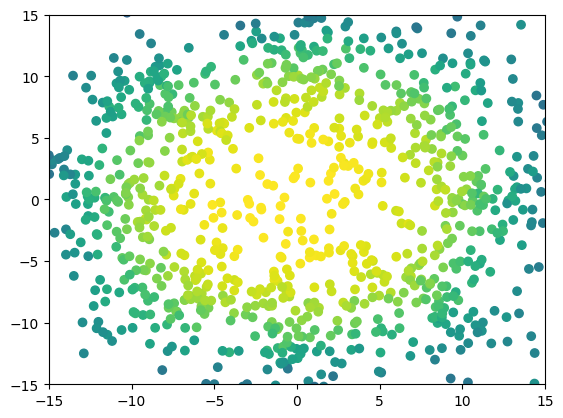}}
  \subfloat{\includegraphics[width=0.12\linewidth]{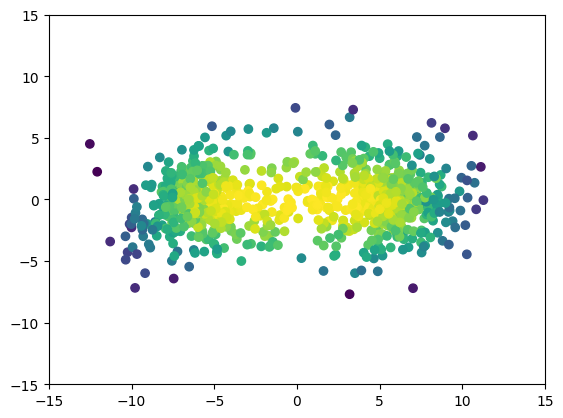}}
  \subfloat{\includegraphics[width=0.12\linewidth]{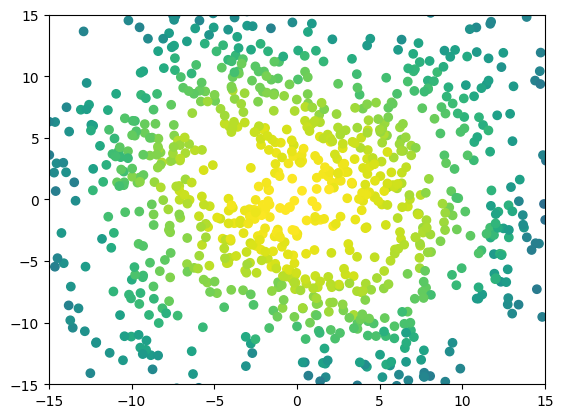}}
  \subfloat{\includegraphics[width=0.12\linewidth]{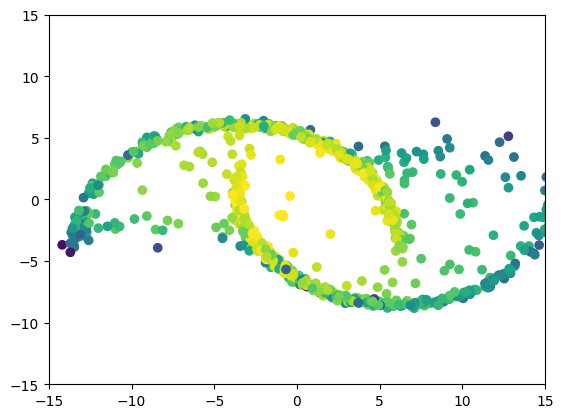}}
  \subfloat{\includegraphics[width=0.12\linewidth]{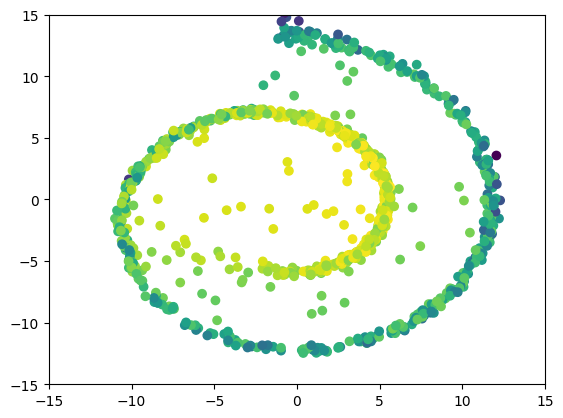}}
  \subfloat{\includegraphics[width=0.12\linewidth]{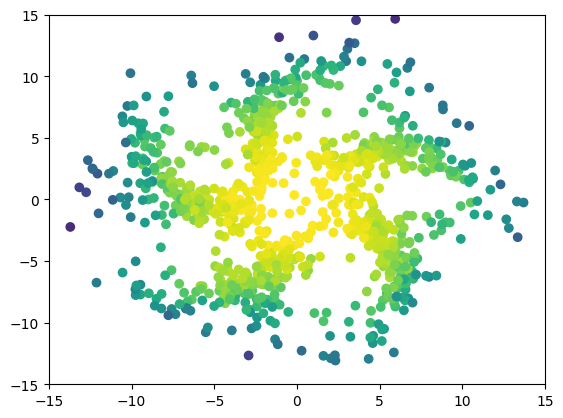}}\\
  \subfloat{\includegraphics[width=0.12\linewidth]{./fig/circle_gaussian_adaptive/im/9999_backward_15_registration_19.png}}
  \subfloat{\includegraphics[width=0.12\linewidth]{./fig/scurve/im/9999_backward_15_registration_19.png}}
  \subfloat{\includegraphics[width=0.12\linewidth]{./fig/8gaussians/im/9999_backward_20_registration_19.png}}
  \subfloat{\includegraphics[width=0.12\linewidth]{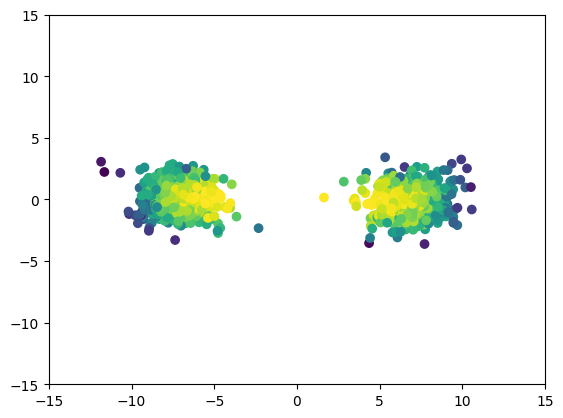}}
  \subfloat{\includegraphics[width=0.12\linewidth]{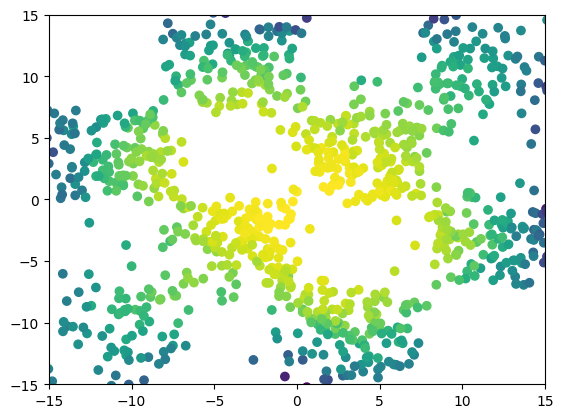}}
  \subfloat{\includegraphics[width=0.12\linewidth]{./fig/moon/im/9999_backward_20_registration_19.png}}
  \subfloat{\includegraphics[width=0.12\linewidth]{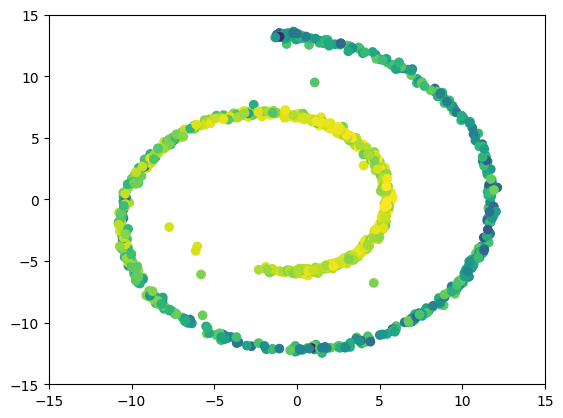}}
  \subfloat{\includegraphics[width=0.12\linewidth]{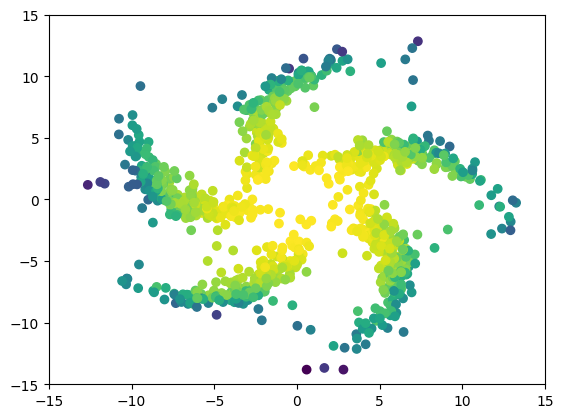}}
  \caption{The first row corresponds to iteration 1 of DSB, the second to
    iteration 3 of DSB, the third to iteration 5 of DSB and the last to
    iteration 20 of DSB.}
  \label{fig:ipf_iteration}
\end{figure}

We found that the variance of $\pprior$ has an impact on the convergence speed
of DSB, see \Cref{fig:effect_init_figure} for an illustration.
This remark is in accordance with \cite[Technique 1]{song2020improved}.
In practice we
recommend to set the variance to be larger than the variance of the target
dataset, see \Cref{sec:gaussian_variance} in \Cref{sec:implementation_details}.

\captionsetup[subfigure]{labelformat=parens}
\begin{figure}[h]
  \centering \subfloat[$\pprior$]{\includegraphics[width=0.25\linewidth]{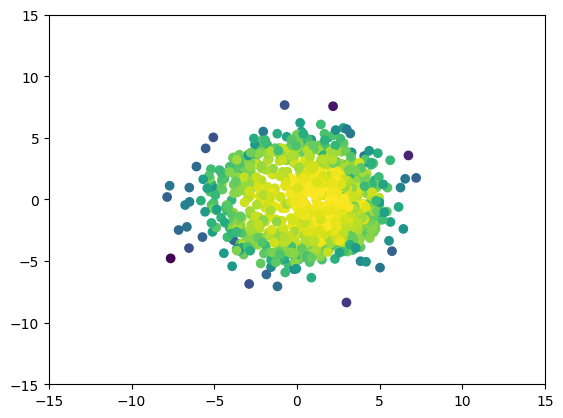}}
  \hfill \subfloat[DSB iteration 5]{\includegraphics[width=0.25\linewidth]{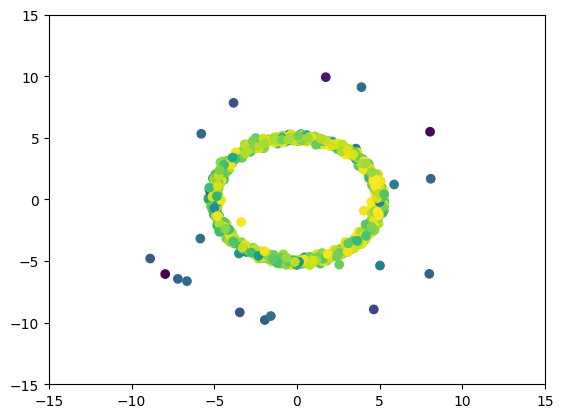}}
  \hfill \subfloat[$\pprior$]{\includegraphics[width=0.25\linewidth]{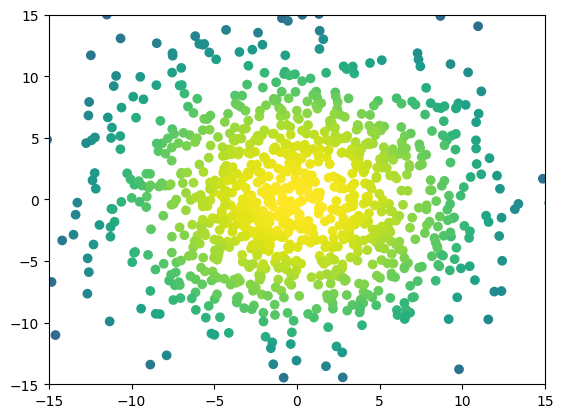}}
  \hfill \subfloat[DSB iteration 
  5]{\includegraphics[width=0.25\linewidth]{./fig/circle_gaussian_adaptive/im/9999_backward_5_registration_19.png}}
  \caption{Effect of the variance of $\pprior$ on the convergence of DSB. If
    $\pprior$ has a small variance $\sigma^2$ (here $\sigma^2 = 5$ in (a) and
    (b)) then DSB converges more slowly. If
    $\sigma^2 \approx \sigma_{\mathrm{data}}^2$, where
    $\sigma_{\mathrm{data}}^2$ is the variance of $\pdata$ then we observe more
    diversity in the samples obtained using DSB even for few iterations.}
  \label{fig:effect_init_figure}
\end{figure}
\captionsetup[subfigure]{labelformat=parens}

\begin{figure}[h]
  \centering
           \begin{tikzpicture}
        \node[inner sep=0pt, label={\small},] (scurve_orig) at (0+\offset,0)
        {\includegraphics[width=0.25\linewidth, trim=0.0cm 0.0cm 0.0cm 0.6cm, clip]{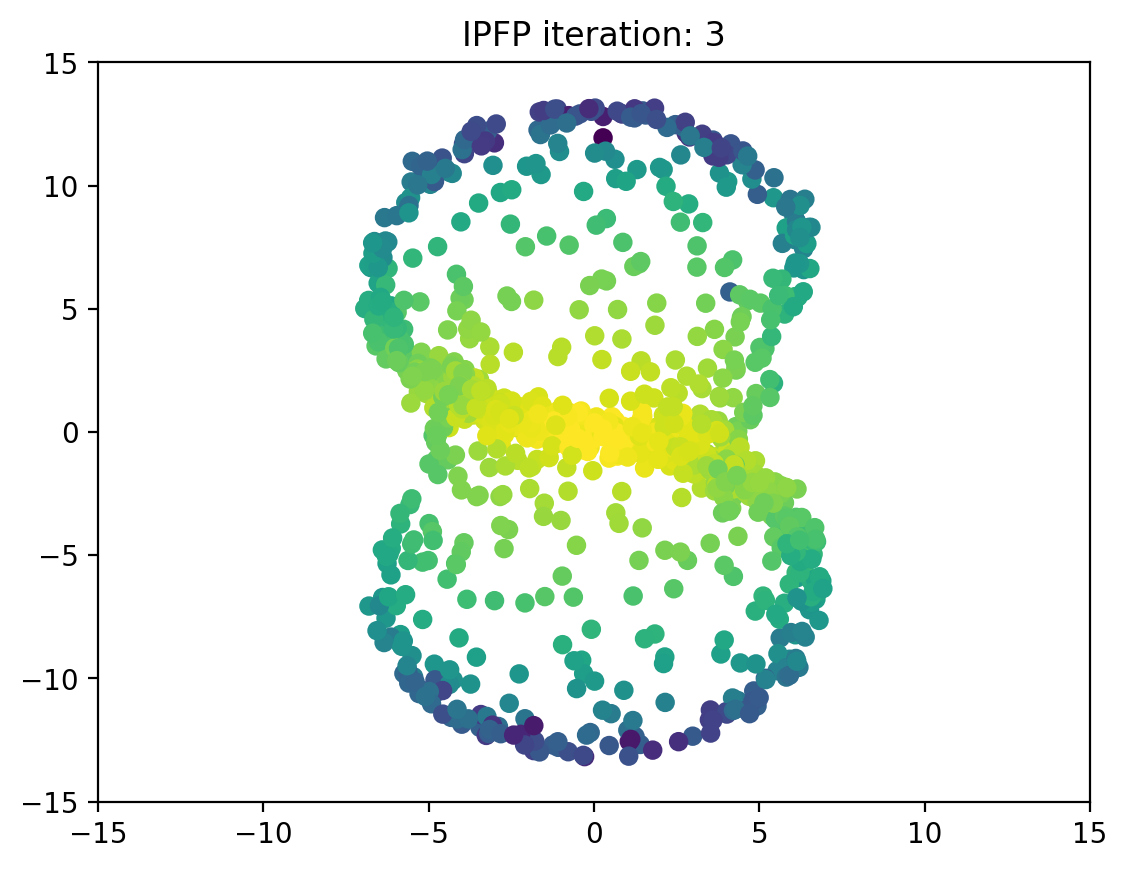}};
        \end{tikzpicture}
  \caption{Failure of DSB for low $N$. DSB iteration $3$ with $N=2$ and $30,000$ training steps per DSB iteration. The results deteriorate significantly after $5$ iterations of the algorithm.}
  \label{fig:failure}
\end{figure}

Finally, since DSB does not require the number of Langevin iterations $N$ to be
large, one may question why not use $N=1$ in order to derive a
feed-forward generative model. In practice this choice of $N$ is not desirable for two reasons.
\begin{enumerate*}[label=(\alph*)]
\item Firstly, since $p_N$ is not a good approximation of $\pprior$, theoretical results
  such as \cite[Corollary 1]{leger2020gradient} indicates that more IPF iterations are needed.
\item Second, in our experiments we observe that in order to obtain similar
  results to $N=10$ with $N=1$ we need to substantially increase the size of the
  networks, even for a large number of IPF iterations, see \Cref{fig:failure}.
\end{enumerate*}

\subsection{Generative Modeling}
\label{sec:generative-modeling-1}

\paragraph{Implementation details}
We use a reduced version of the U-net architecture from
\citet{nichol2021improved} for $F_{\alpha}$ and $B_{\beta}$, where we set the
number of channels to $64$ rather than $128$ for computational resource
purposes. We tried the architecture of \cite{song2020improved}, however we
observed worse results in our framework. Although we observed improvement using
the corrector scheme of \citet{song2020score}, this improvement was
similar to augmenting the number of steps in the Langevin
scheme. We therefore chose to avoid using such techniques altogether because of
the increase in computing time when sampling, often by doubling the number of
passes through the network. 

We chose the sequence $\{\gamma_{k}\}_{k=0}^N$ to be invariant by time reversal,
\ie \ for any $k \in \{0, \dots, N\}$, $\gamma_k = \gamma_{N-k}$.  In practice,
we assume that $N$ is even and let
$\gamma_k = \gamma_0 + (2k/N) (\bar{\gamma} - \gamma_0)$ for
$k \in \{0, \dots, N/2\}$ with $\gamma_0 = 10^{-5}$ and
$\bar{\gamma} = 10^{-1}$. The rest of the sequence is obtained by symmetry.

In the case of the MNIST dataset (dimension $d=28 \times 28 = 784$) we set the
batch size to $128$, the number of samples in the cache to $5 \times 10^4$ with
$10$ time-points sampled from each trajectory for each sample of $\pdata$. We end up with an
effective cache of size $5 \times 10^5$. The cache is refreshed each $10^3$
iterations and the networks are trained for $5 \times 10^3$ iterations. Again we
use the ADAM optimizer with momentum $0.9$ and learning rate
$10^{-4}$. $\pprior$ is a Gaussian density with zero mean and identity
covariance matrix. We have presented results for varying number of diffusion steps, $N$.

In the case of the CelebA dataset (dimension $d = 32 \times 32 \times 3 = 3072$)
we set the batch size to $256$, number of steps $N=50$, the number of samples in the cache to $250$ with
$1$ time-point sampled from each trajectory for each sample of $\pdata$.  The cache is
refreshed each $10^2$ iterations and the networks are trained for
$5 \times 10^3$ iterations. Again we use the ADAM optimizer with momentum $0.9$
and learning rate $10^{-4}$. $\pprior$ is a Gaussian density with zero mean and
identity covariance matrix.

Our results on MNIST and CelebA are computed using up to 4 NVIDIA Tesla V100 from the
Google Cloud Platform.

\paragraph{Additional examples}

In this section we present additional examples for our high-dimensional
generative modeling experiments. In \Cref{fig:interpolation} we perform
interpolation in the latent space. More precisely we let $X_N^0$ and $X_N^1$ be
two samples from $\pprior$. We then compute
$X_N^\lambda = (1 - \lambda) X_N^0 + \lambda X_N^1$ for different values of
$\lambda \in \ccint{0,1}$. For each value of $\lambda \in \ccint{0,1}$ we
associate $X_0^\lambda$ which corresponds to the output sample obtained using
the generative model given by DSB with final condition $X_N^\lambda$. Note that
in order to obtain a deterministic embedding we fix the Gaussian random
variables used in the sampling. One could also have used the deterministic
embedding used by \cite{song2020score}, \ie \ a neural ordinary differential
equation that admits the same marginals as the diffusion thus enabling exact
likelihood computation, see \Cref{sec:likel-comp-schr} for details.

\captionsetup[subfigure]{labelformat=empty}
\begin{figure}[h]  
  \centering
  \hfill \subfloat[]{\includegraphics[width=0.09\linewidth]{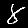}}
  \hfill \subfloat[]{\includegraphics[width=0.09\linewidth]{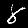}}
  \hfill \subfloat[]{\includegraphics[width=0.09\linewidth]{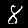}}
  \hfill \subfloat[]{\includegraphics[width=0.09\linewidth]{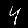}}
  \hfill \subfloat[]{\includegraphics[width=0.09\linewidth]{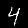}}
  \hfill \subfloat[]{\includegraphics[width=0.09\linewidth]{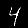}}
  \hfill \subfloat[]{\includegraphics[width=0.09\linewidth]{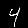}}
  \hfill \subfloat[]{\includegraphics[width=0.09\linewidth]{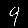}}
  \hfill \subfloat[]{\includegraphics[width=0.09\linewidth]{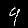}}
  \hfill \subfloat[]{\includegraphics[width=0.09\linewidth]{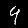}}
  \hfill \subfloat[]{\includegraphics[width=0.09\linewidth]{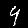}}
  \caption{Interpolation in the latent space for MNIST.}
      \label{fig:interpolation}
\end{figure}
\captionsetup[subfigure]{labelformat=empty}

In \Cref{fig:high_quality} we present high quality samples for MNIST.  In order
to obtain these high quality samples we consider our baseline MNIST
configuration but instead of choosing $N = 10$ time steps we consider $N=
30$. In addition, we train the networks for $15 \times 10^3$ iterations instead
of $5 \times 10^3$. The number of samples in the cache is $M = 500$

\begin{figure}[h]  
  \centering
  \hfill
  \begin{tikzpicture}
        \node[inner sep=0pt, label={\small }] (scurve_uno) at (0,0)
        {\includegraphics[width=0.5\linewidth, trim=1.2cm 0.95cm 0cm 0cm, clip]{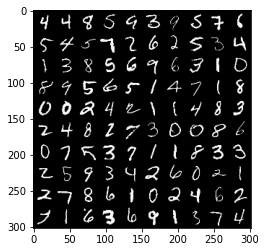}};            
        \node[inner sep=0pt, label={\small},] (scurve_orig) at (7,0)
            {\includegraphics[width=0.5\linewidth, trim=1.2cm 0.95cm 0cm 0cm, clip]{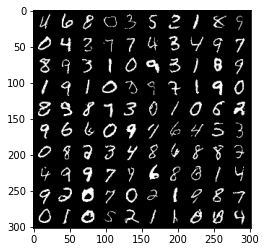}};
      \end{tikzpicture}
      \hfill
  \caption{MNIST samples: original dataset (left) and generated MNIST samples (right) after $12$ DSB iterations}
  \label{fig:high_quality}
\end{figure}

In \Cref{fig:temperature} we present a temperature scaling exploration of the
embedding obtained for CelebA. Similarly to the interpolation experiment we fix
the Gaussian random variables in order to obtain a deterministic mapping from
the latent space to the image space. 

\begin{figure}[h]
  \centering
  \hfill
  \includegraphics[width=0.3\linewidth]{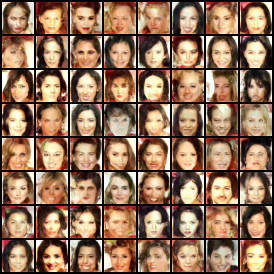}
  \hfill
  \includegraphics[width=0.3\linewidth]{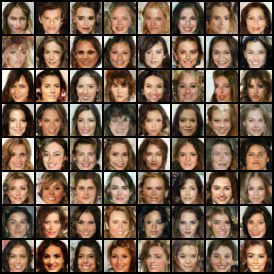}
  \hfill
  \includegraphics[width=0.3\linewidth]{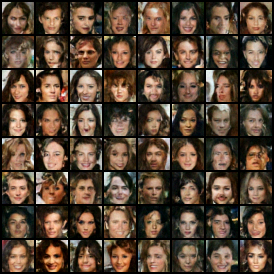}
  \hfill
  \caption{Temperature scaling in the latent space.}
  \label{fig:temperature}
\end{figure}

In \Cref{fig:ornstein_ulhenbeck} we explore the latent space of our embedding of
CelebA. To do so, we obtain samples using a Ornstein-Ulhenbeck process targeting
$\pprior$. We refer to our project page
\href{https://vdeborto.github.io/publication/schrodinger_bridge/}{project webpage} for an animated version of
this latent space exploration.

\begin{figure}[H]
  \centering
  \hfill
  \includegraphics[width=0.24\linewidth]{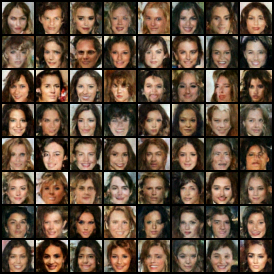}
  \hfill
  \includegraphics[width=0.24\linewidth]{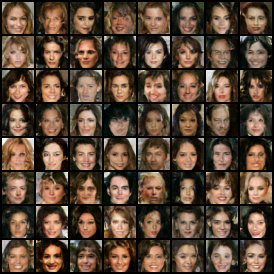}
  \hfill
  \includegraphics[width=0.24\linewidth]{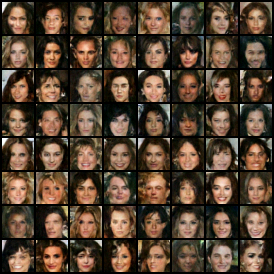}
  \hfill
  \includegraphics[width=0.24\linewidth]{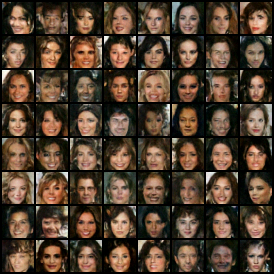}
  \hfill    
  \caption{Exploration of the latent space. Samples are generated using a
    Ornstein-Ulhenbeck process targeting $\pprior$ to obtain the initial
    condition then using the generative model given by DSB. From left to right
    to right: samples at time $t=0, 1.3, 3.6, 8.6$.}
  \label{fig:ornstein_ulhenbeck}
\end{figure}

\newpage

\subsection{Dataset interpolation}
\label{sec:dataset-morphing}

For the dataset interpolation task we keep the same parameters and architecture
as before except that the number of Langevin steps is increased to $50$ steps in
the two-dimensional examples and to $30$ steps in the EMNIST/MNIST interpolation
task. We also change the reference dynamics which is chosen to be the one
obtained with the DSB where $\pprior$ is a Gaussian. This choice allows us to
speed up the training of DSB in this setting. Animated plots are available at
\href{https://vdeborto.github.io/publication/schrodinger_bridge/}{project webpage}.

\paragraph{EMNIST/MNIST}

In order to perform translation between the dataset of handwritten letters
(EMNIST) and handwritten digits (MNIST) we reduce EMNIST to $5$ letters so that
it contains as many classes as MNIST (we distinguish upper-case and lower-case
letters), see \cite{cohen2017emnist} for the original dataset.

 \begin{figure}[H]
        \centering
         \begin{tikzpicture}
        \node[inner sep=0pt, label={\small }] at (0+\offset,0)
            {\includegraphics[width=0.18\linewidth]{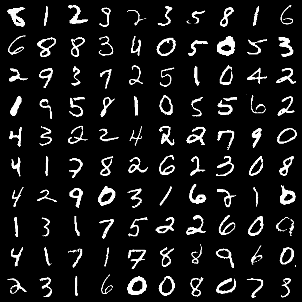}};
        \node[inner sep=0pt, label={\small }]  at (\ww+\offset,0)
        {\includegraphics[width=0.18\linewidth]{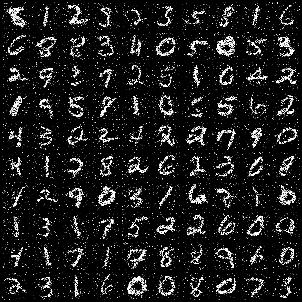}};        
        \node[inner sep=0pt, label={\small }]  at (2*\ww+\offset,0)
            {\includegraphics[width=0.18\linewidth]{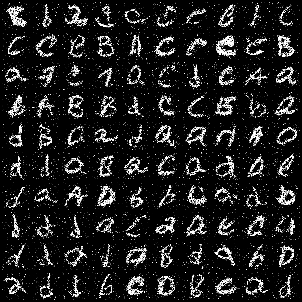}};
        \node[inner sep=0pt, label={\small }] at (3*\ww+\offset,0)
            {\includegraphics[width=0.18\linewidth]{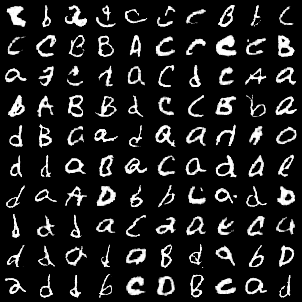}};
          \end{tikzpicture}
          \caption{Iteration 10 of the IPF with $T = 1.5$ ($30$ diffusions
            steps). From left to right: $t=0, 0.4, 1.25, 1.5$.}
  \label{fig:transport_mnist}          
\end{figure}

\paragraph{Two dimensional examples} We present dataset interpolation for a number of classical two-dimensional datasets.

\captionsetup[subfigure]{labelformat=empty}
\begin{figure}[H]
  \centering
  \begin{tikzpicture}
    \node[inner sep=0pt, label={\small},] (scurve_orig) at (0,0)
    {\includegraphics[width=0.24\linewidth, trim=0cm 0cm 0cm 0cm, clip]{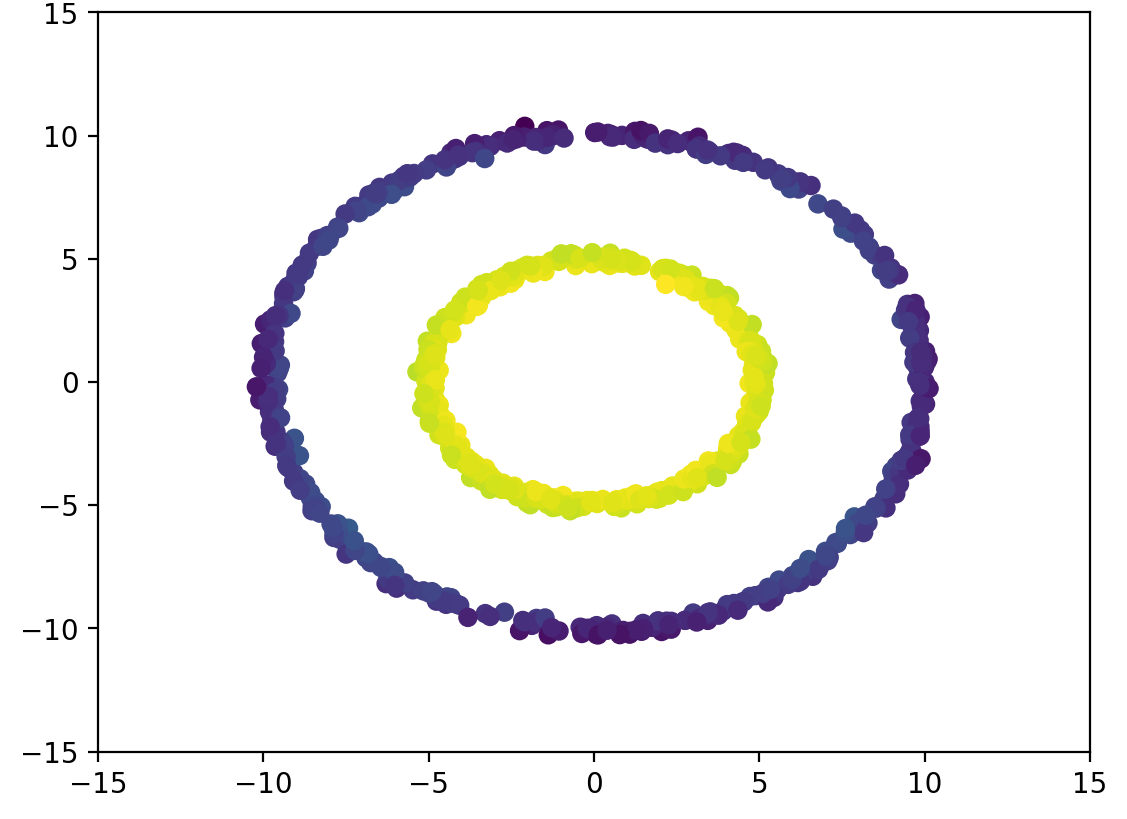}};
    \node[inner sep=0pt, label={\small},] (scurve_orig) at (3.5,0)
    {\includegraphics[width=0.24\linewidth, trim=0cm 0cm 0cm 0cm, clip]{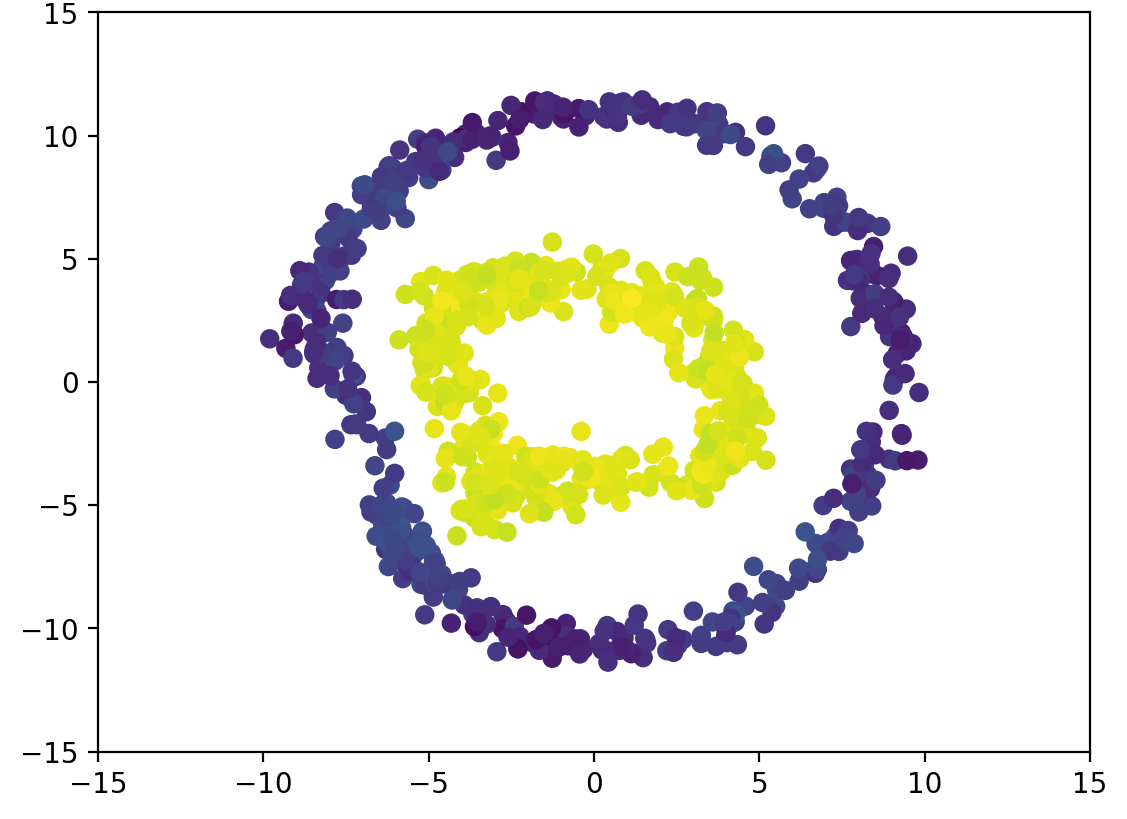}};
    \node[inner sep=0pt, label={\small},] (scurve_orig) at (7,0)
    {\includegraphics[width=0.24\linewidth, trim=0cm 0cm 0cm 0cm, clip]{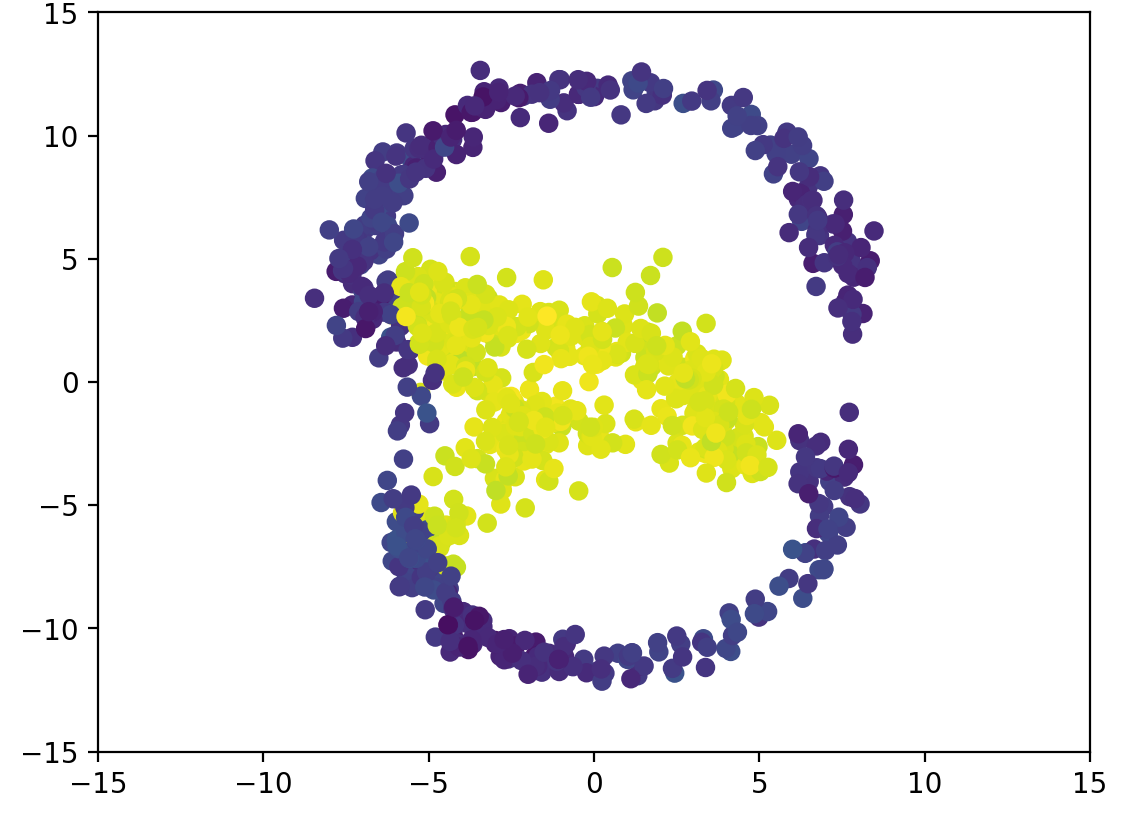}};
    \node[inner sep=0pt, label={\small},] (scurve_orig) at (10.5,0)
    {\includegraphics[width=0.24\linewidth, trim=0cm 0cm 0cm 0cm, clip]{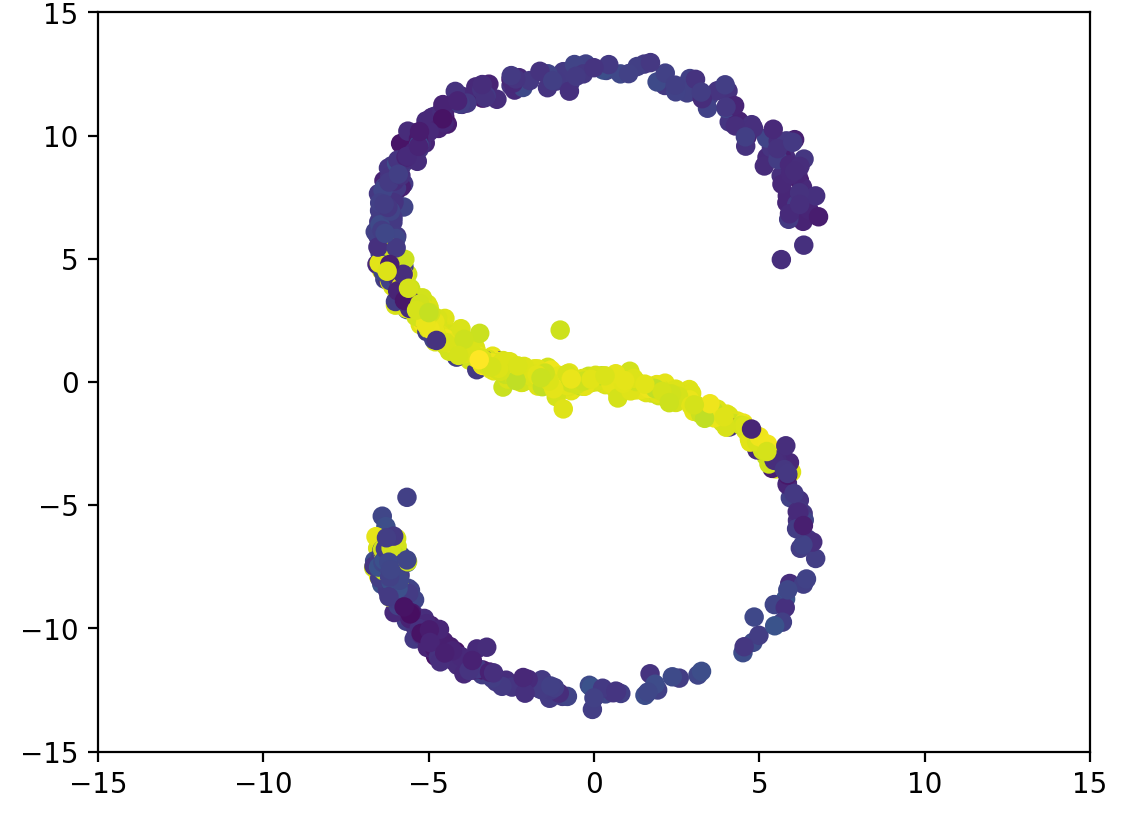}};

    \node[inner sep=0pt, label={\small},] (scurve_orig) at (0,-3)
    {\includegraphics[width=0.24\linewidth, trim=0cm 0cm 0cm 0.75cm, clip]{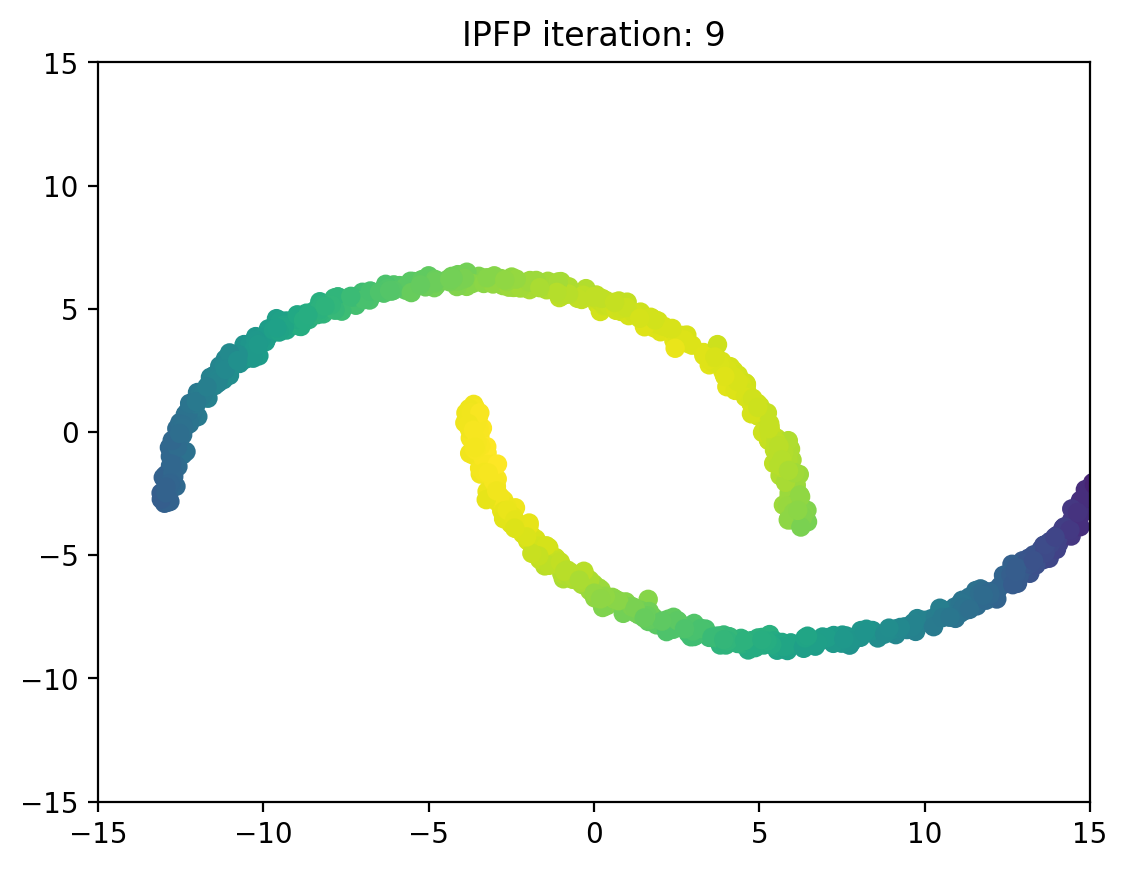}};
    \node[inner sep=0pt, label={\small},] (scurve_orig) at (3.5,-3)
    {\includegraphics[width=0.24\linewidth, trim=0cm 0cm 0cm 0.75cm, clip]{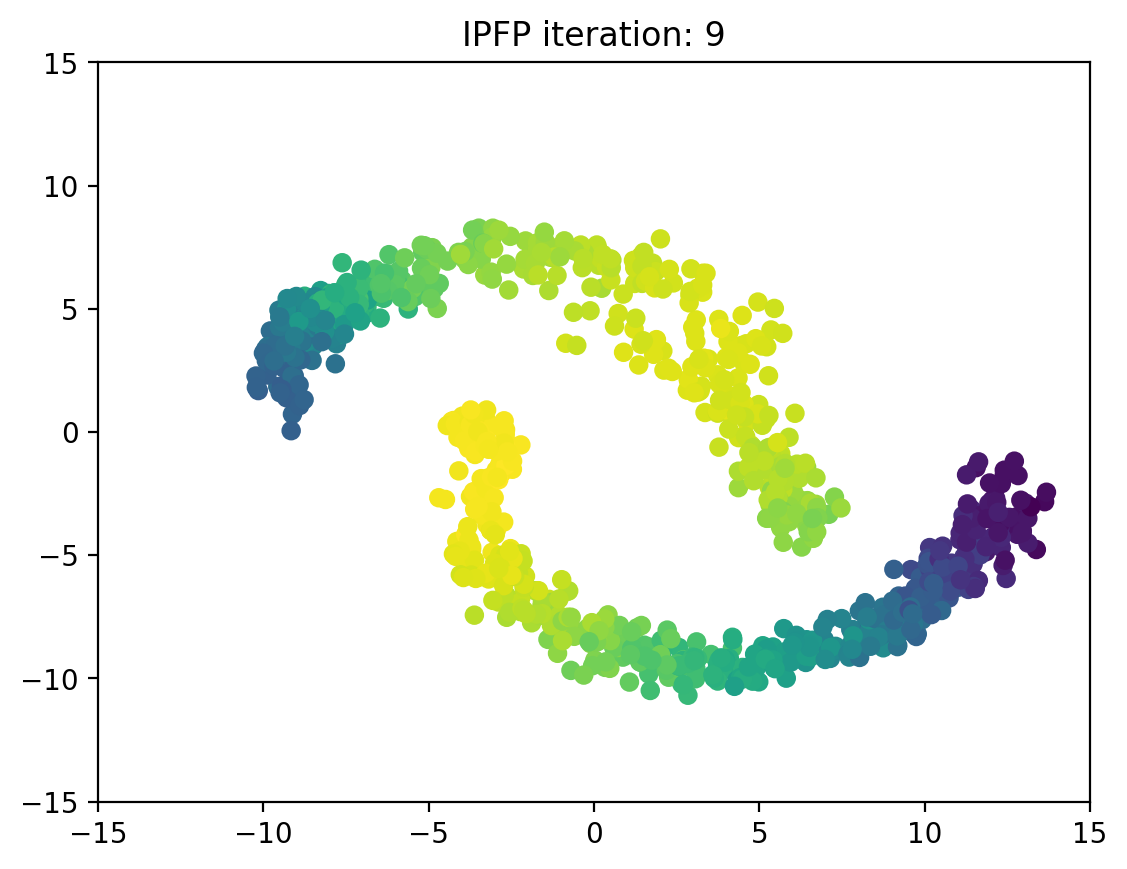}};
    \node[inner sep=0pt, label={\small},] (scurve_orig) at (7,-3)
    {\includegraphics[width=0.24\linewidth, trim=0cm 0cm 0cm 0.75cm, clip]{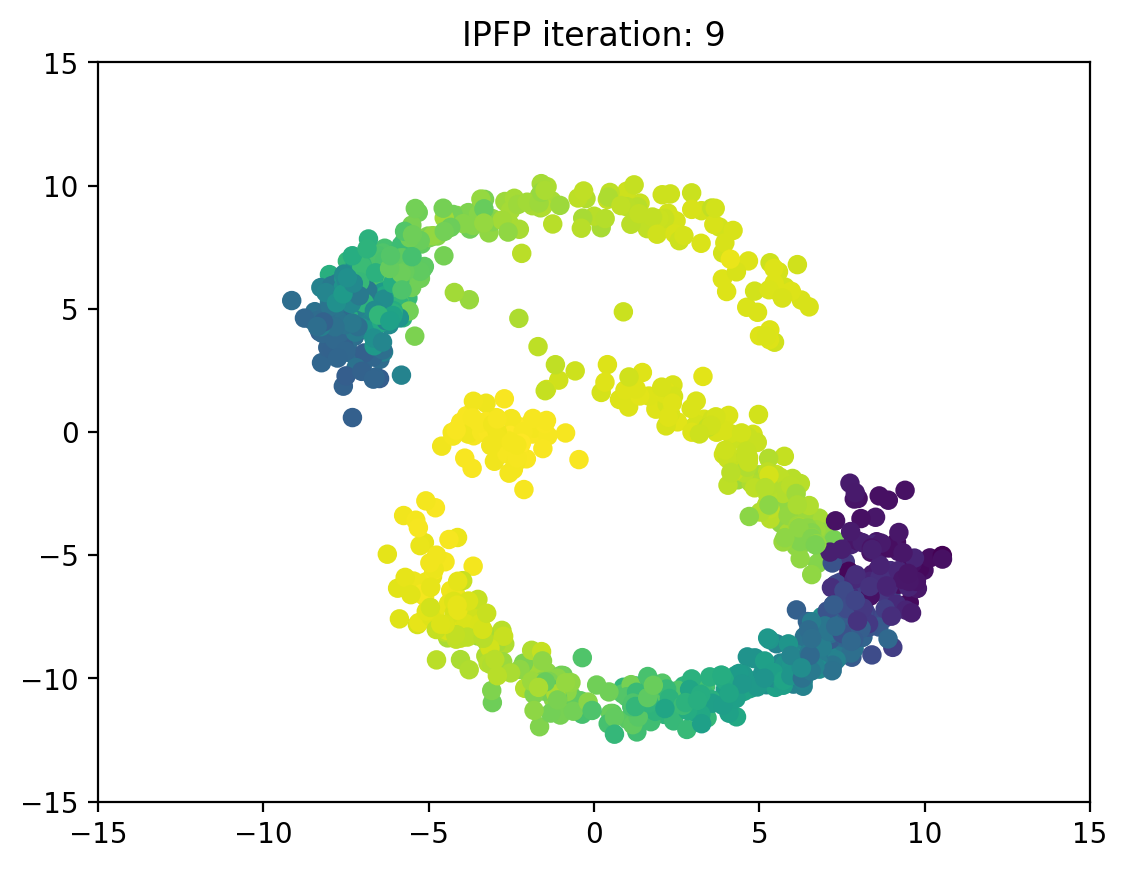}};
    \node[inner sep=0pt, label={\small},] (scurve_orig) at (10.5,-3)
    {\includegraphics[width=0.24\linewidth, trim=0cm 0cm 0cm 0.75cm, clip]{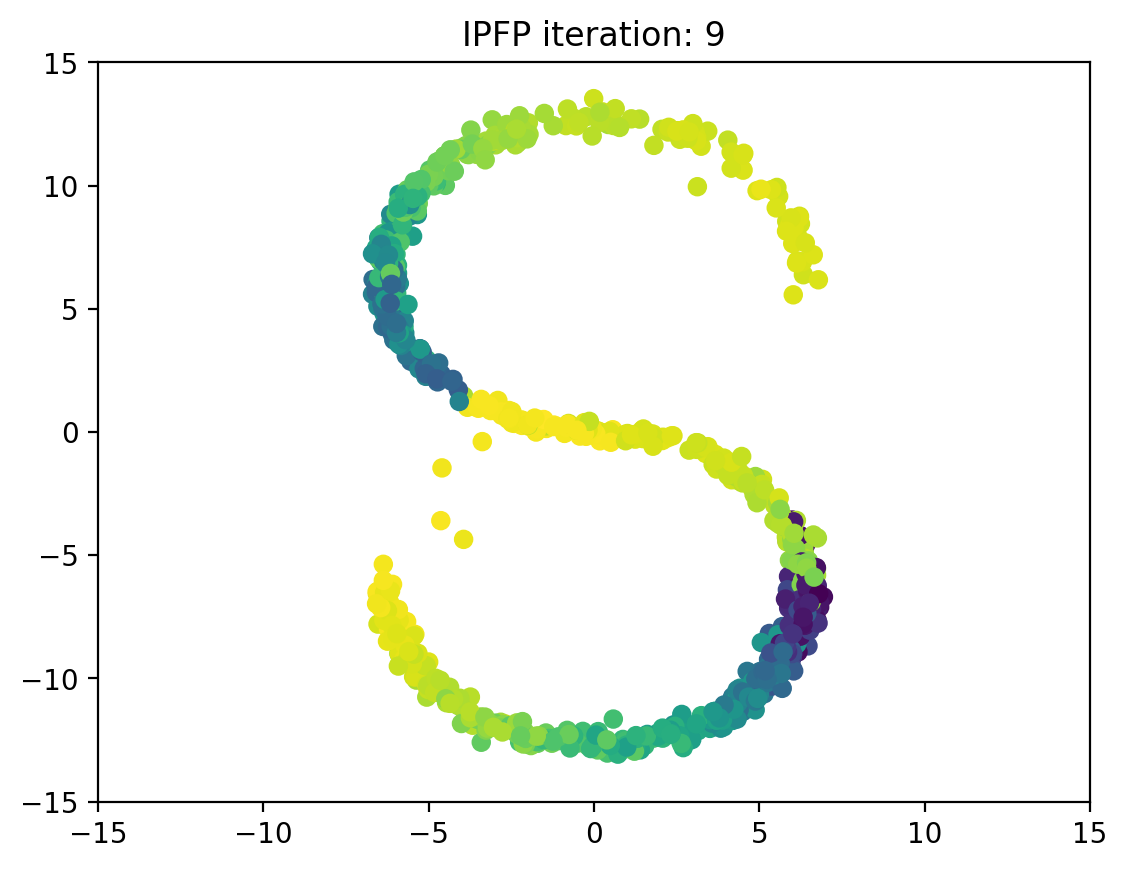}};

    \node[inner sep=0pt, label={\small},] (scurve_orig) at (0,-6)
    {\includegraphics[width=0.24\linewidth, trim=0cm 0cm 0cm 0.75cm, clip]{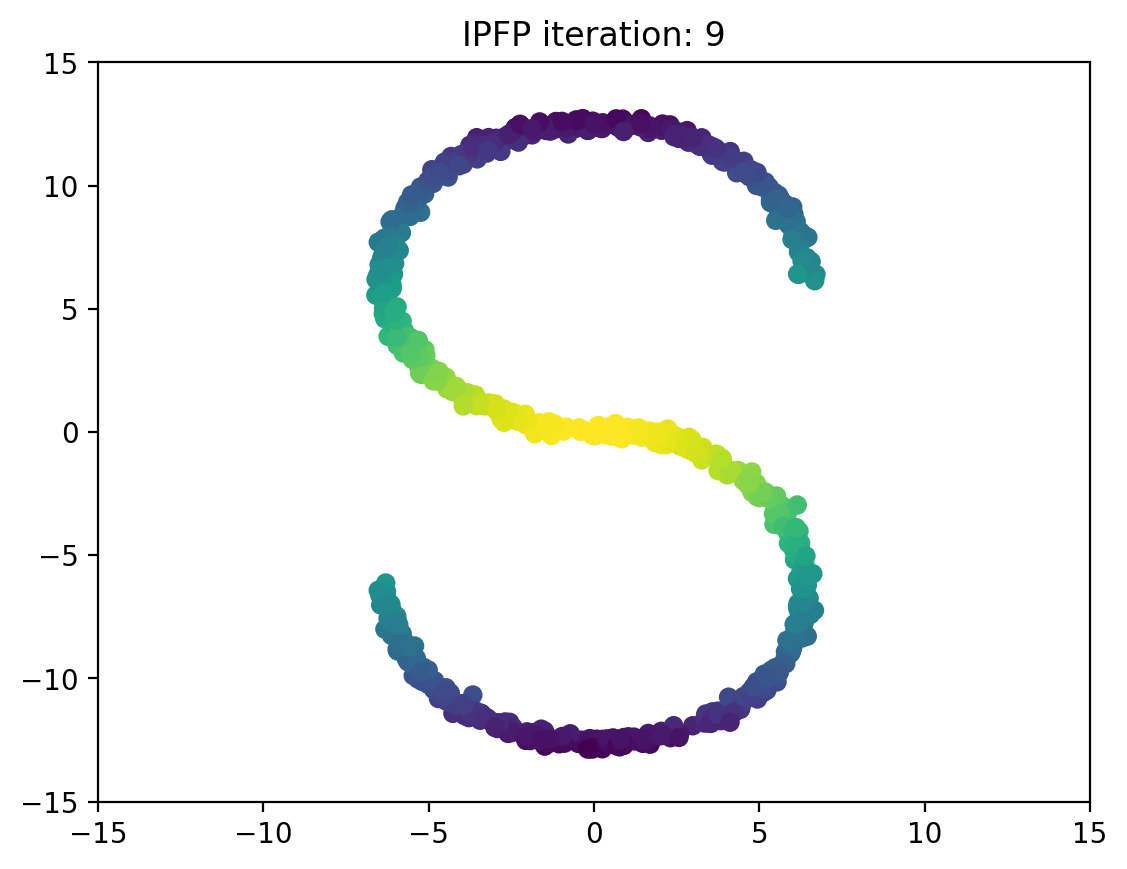}};
    \node[inner sep=0pt, label={\small},] (scurve_orig) at (3.5,-6)
    {\includegraphics[width=0.24\linewidth, trim=0cm 0cm 0cm 0.75cm, clip]{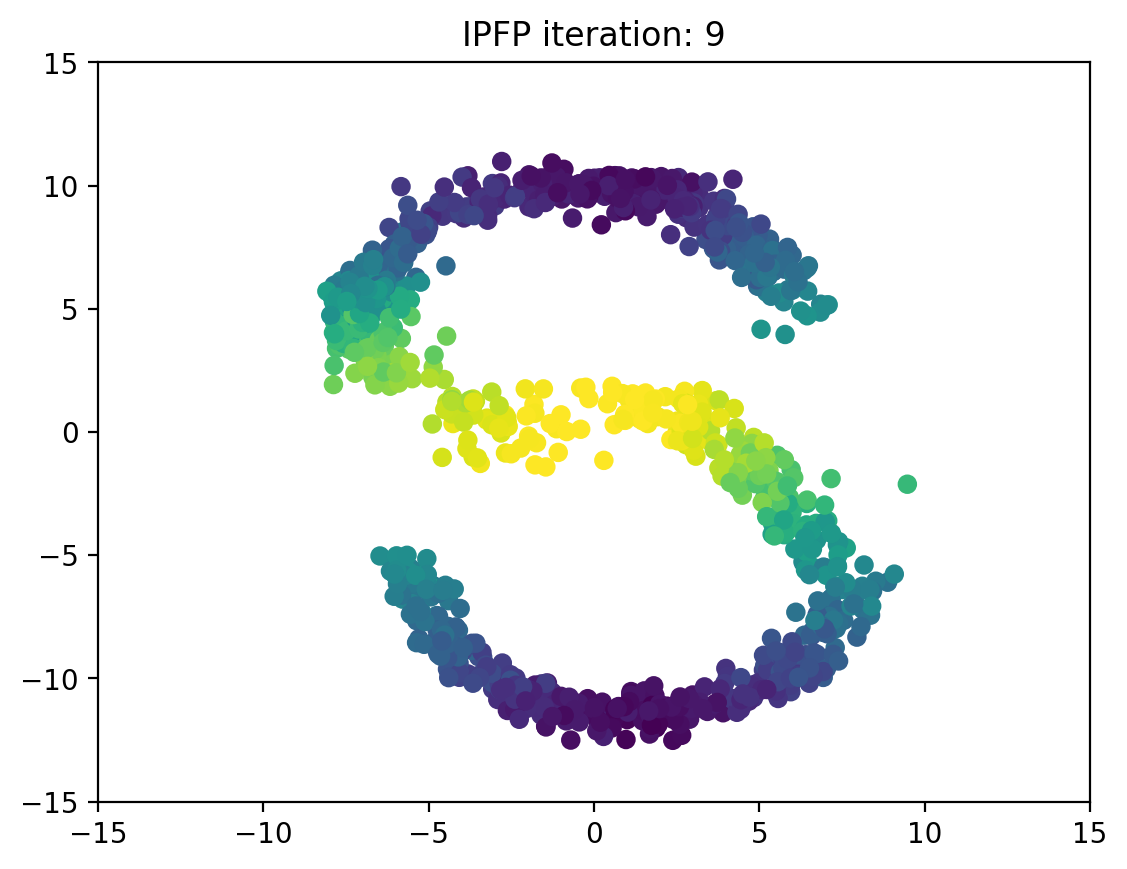}};
    \node[inner sep=0pt, label={\small},] (scurve_orig) at (7,-6)
    {\includegraphics[width=0.24\linewidth, trim=0cm 0cm 0cm 0.75cm, clip]{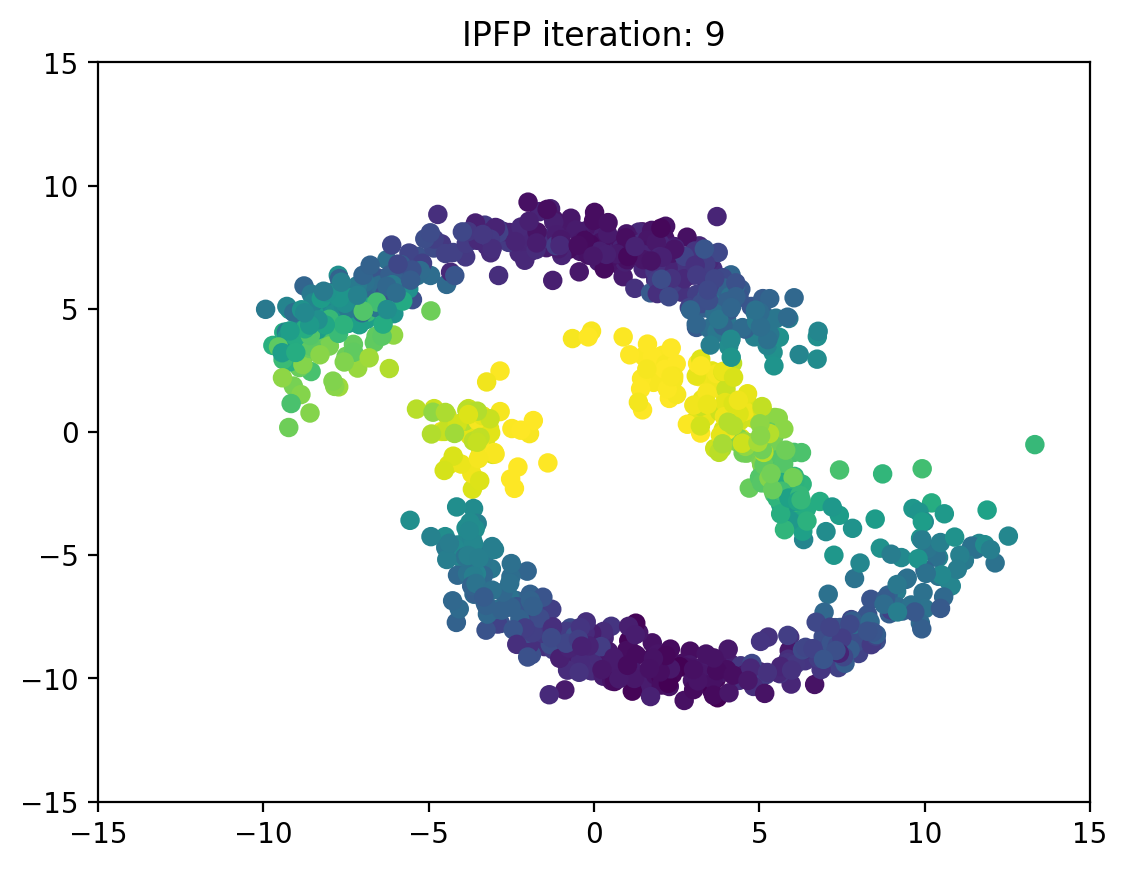}};
    \node[inner sep=0pt, label={\small},] (scurve_orig) at (10.5,-6)
    {\includegraphics[width=0.24\linewidth, trim=0cm 0cm 0cm 0.75cm, clip]{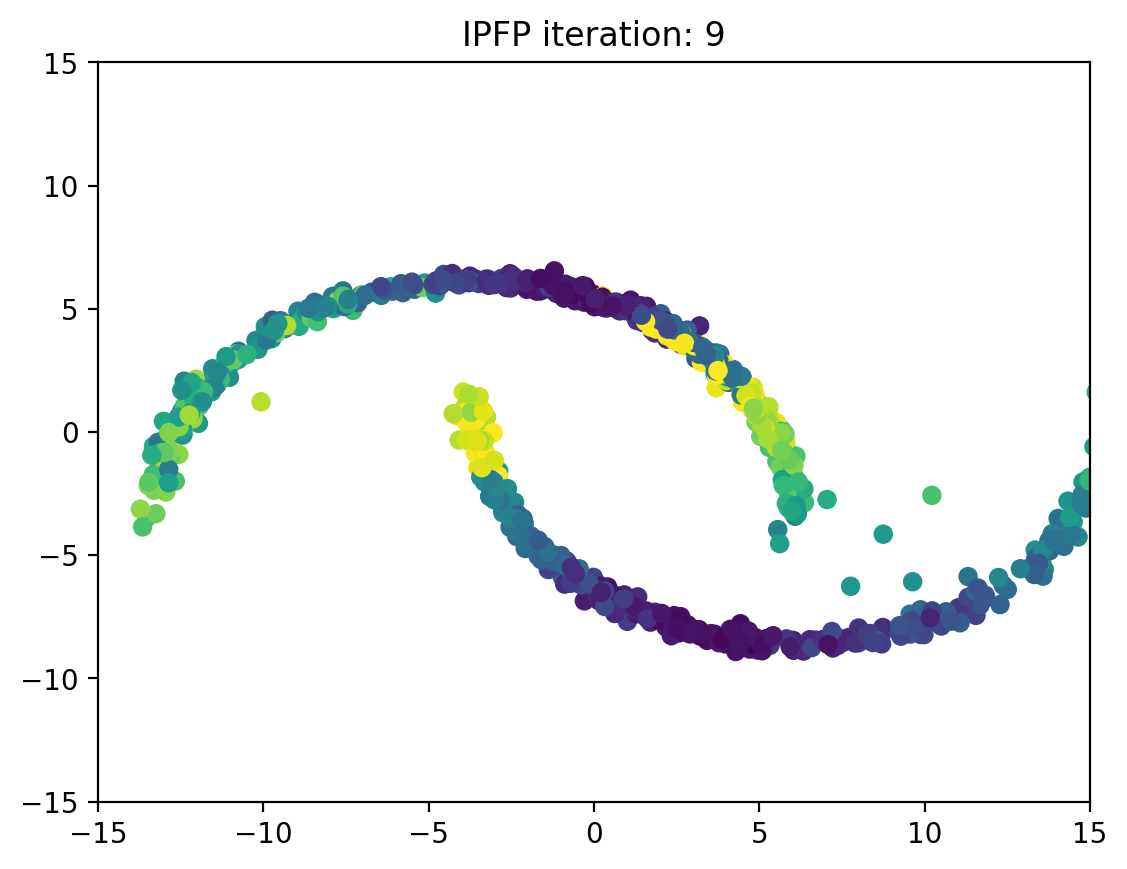}};        
  \end{tikzpicture}
  \label{fig:2d_interp}
  \caption{Dataset interpolation (DSB iteration 9). From left to right: $t=0, 0.15, 0.30, 0.5$.}
\end{figure}

\end{document}